\documentclass[twoside,11pt]{article}

\usepackage{blindtext}

\usepackage[preprint]{jmlr2e}
\usepackage{amsmath,xcolor}
\usepackage{bbm}
\usepackage{subcaption}
\usepackage{url} 
\usepackage{amsfonts} 
\usepackage{microtype} 
\usepackage{times}
\usepackage{bm}

\usepackage[noend,vlined,linesnumbered,ruled]{algorithm2e}
\SetKwInput{KwInput}{Input}               
\SetKwInput{KwOutput}{Output}  

\newtheorem{assumption}{Assumption}[section]

\usepackage{lastpage}
\jmlrheading{23}{2022}{1-\pageref{LastPage}}{12/22; Revised}{12/22}{21-0000}{Xuezhen Li and Can M. Le}

\ShortHeadings{Variational Inference with Posterior Threshold}{Li and Le}
\firstpageno{1}

\begin{document}

\title{Variational Inference: Posterior Threshold Improves Network Clustering Accuracy in Sparse Regimes}

\author{\name Xuezhen Li \email xzhli@ucdavis.edu \\
       \addr Department of Statistics\\
       University of California, Davis\\
       Davis, CA 95616-5270, USA
       \AND
       \name Can M. Le \email canle@ucdavis.edu \\
       \addr Department of Statistics\\
       University of California, Davis\\
       Davis, CA 95616-5270, USA}

\editor{}

\maketitle

\begin{abstract}
Variational inference has been widely used in machine learning literature to fit various Bayesian models. In network analysis, this method has been successfully applied to solve community detection problems. Although these results are promising, their theoretical support is only for relatively dense networks, an assumption that may not hold for real networks. In addition, it has been shown recently that the variational loss surface has many saddle points, which may severely affect its performance, especially when applied to sparse networks. This paper proposes a simple way to improve the variational inference method by hard thresholding the posterior of the community assignment after each iteration. Using a random initialization that correlates with the true community assignment, we show that the proposed method converges and can accurately recover the true community labels, even when the average node degree of the network is bounded. Extensive numerical study further confirms the advantage of the proposed method over the classical variational inference and another state-of-the-art algorithm.
\end{abstract}


\section{Introduction}\label{sec:intro}
Variational inference (VI) is arguably one of the most important methods in Bayesian statistics \citep{jordan1999anintroduction}. It is often used to approximate posterior distributions in large-scale Bayesian inference problems when the exact computation of posterior distributions is not feasible. For example, this is the case for many high-dimensional and complex models which involve complicated latent structures. Mean-field variational method \citep{beal2003variational,jordan1999anintroduction} is the simplest VI algorithm, which approximates the posterior distribution of latent variables by a product measure. This method has been applied in a wide range of fields including network analysis \citep{airoldi2008mixed,celisse2012consistency}, computer science \citep{wang2013variational} and neuroscience \citep{grabska2017probabilistic}. The mean-field variational approximation is especially attractive in large-scale data analysis compared to alternatives such as Markov chain Monte Carlo (MCMC) \citep{gelfand1990sampling} due to its computational efficiency and scalability \citep{blei2017variational}.   

Although mean-field VI has been successfully applied to various Bayesian models, the theoretical behavior of these algorithms has not been fully understood. Most of the existing theoretical work focused on the global minimum of the variational method \citep{blei2003latent,bickel2013asymptotic,zhang2020convergence,wang2019frequentist}. 
However, it is often intractable to compute the exact global minimizers of high-dimensional and complex models in practice. Instead, iterative algorithms such as Batch Coordinate Ascent Variational Inference (BCAVI) are often applied to estimate them \citep{blei2017variational}. 

This paper focuses on the statistical properties of BCAVI for estimating node labels of a network generated from stochastic block models (SBM) \citep{holland1983stochastic} or degree-corrected stochastic block models (DCSBM) \citep{karrer2011stochastic}. Let $n$ be the  number of nodes, indexed by integers $i\in[n]=\{1,2,...,n\}$, and $A\in\{0,1\}^{n\times n}$ be the adjacency matrix with $A_{ij}=1$ if $i$ and $j$ are connected. We only consider undirected networks without self-loops, so $A$ is symmetric, and $A_{ii}=0$ for all $i$. Under SBM or DCSBM, nodes are partitioned into $K$ communities with community labels $z_i\in[K]$. 
The label vector $z=(z_1,...,z_n)^T\in[K]^n$ can also be encoded by a membership matrix $Z\in\mathbb{R}^{n\times K}$ such that the $i$-th row of $Z$ represents the membership of node $i$ with $Z_{iz_i}=1$ and $Z_{ik}=0$ if $k\neq z_i$. Conditioned on $Z$, $\{A_{ij},i<j\}$ are independent Bernoulli random variables with the corresponding probabilities $\{P_{ij}, i<j\}$. For SBM, the edge probability $P_{ij} = B_{z_i z_j}$ is determined by the block memberships of node $i$ and $j$, where $B\in\mathbb{R}^{K\times K}$ is the block probability matrix that may depend on $n$. For DCSBM, $P_{ij} = \theta_i\theta_j B_{z_i z_j}$, where $\theta_i$, $1\le i\le n$, are the parameters controlling the node degrees. Community detection aims to recover $z$ and estimate $B$ and $\theta_i$.

Many algorithms have been proposed for solving the community detection problem; we refer the reader to  \citep{abbe2017community} for a thorough review of this topic. 
Among all these promising algorithms, BCAVI remains one of the most popular approaches. Given the adjacency matrix $A$ and an initial estimate $Z ^{(0)}\in\mathbb{R}^{n\times K}$ of the true membership matrix $Z$, BCAVI aims to calculate the posterior $\mathbb{P}(Z|A)$ and use it to recover the true label vector $z$, for example, by setting the estimated label of node $i$ to $\mathrm{argmax}_{k\in[K]}\mathbb{P} (Z_{ik}|A)$. Since the exact calculation of $\mathbb{P}(Z|A)$ is infeasible because it involves summing $\mathbb{P}(A,Z)$ over an exponentially large set of label assignments $Z$, BCAVI approximates this posterior by a much simpler and tractable distribution $\mathbb{Q}(Z^{(t)})$ and iteratively updates $\mathbb{Q}(Z^{(t)})$, $t\in\mathbb{N}$, to improve the accuracy of this approximation. 

Regarding the consistency of BCAVI, \cite{zhang2020theoretical} shows that if $Z^{(0)}$ is sufficiently close to the true membership matrix $Z$ then $Z^{(t)}$ converges to $Z$ at an optimal rate. In \citet{sarkar2021random}, the authors give a complete characterization of all the critical points for optimizing $\mathbb{Q}(Z)$ and establish the behavior of BCAVI when $Z^{(0)}$ is randomly initialized in such a way so that it is correlated with $Z$. However, their theoretical guarantees only hold for relatively dense networks, an assumption that may not be satisfied for many real networks \citep{chin2015stochastic, le2017concentration,gao2017achieving}. Moreover, BCAVI often converges to uninformative saddle points that contain no information about community labels when networks are sparse. \citet{yin2020theoretical} addresses this problem by introducing dependent posterior approximation, but they still require strong assumptions on the network density. 

This paper considers the application of BCAVI in the sparse regime when the average node degree may be bounded. Our contribution is two-fold. First, we propose to improve BCAVI by hard thresholding the posterior of the label assignment at each iteration: After $Z^{(t)}$ has been calculated by BCAVI in the $t$-th iteration, the largest entry of each row $Z^{(t)}_i$ is set to one, and all other entries are set to zero. In view of the uninformative saddle points \citep{sarkar2021random} of BCAVI, this step appears to be a naive way to project the posterior back to the set of ``reasonable'' label assignments. However, as the hard threshold discards part of the label information in $Z^{(t)}$, it is not a priori clear why such a step is beneficial. Surprisingly, exhaustive simulations show that this adjustment often leads to significant improvement of BCAVI, especially when the network is sparse or the accuracy of the label initialization is poor. The resulting algorithm also outperforms a state-of-the-art method of \cite{gao2017achieving}, which is rate-optimal only when the average degree is unbounded. 

Second, we prove that when the network is generated from either SBM or DCSBM, BCAVI, with the threshold step, accurately estimates the community labels and model parameters even when the average node degree is bounded. This nontrivial result extends the theoretical guarantees in network literature for BCAVI under SBM and is new under DCSBM. It is in parallel with the extension of spectral clustering to the sparse regime where nodes of unusually large degrees are removed before spectral clustering is applied \cite{chin2015stochastic,le2017concentration}. However, in contrast to that approach, our algorithm does not require any pre-processing step for the observed network.

Similar to existing work on BCAVI \citep{zhang2020theoretical,zhang2020convergence,sarkar2021random,yin2020theoretical}, we emphasize that the goal of this paper is not to propose the most competitive algorithm for community detection, which may not exist for all model settings. Instead, we aim to refine VI, a popular but still not very well-understood machine learning method, and provide a deeper understanding of its properties, with the hope that the new insights developed in this paper can help improve the performance of VI for other problems \citep{blei2017variational}.

\section{BCAVI with Posterior Threshold}\label{sec:TBCAVI}
\subsection{Classical BCAVI for SBM} 
\label{sec: bcavi in sbm}
To simplify the presentation, we first describe the classical BCAVI for SBM without the threshold step. In particular, consider SBM with $K$ communities and the likelihood given by 
\begin{equation*}
  \mathbb{P}(A,Z) = \prod_{1 \leq i<j \leq n}\prod_{1 \leq a,b \leq K} \left(B_{ab}^{A_{ij}}(1-B_{ab})^{1-A_{ij}}\right)^{Z_{ia}Z_{jb}} \cdot \prod_{i=1}^n \prod_{a=1}^K \pi_a^{Z_{ia}}.
\end{equation*}
The goal of mean-field VI is to compute the posterior distribution of $Z$ given $A$:
\[\mathbb{P}(Z|A) = \frac{\mathbb{P}(Z,A)}{\mathbb{P}(A)}, \quad 
\mathbb{P}(A)=\sum_{Z\in\mathcal{Z}}\mathbb{P}(Z,A),\] 
where
$\mathcal{Z}=\{0,1\}^{n\times K}$ is the set of membership matrices.
Since $\mathbb{P}(A)$ involves a sum over the exponentially large set $\mathcal{Z}$, computing $\mathbb{P}(Z|A)$ exactly is not infeasible. The mean-field VI  approximates this posterior distribution by a family of product probability measures $\mathbb{Q}(Z) = \prod_{i=1}^n q_i(Z_i)$, where $Z_i$ is the $i$-th row of $Z$. The optimal $\mathbb{Q}(Z)$ is chosen to minimize the Kullback–Leibler divergence: 
\begin{equation*}
    \mathbb{Q}^{*}(Z) = \underset{\mathbb{Q}}{\mathrm{argmin}} \ \text{KL}(\mathbb{Q}(Z),\mathbb{P}(Z|A)).
    \label{qz}
\end{equation*} 
An important property of this divergence is that
\[\text{KL}(\mathbb{Q}(Z),\mathbb{P}(Z|A)) = \mathbb{E}_\mathbb{Q}[\log \mathbb{Q}(Z)] - \mathbb{E}_\mathbb{Q}[\log \mathbb{P}(Z,A)]+\log \mathbb{P}(A),\] 
where $\mathbb{E}_\mathbb{Q}$ denotes the expectation with respect to $\mathbb{Q}(Z)$. This property enables us to approximate the likelihood of the observed data $\mathbb{P}(A)$ by its evidence lower bound (ELBO), defined as 
\begin{equation*}
\text{ELBO}(\mathbb{Q}) = \mathbb{E}_\mathbb{Q}[\log \mathbb{P}(Z,A)]-\mathbb{E}_\mathbb{Q}[\log \mathbb{Q}(Z)].
\end{equation*}
Optimizing $\text{ELBO}(\mathbb{Q})$ also yields the optimal $\mathbb{Q}^*(Z)$ and $\text{ELBO}(\mathbb{Q}^{*})$ is the best approximation of $\mathbb{P}(A)$ we can obtain.

We consider $\mathbb{Q}(Z)$ of the form $\mathbb{Q}(Z)=\prod_{i}\prod_{a} \Psi_{ia}^{Z_{ia}}$, where $\Psi = \mathbb{E}_\mathbb{Q} [Z]$. Then 
\begin{eqnarray}
\label{elbo expression}
    \text{ELBO}(\mathbb{Q}) &=& \sum_{i<j}\sum_{ a,b } \Psi_{ia}\Psi_{jb}\Big[A_{ij}\log B_{ab} + (1-A_{ij})\log (1-B_{ab})\Big] \nonumber \\
    && +   \sum_{i} \sum_{a} \Psi_{ia}\log\left(\pi_a/\Psi_{ia}\right).
\end{eqnarray}
To optimize $\text{ELBO}(\mathbb{Q})$, we alternatively update $\Psi$ and $\{B,\pi\}$ at every iteration. 
In particular, we first fix $\Psi$ (the first iteration requires an initial value for $\Psi$) and set new values for $\{B,\pi\}$ to be the solution of the following equations: 
$$\frac{\partial \text{ELBO}(\mathbb{Q})}{\partial B_{ab}}=0, \quad \frac{\partial\text{ELBO}(\mathbb{Q})}{\partial \pi_{a}}=0, \quad a,b\in[K].$$    
Depending on whether $a=b$, the first equation yields the following update for $B_{ab}$:
\begin{align}
\label{update B}
    B_{aa} = \frac{\sum_{i<j} A_{ij} \Psi_{ia} \Psi_{ja}}{\sum_{i<j} \Psi_{ia} \Psi_{ja} } , \quad
    B_{ab} = \frac{\sum_{i<j} A_{ij} \left(\Psi_{ia} \Psi_{jb} + \Psi_{ib} \Psi_{ja}\right)}{\sum_{i<j} \left(\Psi_{ia} \Psi_{jb} + \Psi_{ib} \Psi_{ja}\right) }, \quad a < b.
\end{align}
To update $\pi$, we treat $\pi_a$, $a \in [K-1]$, as independent parameters and $\pi_K = 1-\sum_{a=1}^{K-1} \pi_a$,  For each $a \in [K-1]$, taking the derivative with respect to $\pi_a$ and setting it to zero, we get
\[ \frac{\partial\text{ELBO}(\mathbb{Q})}{\partial \pi_{a}} = \frac{\sum_{i} \Psi_{ia}}{ \pi_a} - \frac{\sum_{i} \Psi_{iK}}{ \pi_K} = 0,\]
which yields 
\begin{equation}
\label{update pi}
  \pi_a = \frac{\sum_{i} \Psi_{ia}}{\sum_{i}\sum_{b} \Psi_{ib}}, \quad a \in [K].
\end{equation}
Holding $\{B,\pi\}$ fixed, we then update $\Psi$ based on the equation  
\begin{align*}
   \frac{\partial \text{ELBO}(\mathbb{Q})}{\partial \Psi_{ia}} = 0, \quad i\in[n], \quad a \in [K-1],
\end{align*}
which gives for each $i\in[n]$ and $a\in[K]$:
\begin{equation}
\label{update Psi}
\Psi_{ia} = \frac{\pi_a \exp \left\{ \sum_{j: j\neq i} \sum_{b=1}^K \Psi_{jb}\left[A_{ij} \log B_{ab}+ (1-A_{ij}) \log (1 - B_{ab})\right]  \right\} } { \sum_{a'} \pi_{a'} \exp \left\{ \sum_{j: j\neq i} \sum_{b=1}^K \Psi_{jb}\left[A_{ij} \log B_{a'b}+ (1-A_{ij}) \log (1 - B_{a'b})\right]  \right\} }.    
\end{equation}

\subsection{Classical BCAVI for DCSBM}\label{sec: bcavi in dcsbm}

In this section, we extend the classical BCAVI to DCSBM, a more general class of network models that account for node-degree heterogeneity. Recall that under DCSBM, the edge probabilities are
$P_{ij} = \theta_i \theta_j B_{z_i z_j}$, where $\theta = (\theta_1,...,\theta_n)^T$ is the vector of degree parameters. For identifiability, 
we follow \cite{zhao2012consistency} and make the following assumption.

\begin{assumption}[Degree parameters]
\label{ass:theta}
First, $\theta_1,...,\theta_n$ are independently and identically generated from an unknown distribution with expected value one. Second, there exist positive constants $C_1$ and $C_2$ such that $C_1 \le \theta_i \le C_2 $ for all $1 \leq i \leq n$.
\end{assumption}

Conditioning on the degree parameters, we now derive the BCAVI update equations. In particular, we follow \cite{karrer2011stochastic} and use Poisson distributions to approximate Bernoulli distributions in the full likelihood. As pointed out in \cite{karrer2011stochastic}, this approximation simplifies the likelihood function and leads to more tractable updates of the unknown parameters.
Thus, instead of using the likelihood function of the DCSBM, we consider
\begin{equation*}
  \mathbb{P}(A,Z) = \prod_{ i<j}\prod_{ a,b } \left( \frac{ (\theta_i \theta_j B_{ab}) ^ { A_{ij} } e^{- \theta_i \theta_j B_{ab}} }{ (A_{ij} )!} \right)^{Z_{ia}Z_{jb}} \cdot \prod_{i} \prod_{a} \pi_a^{Z_{ia}}.
\end{equation*}
Similar to Section~\ref{sec: bcavi in sbm}, the evidence lower bound is given by
\begin{eqnarray}
\label{elbo expression in dc}
    \text{ELBO}(\mathbb{Q}) &=& \sum_{i<j}\sum_{ a,b } \Psi_{ia}\Psi_{jb}\Big[A_{ij}\log ( \theta_i \theta_j B_{ab}) - \theta_i \theta_j B_{ab} \Big] + \sum_{i} \sum_{a} \Psi_{ia}\log\left(\pi_a/\Psi_{ia}\right).
\end{eqnarray}
To optimize $\text{ELBO}(\mathbb{Q})$, we alternatively update $\{\Psi,\theta\}$ and $\{B,\pi\}$ at every iteration. 
In particular, we first fix $\{\Psi,\theta\}$ (the first iteration requires an initial value for $\{\Psi,\theta\}$) and set new values for $\{B,\pi\}$ to be the solution of the following equations: 
$$\frac{\partial \text{ELBO}(\mathbb{Q})}{\partial B_{ab}}=0, \quad \frac{\partial\text{ELBO}(\mathbb{Q})}{\partial \pi_{a}}=0, \quad a,b\in[K].$$    
Depending on whether $a=b$, the first equation yields the following update for $B_{ab}$:
\begin{align}
\label{update B in dc}
    B_{aa} = \frac{\sum_{i<j} A_{ij} \Psi_{ia} \Psi_{ja}}{\sum_{i<j} \Psi_{ia} \Psi_{ja} \theta_i \theta_j } , \quad
    B_{ab} = \frac{\sum_{i<j} A_{ij} \left(\Psi_{ia} \Psi_{jb} + \Psi_{ib} \Psi_{ja}\right)}{\sum_{i<j} \left(\Psi_{ia} \Psi_{jb} + \Psi_{ib} \Psi_{ja}\right) \theta_i \theta_j }, \quad a < b.
\end{align}
The update equation for $\pi$ is the same as in Section~\ref{sec: bcavi in sbm}:
\begin{equation}
\label{update pi in dc}
  \pi_a = \frac{\sum_{i} \Psi_{ia}}{\sum_{i}\sum_{b} \Psi_{ib}}, \quad a \in [K].
\end{equation}
Holding $\{B,\pi\}$ fixed, we then update $\Psi$ based on the equation  
\begin{align*}
   \frac{\partial \text{ELBO}(\mathbb{Q})}{\partial \Psi_{ia}} = 0, \quad i\in[n], \quad a \in [K-1],
\end{align*}
which gives for each $i\in[n]$ and $a\in[K]$:
\begin{equation}
\label{update Psi in dc}
\Psi_{ia} = \frac{\pi_a \exp \left\{ \sum_{j: j\neq i} \sum_{b=1}^K \Psi_{jb}\left[A_{ij} \log (\theta_i \theta_j B_{ab} ) - \theta_i \theta_j B_{a b} \right]  \right\} } { \sum_{a'} \pi_{a'} \exp \left\{ \sum_{j: j\neq i} \sum_{b=1}^K \Psi_{jb}\left[A_{ij} \log (\theta_i \theta_j B_{a'b} ) - \theta_i \theta_j b_{a' b} \right]  \right\} }.    
\end{equation}
To update the degree parameters, we take the derivative and set it to zero:
\begin{align*}
\frac{\partial \text{ELBO}(\mathbb{Q})}{\partial \theta_i} = 0, \quad i \in [n].
\end{align*}
The above equation can be simplified to
\begin{align}
\label{update theta in dcsbm}
\frac{\sum_j A_{ij}}{ \theta_i} = \sum_{j:j \neq i } \sum_{a, b} \Psi_{ia} \Psi_{jb} \theta_j B_{ab} , \quad i \in [n].
\end{align}

\subsection{Posterior Threshold}\label{sec:posterior threshold}
The only difference between our proposed method and the classical BCAVI in previous sections is that we add a threshold step after each update of $\Psi$, setting the largest value of each row $\Psi_i$ of $\Psi$ to one and all other values of the same row to zero. That is, for each $i \in [n]$, we further update:
\begin{equation}
\label{threshold Psi}
\Psi_{ i k_i} = 1, \quad \Psi_{ i a} = 0, \quad a \neq k_i,
\end{equation}
where $k_i = \arg \max_{a} \{\Psi_{ia}\}$ and if there are ties, we arbitrarily choose one of the values of $k_i$. The summary of this algorithm for SBM, which we will refer to as Threshold BCAVI or T-BCAVI for SBM, is provided in Algorithm~\ref{alg}.

\begin{algorithm}[h]
  \KwInput{Adjacency matrix A, initialization $\Psi^{(0)}$ and number of iterations $s$.}

\For{$k\leftarrow 1$ \KwTo s} 
{ 
  Compute $B^{(k)}$ according to Equation~\eqref{update B} \;
  Compute $\pi^{(k)}$ according to Equation~\eqref{update pi}\;
  Compute $\Psi^{(k)}$ according to Equation~\eqref{update Psi}\;
  Update $\Psi^{(k)}$ by hard threshold according to Equation~\eqref{threshold Psi}\;
}

\KwOutput{Estimation of label matrix $\Psi^{(s)}$, estimation of block probability matrix $B^{(s)}$ and estimation of parameters $\pi^{(s)}$.}

\caption{Threshold BCAVI for SBM}
\label{alg}
\end{algorithm}

The summary of the algorithm for DCSBM is provided in Algorithm~\ref{alg in DCSBM}. 

\begin{algorithm}
  
  \KwInput{Adjacency matrix A, initialization $\Psi^{(0)}$ and number of iterations $s$.}
 Initialize the degree parameters: $\theta^{(0)}_i = \frac{\sum_{j} A_{i j} n}{\sum_{i j} A_{i j}}, i \in [n].$

\For{$k\leftarrow 1$ \KwTo s} 
{ 
  
  Compute $B^{(k)}$ according to Equation~\eqref{update B in dc}\;
   Compute $\pi^{(k)}$ according to Equation~\eqref{update pi in dc}\; 
  Compute $\Psi^{(k)}$ and update $\Psi^{(k)}$ by hard threshold according to Equation~\eqref{update Psi in dc} and Equation~\eqref{threshold Psi} \; 
  Compute $\theta^{(k)}$ according to Equation~\eqref{update theta in dcsbm} \;
}

\KwOutput{Estimation of label matrix $\Psi^{(s)}$, estimation of block probability matrix $B^{(s)}$, estimation of size parameters $\pi^{(s)}$ and estimation of degree parameters $\theta^{(s)}$.}

\caption{Threshold BCAVI for DCSBM}
\label{alg in DCSBM}
\end{algorithm}

Intuitively, thresholding the label posterior forces variational inference to avoid saddle points by performing a majority vote at each iteration. From an algorithmic point of view, it also performs a selection similar to the majority vote. As some versions of the majority vote have been proven consistent in estimating the community labels \cite{chin2015stochastic,gao2017achieving,lu2016statistical}, this explains why our proposed method is always biased toward the correct membership. The main difference between the selection induced by our harsh thresholding step and other existing versions of the majority vote is that it does not perform on the adjacency matrix. Instead, it is computed based on an adaptive matrix that also depends on the parameters of the label posterior approximation, resulting in a better performance in challenging settings when the networks are sparse, or the initializations are poor. 

Moreover, unlike the naive majority vote that does not take into account model parameters or the majority vote with penalization \cite{gao2017achieving}, which requires a careful design of the penalty term, T-BCAVI performs the selection step automatically by completely relying on its posterior approximation. Once a strategy for posterior approximation is chosen, T-BCAVI does not require an extra model-specific (and often nontrivial) step to determine efficient majority vote updates. From a methodology point of view, this property is desirable because it allows us to readily extend the algorithm to other applications of variational inference, such as clustering mixtures of exponential families.

\subsection{Initialization}\label{sec:initialization}

Similar to other variants of BCAVI \citep{zhang2020theoretical,sarkar2021random,yin2020theoretical}, our algorithm requires an initial value $\Psi^{(0)}=Z^{(0)}$ that is correlated with the true membership matrix $Z$. A natural way to obtain such an initialization is through network data splitting \citep{chin2015stochastic,li2016network}, a popular approach in machine learning and statistics. For this purpose, we fix $\tau\in (0,1/2)$ and create $A^{(\text{init})}\in\{0,1\}^{n\times n}$ by setting all entries of the adjacency matrix $A$ to zero independently with probability $1-\tau$. A spectral clustering algorithm \citep{abbe2017community, qin2013regularized} is applied to $A^{(\text{init})}$ to find the membership matrix $Z^{(0)}$. We then apply our algorithm on $A-A^{(\text{init})}$ using the initialization $\Psi^{(0)}=Z^{(0)}$. 
It is known that under certain conditions on SBM or DCSBM, $Z^{(0)}$ is provably correlated with the true membership matrix, even when the average node degree is bounded \citep{chin2015stochastic,le2017concentration,qin2013regularized}. Numerical study of this approach is given in Section~\ref{simulation}. For the initialization step for $\theta$, we set each element proportional to the observed degree \citep{karrer2011stochastic}. 

\medskip

\section{Theoretical Results}
\label{sec:thoery}
This section provides the theoretical guarantees for the  Threshold BCAVI algorithm described in Section~\ref{sec:TBCAVI}. Although this algorithm works for any block model (either SBM or DCSBM), we will focus on the settings studied by \cite{sarkar2021random} and \cite{yin2020theoretical} for the theoretical analysis. Specifically, we make the following assumption, where, for two sequences $(a_n)_{n=1}^\infty$ and $(b_n)_{n=1}^\infty$ of positive numbers, we write $a_n\asymp b_n$ if there exists a constant $C>0$ such that $a_n/C\le b_n\le Ca_n$. 

\begin{assumption}[Block model assumption]\label{ass:model assumption}
    Consider block models with $K$ communities of equal sizes and the block probability matrix given by $B_{aa} = p$ for all $a$ and $B_{ab} = q$ for all  $a\neq b$. Moreover, $p>q>0$ can vary with $n$ so that $p \asymp q \asymp p - q \asymp \rho_n$, where $\rho_n \to 0$ if $n \to \infty$.
\end{assumption}

Network model satisfying Assumption~\ref{ass:model assumption}  provides a benchmark for studying the behavior of various community detection algorithms \citep{mossel2012stochastic,amini2013pseudo,gao2017achieving,abbe2017community}. Although we only provide theoretical guarantees for networks with equal community sizes, extending the results to networks for which the ratios of community sizes are bounded is straightforward. In the context of VI, these model assumptions significantly simplify the update rules of BCAVI and make the theoretical analysis tractable \citep{sarkar2021random,yin2020theoretical}.

In addition, we assume that the initialization $Z^{(0)}$ can be obtained from the true membership matrix $Z$ by independently perturbing its entries as follows (note that $\varepsilon=(K-1)/K$ corresponds to random guessing). 

\begin{assumption}[Random initialization]
\label{ass:perturb}
Let $\varepsilon\in(0,(K-1)/K)$ be a fixed error rate. Let $z$ be the true label vector and $Z$ be the encoded true membership matrix. Assume that the initialized label vector $z^{(0)}$ is a vector of independent random variables with 
$\mathbb{P}(z^{(0)}_i = z_i) = 1-\varepsilon$ and $\mathbb{P}(z^{(0)}_i = a) = \varepsilon/(K-1), 
\quad a \neq z_i$. Let $Z^{(0)}$ be the encoded membership matrix according to $z^{(0)}$. Moreover, $Z^{(0)}$ and the adjacency matrix $A$ are independent.
\end{assumption}

This assumption is slightly stronger than what we can obtain from the practical initialization involving data splitting and spectral clustering described in Section~\ref{sec:initialization}. It simplifies our analysis and helps us highlight the essential difference between T-BCAVI and BCAVI; a similar assumption is also used in \cite{sarkar2021random}. Based on the proofs provided in Appendix~\ref{all_proof}, it is also possible to drop the independent assumption if the initialization is sufficiently close to the true membership matrix.

To evaluate the clustering accuracy of Threshold BCAVI, we use $\ell_1$ norm to measure the difference between estimated membership matrix $\Psi^{(s)}$ and true membership matrix $Z$ after $s$ steps, where for a matrix $M \in \mathbb{R}^{n \times m } $, we use $||M||_1 = \sum_{i,j } |M_{i j}|$ to denote its $\ell_1$ norm. In \cite{zhang2020theoretical}, they considered all bijections $\phi$ from $[K]$ to $[K]$ and evaluated the error by $ \inf_{\phi} ||\Psi^{(s)} -  \phi(Z)||_1$. In our paper, since we assume that the initialization is from the true membership matrix $Z$ by independently perturbing its entries, we will use $||\Psi^{(s)} -Z||_1$ throughout the paper. It is obvious to notice that $ \inf_{\phi} ||\Psi^{(s)} -  \phi(Z)||_1 \leq ||\Psi^{(s)} -Z||_1$, which means the error defined in \cite{zhang2020theoretical} is always bounded by our expression.

\subsection{Stochastic Block Models}\label{sec:theory sbm}

Under Assumption~\ref{ass:model assumption}, the BCAVI algorithm for SBM has the following update rule:
\begin{align}
\label{update of Psi in proof}
  \Psi_{ia}^{(s)} &\propto \exp\left\{2t^{(s)} \sum_{j:j\neq i} \Psi_{ja}^{(s-1)} (A_{ij} - \lambda^{(s)})\right\}, \quad i = 1,2,...,n,
\end{align}
where the parameters $t^{(s)}$ and $\lambda^{(s)}$ are described below.
After this update we add a threshold step to update $\Psi^{(s)}$ by:
\begin{equation*}
\Psi_{ i k_i}^{(s)} = 1, \quad \Psi_{ i a}^{(s)} = 0, \quad a \neq k_i,\quad i = 1,2,...,n,
\end{equation*}
where $k_i = \arg \max_{a} \{\Psi_{ia}^{(s)}\}$ and if there are ties, we arbitrarily choose one of the values of $k_i$. The update equation for $t^{(s)}$ and $\lambda^{(s)}$ are given by:
\begin{align}
 t^{(s)} = \frac{1}{2}\log \frac{p^{(s)}\left(1-q^{(s)}\right)}{q^{(s)}\left(1-p^{(s)}\right)},\quad
 \lambda^{(s)} = \frac{1}{2t^{(s)}} \log \frac{1-q^{(s)}}{1-p^{(s)}},    
 \label{eq:tlambda}
\end{align}
where $p^{(s)},q^{(s)}$ are the estimates of $p,q$ at the $s$-th iteration, given as follows:
\begin{align}
\label{update of pq in proof}
  p^{(s)} = \frac{\sum_{i<j} \sum_a \Psi_{ia}^{(s-1)} \Psi_{ja}^{(s-1)} A_{ij} }{ \sum_{i<j} \sum_a \Psi_{ia}^{(s-1)} \Psi_{ja}^{(s-1)}}, \quad
  q^{(s)} = \frac{\sum_{i<j} \sum_{a\neq b} \Psi_{ia}^{(s-1)} \Psi_{jb}^{(s-1)} A_{ij} }{ \sum_{i<j} \sum_{a \neq b} \Psi_{ia}^{(s-1)} \Psi_{jb}^{(s-1)}}.
\end{align}
These formulas follow from a direct calculation; for details, see Appendix~\ref{derivation of updated eq}.

We are now ready to state the main theoretical results for SBM. 

\begin{proposition}[Parameter estimation for SBM]
\label{prop:parameter estimation}
Consider SBM that satisfies Assumption~\ref{ass:model assumption}. In addition, assume that the initialization for the Threshold BCAVI satisfies Assumption~\ref{ass:perturb} with fixed error rate $\varepsilon\in(0,(K-1)/K)$. Let 
$d=(n/K - 1)p + n(K-1)q/K $
denote the expected average degree.
Then there exist constants $C,C',c>0$ only depending on $\varepsilon$ and $K$ such that if $d > C$ then with high probability $1-n^{-r}$ for some constant $r>0$,   
\[ t^{(1)} \ge c,\quad \lambda^{(1)} \le C' \rho_n.\]
Moreover, for $s \geq 2$,
\begin{align*}
\big|p^{(s)}-p\big| &\le p\exp(-c d),\quad \big|q^{(s)}-q\big| \le q\exp(-c d),\\
\big|t^{(s)}-p\big| &\le t\exp(-c d),\quad \big|\lambda^{(s)}-\lambda\big| \le \lambda\exp(-c d).
\end{align*} 
\end{proposition}

Proposition~\ref{prop:parameter estimation} shows that if the initialization is better than random guessing, then $p^{(s)}$ and $q^{(s)}$ are sufficiently accurate after only a few iterations, even when $d$ is bounded but sufficiently large. This result is important in proving the accuracy of the Threshold BCAVI in estimating community labels.

\begin{theorem}[Clustering accuracy for SBM]
\label{unknown pa}
Suppose that conditions of Proposition~\ref{prop:parameter estimation} hold.
Then there exist constants $C, c>0$ only depending on $\varepsilon$ and $K$ such that if $d \ge C$ then with high probability $1-n^{-r}$ for some constant $r>0$,
\[\|\Psi^{(s)} - Z\|_1 \leq n\exp(-c d),\]
for every $s\ge 1$, where $\|.\|_1$ denotes the $\ell_1$ norm.
Moreover, for $s \geq 2$, 
\begin{align*}
||\Psi^{(s)} - Z||_1   
 \leq  \frac{C_1 K}{ \delta^2 d} || \Psi^{(s-1)}- Z ||_1 + \frac{nK\exp(-C_2 d)}{\delta d} +  Kn \exp\left(\frac{-\delta^2d}{1+\delta}\right),
\end{align*}
where $\delta = \frac{C_3(p-q)}{p + Kq}$ and $C_1,C_2,C_3>0$ are some absolute constants.
\end{theorem}

The bound in Theorem~\ref{unknown pa} implies that the fraction of incorrectly labeled nodes is exponentially small in the average degree $d$. In addition, we can recursively use the second inequality in Theorem~\ref{unknown pa} to show the estimation error as a function of the iteration number $s$. To that end, we rewrite the second inequality in Theorem~\ref{unknown pa} as follows:  \[||\Psi^{(s)} - Z||_1 \leq a || \Psi^{(s-1)} - Z||_1 + b,\] where $a$ and $b$ are some constants only depending on $p$, $q$, $n$, $\epsilon$ and $K$. Note that $a<1$ because of the assumption $n \rho_n > C$ for some constant $C$ only depending on $K$ and $\varepsilon$. In addition, $b=O(ne^{-cd})$ for some constant $c>0$. Applying this inequality recursively $s$ times yields 
\[ ||\Psi^{(s)} - Z||_1 \leq a^s || \Psi^{(0)} - Z||_1 + \frac{b(a^s-1)}{a-1},\] which gives the estimation error of $\Psi^{(s)}$ as a function of the iteration number.

The result in Theorem~\ref{unknown pa} is comparable with the bound obtained by \cite{sarkar2021random} for BCAVI in the regime when $d$ grows at least as $\log n$. When $d$ is bounded, CommuLloyd algorithm proposed in \cite{lu2016statistical} satisfies a similar bound. 
One advantage of our method over the CommuLloyd algorithm is that it only requires the initialization to be better than the random guessing, while the initialization for the CommuLloyd algorithm must be smaller than $1/4$ (the assumption that entries of initialization are independent in our analysis can also be dropped if the initialization is sufficiently close to the true membership matrix). In addition, the theoretical guarantee in \cite{lu2016statistical} is only proved if the number of iterations $s$ is at most $3\log n$ while our result holds for any $s$. 
Finally, the extensive numerical study in Section 4 shows that the proposed method outperforms the CommuLloyd algorithm (which is essentially equivalent to a majority vote algorithm) in almost all settings, especially when the networks are sparse or the initializations are poor.

Besides estimating community labels, Proposition~\ref{prop:parameter estimation} shows that the Threshold BCAVI also provides good estimates of the unknown parameters. This property is essential not only for proving Theorem~\ref{unknown pa} but also for statistical inference purposes. In this direction, \cite{bickel2013asymptotic} shows that if the exact optimizer of the variational approximation can be calculated, then the parameter estimates of VI for SBM converge to normal random variables. Our following result shows that the same property also holds for the Threshold BCAVI in the regime that $d$ grows at least as $\log n$. It will be interesting to see if the condition $d\gg\log n$ can be removed.

\begin{theorem}[Limiting distribution of parameter estimates for SBM]
\label{dist}
Suppose that conditions of Proposition~\ref{prop:parameter estimation} hold and $d\ge C \log n$ for some large constant $C$ only depending on $\varepsilon$ and $K$. Then for every $s\ge 2$, as $n$ tends to infinity,     
$(p^{(s)},q^{(s)})$ converges in distribution to a bi-variate Gaussian vector:  
\[ \begin{pmatrix}n (p^{(s)} - p)/\sqrt{p}\\n(q^{(s)} - q)/ \sqrt{q} \end{pmatrix} \to N\left(\begin{pmatrix} 0\\0\end{pmatrix},\begin{pmatrix} 2K & 0\\0 & 2K/(K-1)\end{pmatrix}\right).
\]
\end{theorem}

Given the asymptotic distribution of $(p^{(s)},q^{(s)})$, we can construct joint confidence intervals for the unknown parameters $p$ and $q$.

\subsection{Degree-corrected Stochastic Block Models}\label{sec:theory dcsbm}

We now consider DCSBM satisfying Assumption~\ref{ass:model assumption}. In this setting, the BCAVI algorithm for DCSBM has the following update rule:

\begin{align}
\label{update of Psi in dc in proof}
  \Psi_{ia}^{(s)} &\propto \exp\left\{2t^{(s)} \sum_{j:j\neq i} \Psi_{ja}^{(s-1)} (A_{ij} - \theta_i^{(s-1)} \theta_j^{(s-1)} \lambda^{(s)})\right\}, \quad i = 1,2,...,n,
\end{align}
where the parameters $t^{(s)}$ and $\lambda^{(s)}$ are described below.
After this update we add a threshold step to update $\Psi^{(s)}$ by:
\begin{equation*}
\Psi_{ i k_i}^{(s)} = 1, \quad \Psi_{ i a}^{(s)} = 0, \quad a \neq k_i,\quad i = 1,2,...,n,
\end{equation*}
where $k_i = \arg \max_{a} \{\Psi_{ia}^{(s)}\}$ and if there are ties, we arbitrarily choose one of the values of $k_i$. The update equation for $\theta^{(s)}$ is given by: 
\begin{align}
\label{update of theta in proof}
\frac{\sum_j A_{ij}}{ \theta_i^{(s)}} = \sum_{j:j \neq i } \sum_{a \neq b} \Psi_{ia}^{(s-1)} \Psi_{jb}^{(s-1)} \theta_j^{(s-1)} q^{(s)} + \sum_{j:j \neq i } \sum_{a} \Psi_{ia}^{(s-1)} \Psi_{ja}^{(s-1)} \theta_j^{(s-1)} p^{(s)}, \quad i \in [n],
\end{align}
where the parameters $p^{(s)}$ and $q^{(s)}$ are described below. For theoretical analysis, we follow \citep{karrer2011stochastic} to rescale $\theta_i^{(s)}$ within each estimated community so that their sum is exactly $n/K$. That is, for each $a \in [K]$, we update $\theta_i$ by:
\begin{align}
\label{threshold theta in dcsbm in proof}
\theta_i^{(s)} = \frac{ \theta_i^{(s)} n/K }{ \sum_i \Psi_{ia}^{(s)} \theta_i^{(s)}}, \quad \text{if } \Psi_{ia}^{(s)} = 1.
\end{align}
Similar to hard thresholding $\Psi$, this step simplifies our theoretical analysis. It is also consistent with Assumption~\ref{ass:theta}, which implies that the sum of the true degree parameters within each true community highly concentrates on the community size. In practice, we find that the performance of T-BCAVI with initialization satisfying Assumption~\ref{ass:theta} essentially does not depend on whether equation \eqref{threshold theta in dcsbm in proof} is used. For this reason, we omit the rescaling step in practice and implement  Algorithm~\ref{alg in DCSBM}; for details, see Section~\ref{real data} and Appendix~\ref{sec: simu of dcsbm}. Finally, the update equations for $t^{(s)}$ and $\lambda^{(s)}$ are given by:
\begin{align}
 t^{(s)} = \frac{1}{2}\log \frac{p^{(s)}}{q^{(s)}},\quad
 \lambda^{(s)} = \frac{1}{2t^{(s)}} ( p^{(s)} - q^{(s)}),    
 \label{eq:tlambda in dc}
\end{align}
where $p^{(s)},q^{(s)}$ are the estimates of $p,q$ at the $s$-th iteration, given as follows:
\begin{align}
\label{update of pq in dc in proof}
  p^{(s)} = \frac{\sum_{i<j} \sum_a \Psi_{ia}^{(s-1)} \Psi_{ja}^{(s-1)} A_{ij} }{ \sum_{i<j} \sum_a \Psi_{ia}^{(s-1)} \Psi_{ja}^{(s-1)} \theta_i^{(s-1)} \theta_j^{(s-1)} }, \quad
  q^{(s)} = \frac{\sum_{i<j} \sum_{a\neq b} \Psi_{ia}^{(s-1)} \Psi_{jb}^{(s-1)} A_{ij} }{ \sum_{i<j} \sum_{a \neq b} \Psi_{ia}^{(s-1)} \Psi_{jb}^{(s-1)} \theta_i^{(s-1)} \theta_j^{(s-1)} }.
\end{align}
These formulas follow from a direct calculation; for details, see Appendix~\ref{derivation of updated eq in dc}.

The following proposition gives error bounds for parameter estimation for DCSBM. 

\begin{proposition}[Parameter estimation for DCSBM]
\label{prop:parameter estimation in DCSBM}
Consider DCSBM satisfying Assumption~\ref{ass:theta} and Assumption~\ref{ass:model assumption}. In addition, assume that the initialization satisfies Assumption~\ref{ass:perturb} with a fixed error rate $\varepsilon\in(0,(K-1)/K)$. Let
$$d=(n/K - 1)p + n(K-1)q/K, \qquad d_{\min} = \min_{1\le i\le n} \mathbb{E}\sum_{j=1}^n A_{ij}$$ 
denote the expected average degree and minimum expected node degree, respectively.
Then there exist constants $C,C',c>0$ only depending on $\varepsilon$ and $K$ such that if $d > C$ then with probability at least $1-n^{-r}$ for some constant $r>0$,  
\[ t^{(1)} \ge c,\quad \lambda^{(1)} \le C' \rho_n.\]
Moreover, for $s \geq 2$,
\begin{align*}
\big|p^{(s)}-p\big| &\le p \cdot \left( \exp(-c d) + C' n^{-1/3} \right) ,\quad \big|q^{(s)}-q\big| \le q \cdot \left( \exp(-c d) + C' n^{-1/3} \right),\\
\big|t^{(s)}-p\big| &\le t \cdot \left( \exp(-c d) + C' n^{-1/3} \right),\quad \big|\lambda^{(s)}-\lambda\big| \le \lambda \cdot \left( \exp(-c d)+ C' n^{-1/3} \right).
\end{align*} 
\end{proposition}

Compared to the bounds for SBM, the bounds in Proposition~\ref{prop:parameter estimation in DCSBM} contain an extra term of order $n^{-1/3}$. Although it is unclear if this term is of optimal order, its presence is expected. This is because the threshold step in \eqref{threshold theta in dcsbm in proof} forces the sum of estimated degree parameters within each community to be exactly $n/K$, while under Assumption~\ref{ass:theta}, the sum of the true degree parameters within each community typically deviates from $n/K$ by an order of $n^{1/2}$. Nevertheless, $n^{-1/3}$ is sufficiently small for our purposes.

\begin{theorem}[Clustering accuracy of Threshold BCAVI in DCSBM]
\label{unknown pa (dcsbm)}
Suppose that conditions of Proposition~\ref{prop:parameter estimation in DCSBM} hold.
Then there exist constants $C, c>0$ only depending on $\varepsilon$ and $K$ such that if $n\rho_n \ge C$ then with high probability $1-n^{-r}$ for some constant $r>0$,
\[\|\Psi^{(s)} - Z\|_1 \leq n\exp(-c d),\]
for every $s\ge 1$, where $\|.\|_1$ denotes the $\ell_1$ norm.
Moreover, for $s \geq 2$, 
\begin{align*}
||\Psi^{(s)} - Z||_1   
 \leq  \frac{C_1 K}{ \delta^2 d_{\min}} || \Psi^{(s-1)}- Z ||_1 + \frac{ndK\exp(-C_2 d)}{\delta d_{\min}^2} +  K \sum_{i=1}^n \exp\left(\frac{-\delta^2d_i}{1+\delta}\right),
\end{align*}
where $\delta = \frac{C_3 (p-q)}{p + Kq}$ and $C_1,C_2,C_3>0$ are some absolute constants. 
\end{theorem}

 To the best of our knowledge, this is the first theoretical result of BCAVI for DCSBM. Furthermore, the bounds provided here are comparable with the minimax risk results of DCSBM in \cite{gao2018community}. Thus, we have proved that T-BCAVI can attain the rate-optimal clustering accuracy under DCSBM. Note also that the bounds in this theorem are similar to those for SBM in Theorem~\ref{unknown pa}; for a more detailed discussion about these bounds, see the paragraph following Theorem~\ref{unknown pa}.

\section{Numerical Studies}\label{simulation}
This section compares the performance of Threshold BCAVI (T-BCAVI), the classical version without the threshold step (BCAVI), majority vote (MV) (which iteratively updates node labels by assigning nodes to the communities they have the most connections), and the version of majority vote with penalization (P-MV) of \cite{gao2017achieving}. For a thorough numerical analysis, we will consider both balanced (equal community sizes) and unbalanced (different community sizes) network models with $K=2$ and $K=3$ communities, although our theoretical results in Section~\ref{sec:thoery} are only proved for balanced models. Overall, the qualitative behaviors of the two versions of T-BCAVI based on SBM and DCSBM are very similar for networks drawn from SBM and DCSBM, respectively. For this reason, the current section mainly presents the results for SBM-based T-BCAVI and networks generated from SBM (except Section~\ref{real data} where both SBM and DCSBM-based T-BCAVI are applied to fit real data); the corresponding study for DCSBM is provided in Appendix~\ref{sec: simu of dcsbm}.  


\subsection{Simulated Networks}\label{sec:simulated}

\medskip
\noindent
{\em (a) Communities with idealized initialization}. 
We first provide numerical support for the theoretical results in Section~\ref{sec:thoery} and compare the methods listed above. To this end, we first consider SBM with $n=600$ nodes and $K=2$ or $K=3$ communities. When $K=2$, the sizes of communities are fixed to be $n_1=n_2=300$ in balanced settings and $n_1 = 240, n_2 = 360$ in unbalanced settings. When $K=3$, the sizes of communities are fixed to be $n_1=n_2=n_3 = 200$ in balanced settings and $n_1 = 150, n_2 = 210, n_3 = 240$ in unbalanced settings. The block probability matrix is generated according to Assumption~\ref{ass:model assumption}, where the true model parameters $p$ and $q$ are chosen so that $p/q=10/3$ and the expected average degree $d$ can vary. We generate initializations $Z^{(0)}$ for all methods from true membership matrix $Z$ according to Assumption~\ref{ass:perturb} with various values of the error rate $\varepsilon$. The accuracy of the estimated membership matrix $\Psi \in \{0,1\}^n$ is measured by the fraction of correctly labeled nodes $\inf_{\phi} || \Psi - \phi (Z) ||_1 / (2n)$, where the infimum is over all label permutations $\phi$. Note that according to this measure, choosing node labels independently and uniformly at random results in the baseline accuracy of approximately $1/2$ when $K=2$ and $1/3$ when $K=3$ in the balanced network. For each setting, we report this clustering accuracy averaged over 100 replications with one standard deviation band. For reference, the actual accuracy of initializations $Z^{(0)}$ (RI), which is similar to $1-\varepsilon$, is also included. Since the simulation result is similar in DCSBM, we will include it in the appendix.

Figure~\ref{eps} shows that T-BCAVI performs uniformly better than BCAVI, MV, and P-MV in balanced settings for both $K=2$ and $K=3$ networks. In particular, the improvement is significant when $d$ is small or $\varepsilon$ is large. Since real-world networks are often sparse and the accuracy of initialization is usually poor, this improvement highlights the practical importance of our threshold strategy. Figure~\ref{Ueps} shows that T-BCAVI also completely outperforms other methods in unbalanced settings for both $K=2$ and $K=3$ networks. In general, T-BCAVI does not require accurate initializations, even in sparse networks.

\begin{figure}[ht]
\begin{subfigure}{.33\textwidth}
  \centering
  \includegraphics[width=1\linewidth]{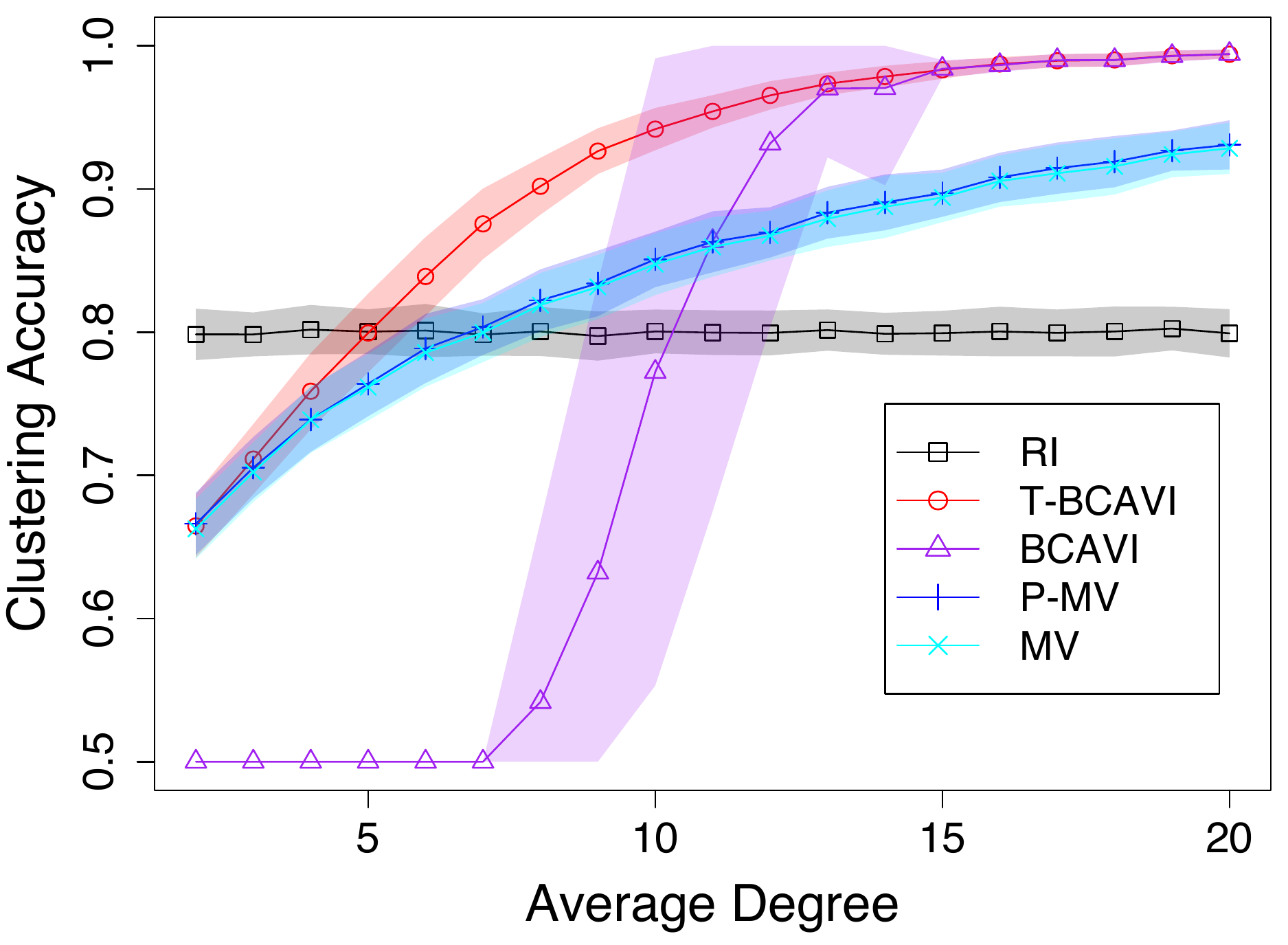}  
  \caption{$\varepsilon = 0.2$}
  \label{eps0.2}
\end{subfigure}
\begin{subfigure}{.33\textwidth}
  \centering
  \includegraphics[width=1\linewidth]{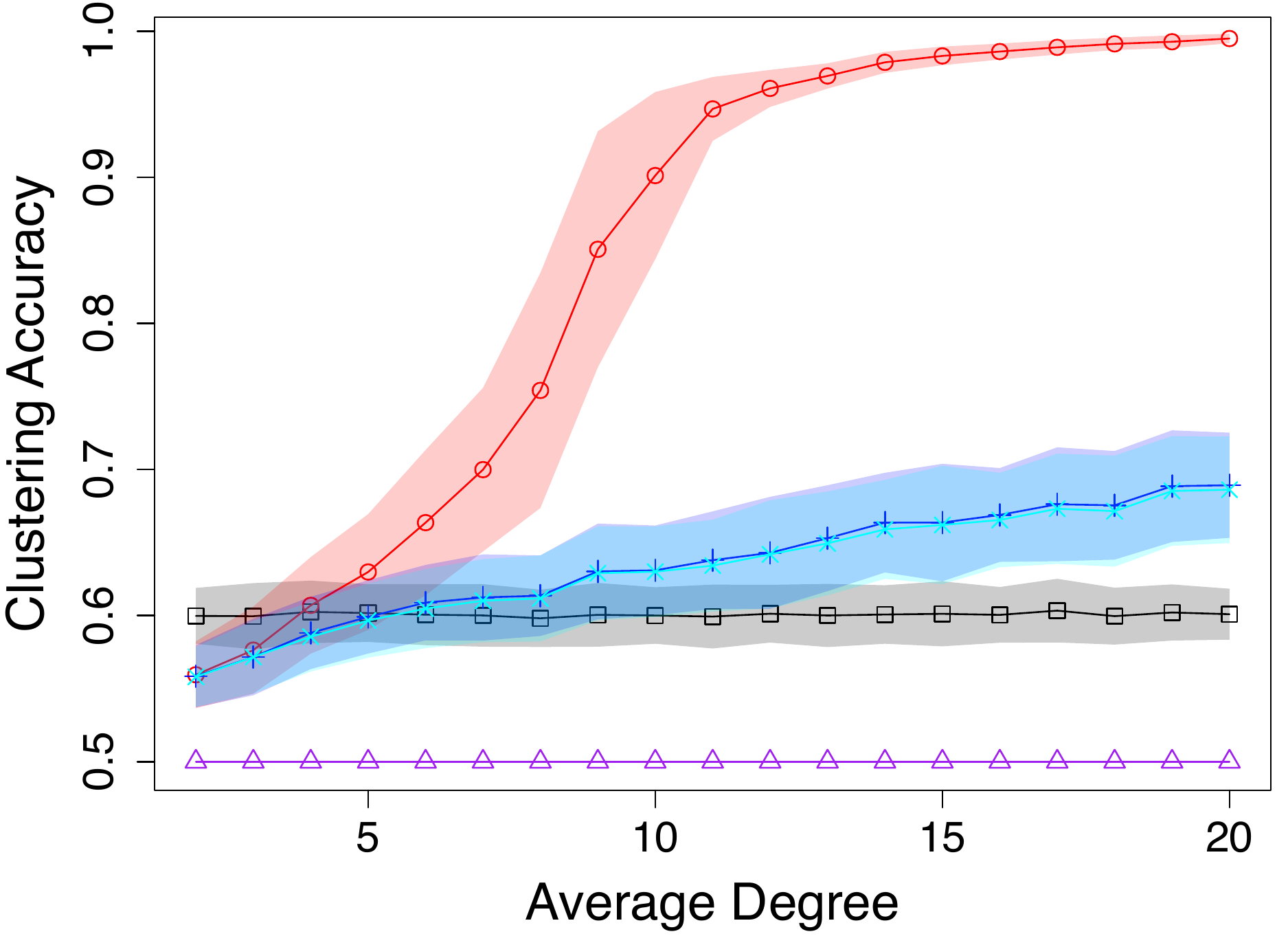}  
  \caption{$\varepsilon = 0.4$}
  \label{eps0.4}
\end{subfigure}
\begin{subfigure}{.33\textwidth}
  \centering
  \includegraphics[width=1\linewidth]{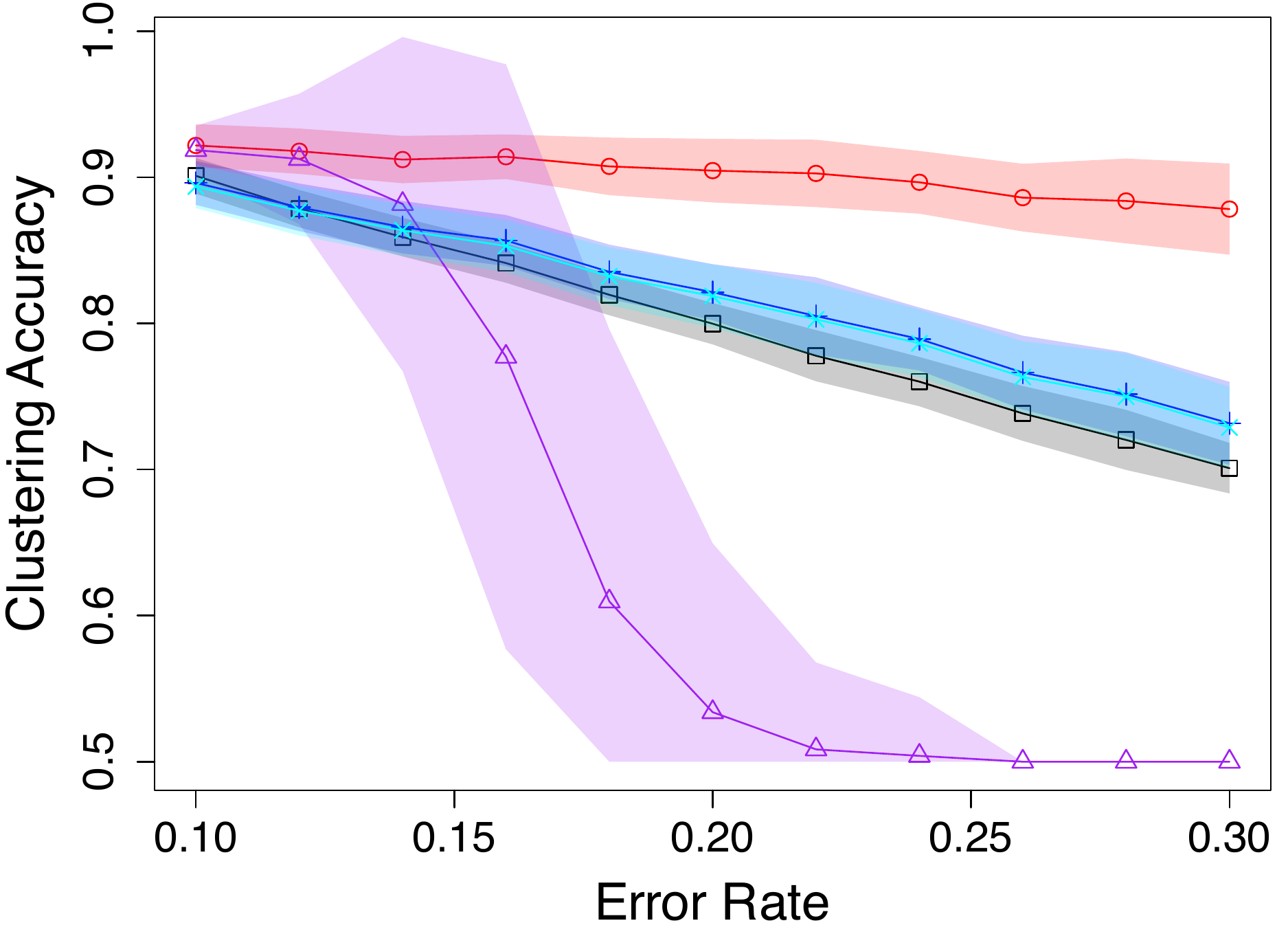}  
  \caption{$d = 8$}
  \label{avd8}
\end{subfigure}
\begin{subfigure}{.33\textwidth}
  \centering
  \includegraphics[width=1\linewidth]{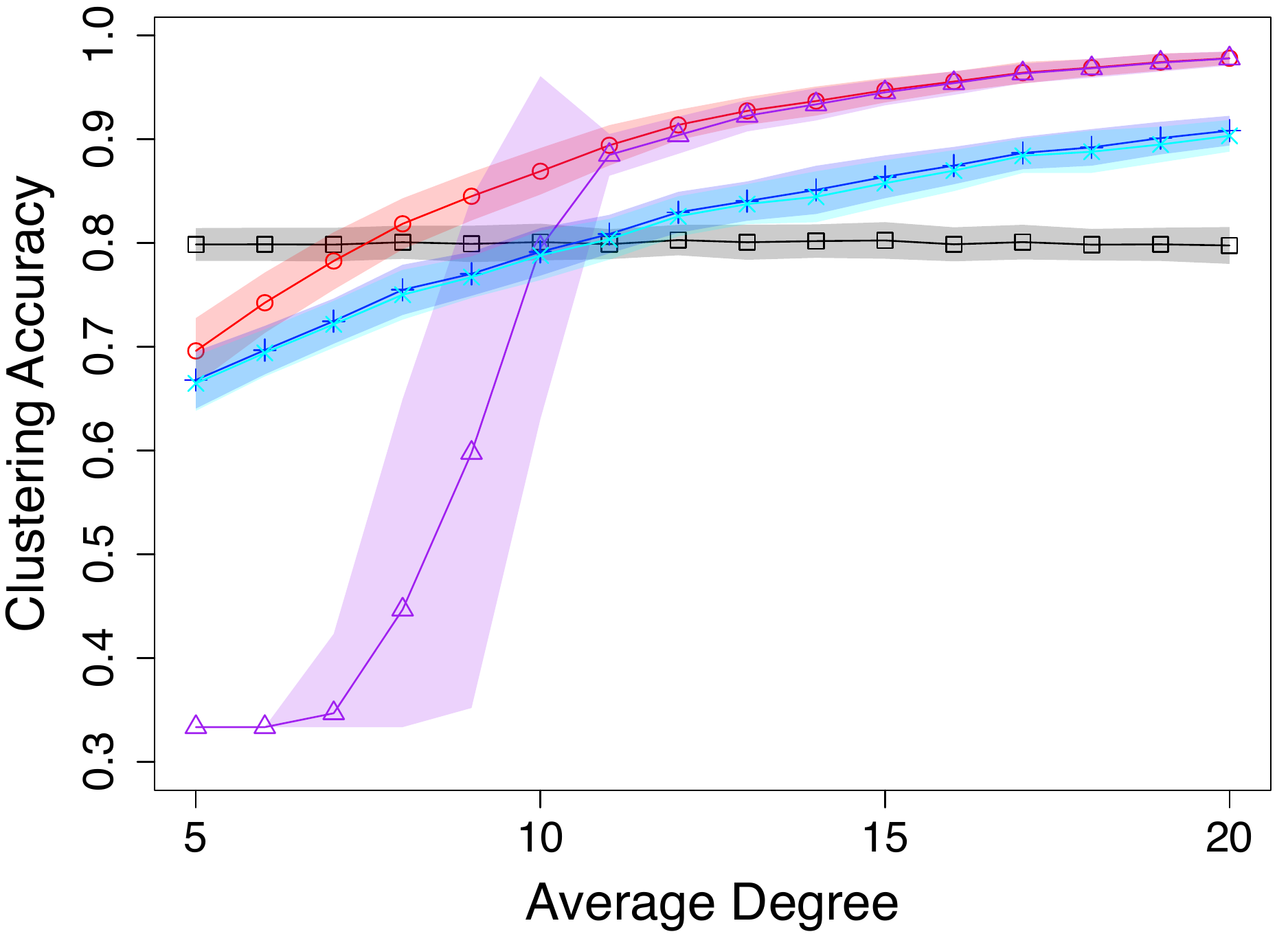}  
  \caption{$\varepsilon = 0.2$}
  \label{eps0.2(k=3)}
\end{subfigure}
\begin{subfigure}{.33\textwidth}
  \centering
  \includegraphics[width=1\linewidth]{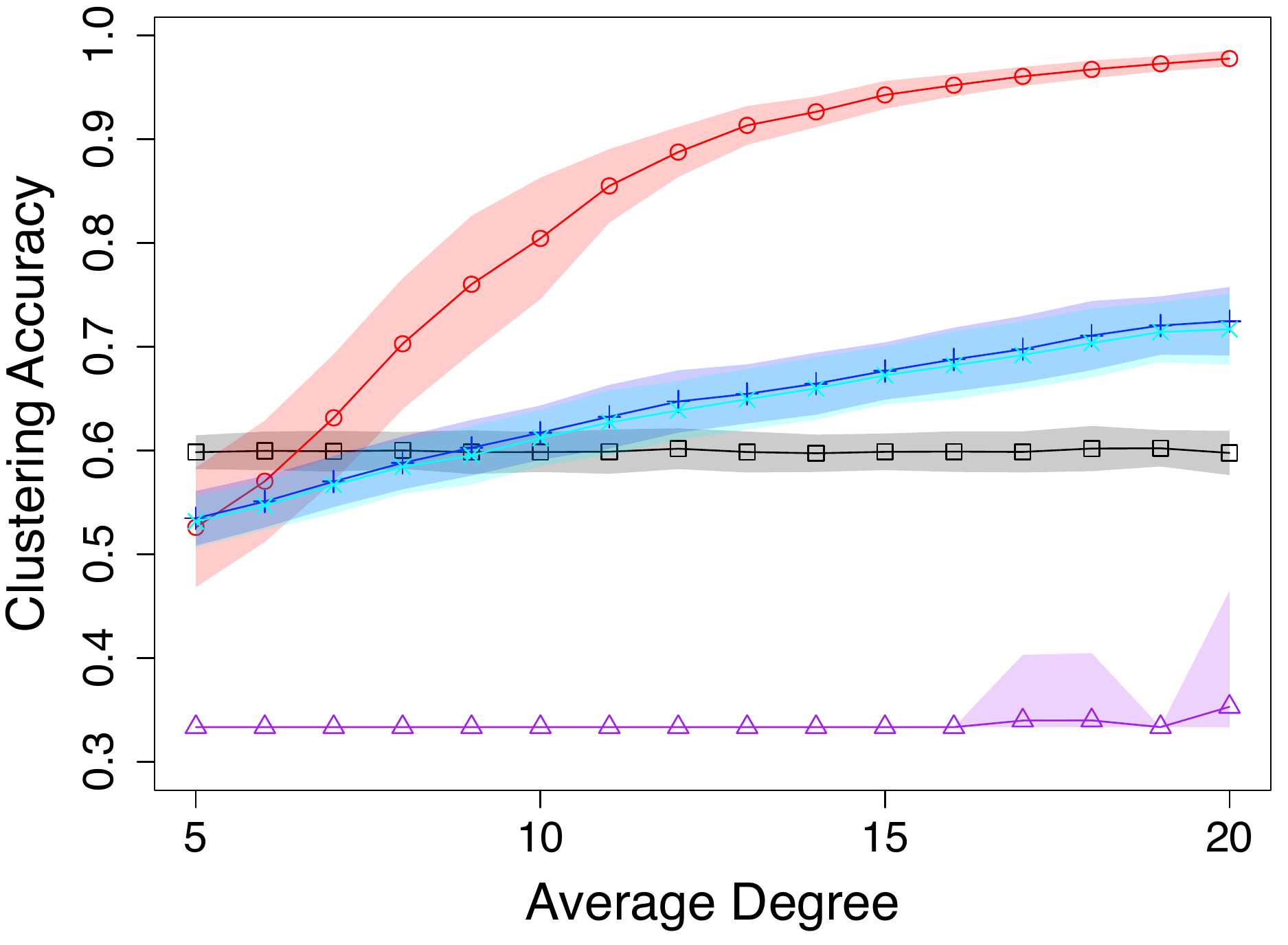}  
  \caption{$\varepsilon = 0.4$}
  \label{eps0.4(k=3)}
\end{subfigure}
\begin{subfigure}{.33\textwidth}
  \centering
  \includegraphics[width=1\linewidth]{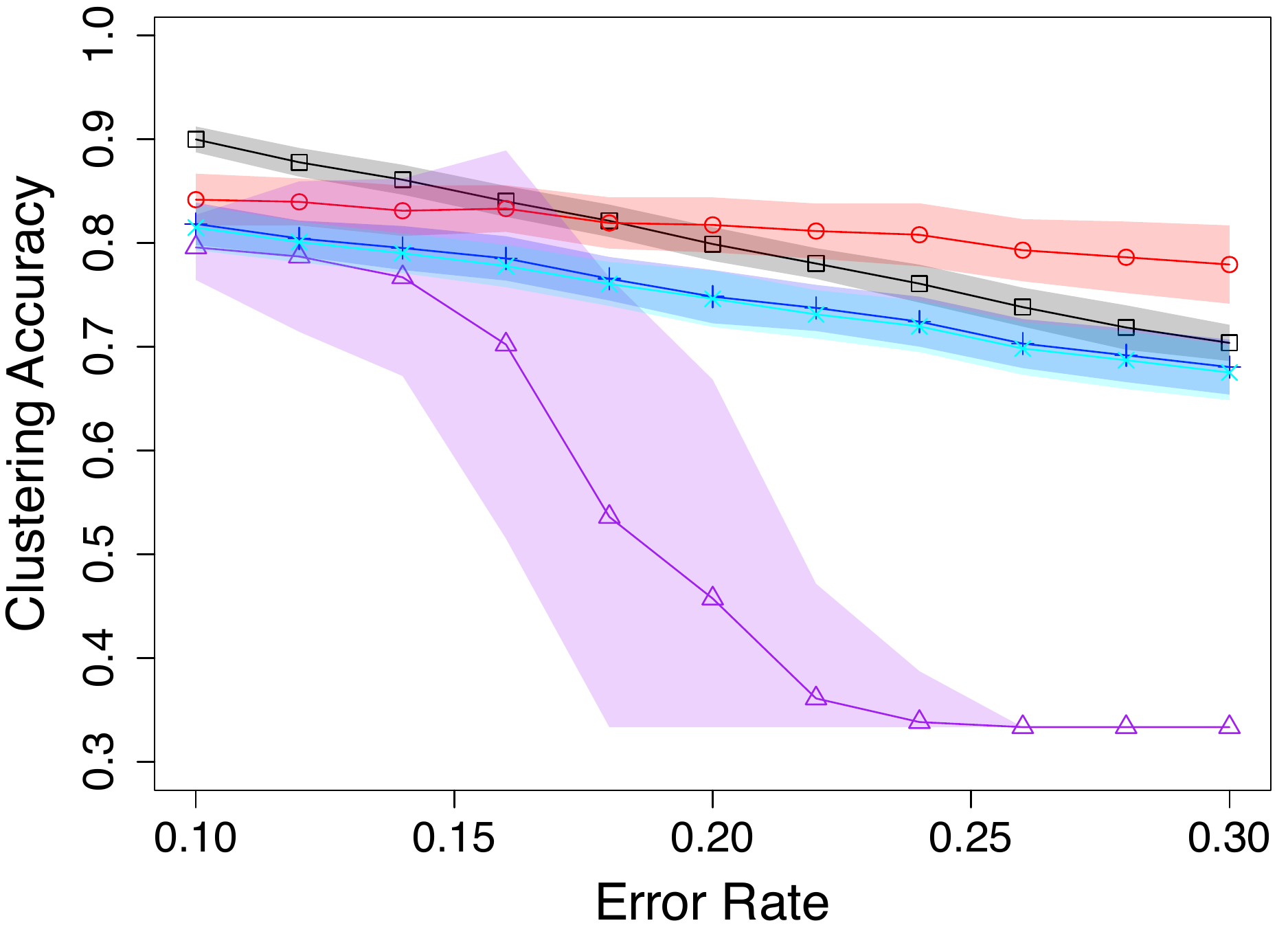}  
  \caption{$d = 8$}
  \label{avd8(k=3)}
\end{subfigure}
\caption{Performance of Threshold BCAVI (T-BCAVI), the classical BCAVI, majority vote (MV), and majority vote with penalization (P-MV) in balanced settings. (a)-(c): Networks are generated from SBM with $n=600$ nodes, $K=2$ communities of sizes $n_1 = n_2 = 300$. (d)-(f): Networks are generated from SBM with $n=600$ nodes, $K=3$ communities of sizes $n_1 = n_2 = n_3 = 200$. Initializations are generated from true node labels according to Assumption~\protect\ref{ass:perturb} with error rate $\varepsilon$, resulting in actual clustering initialization accuracy (RI) of approximately $1-\varepsilon$.}
\label{eps}
\end{figure}

\begin{figure}[ht]
\begin{subfigure}{.33\textwidth}
  \centering
  \includegraphics[width=1\linewidth]{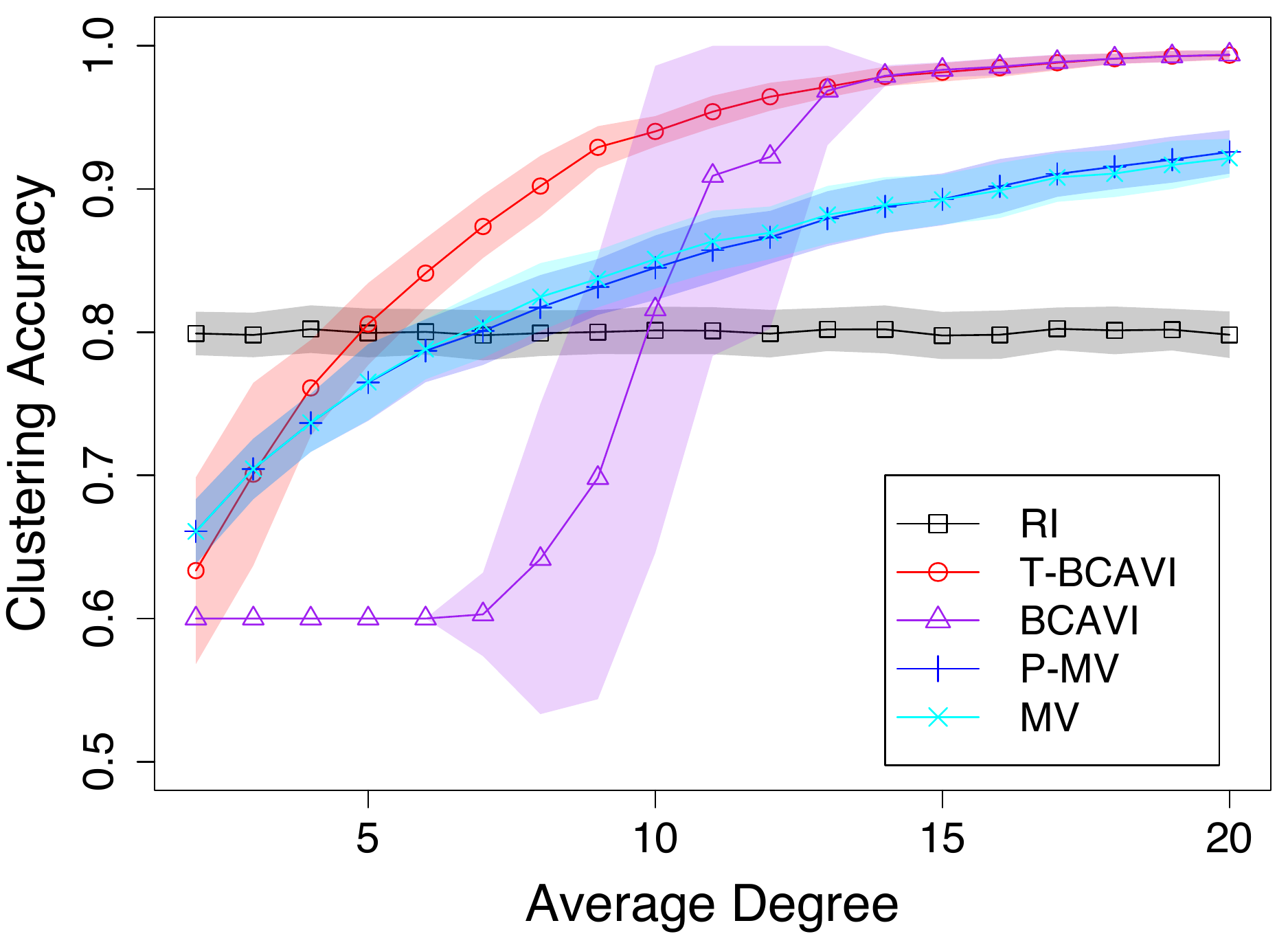}  
  \caption{$\varepsilon = 0.2$}
  \label{Ueps0.2}
\end{subfigure}
\begin{subfigure}{.33\textwidth}
  \centering
  \includegraphics[width=1\linewidth]{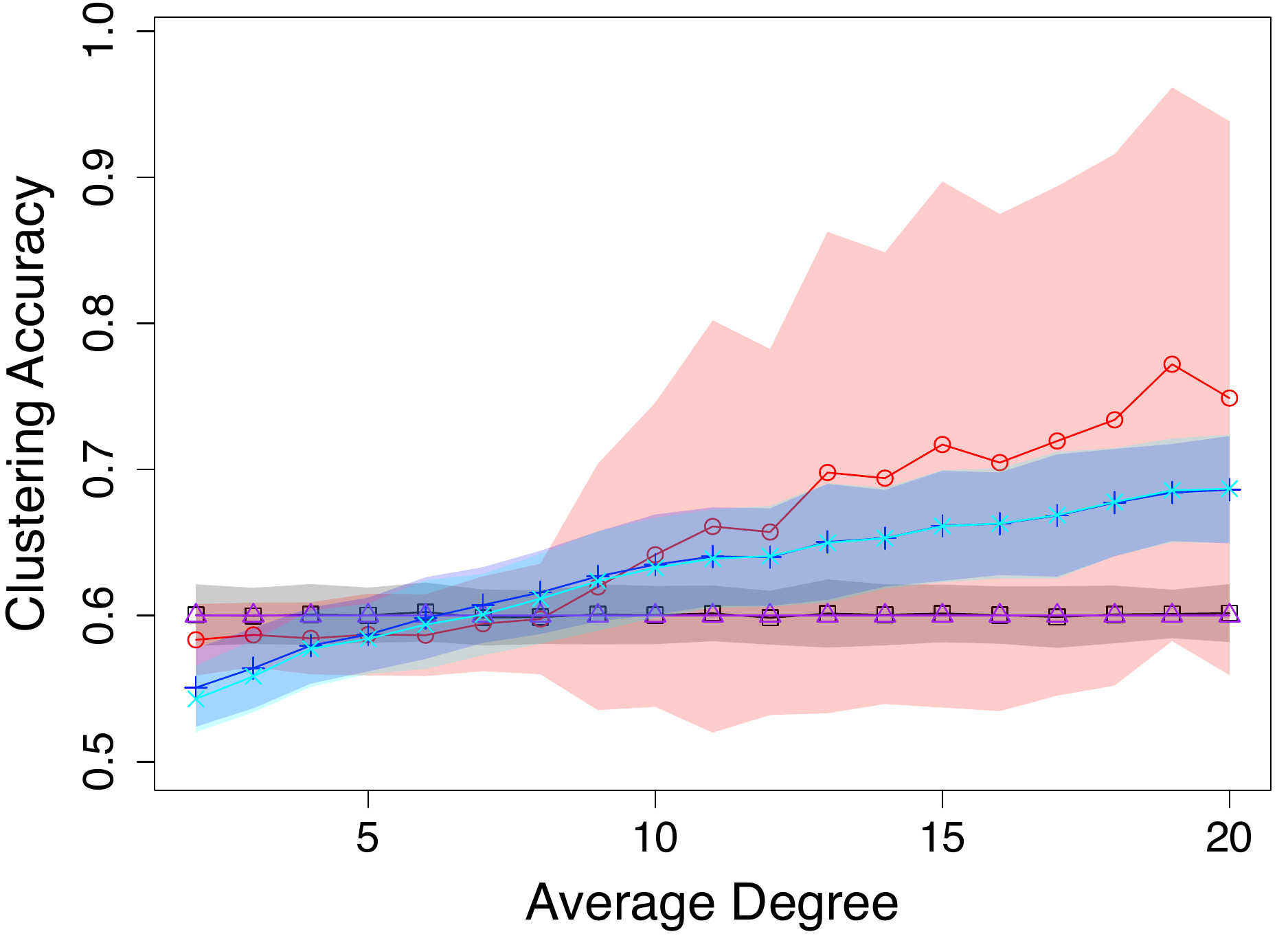}  
  \caption{$\varepsilon = 0.4$}
  \label{Ueps0.4}
\end{subfigure}
\begin{subfigure}{.33\textwidth}
  \centering
  \includegraphics[width=1\linewidth]{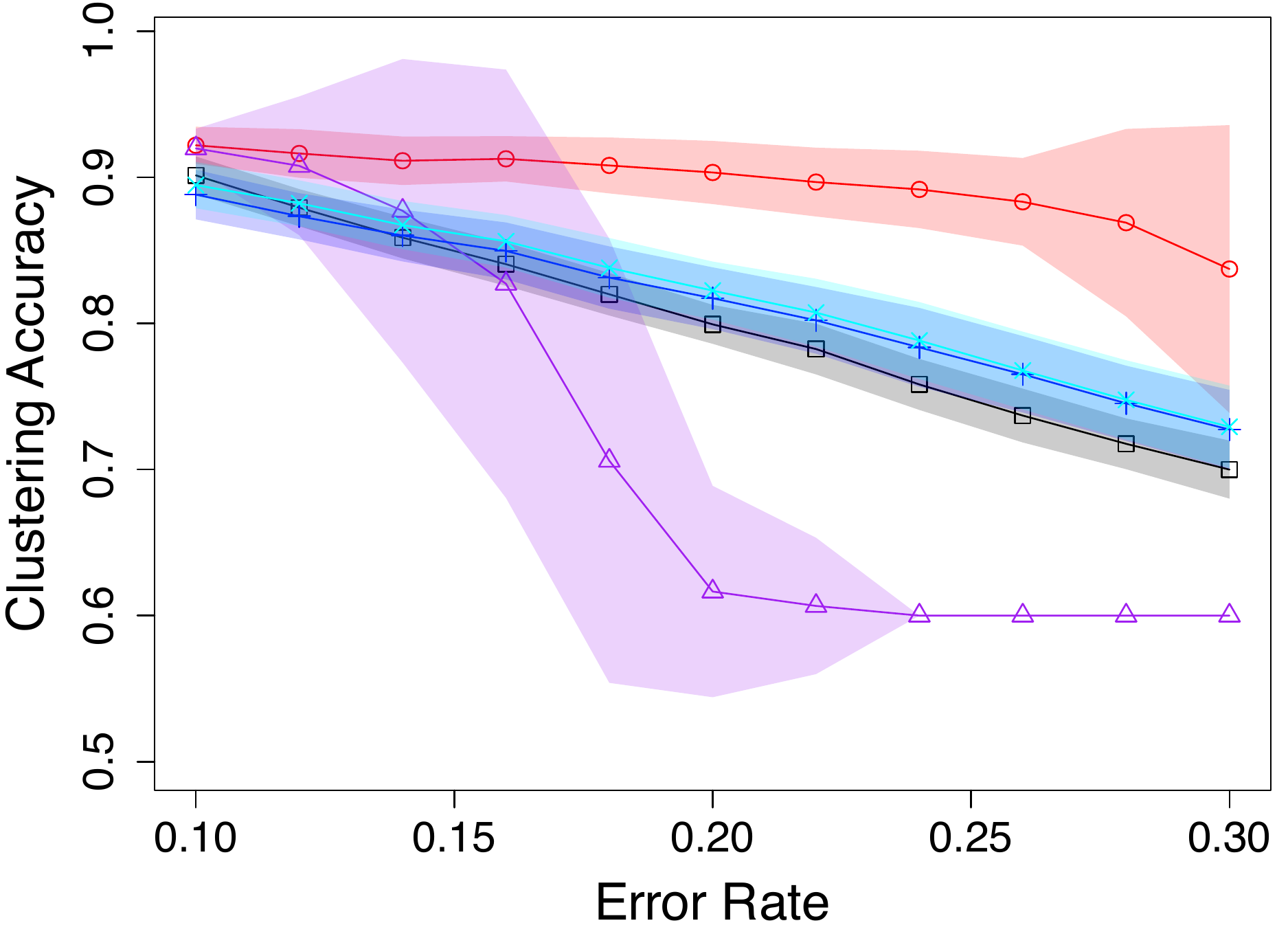}  
  \caption{$d = 8 $}
  \label{Uavd8}
\end{subfigure}
\begin{subfigure}{.33\textwidth}
  \centering
  \includegraphics[width=1\linewidth]{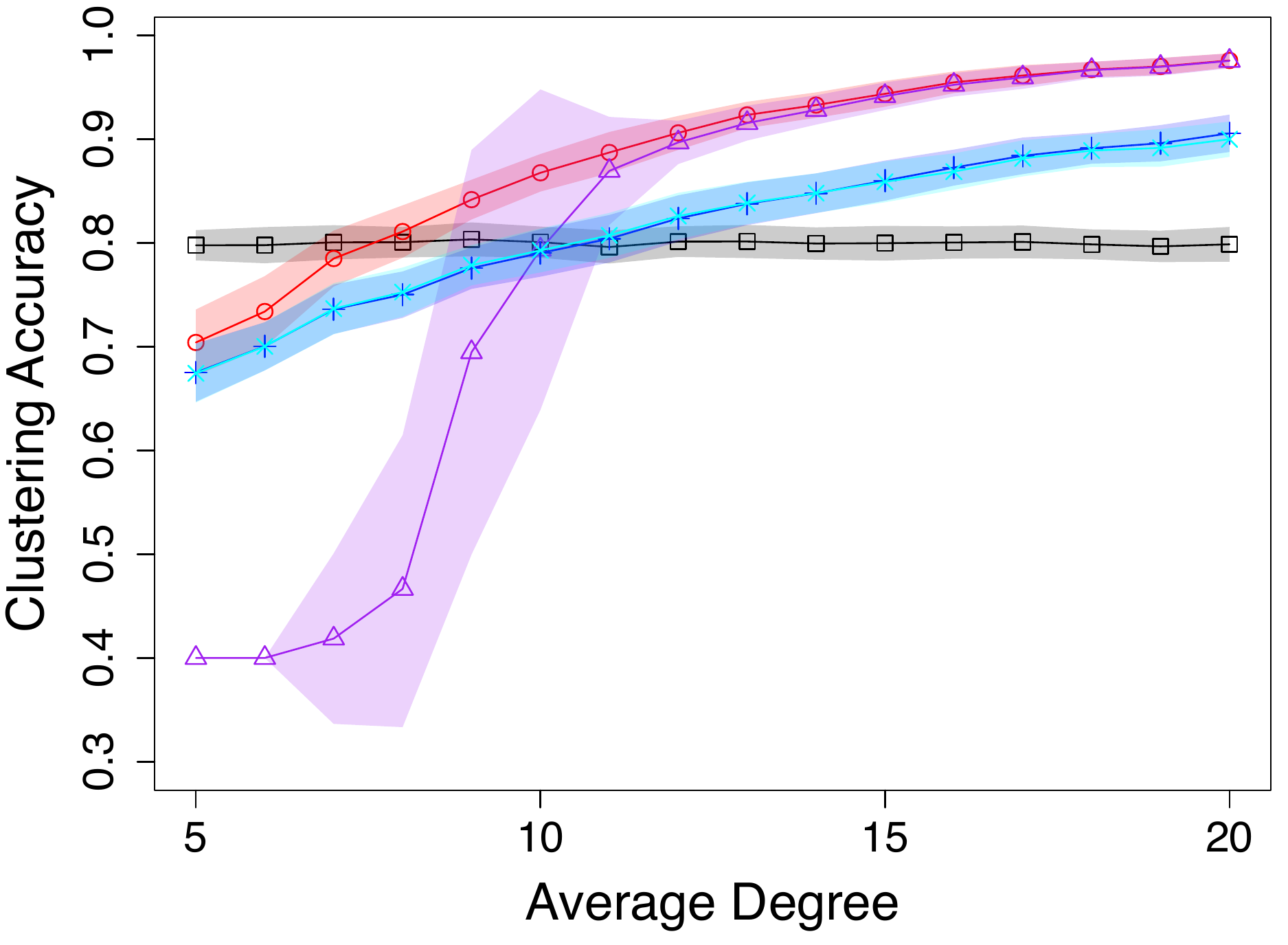}  
  \caption{$\varepsilon = 0.2$}
  \label{Ueps0.2(k=3)}
\end{subfigure}
\begin{subfigure}{.33\textwidth}
  \centering
  \includegraphics[width=1\linewidth]{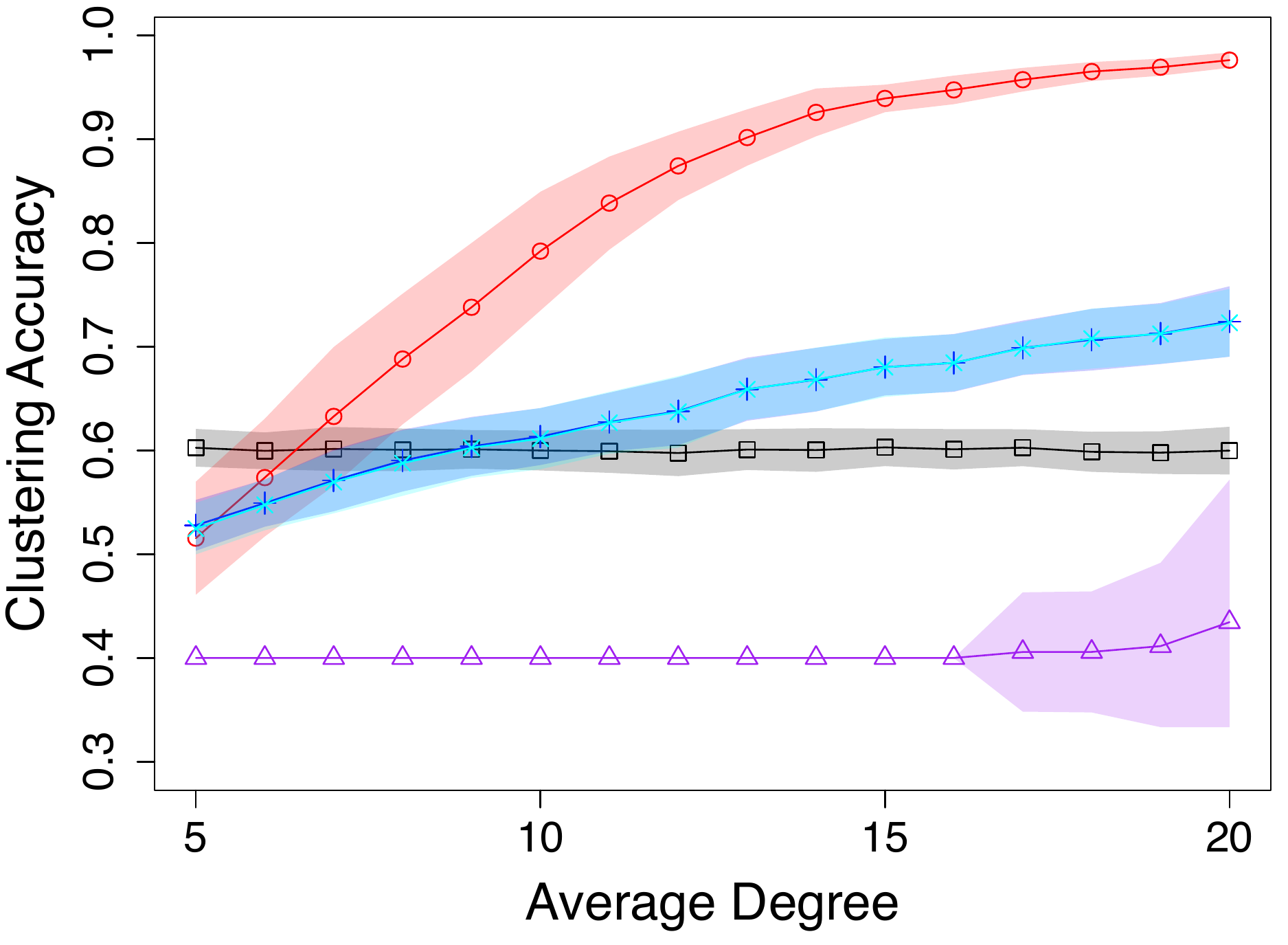} 
  \caption{$\varepsilon = 0.4$}
  \label{Ueps0.4(k=3)}
\end{subfigure}
\begin{subfigure}{.33\textwidth}
  \centering
  \includegraphics[width=1\linewidth]{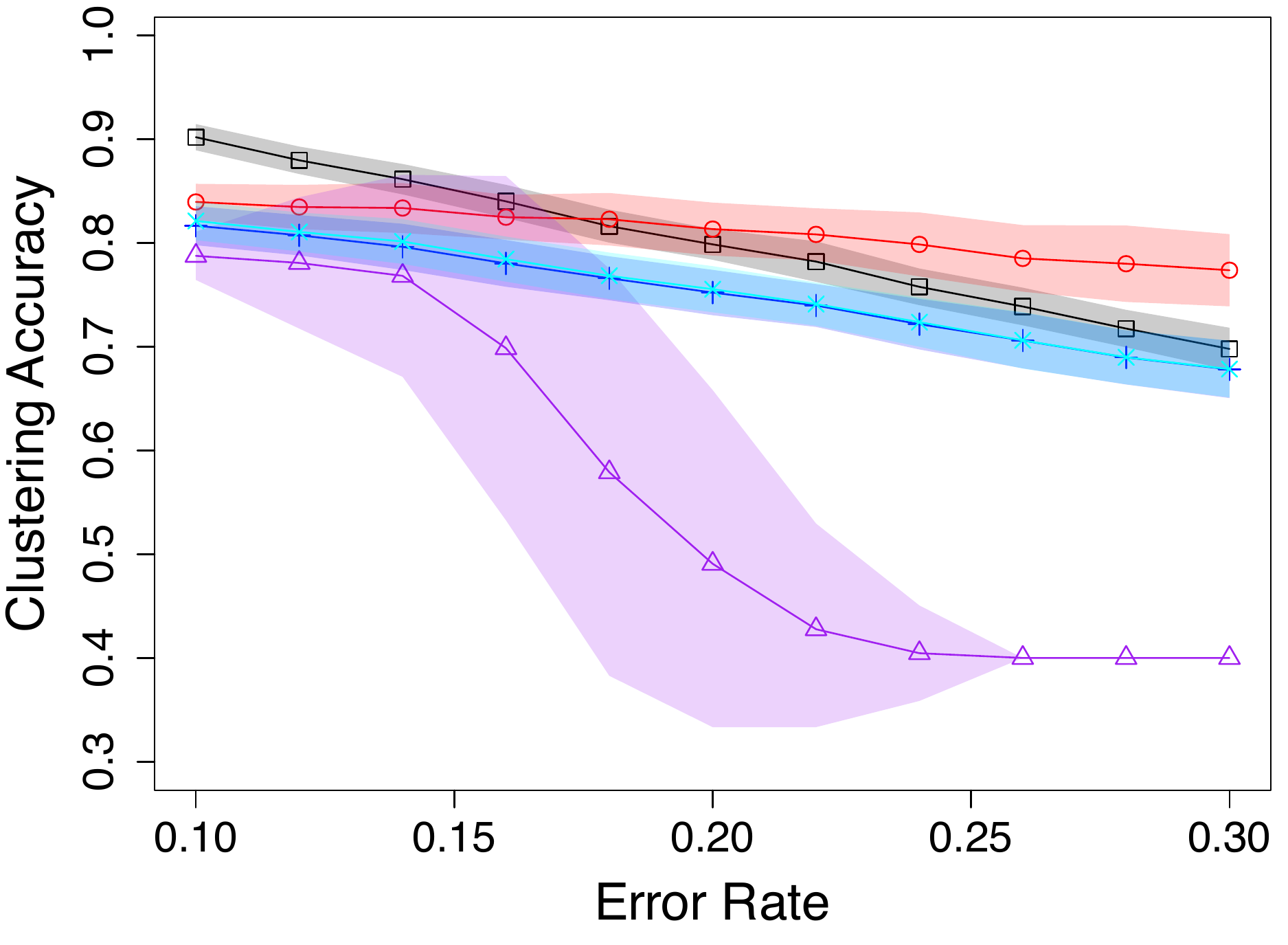}  
  \caption{$d = 8$}
  \label{Uavd8(k=3)}
\end{subfigure}
\caption{Performance of Threshold BCAVI (T-BCAVI), the classical BCAVI, majority vote (MV), and majority vote with penalization (P-MV) in unbalanced settings. (a)-(c): Networks are generated from SBM with $n=600$ nodes, $K=2$ communities of sizes $n_1 = 240, n_2 = 360$. (d)-(f): Networks are generated from SBM with $n=600$ nodes, $K=3$ communities of sizes $n_1 = 150, n_2 =210, n_3 = 240$. Initializations are generated from true node labels according to Assumption~\protect\ref{ass:perturb} with error rate $\varepsilon$, resulting in actual clustering initialization accuracy (RI) of approximately $1-\varepsilon$.}
\label{Ueps}
\end{figure}

In addition, we provide experimental results for the accuracy of parameter estimation in the balanced setting. We report the relative errors of $p^{(s)}$, $q^{(s)}$ and $p^{(s)}/q^{(s)}$, defined by $\frac{p^{(s)} - p}{p}$, $\frac{q^{(s)} - q}{q}$, and $\frac{p^{(s)}/q^{(s)} - p/q}{p/q}$, respectively. Figures~\ref{REeps0.2} and \ref{REeps0.2(k=3)} show that, overall, T-BCAVI is much more accurate than BCAVI in parameter estimation for sparse networks. Both algorithms overestimate $q$ and underestimate $p$, although the bias is smaller for T-BCAVI. When the network is dense, two algorithms have a negligible relative error. Figures~\ref{REeps0.4} and \ref{REeps0.4(k=3)} show that when the initialization contains a large amount of wrong labels ($\varepsilon=0.4$), T-BCAVI still provides meaningful estimates of $p$ and $q$ while the estimation errors of BCAVI do not reduce as the average degree increases. This is because BCAVI always produces $p^{(s)} = q^{(s)}$ when $s$ is sufficiently large. Figures~\ref{REavd8} and \ref{REavd8(k=3)} further confirm that T-BCAVI is much more robust with respect to initialization than BCAVI is.

\begin{figure}[ht]
\begin{subfigure}{.33\textwidth}
  \centering
  \includegraphics[width=1\linewidth]{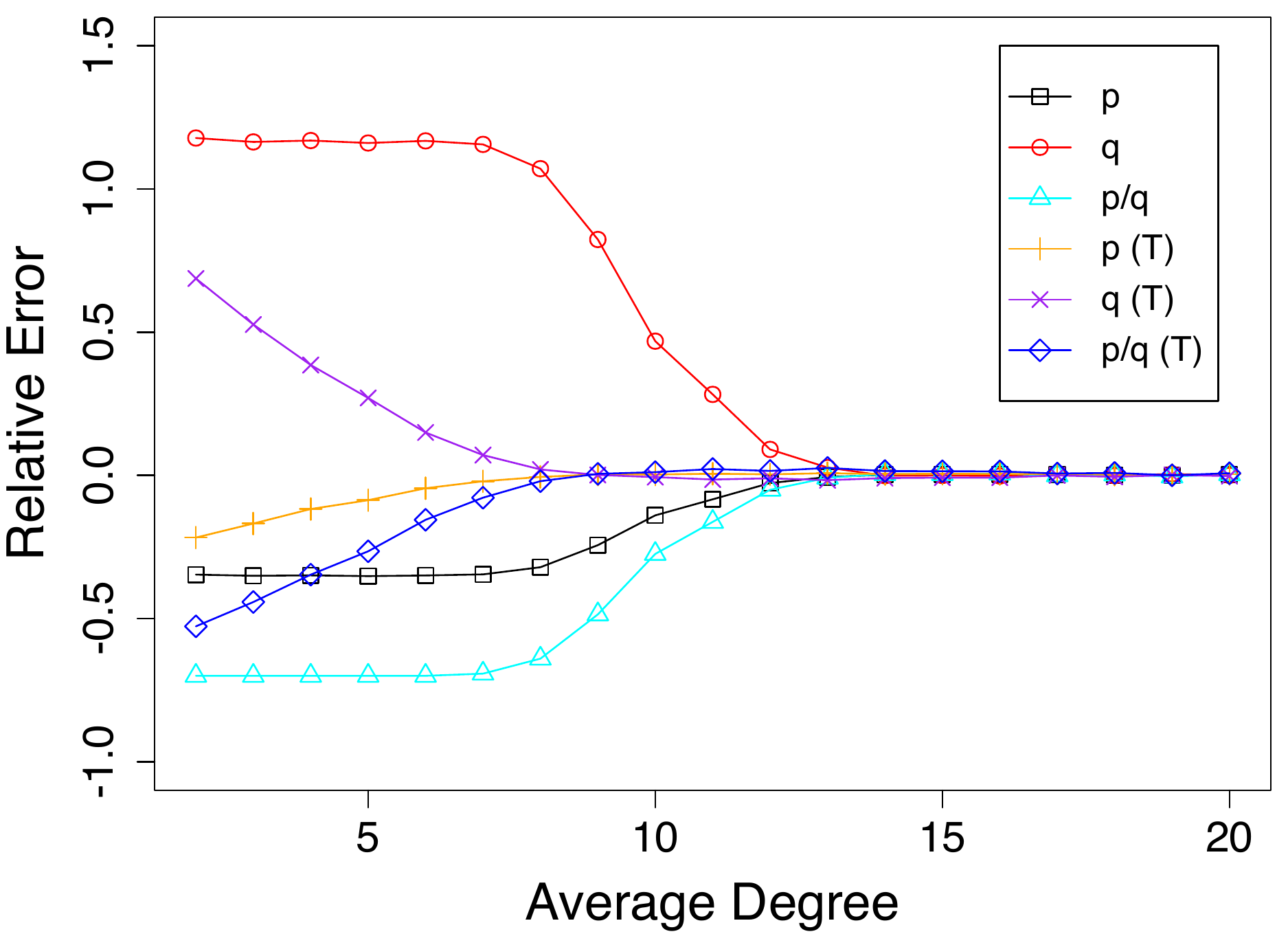}  
  \caption{$\varepsilon = 0.2$}
  \label{REeps0.2}
\end{subfigure}
\begin{subfigure}{.33\textwidth}
  \centering
  \includegraphics[width=1\linewidth]{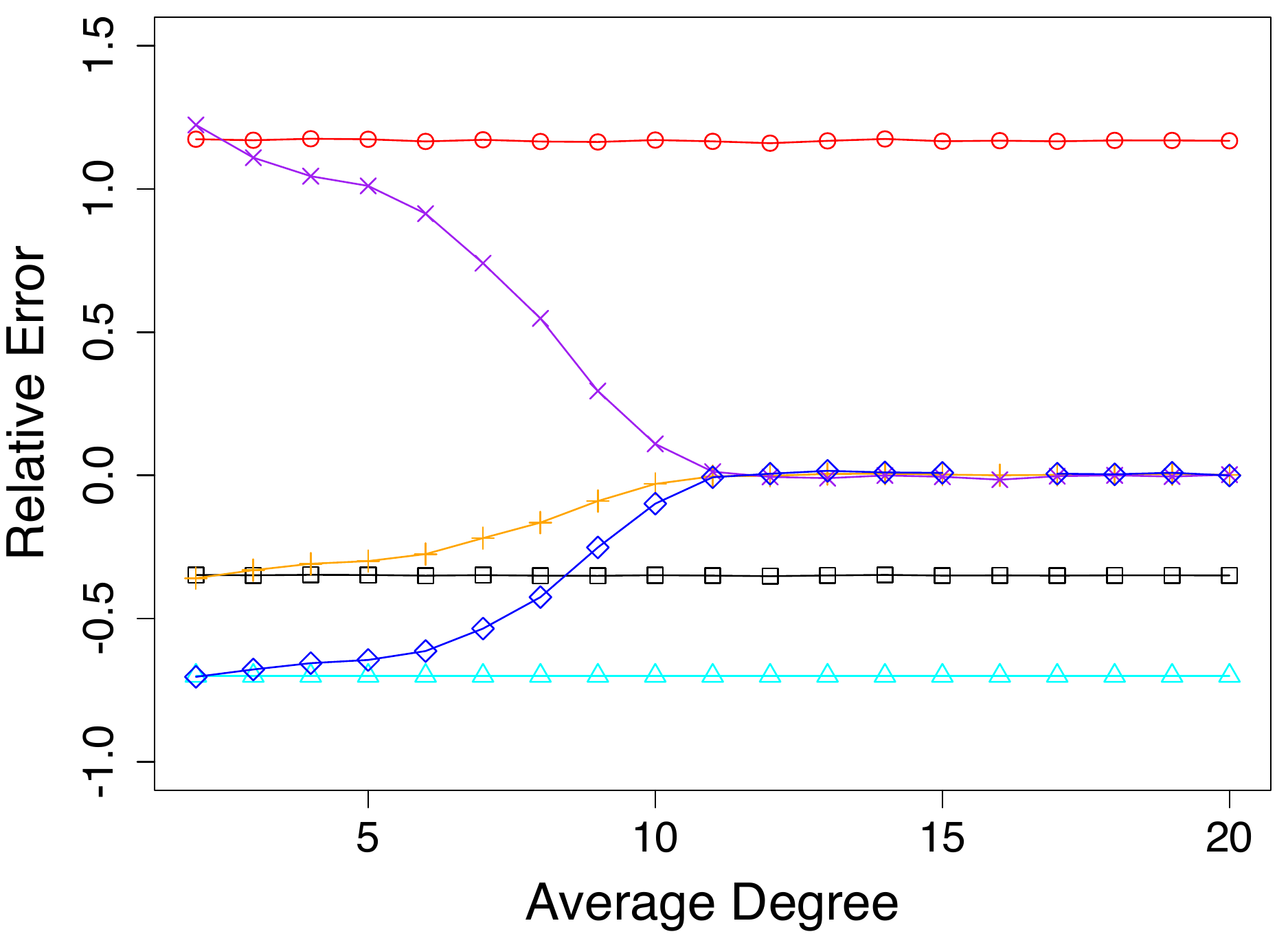}  
  \caption{$\varepsilon = 0.4$}
  \label{REeps0.4}
\end{subfigure}
\begin{subfigure}{.33\textwidth}
  \centering
  \includegraphics[width=1\linewidth]{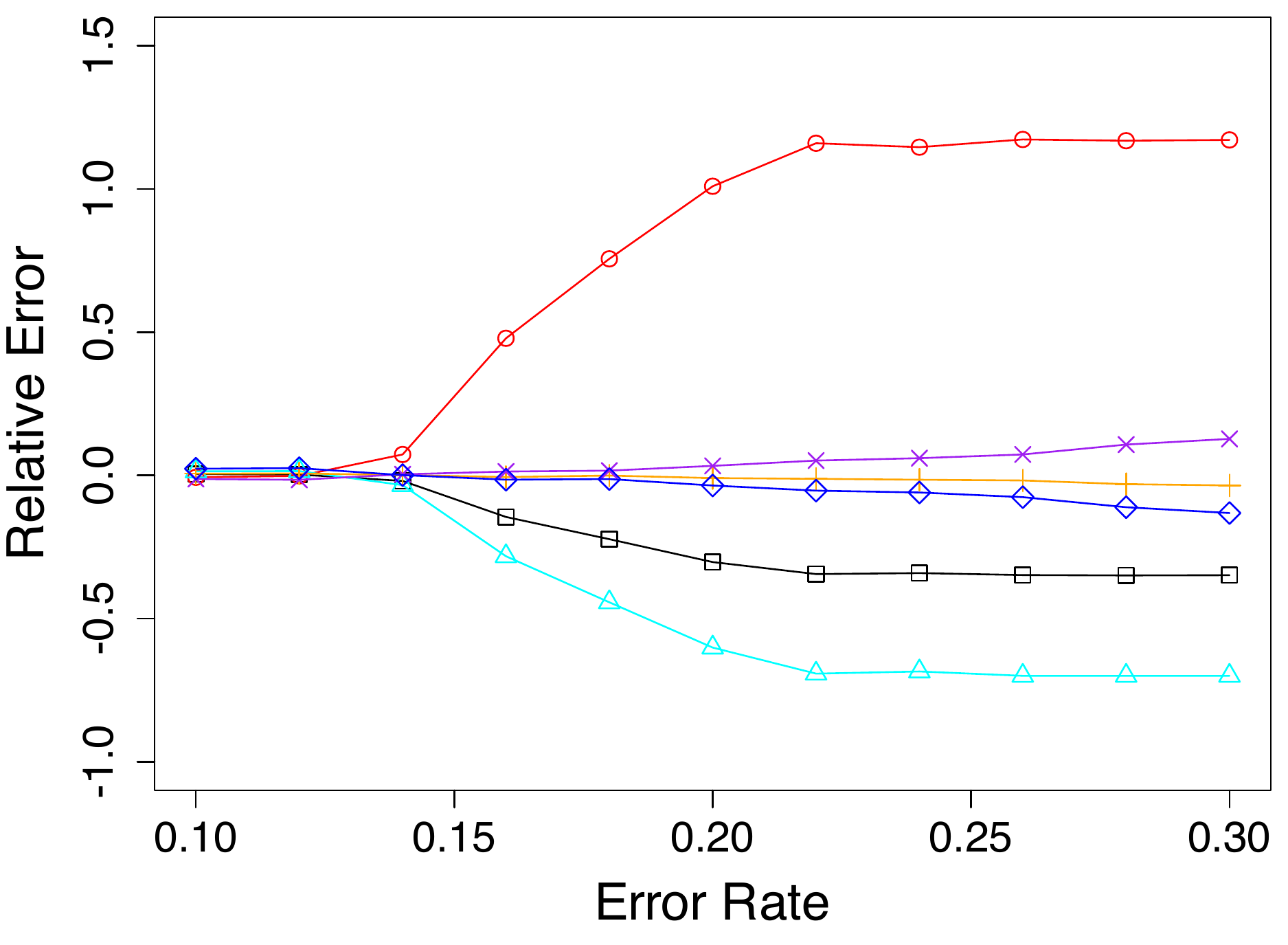}  
  \caption{$d = 8$}
  \label{REavd8}
\end{subfigure}
\begin{subfigure}{.33\textwidth}
  \centering
  \includegraphics[width=1\linewidth]{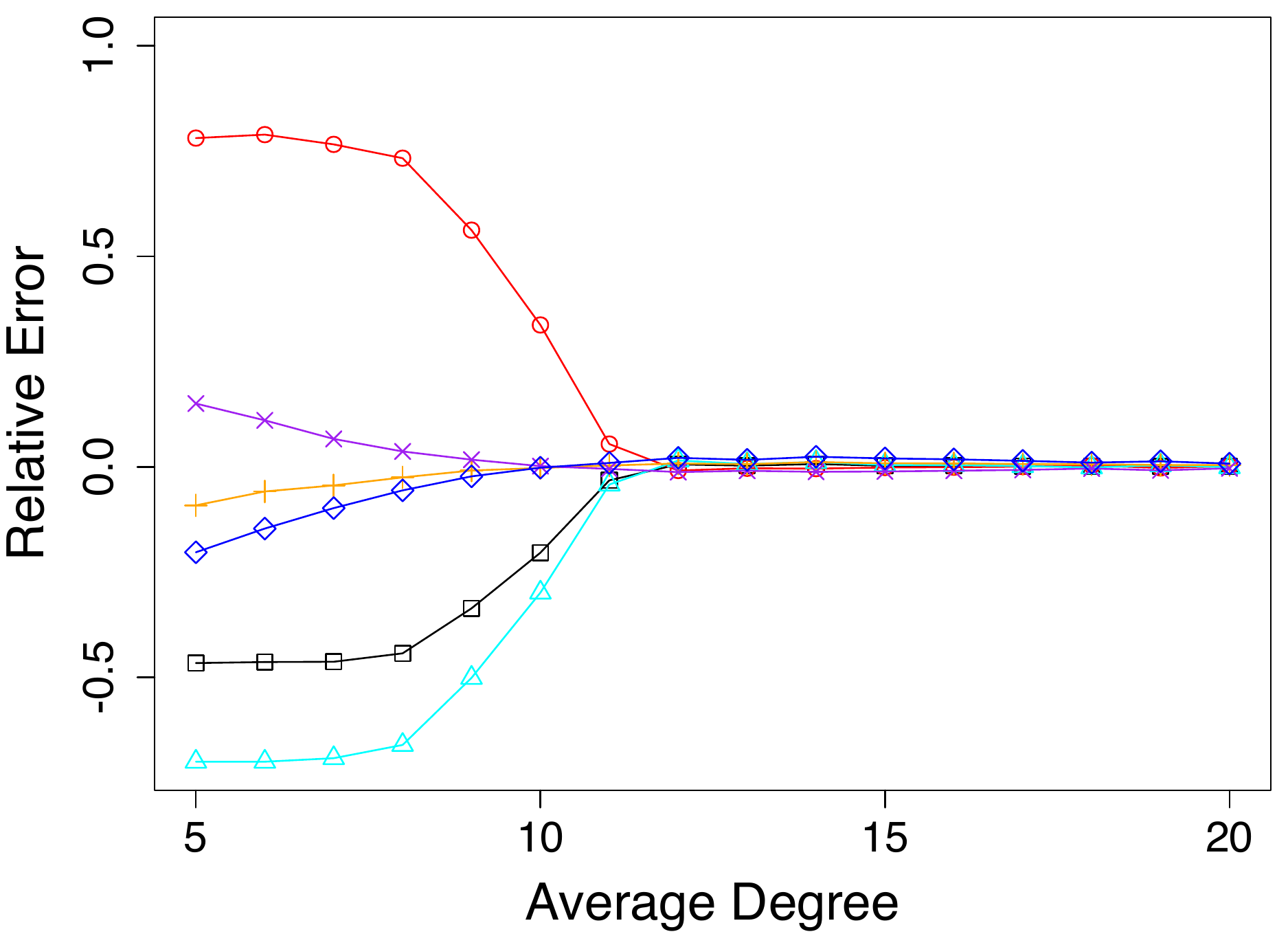}  
  \caption{$\varepsilon = 0.2$}
  \label{REeps0.2(k=3)}
\end{subfigure}
\begin{subfigure}{.33\textwidth}
  \centering
  \includegraphics[width=1\linewidth]{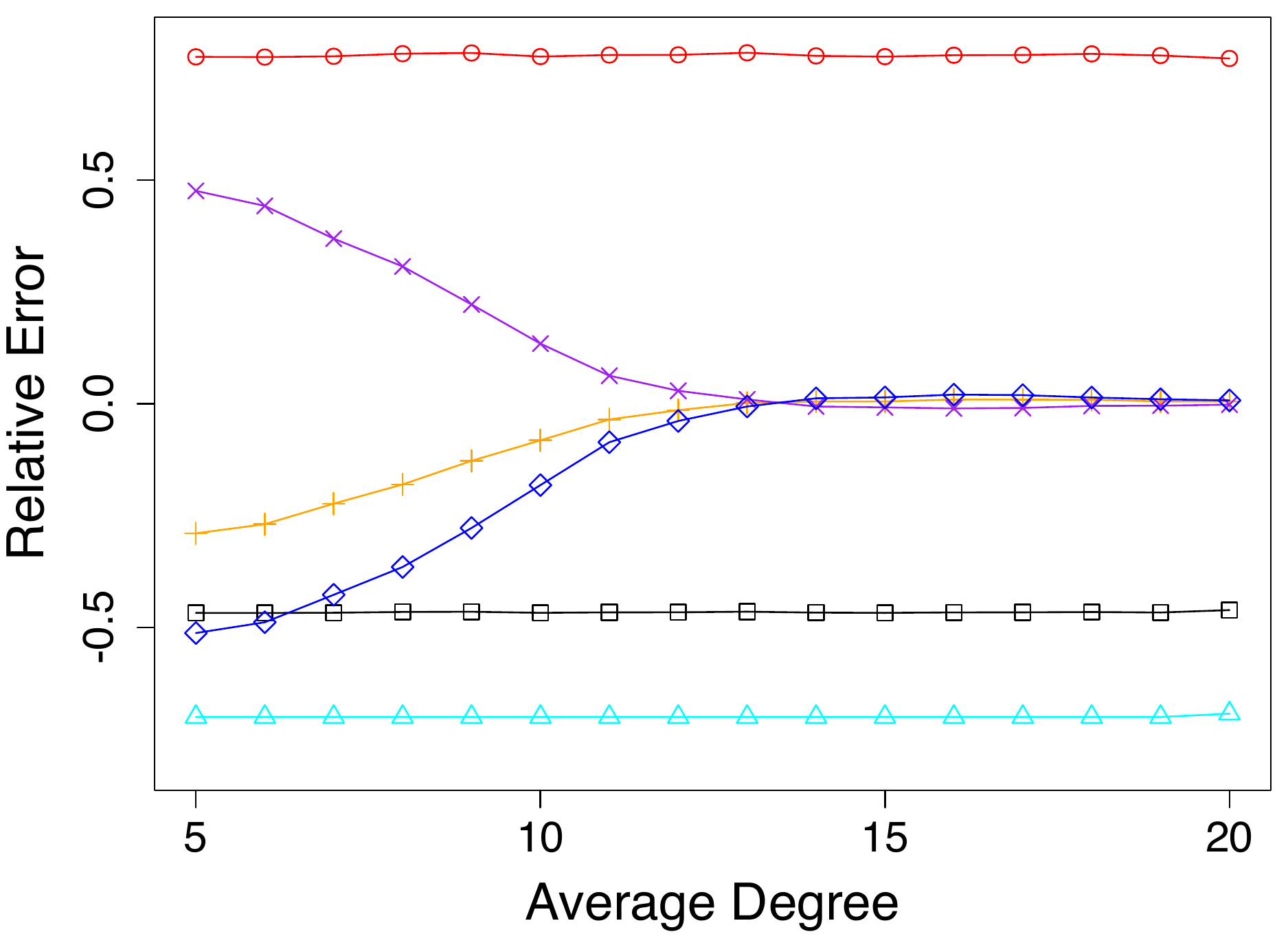}  
  \caption{$\varepsilon = 0.4$}
  \label{REeps0.4(k=3)}
\end{subfigure}
\begin{subfigure}{.33\textwidth}
  \centering
  \includegraphics[width=1\linewidth]{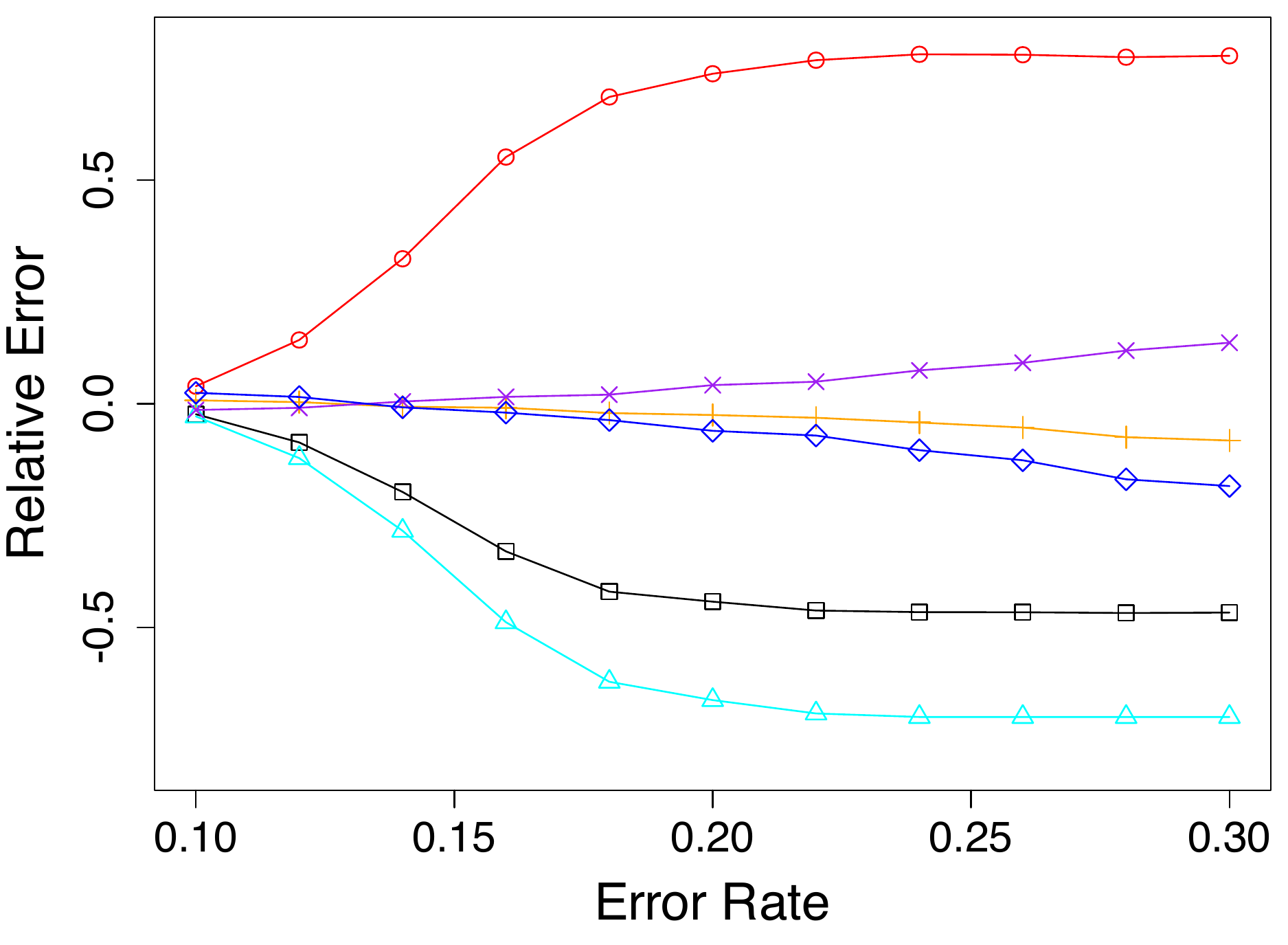}  
  \caption{$d = 8$}
  \label{REavd8(k=3)}
\end{subfigure}
\caption{Relative errors of parameter estimation by Threshold BCAVI (T) and the classical BCAVI in balanced settings. (a)-(c): Networks are generated from SBM with $n=600$ nodes, $K=2$ communities of sizes $n_1 = n_2 = 300$. (d)-(f): Networks are generated from SBM with $n=600$ nodes, $K=3$ communities of sizes $n_1 = n_2 = n_3 = 200$. Initializations are generated from true node labels according to Assumption~\protect\ref{ass:perturb} with error rate $\varepsilon$.}
\label{REperturb}
\end{figure}

\medskip
\noindent
{\em (b) Communities with initialization by spectral clustering}. 
Since the initialization in Assumption~\ref{ass:perturb} is not available in practice, we now evaluate the performance of all methods using initialization generated from network data splitting \citep{chin2015stochastic,li2016network}. According to the discussion in Section~\ref{sec:TBCAVI}, we fix a sampling probability $\tau\in(0,1/2)$ and sample edges in $A$ independently with probability $\tau$; denote by $A^{(\text{init})}$ the adjacency matrix of the resulting sampled network. To get a warm initialization $Z^{(0)}$, we apply spectral clustering algorithm \citep{von2007tutorial} on $A^{(\text{init})}$. All methods are then performed on the remaining sub-network $A - A^{(\text{init})}$ using the initialization $Z^{(0)}$. For reference, we also report the accuracy of $Z^{(0)}$ (SCI). 

Similar to the previous setting when $Z^{(0)}$ are generated according to Assumption~\ref{ass:perturb}, Figure~\ref{avd} shows that T-BCAVI performs much better than other methods for balanced networks, especially when the spectral clustering initialization is almost uninformative (accuracy close to random guess of 1/2 when $K=2$ and 1/3 when $K=3$ in balanced network). This observation again highlights that T-BCAVI often requires a much weaker initialization than what BCAVI and variants of majority vote need. Network data splitting also provides a practical way to implement our algorithm in real-world data analysis. Figure~\ref{Uavd} further shows that T-BCAVI outperforms BCAVI and other methods in unlanced networks. When the average degree is large, both algorithms have similar and negligible errors. This is consistent with the results in Figures~\ref{eps} and \ref{Ueps} when initialization is generated according to Assumption~\ref{ass:perturb}. Generally, T-BCAVI does not require accurate spectral clustering initializations and consistently outperforms other methods for both balanced and unbalanced networks.

\begin{figure}[ht]
\begin{subfigure}{.33\textwidth}
  \centering
  \includegraphics[width=1\linewidth]{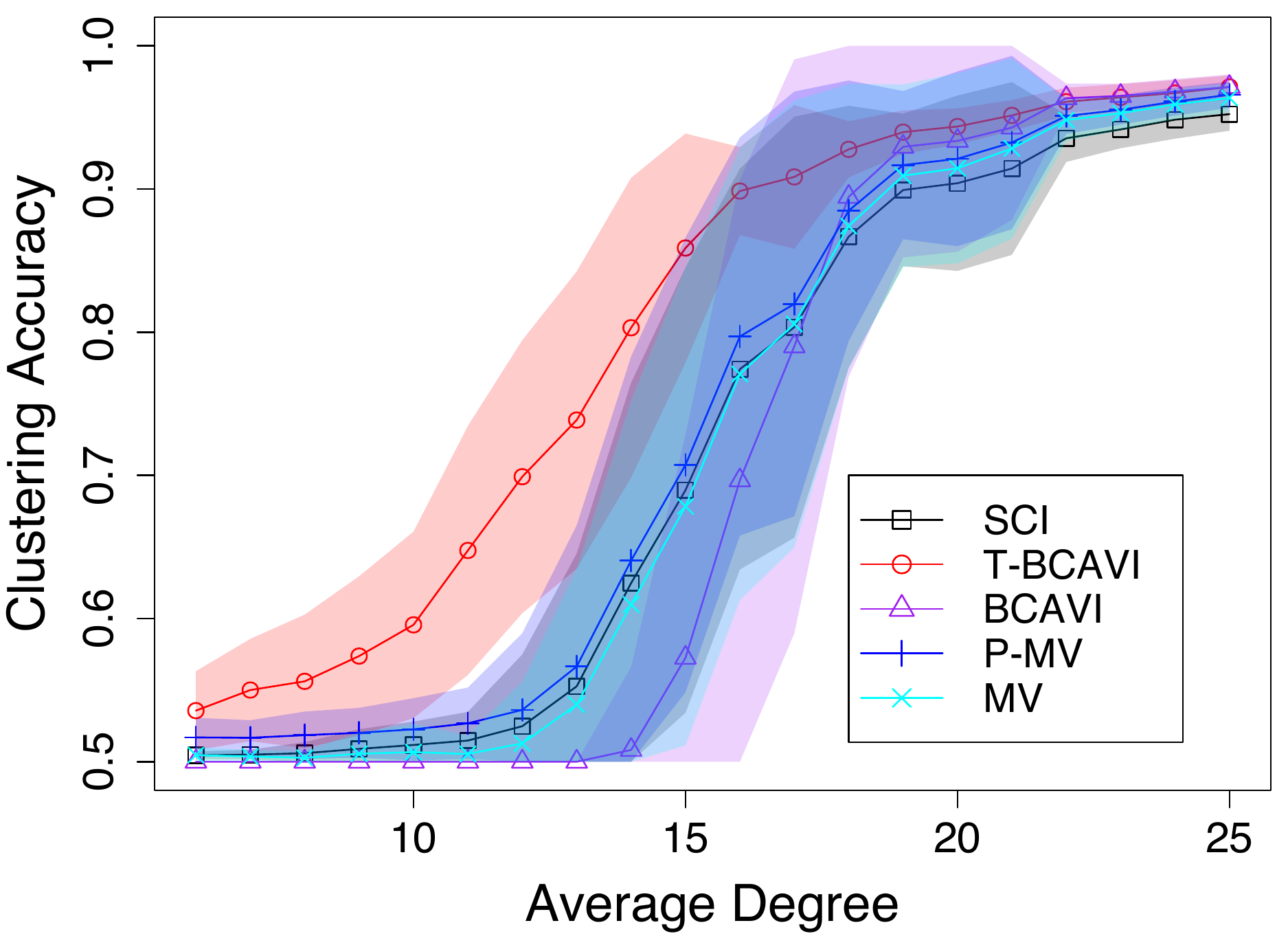}  
  \caption{$\tau = 0.5$}
  \label{samplep0.5}
\end{subfigure}
\begin{subfigure}{.33\textwidth}
  \centering
  \includegraphics[width=1\linewidth]{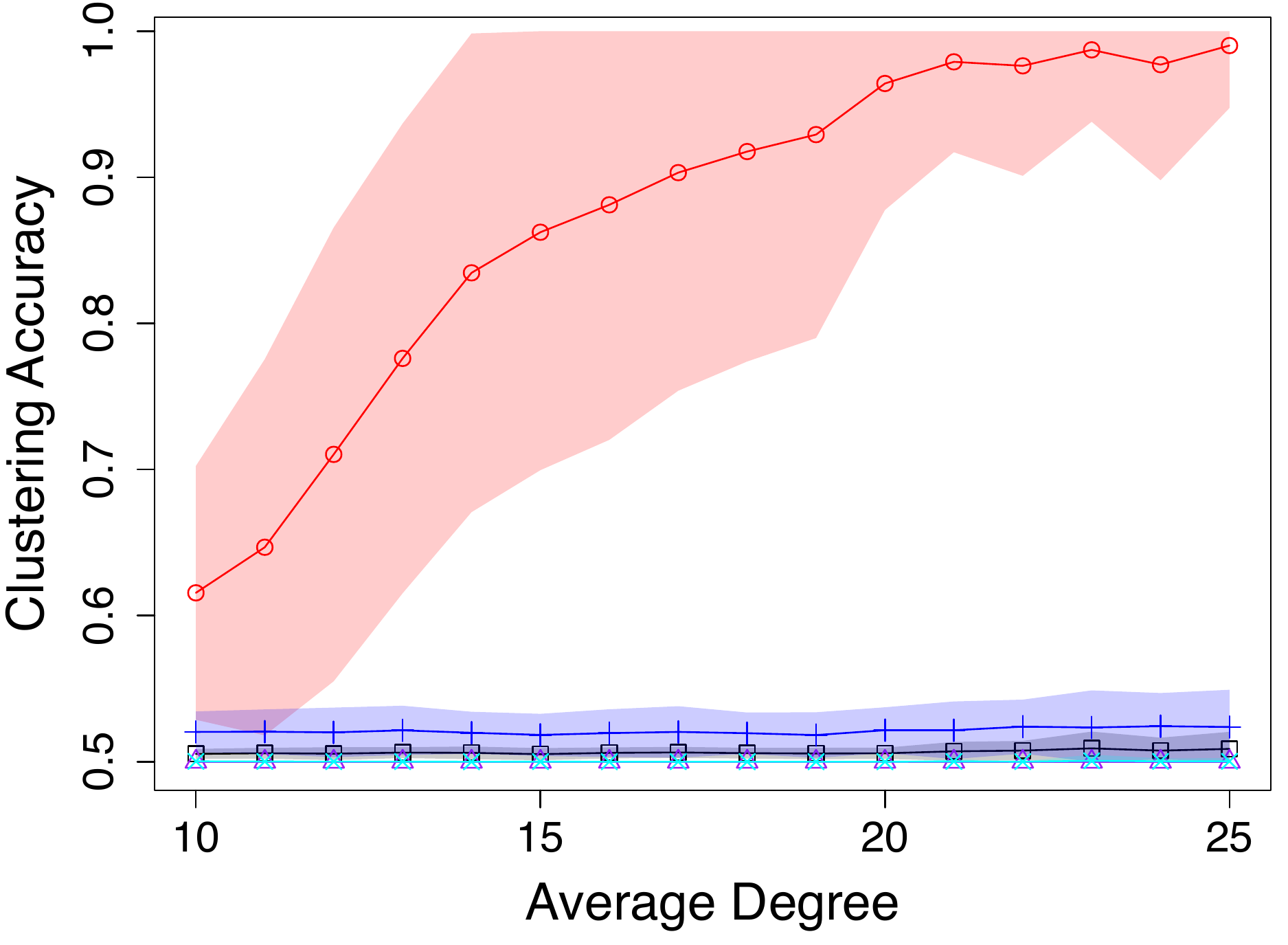}  
  \caption{$\tau = 0.2$}
  \label{samplep0.2}
\end{subfigure}
\begin{subfigure}{.33\textwidth}
  \centering
  \includegraphics[width=1\linewidth]{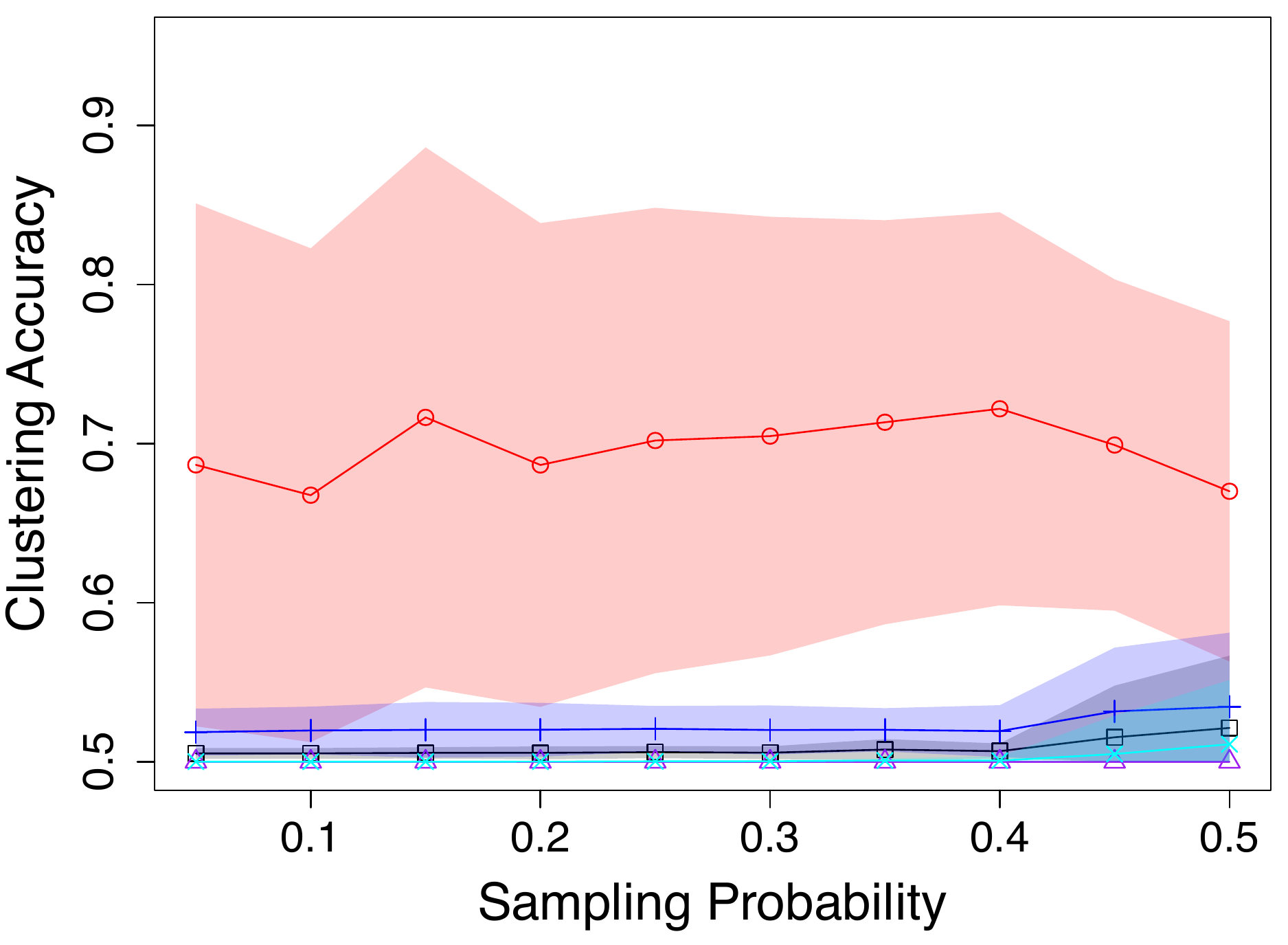}  
  \caption{ $d = 12$}
  \label{sample_avd12}
\end{subfigure}
\begin{subfigure}{.33\textwidth}
  \centering
  \includegraphics[width=1\linewidth]{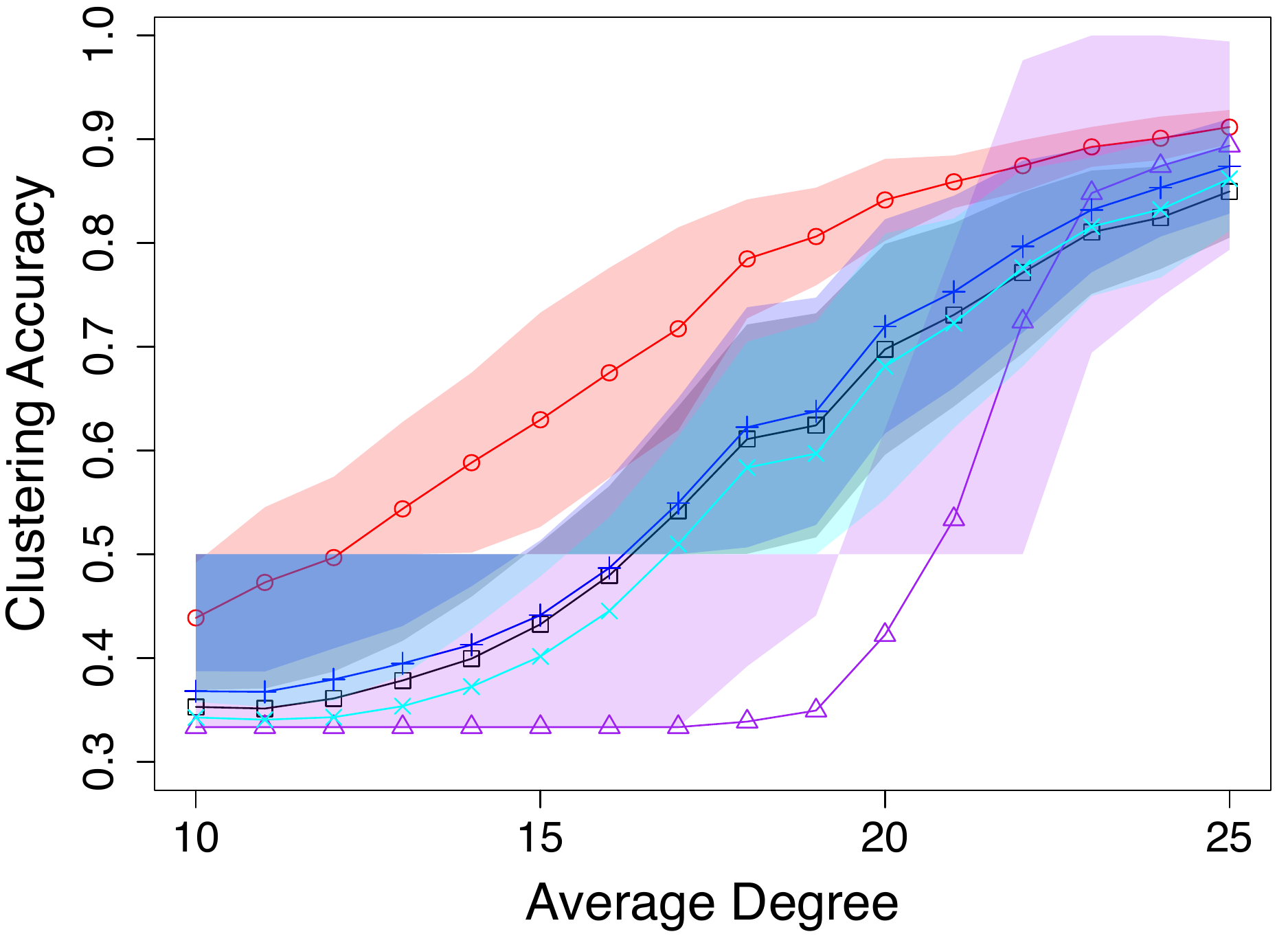}  
  \caption{$\tau = 0.5$}
  \label{samplep0.5(k=3)}
\end{subfigure}
\begin{subfigure}{.33\textwidth}
  \centering
  \includegraphics[width=1\linewidth]{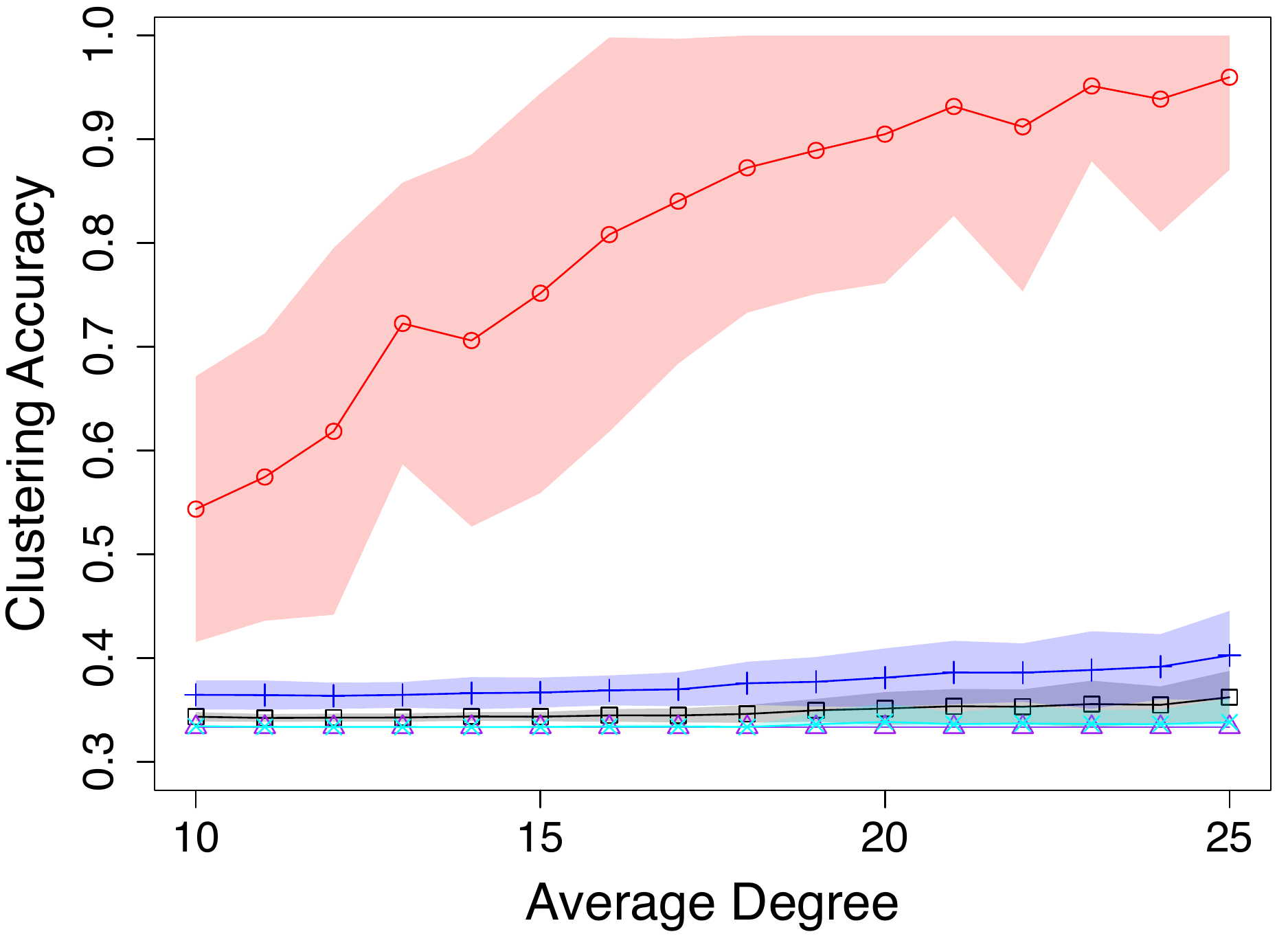}  
  \caption{$\tau = 0.25$}
  \label{samplep0.25(k=3)}
\end{subfigure}
\begin{subfigure}{.33\textwidth}
  \centering
  \includegraphics[width=1\linewidth]{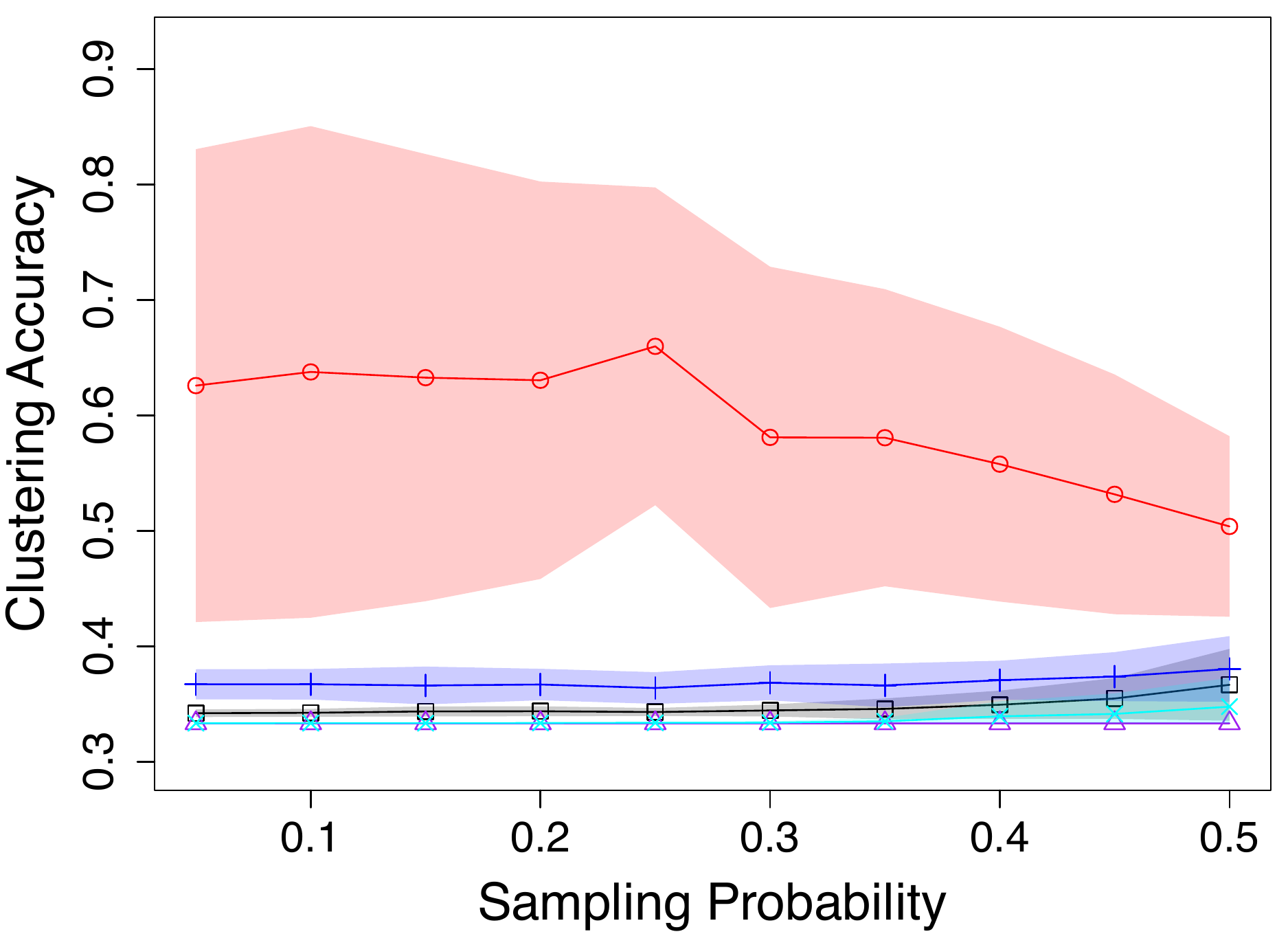}  
  \caption{$d = 12$}
  \label{sample_avd12(k=3)}
\end{subfigure}

\caption{Performance of Threshold BCAVI (T-BCAVI), the classical BCAVI, majority vote (MV), and majority vote with penalization (P-MV) in balanced settings. (a)-(c): Networks are generated from SBM with $n=600$ nodes, $K=2$ communities of sizes $n_1 = n_2 = 300$. (d)-(f): Networks are generated from SBM with $n=600$ nodes, $K=3$ communities of sizes $n_1 = n_2 = n_3 = 200$. Initializations are computed by spectral clustering (SCI) applied to sampled sub-networks $A^{(\text{init})}$ with sampling probability $\tau$  while T-BCAVI and BCAVI are performed on remaining sub-networks $A-A^{(\text{init})}$.}
\label{avd}
\end{figure}

\begin{figure}[ht]
\begin{subfigure}{.33\textwidth}
  \centering
  \includegraphics[width=1\linewidth]{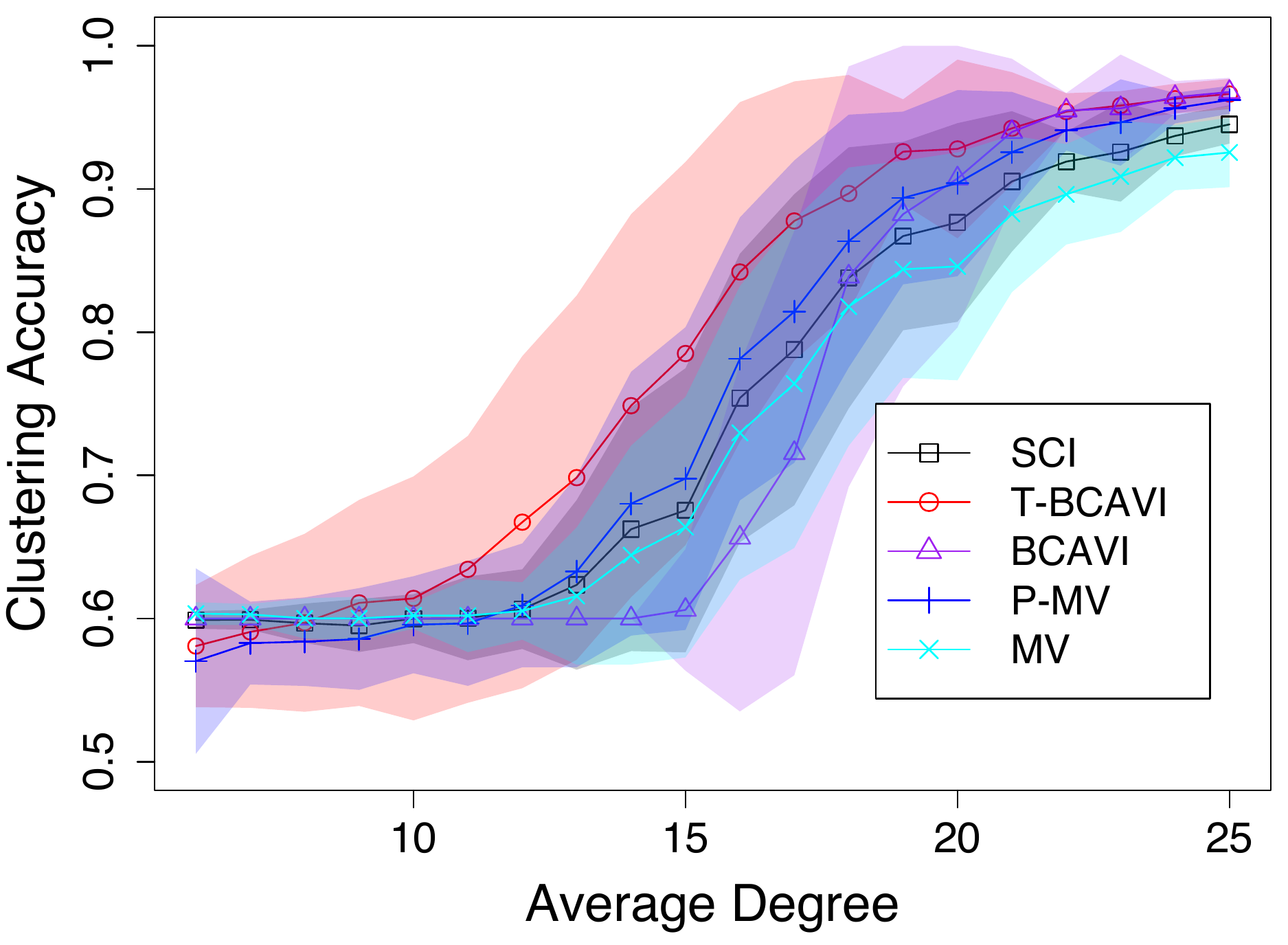}  
  \caption{$\tau = 0.5$}
  \label{Usamplep0.5}
\end{subfigure}
\begin{subfigure}{.33\textwidth}
  \centering
  \includegraphics[width=1\linewidth]{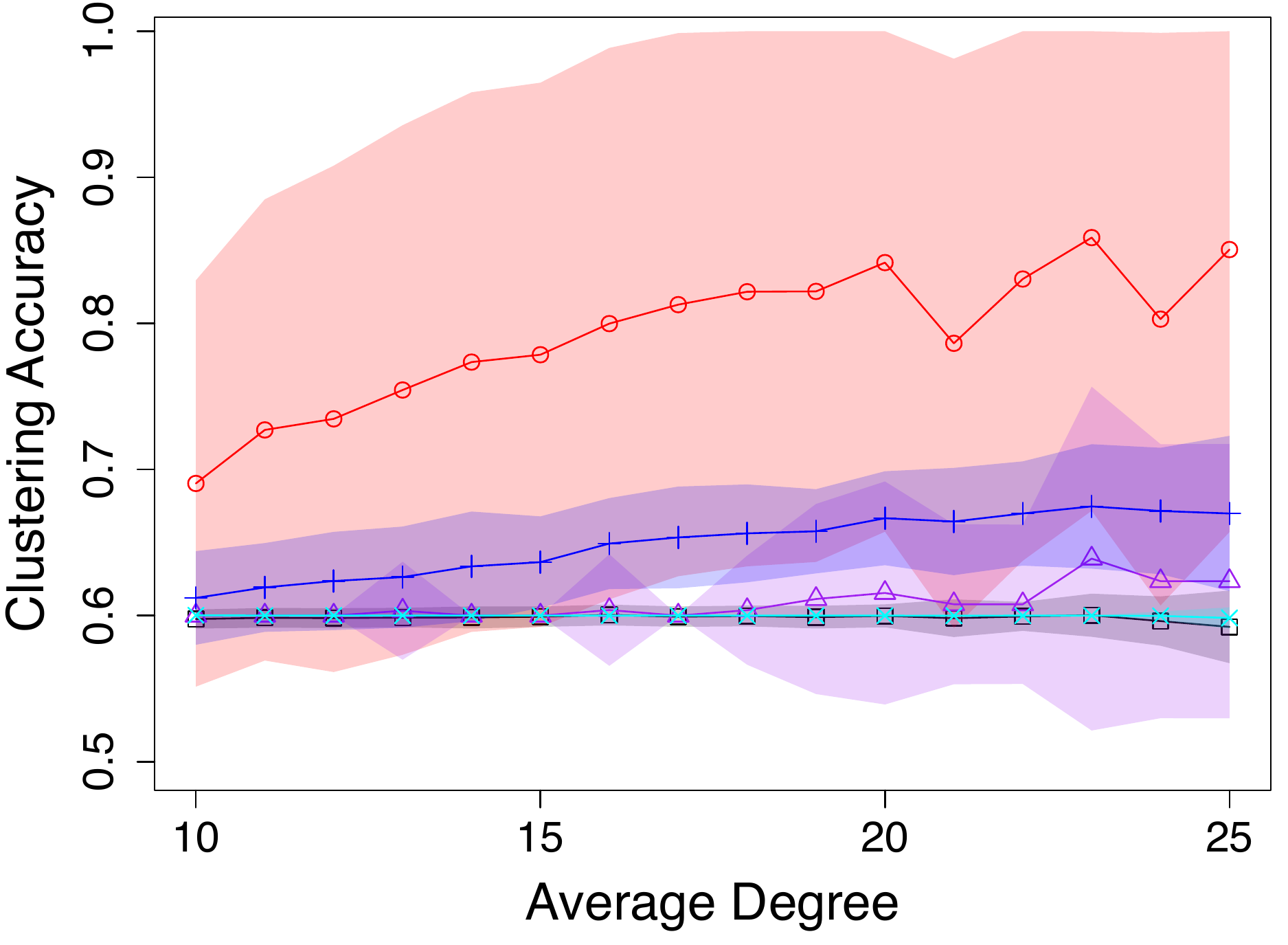}  
  \caption{$\tau = 0.2$}
  \label{Usamplep0.2}
\end{subfigure}
\begin{subfigure}{.33\textwidth}
  \centering
  \includegraphics[width=1\linewidth]{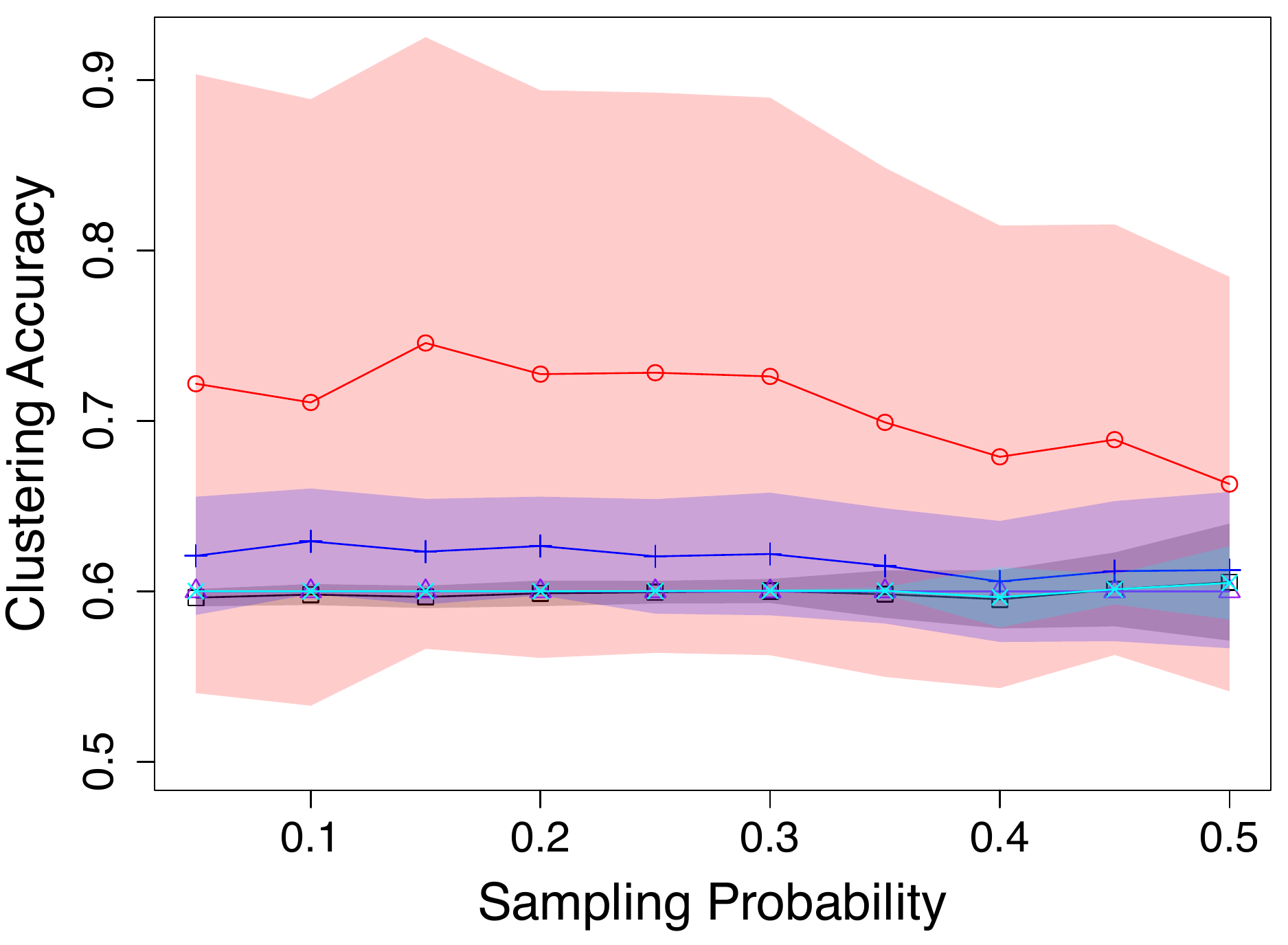}  
  \caption{$d = 12$}
  \label{Usample_avd12}
\end{subfigure}
\begin{subfigure}{.33\textwidth}
  \centering
  \includegraphics[width=1\linewidth]{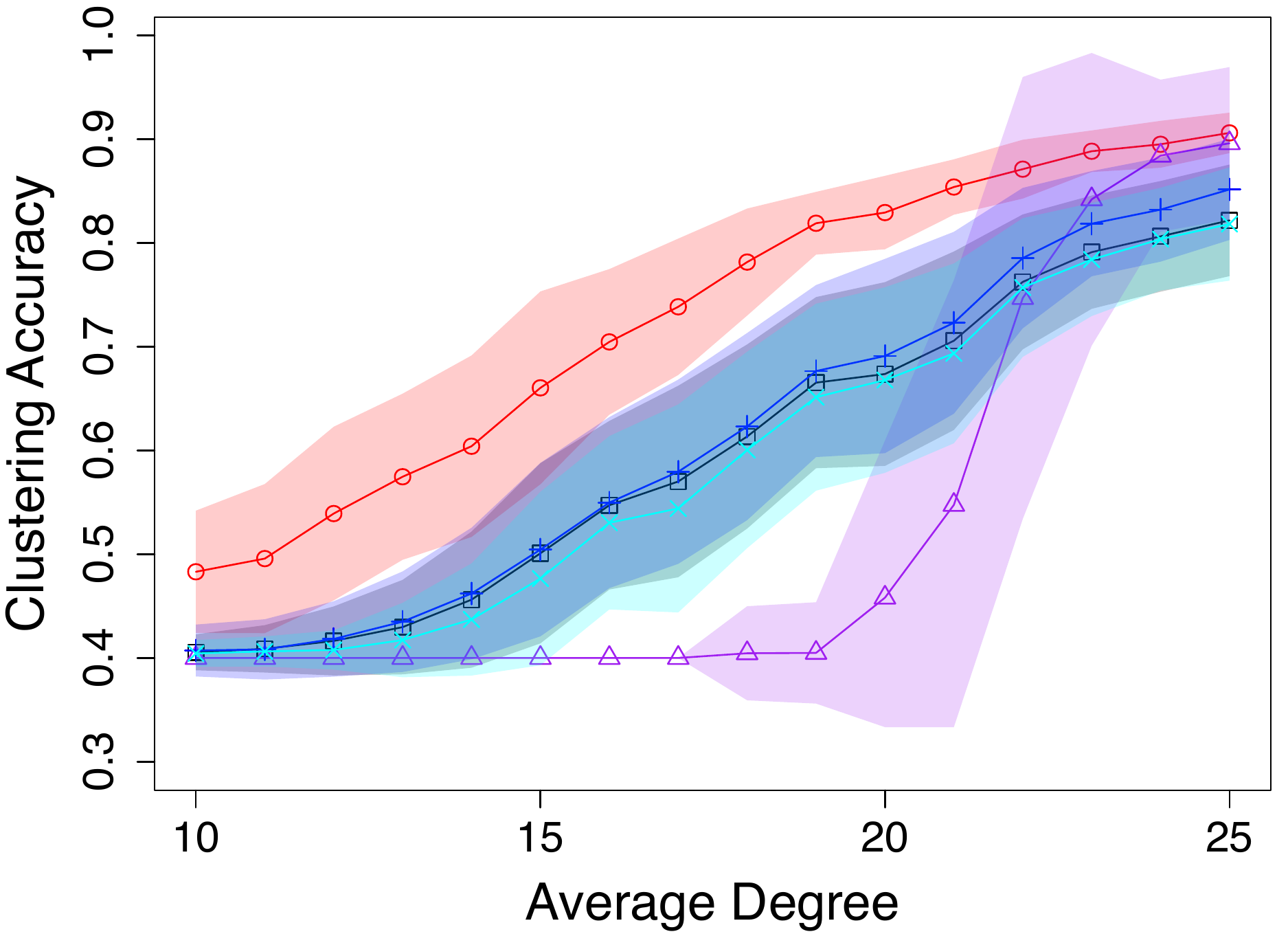}  
  \caption{$\tau = 0.5$}
  \label{Usamplep0.5(k=3)}
\end{subfigure}
\begin{subfigure}{.33\textwidth}
  \centering
  \includegraphics[width=1\linewidth]{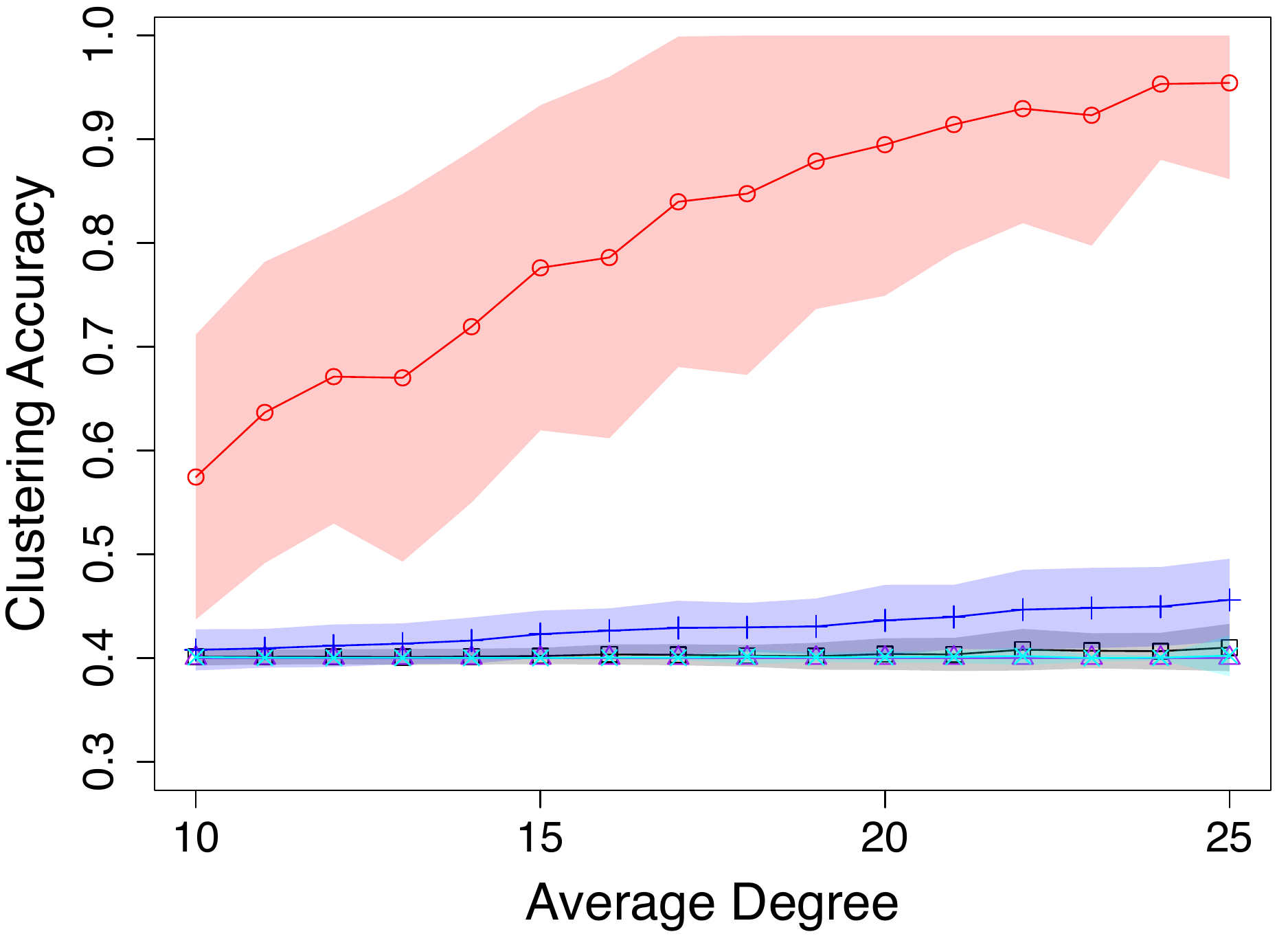}  
  \caption{$\tau = 0.2$}
  \label{Usamplep0.25(k=3)}
\end{subfigure}
\begin{subfigure}{.33\textwidth}
  \centering
  \includegraphics[width=1\linewidth]{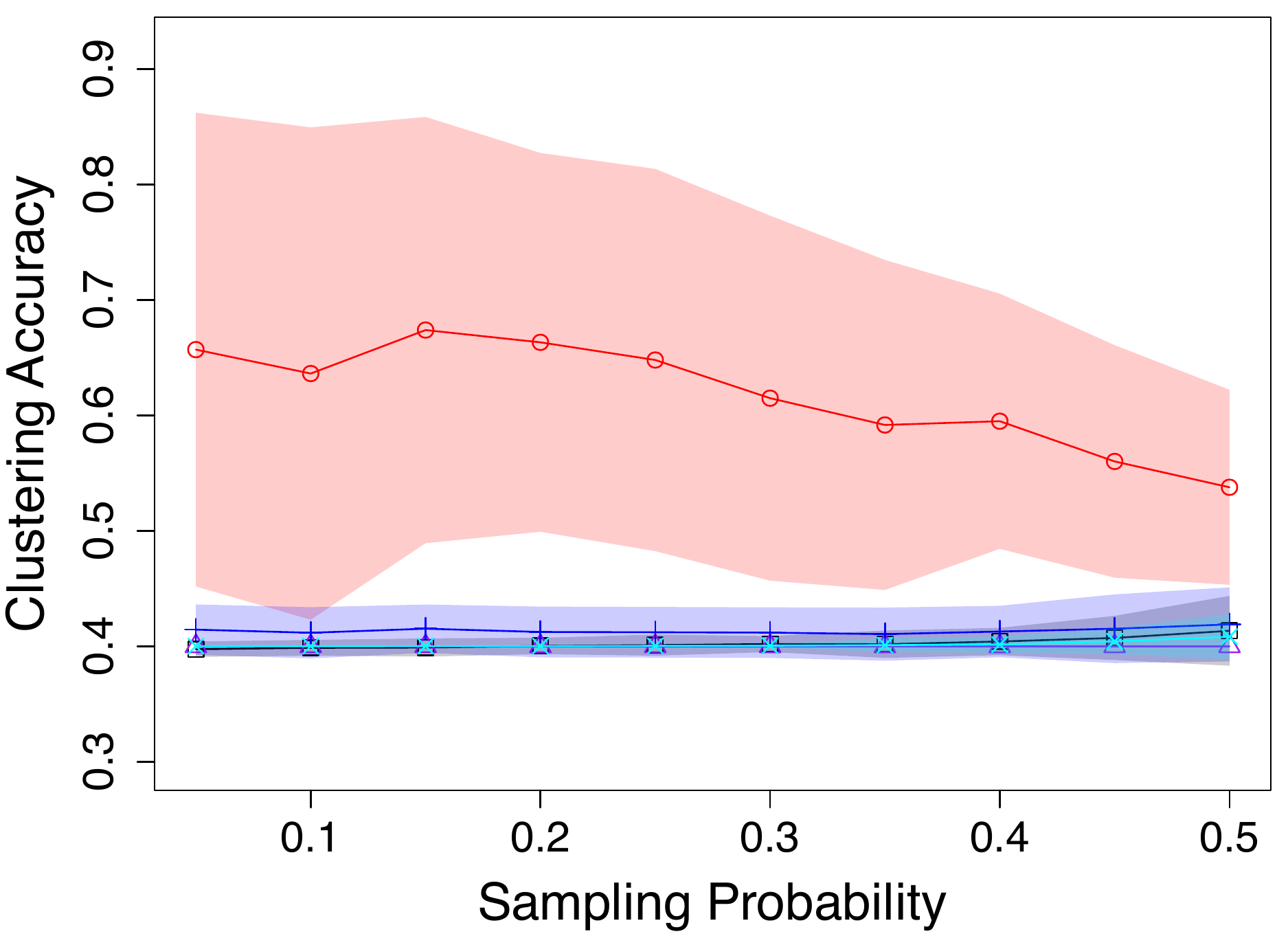}  
  \caption{$d = 12$}
  \label{Usample_avd12(k=3)}
\end{subfigure}

\caption{Performance of Threshold BCAVI (T-BCAVI), the classical BCAVI, majority vote (MV), and majority vote with penalization (P-MV) in unbalanced settings. (a)-(c): Networks are generated from SBM with $n=600$ nodes, $K=2$ communities of sizes $n_1 = 240, n_2 = 360$. (d)-(f): Networks are generated from SBM with $n=600$ nodes, $K=3$ communities of sizes $n_1 = 150, n_2 =210, n_3 = 240$. Initializations are computed by spectral clustering (SCI) applied to sampled sub-networks $A^{(\text{init})}$ with sampling probability $\tau$  while T-BCAVI and BCAVI are performed on remaining sub-networks $A-A^{(\text{init})}$.}
\label{Uavd}
\end{figure}

\subsection{Real Data Examples}
\label{real data}

We first consider the network of common adjectives and nouns in the novel ``David Copperfield'' by Charles Dickens, as described by \citet{newman2006finding}. The nodes represent the adjectives and nouns that occur most frequently in the book. The node labels are 0 for adjectives and 1 for nouns. Edges connect any pair of words that occur in adjacent positions in the book. The network has 58 adjectives and 54 nouns with an average degree of 7.59.

Another data set we analyze is a network of books about US politics published around the 2004 presidential election and sold by the online bookseller Amazon.com. Edges between books represent frequent co-purchasing of books by the same buyers. The network was compiled by \citet{orgnet} and was posted by \citet{ndt}. The books are labeled as "liberal," "neutral," or "conservative" based on their descriptions and reviews of the books \citep{ndt}. This network has an average degree of 8.40, and it contains 49 conservative books, 43 liberal books, and 13 neutral books.

The last data set that we analyze in this section is the political blogosphere network data set \citep{adamic2005political}. It is also related to the 2004 U.S. presidential election, where the nodes are blogs focused on American politics and the edges are hyperlinks connecting the blogs. These blogs have been manually labeled as 'liberal' or 'conservative' \citep{adamic2005political}. We ignore the direction of the hyperlinks and perform our analysis on the largest connected component of the network \citep{karrer2011stochastic}, which contains 1490 nodes, and the average node degree is 22.44. This network includes 732 conservative blogs and 758 liberal blogs.

Similar to Section~\ref{sec:simulated}, to implement all methods, we first randomly sample a sub-network and apply a spectral clustering algorithm to get a warm initialization. We then run these algorithms on the remaining sub-network for comparison. Since it is unclear whether these real networks are more similar to networks drawn from SBM or DCSBM, we implement both versions of the variational inference methods that fit either
SBM or DCSBM. In particular, we denote the classical BCAVI algorithms that fit SBM in Section~\ref{sec: bcavi in sbm} and DCSBM in Section~\ref{sec: bcavi in dcsbm} by BCAVI(sbm) and BCAVI(dcsbm), respectively. Similarly, T-BCAVI(sbm) and T-BCAVI(dcsbm) are the Threshold BCAVI algorithms that fit SBM in Algorithm~\ref{alg} and DCSBM in Algorithm~\ref{alg in DCSBM}, respectively. For initialization, we denote the standard spectral clustering \citep{von2007tutorial} and regularized spectral clustering \citep{qin2013regularized} by SCI(sbm) and SCI(dcsbm), respectively.

The numerical results for BCAVI(sbm) and T-BCAVI(sbm) are shown in Figures~\ref{adj} - \ref{pol1}, while those for BCAVI(dcsbm) and T-BCAVI(dcsbm) are in Figures~\ref{adj-dc} - \ref{pol1-dc}. Figures~\ref{adj} and \ref{adj-dc} show that T-BCAVI(sbm) outperforms other methods that fit SBM and T-BCAVI(dcsbm) is more accurate than other methods that fit DCSBM, although initializations are close to random guess, resulting in a relatively small improvement of all algorithms. Also, the improvement of T-BCAVI(dcsbm) is more visible than that of T-BCAVI(sbm). This is within our expectations because real networks are usually degree-inhomogeneous, which means that DCSBM is often more suitable than SBM in fitting them. Figures~\ref{pol2} and \ref{pol2-dc} for the book network show that BCAVI(sbm), BCAVI(dcsbm), MV, and P-MV do not improve the initialization much, while T-BCAVI(sbm) and especially T-BCAVI(dcsbm) improve the accuracy of the initialization considerably. Finally, Figures~\ref{pol1} and \ref{pol1-dc} for political blogosphere show that they all improve the accuracy of the initializations, but T-BCAVI(sbm) outperforms other SBM-based methods for almost all sampling probabilities $\tau$. In addition, T-BCAVI(dcsbm) is still consistently better than other DCSBM-based methods and has a larger improvement than T-BCAVI(sbm).

\begin{figure}[ht]
\begin{subfigure}{.33\textwidth}
  \centering
  \includegraphics[width=.99\linewidth]{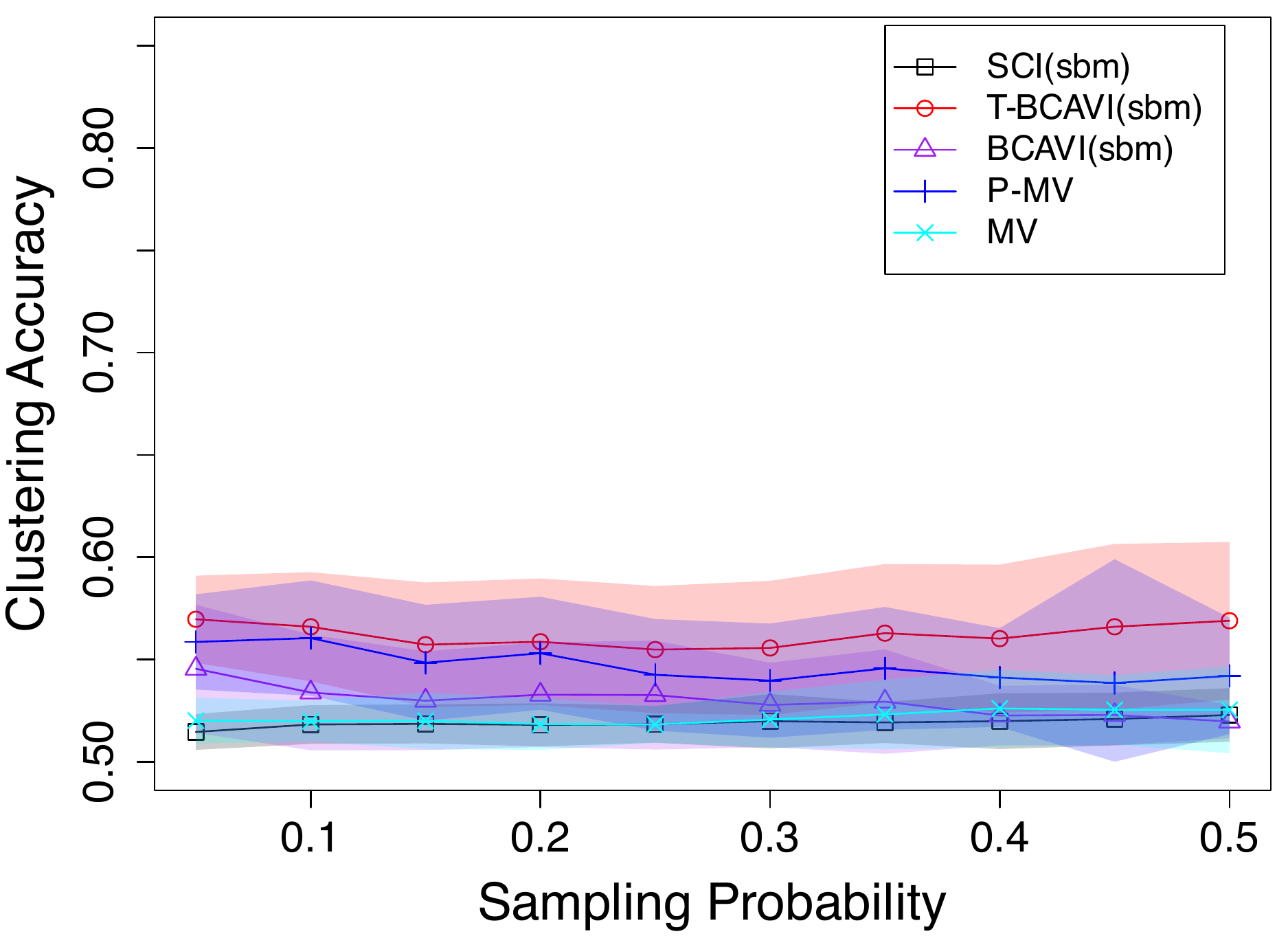}  
  \caption{Adj and nouns: SBM fit}
\label{adj}
\end{subfigure}
\begin{subfigure}{.33\textwidth}
  \centering
  \includegraphics[width=.99\linewidth]{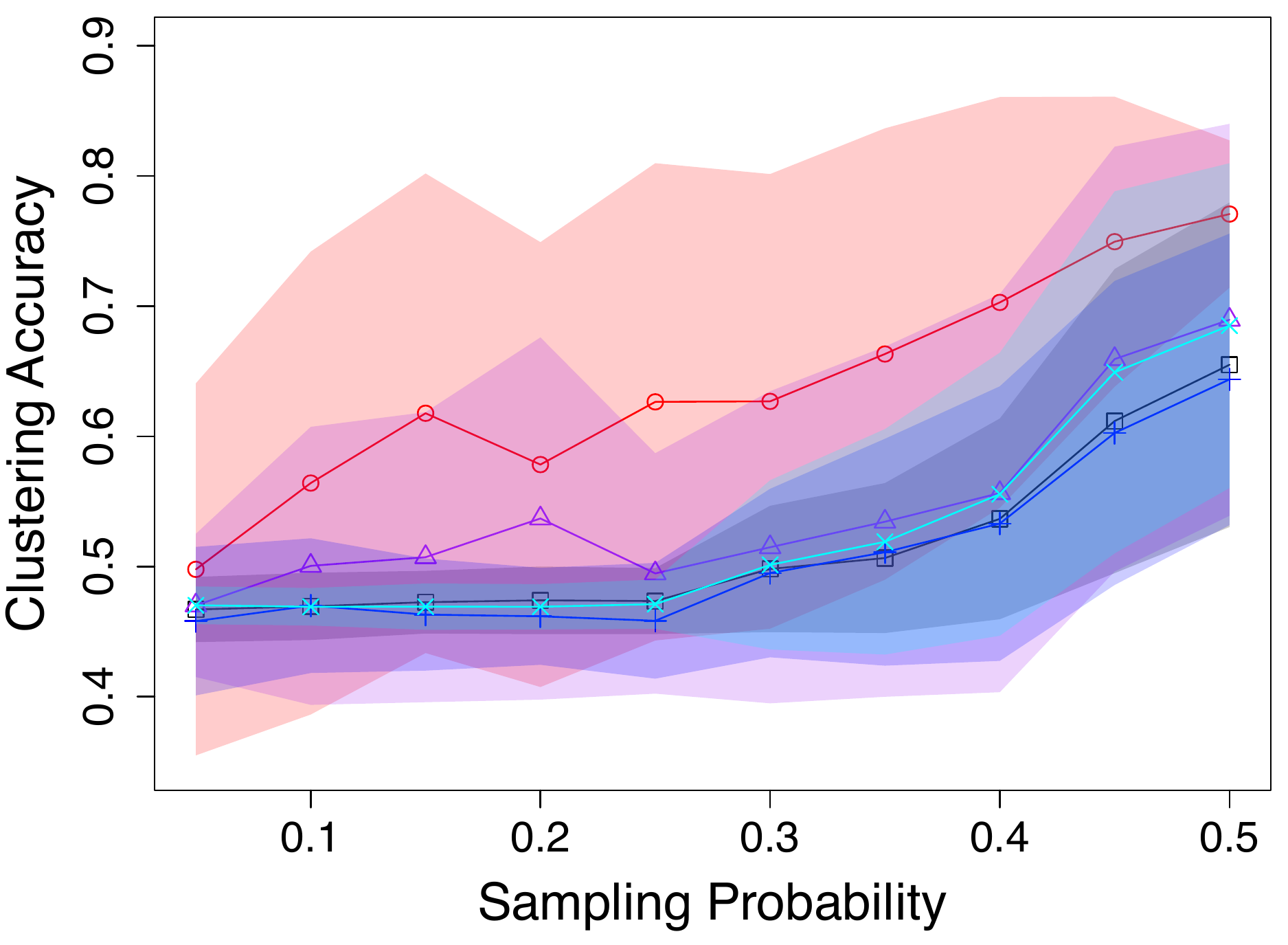} 
  \caption{Politics book: SBM fit}
  \label{pol2}
\end{subfigure}
\begin{subfigure}{.33\textwidth}
  \centering
  \includegraphics[width=.99\linewidth]{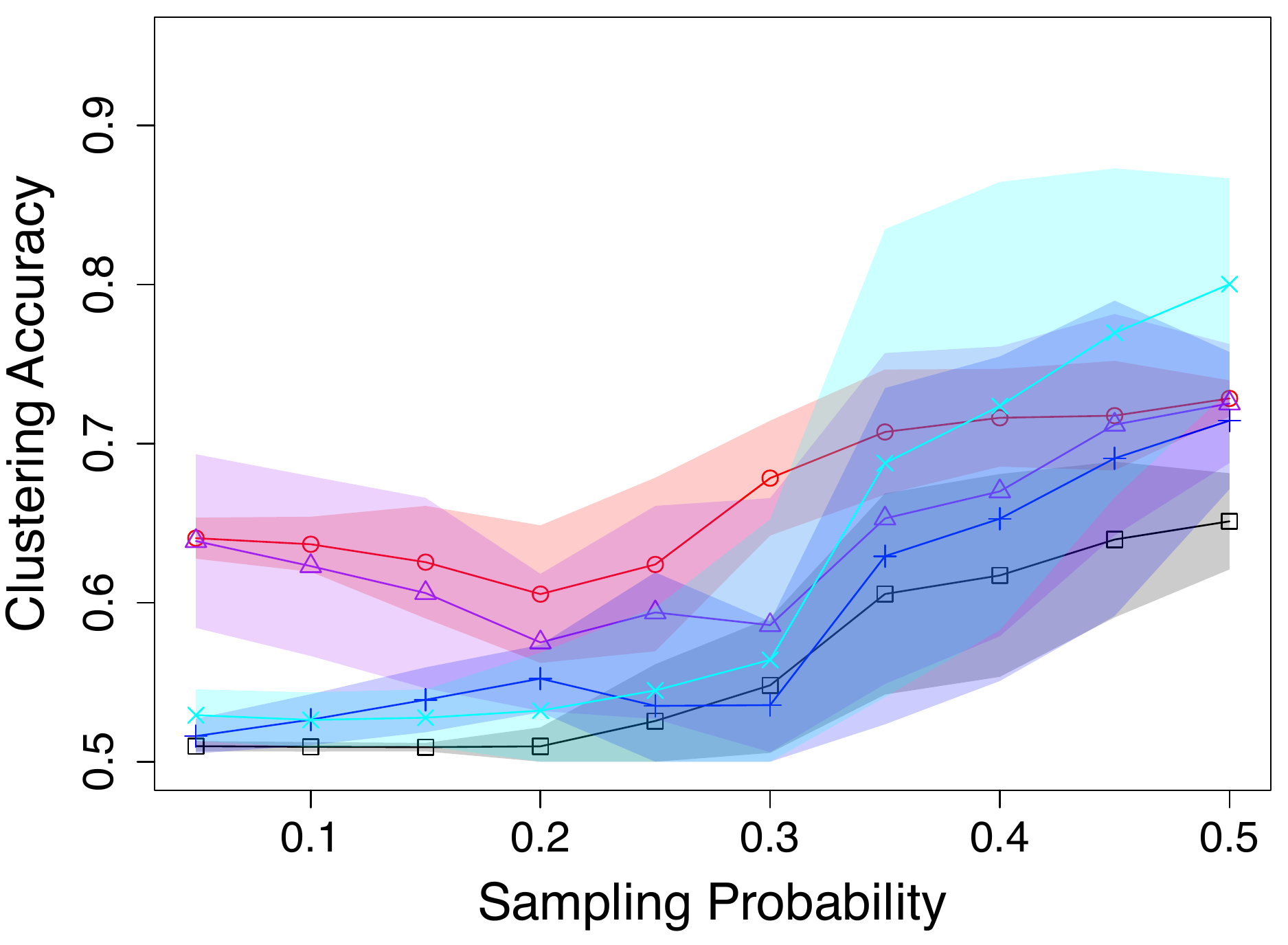}  
  \caption{Politics blogosphere: SBM fit}
  \label{pol1}
\end{subfigure}
\begin{subfigure}{.33\textwidth}
  \centering
  \includegraphics[width=.99\linewidth]{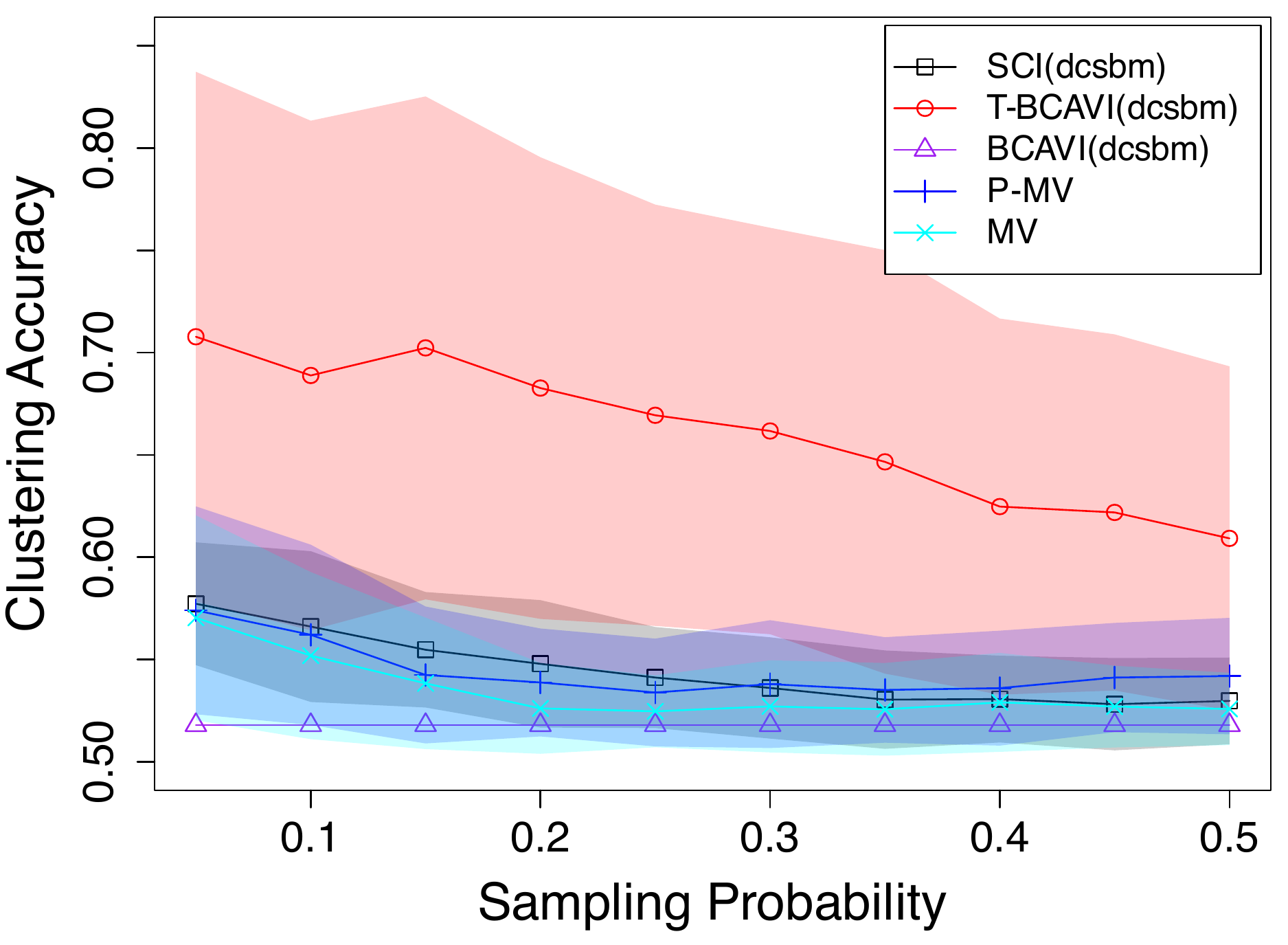}  
  \caption{Adj and nouns: DCSBM fit}
\label{adj-dc}
\end{subfigure}
\begin{subfigure}{.33\textwidth}
  \centering
  \includegraphics[width=.99\linewidth]{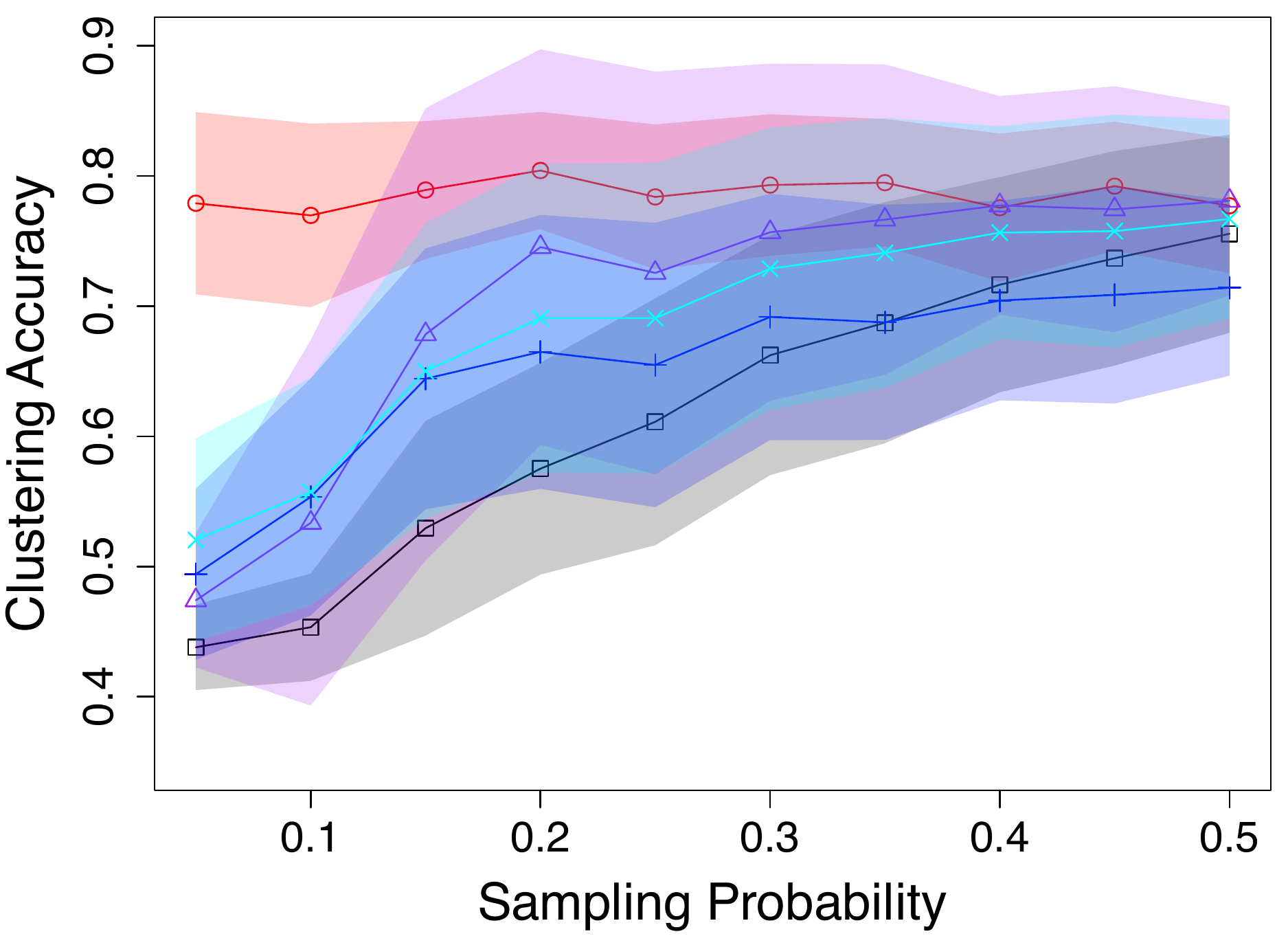} 
  \caption{Politics book: DCSBM fit}
  \label{pol2-dc}
\end{subfigure}
\begin{subfigure}{.33\textwidth}
  \centering
  \includegraphics[width=.99\linewidth]{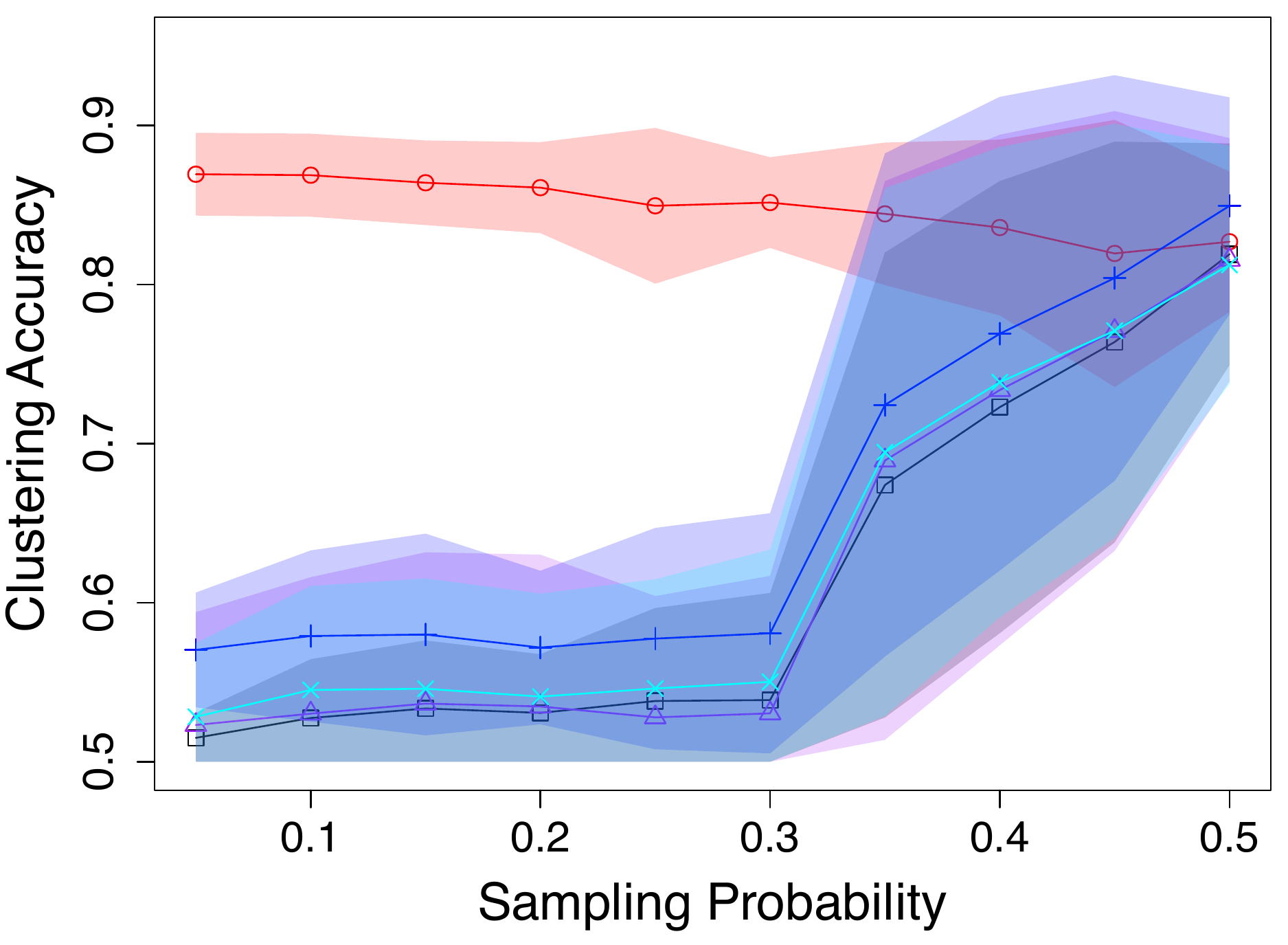}  
  \caption{Politics blogosphere: DCSBM fit}
  \label{pol1-dc}
\end{subfigure}
\caption{Plot (a) - (c): Performance of T-BCAVI(sbm) and BCAVI(sbm) in real data examples. Initializations are computed by standard spectral clustering (SCI(sbm)) applied to sampled sub-networks $A^{(\text{init})}$ with sampling probability $\tau$ while all the algorithms are performed on remaining sub-networks $A - A^{(\text{init})}$. Plot (d) - (f): Performance of T-BCAVI(dcsbm) and BCAVI(dcsbm) in real data examples. Initializations are computed by regularized spectral clustering (SCI(dcsbm)) applied to sampled sub-networks $A^{(\text{init})}$ with sampling probability $\tau$ while all the algorithms are performed on remaining sub-networks $A - A^{(\text{init})}$.} 
\label{pol}
\end{figure}

\section{Discussion}
This paper studies the batch coordinate update variational inference algorithm for community detection under the stochastic block model.
Existing work in this direction only establishes the theoretical support for this iterative algorithm in the relatively dense regime. The contribution of this paper is two-fold. First, we prove the validity of the variational approach under DCSBM and extend its guarantee under SBM to the sparse setting when node degrees may be bounded. Second, we propose a simple but novel threshold strategy that significantly improves the accuracy of the classical variational inference method, especially in the sparse regime or when the accuracy of the initialization is poor. 

While we only have theoretical results for SBM and DCSBM with $K$ equal-sized communities, we believe that similar results also hold for more general settings, for example, the covariate stochastic blockmodels \citep{sweet2015incorporating} and leave their analysis for future work. In addition, Assumption~\ref{ass:perturb} requires that each entry of the initialization is perturbed independently with the same error rate. Simulations in Section~\ref{simulation} suggest that our algorithm works for much weaker assumptions on the initialization.  

\newpage

\appendix

\renewcommand{\theequation}{\thesection\arabic{equation}}

\section{Derivation of updated equations of SBM in Section~\ref{sec:thoery} }
\label{derivation of updated eq}
First we derive the BCAVI update equations in the simplified settings where the block probability matrix only has two different parameters: diagonal element $p$ and off-diagonal element $q$. The full likelihood of SBM with $K$ communities is given by 
\begin{equation*}
  \mathbb{P}(A,Z) = \prod_{1 \leq i<j \leq n}\prod_{1 \leq a,b \leq K} \left(B_{ab}^{A_{ij}}(1-B_{ab})^{1-A_{ij}}\right)^{Z_{ia}Z_{jb}} \cdot \prod_{i=1}^n \prod_{a=1}^K \pi_a^{Z_{ia}}.
\end{equation*}
We omit the natural range for those indexes for simplicity in the following content unless we need specify it. We consider the mean-field variational approximation $\mathbb{Q}(Z)$, which takes the form 
\[\mathbb{Q}(Z)=\prod_{i}\prod_{a}\Psi_{ia}^{Z_{ia}},\] where $\Psi = \mathbb{E}_\mathbb{Q} [Z]$. Then the evidence lower bound (ELBO) is 
\begin{align*}
    \text{ELBO}( \mathbb{Q}) =& 
 \mathbb{E}_{\mathbb{Q}} \left[ \sum_{i<j} \sum_{a,b} Z_{ia} Z_{jb} [ A_{ij} \log B_{ab} + (1 - A_{ij} ) \log (1 - B_{ab} ) ] + \sum_i \sum_a (Z_{ia} \log \pi_a - Z_{ia} \log \Psi_{ia}) \right]\\
    =& \sum_{i<j}\sum_{a\neq b} \Psi_{ia}\Psi_{jb}[  A_{ij} \log q + (1 - A_{ij} ) \log (1 - q )] \\ &+ \sum_{i<j}\sum_{a} \Psi_{ia}\Psi_{ja}[  A_{ij} \log p + (1 - A_{ij} ) \log (1 - p )] + \sum_i \sum_a \Psi_{ia} \log \frac{\pi_a}{ \Psi_{ia}}.
\end{align*}
By solving the following two equations: 
\begin{align*}
    \frac{\partial \text{ELBO}(\mathbb{Q})}{\partial p} =& \sum_{i<j} \sum_a \Psi_{ia} \Psi_{ja} \left( \frac{A_{ij}}{p} - \frac{1 - A_{ij}}{1-p} \right) = 0, \\
    \frac{\partial \text{ELBO}(\mathbb{Q})}{\partial q} =& \sum_{i<j} \sum_{a \neq b} \Psi_{ia} \Psi_{jb} \left( \frac{A_{ij}}{q} - \frac{1 - A_{ij}}{1-q} \right) = 0,   
\end{align*}
we derived the following updates for $p$ and $q$:
\begin{align*}
    p = \frac{\sum_{i<j} \sum_a \Psi_{ia} \Psi_{ja} A_{ij}}{ \sum_{i<j} \sum_a \Psi_{ia} \Psi_{ja} }, \quad q =& \frac{ \sum_{i<j} \sum_{a \neq b} \Psi_{ia} \Psi_{jb} A_{ij} }{ \sum_{i<j} \sum_{a \neq b} \Psi_{ia} \Psi_{jb} }. 
\end{align*}
To update the variational parameters under the constraints $\sum_a \Psi_{ia} = 1, \forall i \in [n]$, we take the derivatives:
\begin{align*}
\frac{\partial \text{ELBO}(\mathbb{Q})}{\partial \Psi_{ia}} = & \sum_{j: j \neq i} \sum_{b:b \neq a} \Psi_{jb}[ A_{ij} \log q + (1 - A_{ij})\log (1-q) ] + \sum_{j: j\neq i} \Psi_{ja} [ A_{ij} \log p + (1 - A_{ij})\log (1-p) ] \\ &+ \log\frac{\pi_a}{\Psi_{ia}} - 1 \\
= & \sum_{j: j \neq i} \sum_{b} \Psi_{jb}[ A_{ij} \log q + (1 - A_{ij})\log (1-q) ] + \sum_{j: j\neq i} \Psi_{ja} [ A_{ij} \log \frac{p}{q} + (1 - A_{ij})\log \frac{1-p}{1-q} ] \\ &+ \log\frac{\pi_a}{\Psi_{ia}} - 1, \quad i \in [n], \quad a \in [K].
\end{align*}
Therefore, we can apply the Lagrange approach, which gives for each $i \in [n]$ and $a \in [K]$:
\begin{equation*}
\Psi_{ia} \propto \pi_a \exp\left( 2t \sum_{j: j \neq i} \Psi_{ja} [A_{ij} - \lambda] \right), 
\end{equation*}
where $t$ and $\lambda$ are defined by \[ t = \frac{1}{2} \log \frac{p(1-q)}{q(1-p)}, \quad \lambda = \frac{1}{2t} \log \frac{1-q}{1-p} .\]
Since we consider the balanced setting, by letting $\pi_a = 1/K$ for all $a$ we can derive the updated equations mentioned in Section~\ref{sec:thoery}.

\section{Derivation of updated equations of DCSBM in Section~\ref{sec:thoery} }
\label{derivation of updated eq in dc}
The full likelihood of DCSBM with $K$ communities is given by \begin{equation*}
  \mathbb{P}(A,Z) = \prod_{1 \leq i<j \leq n}\prod_{1 \leq a,b \leq K} \left( \frac{ (\theta_i \theta_j B_{ab}) ^ { A_{ij} } e^{- \theta_i \theta_j B_{ab}} }{ (A_{ij} )!} \right)^{Z_{ia}Z_{jb}} \cdot \prod_{i=1}^n \prod_{a=1}^K \pi_a^{Z_{ia}}.
\end{equation*}
We consider the mean-field variational approximation $\mathbb{Q}(Z)$, which takes the form \[\mathbb{Q}(Z)=\prod_{i}\prod_{a}\Psi_{ia}^{Z_{ia}},\] where $\Psi = \mathbb{E}_\mathbb{Q} [Z]$. Then the evidence lower bound (ELBO) is 
\begin{align*}
    \text{ELBO}( \mathbb{Q}) =& 
 \mathbb{E}_{\mathbb{Q}} \left[ \sum_{i<j} \sum_{a,b} Z_{ia} Z_{jb} [ A_{ij} \log( \theta_i \theta_j B_{a b} ) - \theta_i \theta_j B_{a b} - \log (A_{ij} !)
 ] + \sum_i \sum_a (Z_{ia} \log \pi_a - Z_{ia} \log \Psi_{ia}) \right]\\
    =& \sum_{i<j}\sum_{a\neq b} \Psi_{ia}\Psi_{jb}[A_{ij} \log( \theta_i \theta_j q ) - \theta_i \theta_j q - \log (A_{ij} ! )] \\ &+ \sum_{i<j}\sum_{a} \Psi_{ia}\Psi_{ja}[  A_{ij} \log( \theta_i \theta_j p ) - \theta_i \theta_j p - \log (A_{ij} !) ] + \sum_i \sum_a \Psi_{ia} \log \frac{\pi_a}{ \Psi_{ia}}.
\end{align*}
By solving the following two equations: 
\begin{align*}
    \frac{\partial \text{ELBO}(\mathbb{Q})}{\partial p} =& \sum_{i<j} \sum_a \Psi_{ia} \Psi_{ja} \left( \frac{A_{ij}}{p} - \theta_i \theta_j \right) = 0, \\
    \frac{\partial \text{ELBO}(\mathbb{Q})}{\partial q} =& \sum_{i<j} \sum_{a \neq b} \Psi_{ia} \Psi_{jb} \left( \frac{A_{ij}}{q} - \theta_i \theta_j \right) = 0,   
\end{align*}
we derived the following updates for $p$ and $q$:
\begin{align*}
    p = \frac{\sum_{i<j} \sum_a \Psi_{ia} \Psi_{ja} A_{ij}}{ \sum_{i<j} \sum_a \Psi_{ia} \Psi_{ja}\theta_i \theta_j }, \quad q =& \frac{ \sum_{i<j} \sum_{a \neq b} \Psi_{ia} \Psi_{jb} A_{ij} }{ \sum_{i<j} \sum_{a \neq b} \Psi_{ia} \Psi_{jb} \theta_i \theta_j }. 
\end{align*}
To update the variational parameters under the constraints $\sum_a \Psi_{ia} = 1, \forall i \in [n]$, we take the derivatives:
\begin{align*}
\frac{\partial \text{ELBO}(\mathbb{Q})}{\partial \Psi_{ia}} = & \sum_{j: j \neq i} \sum_{b:b \neq a} \Psi_{jb}[ A_{ij} \log (\theta_i \theta_j q) - \theta_i \theta_j q - \log (A_{ij} !) ] + \sum_{j: j\neq i} \Psi_{ja} [ A_{ij} \log (\theta_i \theta_j p) - (\theta_i \theta_j p) \\ & - \log (A_{ij} !) ] + \log\frac{\pi_a}{\Psi_{ia}} - 1 \\
= & \sum_{j: j \neq i} \sum_{b} \Psi_{jb} [ A_{ij} \log (\theta_i \theta_j q) - \theta_i \theta_j q - \log (A_{ij} !) ]  + \sum_{j: j\neq i} \Psi_{ja} [ A_{ij} \log \frac{p}{q} - \theta_i \theta_j (p - q) ] \\ &+ \log\frac{\pi_a}{\Psi_{ia}} - 1, \quad i \in [n], \quad a \in [K].
\end{align*}
Therefore, we can apply the Lagrange approach, which gives for each $i \in [n]$ and $a \in [K]$:
\begin{equation*}
\Psi_{ia} \propto \pi_a \exp\left( 2t \sum_{j: j \neq i} \Psi_{ja} [A_{ij} - \theta_i \theta_j \lambda] \right), 
\end{equation*}
where $t$ and $\lambda$ are defined by \[ t = \frac{1}{2} \log \frac{p}{q}, \quad \lambda = \frac{1}{2t} (p-q) .\]
In the analysis of the balanced setting, we set $\pi_a = 1/K$ for all $a$. To update the degree parameters, we take the derivatives and set them to zero:
\begin{align*}
\frac{\partial \text{ELBO}(\mathbb{Q})}{\partial \theta_i} = \sum_{j: j \neq i}\sum_{a\neq b} \Psi_{ia}\Psi_{jb}[A_{ij} / \theta_i - \theta_j q ] + \sum_{j: j \neq i}\sum_{a} \Psi_{ia}\Psi_{ja}[  A_{ij} / \theta_i - \theta_j p] = 0, \quad i \in [n].
\end{align*}
The above equations can be simplified to
\begin{align*}
\frac{\sum_j A_{ij}}{ \theta_i} = \sum_{j:j \neq i } \sum_{a \neq b} \Psi_{ia} \Psi_{jb} \theta_j q + \sum_{j:j \neq i } \sum_{a} \Psi_{ia} \Psi_{ja} \theta_j p, \quad i \in [n].
\end{align*}

\section{Proofs of Results for SBM in Section~\ref{sec:theory sbm}}
\label{all_proof}

Unlike \cite{sarkar2021random}, we need to control the irregularity of sparse networks carefully. To that end, denote by $A^{\prime}$ the adjacency matrix of the network obtained from the observed network by removing some edges (in an arbitrary way) of nodes with degrees greater than $C_0 d$ so that all node degrees of the new network are at most $C_0 d$, where $C_0>0$ is a large fixed absolute constant. In this paper we choose $C_0$ such that $C_0 d \geq 20np$. Note that $A'$ is only used for the proofs, and our algorithm does not need to specify it.

Now, we state and prove all the lemmas needed in the main proof. In the following lemmas, we assume that the assumptions of Proposition~\ref{prop:parameter estimation} and Theorem~\ref{unknown pa} are satisfied.

\begin{lemma}[Concentration inequality]
\label{lemma: Concentration inequality}
Denote 
$$S = \Big\{ i \in [n], \sum_{j=1}^n A_{ij} \geq C_0 d \Big\}, \quad H_i = \sum_{j=1}^n |A_{ij} - \mathbb{E} A_{ij}| \mathbbm{1}\{ i \in S\}.$$ Then with probability at least $1 - \exp(-C_0 d /4)$, we have
\begin{equation}
   \sum_{i=1}^n H_i \leq n C_0 d \exp(-C_0 d / 4).
\end{equation}
\end{lemma}
\begin{proof}
This lemma is similar to Lemma C.5 in \cite{zhang2020theoretical}. First, note that
$$\mathbb{E} \left[ \sum_j |A_{ij} - \mathbb{E} A_{ij}| \right] \leq 2np(1-p) \leq 2np \leq C_0 d /10.$$  
For any $s \leq C_0 d$, by Bernstein inequality, we have 
\begin{align*}
\mathbb{P}(H_i > s) & \leq \mathbb{P} \left( \sum_j |A_{ij} - \mathbb{E} A_{ij}| -   \mathbb{E} \left[ \sum_j |A_{ij} - \mathbb{E} A_{ij}| \right] > s - 2np \right) \\
&  \leq \exp\left(  -\frac{ (s-2np)^2 / 2 }{np+ (s-2np)/3}  \right)\\
& \leq \exp(-s/2).
\end{align*}
Applying Bernstein inequality again, we get
\begin{align*}
\mathbb{P}(H_i>0) & = \mathbb{P} \left( \sum_i A_{ij} \geq C_0 d  \right) \\ & \leq \mathbb{P} \left( \sum_i A_{ij} - \mathbb{E}\left[ \sum_i A_{ij}\right] \geq 9C_0 d /10  \right) \\ & \leq \exp\left(  -\frac{(9C_0 d /10)^2/2}{ C_0 d /20 +(9C_0 d /10) / 3 }\right) \\ & \leq \exp( - 21/40 C_0 d).
\end{align*}
Thus, we are able to bound $\mathbb{E}[H_i]$ by 
\begin{align*}
\mathbb{E}[H_i] & \leq \int_{0}^{C_0 d} \mathbb{P}(H_i > 0) ds + \int_{C_0 d}^{\infty}  \mathbb{P}(H_i > s) ds \\ & \leq C_0 d \exp( - 21/40 C_0 d) + \int_{C_0 d}^{\infty}  \exp(-s/2) ds \\ &  \leq C_0 d \exp( - C_0 d / 2).
\end{align*}
By Markov inequality, we have
\begin{align*}
\mathbb{P}\left( \sum_i H_i \geq n C_0 d \exp(-C_0 d / 4) \right) \leq \frac{ n \mathbb{E}[H_1] }{n C_0 d \exp(-C_0 d / 4)} \leq \exp(-C_0 d /4).
\end{align*}
The proof is complete.
\end{proof}

\begin{lemma}[Number of removed edges] 
\label{lemma:quantity of removed edges}
Assume that $ d \geq \kappa $ for a large constant $\kappa>0$. Then with high probability $1 - n^{-r}$ for some absolute constant $r$, we have  
\[\sum_{i ,j=1}^n (A_{i j} - A_{i j}^{\prime}) \leq 2n \exp(- C_1 d),\]
where $C_1$ is an absolute constant.
\end{lemma}

\begin{proof}
Case 1: $ \kappa \leq d \leq \eta \log n$ for some constant $\eta \in (0,1)$. It follows from Proposition 1.12 of \cite{benaych2019largest} that with high probability, \[\sum_{i,j} (A_{i j} - A^{\prime}_{i j}) \leq 2 \sum_{k = C _0 d+1}^{\infty} k n  \exp(-f(k)) ,\]
where \[f(x) = x\log\left(\frac{x}{d}\right) - (x-d) - \log\sqrt{2\pi x}.
\]
Notice that for $k > C_0 d$,  
$$f(k) - \log k \geq k ( \log C_0 - 1) - \log \sqrt{2 \pi k} - \log k \geq k( \log C_0 - 1)/2.$$ 
This implies the following inequality
\[ \sum_{i,j} (A_{i j} - A^{\prime}_{i j}) \leq 2n  \sum_{k = C_0 d+1}^{\infty} \exp\left(-\frac{\log C_0 -1}{2} k\right) \leq 2 n \exp(- C_1 d)\]
occurs with high probability for some constant $C_1$.

Case 2: $d >  \eta \log n$. Define $S = \{ i \in [n], \sum_j A_{ij} \geq C_0 d \}$, $H_i = \sum_j |A_{ij} - \mathbb{E} A_{ij}| \mathbbm{1}\{ i \in S\}$ and $M_i = \sum_j |A_{ij}^{\prime} - \mathbb{E} A_{ij} | \mathbbm{1}\{ i \in S\}$. From Lemma~\ref{lemma: Concentration inequality} we know that with probability at least $1 - \exp(-C_0 d/4)$, it holds that 
\begin{equation}\label{first bound in removed edges}
   \sum_i H_i \leq n C_0 d \exp(-C_0 d / 4).
\end{equation}
In addition, from Lemma~\ref{lemma: Concentration inequality} we know \[\mathbb{P}(i \in S) = \mathbb{P}(H_i>0) \leq \exp(-21C_0 d / 40).\]
Therefore, $\mathbb{E}[M_i] \leq (20np + np) \exp(-21C_0 d / 40 ) = 21C_0d/20 \exp(-21C_0 d / 40 ) $. By Markov inequality, we have 
\begin{equation}\label{second bound in removed edges}
  \mathbb{P}\left( \sum_i M_i > 21n C_0 d/20 \exp(-11 C_0 d / 40)\right) \leq \exp(-C_0 d /4).
\end{equation}
It follows from \eqref{first bound in removed edges} and \eqref{second bound in removed edges} that 
\begin{align*}
    \sum_{i ,j} (A_{i j} - A_{i j}^{\prime}) & \leq \sum_i H_i +  \sum_i M_i \\
    & \leq n C_0 d \exp(-C_0 d / 4) +  21n C_0 d/20 \exp(-11 C_0 d / 40) \\
    & \leq 2 n \exp(- C_1 d)
\end{align*}
occurs with high probability for some constant $C_1$.
\end{proof}

\begin{lemma}[Concentration of regularized adjacency matrices]
\label{lemma:concentration of regularized adjacency matrices}
Let $P=\mathbb{E}[A|Z]$.  With high probability $1 - n^{-r}$ for some absolute constant $r$, we have
\[||A^{\prime} - \mathbb{E} A||^2 \leq C_2 d,\]
where $\|.\|$ denotes the spectral norm and 
$C_2$ is an absolute constant.
\end{lemma}

\begin{proof}
This lemma is a special case of Theorem 1.1  in \cite{le2017concentration}.
\end{proof}

\begin{lemma}[Propositions of the parameter $\lambda$]
\label{lemma:Propositions of the parameter lambda}
Under the assumption that $p>q>0$ and $p \asymp q \asymp p - q \asymp \rho_n$, we have the following bounds for $\lambda$:
\begin{align*}
    q< \lambda < p, \quad
    \lambda - q \geq C_3 \rho_n, \quad
    \frac{p+q}{2} - \lambda \geq C_4 \rho_n,
\end{align*}
where $C_3$ and $C_4$ are absolute constants.
\end{lemma}

\begin{proof}
This lemma is the direct result from Proposition 1 and 2 in \cite{sarkar2021random}.
\end{proof}

The next lemma provides crude bounds for the parameters of the threshold BCAVI after the first iteration.

\begin{lemma}[First step: Parameter estimation]
\label{lemma:parameter estimation}
There exist constants $C,c_{1},c_{2}$ only depending on $\varepsilon$ and $K$ such that if $n \rho_n > C$ then with high probability $1-n^{-r}$ for some constant $r>0$,  
\[ t^{(1)} \ge c_{1 },\quad \lambda^{(1)} \le c_{2 } \rho_n.\]
\end{lemma}

\begin{proof}
Recall the formulas for $t^{(1)}$ and $\lambda^{(1)}$ in \eqref{eq:tlambda}.
When $s=1$, from \eqref{update of pq in proof} the first update of $p$ is:
\[
    p^{(1)} = \frac{\sum_{i<j} \sum_a \Psi_{i a}^{(0)} \Psi_{j a}^{(0)} A_{i j}}{ \sum_{i<j} \sum_a \Psi_{i a}^{(0)} \Psi_{j a}^{(0)} }, 
\] where $\Psi^{(0)} = Z^{(0)}$ is the random initialization with error rate $\varepsilon$. We will analyze the numerator and denominator of $p^{(1)}$ separately and combine them to have an estimation of $p^{(1)}$.

First we analyze the denominator of $p^{(1)}$. By the Bernstein inequality, for all positive $t$, we have
\[
\mathbb{P} \Big(\Big| \sum_i  \Psi_{ia}^{(0)} - \frac{n}{K}  \Big| > t \Big) \leq 2\exp\Big( -\frac{ \frac{1}{2}t^2}{ \frac{n}{K} \Big[\varepsilon(1 - \varepsilon) + \varepsilon \big(1 - \varepsilon/(K-1) \big) \Big] + \frac{1}{3}t}\Big).
\]
By choosing $t = n^{2/3}$ and taking the union bound for all $a$, we get 
\begin{align}\label{eq:denominator of p1}
   \sum_{i<j} \sum_a \Psi_{i a}^{(0)} \Psi_{j a}^{(0)}  &= \frac{1}{2}\sum_a  \left( \sum_{i } \Psi_{ia}^{(0)} \right)^2 - \frac{1}{2}\sum_a \sum_{i} ( \Psi_{ia}^{(0)} )^2 \nonumber \\
   &= \frac{1}{2}\sum_a  \left( \sum_{i } \Psi_{ia}^{(0)} \right)^2 - \frac{1}{2}\sum_a \sum_{i}  \Psi_{ia}^{(0)}  \nonumber \\
   &= \frac{n^2}{2K} + \mathcal{E},
\end{align}
where $|\mathcal{E}|< r_1 n ^{5/3}$ with probability at least $1 - 2K\exp(- r_2 n^{1/3})$ for some absolute constant $r_1$ and $r_2$.

Second we analyze the numerator of $p^{(1)}$. By replacing $A_{ij}$ with $P_{ij}+(A_{ij}-P_{ij})$, we decompose the numerator of $p^{(1)}$ as the sum of signal and noise terms:
\begin{eqnarray*}
\text{signal of $p^{(1)}$} &:=&\sum_{i<j} \sum_a \Psi_{i a}^{(0)} \Psi_{j a}^{(0)} P_{i j},\\
\text{noise of $p^{(1)}$} &:=&\sum_{i<j} \sum_a \Psi_{i a}^{(0)} \Psi_{j a}^{(0)} (A_{i j} - P_{i j}).
\end{eqnarray*}
Since $P_{ij} = p$ when $z_i=z_j$ and $P_{ij} = q$ when $z_i \neq z_j$, the signal can be written as \[\sum_{i<j} \sum_a \Psi_{i a}^{(0)} \Psi_{j a}^{(0)} P_{i j} = \sum_a \sum_{i<j : z_i = z_j} \Psi_{i a}^{(0)} \Psi_{j a}^{(0)} p + \sum_a \sum_{i<j : z_i \neq z_j} \Psi_{i a}^{(0)} \Psi_{j a}^{(0)} q.\]
By the assumption of $\Psi^{(0)}$, we have 
\begin{align*}
    \mathbb{E} \Big[\sum_{i:z_i=a} \Psi^{(0)}_{ia} \Big] = \frac{n}{K}(1 - \epsilon), \quad \mathbb{E} \Big[\sum_{i:z_i=b, z_i \neq a} \Psi^{(0)}_{ia} \Big] = \frac{n \epsilon}{K(K-1)}.
\end{align*}
Therefore we can apply the Bernstein inequalities for each $\sum_{i:z_i=a} \Psi^{(0)}_{i a}$ and $\sum_{i:z_i=b, z_i \neq a} \Psi^{(0)}_{ia}$. As in the analysis of the denominator of $p^{(1)}$, by choosing $t = n^{2/3}$ and taking the union bound for all those $K^2$ Bernstein inequalities, we have 
\begin{align}\label{eq:signal of p1}
    \text{signal of $p^{(1)}$} =& \sum_{i<j} \sum_a \Psi_{i a}^{(0)} \Psi_{j a}^{(0)} P_{i j} = \sum_a \sum_{i<j : z_i = z_j} \Psi_{i a}^{(0)} \Psi_{j a}^{(0)} p + \sum_a \sum_{i<j : z_i \neq z_j} \Psi_{i a}^{(0)} \Psi_{j a}^{(0)} q  \nonumber \\ =& K\left[ \frac{1}{2}[ \frac{n}{K}(1 - \varepsilon)]^2 + \frac{1}{2}[\frac{n \varepsilon}{K(K-1)}]^2(K-1) \right]p  \nonumber \\ &+ K\left[ \frac{n}{K}(1 - \varepsilon) \frac{n \varepsilon}{K(K-1)} (K-1) + [\frac{n \varepsilon}{K(K-1)}]^2 \binom{K-1}{2} \right]q + \mathcal{E},
\end{align}
where $|\mathcal{E}|< r_1 n ^{5/3} \rho_n$ with probability at least $1 - 2K^2\exp(- r_2 n^{1/3})$ for some absolute constant $r_1$ and $r_2$.

We can analyze the noise of $p^{(1)}$ conditioning on $\Psi^{(0)}$. Since
\begin{equation*}
      \mathbb{E}\left[ \sum_{i<j} \sum_a \Psi_{i a}^{(0)} \Psi_{j a}^{(0)} (A_{i j} - P_{i j}) | \Psi^{(0)} \right] = 0,  
\end{equation*}
we have
\begin{align*}
    \text{Var}\left[\sum_{i<j} \sum_a \Psi_{i a}^{(0)} \Psi_{j a}^{(0)} (A_{i j} - P_{i j}) \right] &= \mathbb{E}\left[ \text{Var} \left[ \sum_{i<j} \sum_a \Psi_{i a}^{(0)} \Psi_{j a}^{(0)} (A_{i j} - P_{i j}) | \Psi^{(0)} \right] \right] \\ &\leq r_1 n^2 K \rho_n,
\end{align*}
where $r_1$ is an absolute constant. By Chebyshev’s inequality, for any real number $k>0$, we have 
\begin{align} \label{eq:noise of p1}
 \mathbb{P} \Big( 
\Big| \sum_{i<j} \sum_a \Psi_{i a}^{(0)} \Psi_{j a}^{(0)} (A_{i j} - P_{i j}) \Big| < kn \sqrt{r_1 K \rho_n} \Big) \geq 1 - \frac{1}{k^2}.
\end{align}
By choosing $k = n^{3/4}\sqrt{\rho_n}$, the noise term is bounded by $ \sqrt{r_1 K} n^{7/4}\rho_n$ with probability at least $1 - \frac{1}{n^{3/2} \rho_n}$ for some absolute constant $r_1$.

It follows from \eqref{eq:denominator of p1}, \eqref{eq:signal of p1} and \eqref{eq:noise of p1} that 
\begin{equation}\label{eq: estimation of p1}
 p^{(1)} = \left[  (1 - \varepsilon)^2 + \frac{\varepsilon^2}{K-1}\right]p+  \left[2(1 - \varepsilon)\varepsilon + \frac{(K-2) \varepsilon^2 }{K-1}\right]q+ \mathcal{E}, 
\end{equation}
where $|\mathcal{E}| < r_1 n^{-1/4}p_n$ with probability at least $1 - n^{-r_2}$ for some absolute constant $r_1$ and $r_2$.

A similar analysis gives
\begin{eqnarray}\label{eq: estimation of q1}
    q^{(1)} = \left[  K-2 + (1 - \varepsilon)^2 + \frac{\varepsilon^2}{K-1}\right] \frac{q}{K-1} + \left[2(1 - \varepsilon)\varepsilon + \frac{(K-2) \varepsilon^2 }{K-1}\right]\frac{p}{K-1}+   \mathcal{E}, 
\end{eqnarray}
where $|\mathcal{E|} < r_1 n^{-1/4}p_n$ with probability at least $1 - n^{-r_2}$ for some absolute constant $r_1$ and $r_2$.

Notice here $p^{(1)}$ is always larger than $q^{(1)}$ with high probability. To verify this, first notice the sum of coefficients in front of $p$ and $q$ is $1$ in both \eqref{eq: estimation of p1} and \eqref{eq: estimation of q1}. Therefore, we only need show that \[
\left[  (1 - \varepsilon)^2 + \frac{\varepsilon^2}{K-1}\right] > \left[2(1 - \varepsilon)\varepsilon + \frac{(K-2) \varepsilon^2 }{K-1}\right] \frac{1}{K-1}
,\]
which is equivalent to \[
(K-1)(1 - \varepsilon)^2 + \frac{\varepsilon^2}{K-1} > 2\varepsilon(1- \varepsilon).\]
The last inequality holds by fundamental inequalities. The similar calculation also yields that 
\begin{equation}
\label{eq: estimation of p1 - q1}
p^{(1)} - q^{(1)} =  \left( (K-1)(1 - \varepsilon)^2 + \frac{\varepsilon^2}{K-1} - 2\varepsilon(1- \varepsilon) \right)\frac{p-q}{K-1} + \mathcal{E}, 
\end{equation}
where $|\mathcal{E|} < r_1 n^{-1/4}p_n$ with probability at least $1 - n^{-r_2}$ for some absolute constant $r_1$ and $r_2$.

Since 
\[t^{(1)} = \frac{1}{2} \log \frac{p^{(1)} ( 1 - q^{(1)})}{ q^{(1)} (1 - p^{(1)})}, \quad \lambda^{(1)} = \frac{1}{2t^{(1)}} \log \frac{1 - q^{(1)}}{ 1 - p^{(1)}},\]
it follows directly from \eqref{eq: estimation of p1}, \eqref{eq: estimation of q1} and \eqref{eq: estimation of p1 - q1} that 
\begin{equation}\label{eq: estimation of t1 and lambda1}
    t^{(1)} \ge c_{1}, \quad \lambda^{(1)} \le c_{2}\rho_n
\end{equation} 
for some constants $c_1$, $c_2$ only depending on $\varepsilon$ and $K$.
\end{proof}

The next lemma provides the accuracy of the label estimates after the first iteration.

\begin{lemma}[First step: Label estimation]\label{lem:label first step}
There exist constants $C, c>0$ only depending on $\varepsilon$ and $K$ such that if $n\rho_n \ge C$ then with high probability $1-n^{-r}$ for some constant $r>0$,
$$||\Psi^{(1)}- Z||_1 \leq n\exp(-c d).$$
\end{lemma}
\begin{proof}
To deal with the estimation of $\Psi$ in the first iteration, we will first quantify the probability of misclassification for each node and then derive the upper bound of $||\Psi^{(1)} - Z||$.

For the node $i$, from \eqref{update of Psi in proof} and \eqref{eq: estimation of t1 and lambda1} we know $\Psi_{i}^{(1)} = Z_i$ after threshold if any only if the following inequality 
 \[ \sum_{j:j\neq i} \Psi_{j z_i}^{(0)} (A_{ij} -  \lambda^{(1)}) > \sum_{j:j\neq i} \Psi_{ja}^{(0)} (A_{ij} -  \lambda^{(1)})
\]
holds for any $a \in [K]$ such that $a \neq z_i$. This is equivalent to \[ \sum_{j:j\neq i} (\Psi_{j z_i}^{(0)} - \Psi_{ja}^{(0)}) (A_{ij} -  \lambda^{(1)}) > 0
\] holds for any $a \in [K]$ such that $a \neq z_i$. We can decompose the left-hand side of the above inequality by the sum of signal and noise terms. We will first quantify the signal term and then show that most of the noise terms can be upper bounded by the signal term. Denote the signal and noise terms in the first iteration by $SIG_{ia}^{(0)}$ and $ r_{ia}^{(0)}$ defined as:

\begin{align*}
   \text{SIG}_{ia}^{(0)} :=& \sum_{j:j\neq i} (\Psi_{j z_i}^{(0)} - \Psi_{ja}^{(0)}) (E[A_{ij}] - \lambda^{(1)}), \\
    r_{ia}^{(0)} :=& \sum_{j:j\neq i} (\Psi_{j z_i}^{(0)} - \Psi_{ja}^{(0)}) ( A_{ij} - E[A_{ij}]) := \sum_{j:j \neq i} Y_{ija}.  
\end{align*}

Similarly as in the derivation of \eqref{eq:signal of p1}, if follows from Lemma~\ref{lemma:parameter estimation} that 
\begin{align*}
    \text{SIG}_{ia}^{(0)} &= 
    \sum_{j:j\neq i} (\Psi_{j z_i}^{(0)} - \Psi_{ja}^{(0)}) (E[A_{ij}] - \lambda^{(1)}) \\ &= \sum_{j:j\neq i,z_j = z_i} (\Psi_{j z_i}^{(0)} - \Psi_{ja}^{(0)})p + \sum_{j:j\neq i, z_j \neq z_i } (\Psi_{j z_i}^{(0)} - \Psi_{ja}^{(0)})q -  \sum_{j:j\neq i} (\Psi_{j z_i}^{(0)} - \Psi_{ja}^{(0)})\lambda^{(1)}\\
    & = \frac{n}{K} [ 1- \varepsilon - \frac{\varepsilon}{ K-1 }][p-q] + \mathcal{E},
\end{align*}
where $|\mathcal{E}|<r_1 n^{2/3}\rho_n$ with probability at least $1 - n^{-r_2}$ for some absolute constant $r_1$ and $r_2$. Here the upper bound of $|\mathcal{E}|$ is a union bound for all $i \in [n]$ and $a \in [K]$ such that $a \neq z_i$.

Now we are ready to bound the noise terms $r_{ia}^{(0)}$. Denote
\[ \delta = \frac{[1 - \varepsilon - \frac{\varepsilon}{K-1}](p-q)}{2 [p + (K-1)q]},
\]
and note that $\delta d < SIG_{ia}^{(0)}$. As we discussed, if $|r_{ia}^{(0)}| < \delta d$ holds for all $a \in [K]$ and $a \neq z_i$, then $|r_{ia}^{(0)}|< SIG_{ia}^{(0)}$, which implies $\Psi_i^{(1)} = Z_i$. Otherwise $||\Psi_i^{(1)} -Z_i  ||_1 \leq 2$.  For an event $E$, denote by $\mathbbm{1}(E)$ the indicator of $E$. Therefore, we have the fact that
\begin{equation}\label{bound of Psi1 related to noise}
||\Psi^{(1)} - Z||_1 = \sum_{i} || \Psi^{(1)}_i - Z_i ||_1 \leq 2 \sum_{i} \mathbbm{1}\Big( \bigcup_{a:a \neq z_i}  \Big\{ |r_{ia} ^{(0)} | > \delta d
 \Big\} \Big).
\end{equation}
First we analyze the noise terms $r_{ia}^{(0)}$ conditioning on $\Psi^{(0)}$. 
To deal with the dependency among rows of the adjacency matrix due to symmetry, 
let $\{Y_{i ja}^*\}$ be an independent copy of $\{Y_{i ja}\}$. Then by the triangle inequality, 
 \begin{align*}
 \mathbbm{1}\Big(\bigcup_{a:a \neq z_i}  
 \Big\{ |r_{ia} ^{(0)} | > \delta d
 \Big\} \Big) &= \mathbbm{1} \Big( \bigcup_{a:a \neq z_i} \Big\{ |\sum_{j: j \neq i} Y_{ija }|>\delta d \Big\} \Big)  \leq \mathbbm{1}\Big( \bigcup_{a:a \neq z_i} \Big\{ \Big| \sum_{j=1}^{i-1} Y_{ija}\Big|+\Big|\sum_{j=i+1}^{n} Y_{ija}\Big|> \delta d \Big\} \Big) \\ 
 & \leq \mathbbm{1}\Big( \bigcup_{a:a \neq z_i} \Big\{ \Big| \sum_{j=1}^{i-1} Y_{ija}\Big|+\Big|\sum_{j=i+1}^{n} Y_{ija}\Big| + \Big| \sum_{j=1}^{i-1} Y_{ija}^*\Big| +\Big|\sum_{j=i+1}^{n} Y_{ija}^*\Big| > \delta d \Big\} \Big) \\ 
 & \leq \mathbbm{1}\Big( \bigcup_{a:a \neq z_i} \Big\{ \Big| \sum_{j=1}^{i-1} Y_{ija}\Big|+\Big|\sum_{j=i+1}^{n} Y_{ija}^*\Big| > \frac{\delta d}{2} \Big\} \Big) \\ & \quad + \mathbbm{1}\Big( \bigcup_{a:a \neq z_i} \Big\{ \Big| \sum_{j=1}^{i-1} Y_{ija}^*\Big|+\Big|\sum_{j=i+1}^{n} Y_{ija}\Big| > \frac{\delta d}{2} \Big\} \Big).
 \end{align*}
 Applying Bernstein's inequality for $ \sum_{j: j\neq i} Y_{ija} $ with $\mathbb{E}[Y_{ija}] = 0$, $|Y_{ija}|<1$ and $\sum_{j:j \neq i}\mathbb{E}[Y_{ija}^2] \leq d$, we get
 \begin{align*}
 \mathbb{P}\Big(\Big|\sum_{j=1}^{i-1} Y_{ija}\Big|+\Big|\sum_{j=i+1}^{n} Y_{ija}^*\Big| > \frac{\delta d}{2}\Big) 
  &\leq \mathbb{P}\Big(\Big|\sum_{j=1}^{i-1} Y_{ija}\Big|>\frac{\delta d}{4}\Big) + \mathbb{P}\Big(\Big|\sum_{j=i+1}^{n} Y_{ija}^*\Big|>\frac{\delta d}{4}\Big) \\
  &\leq 4\exp\left(\frac{-(\delta d/4)^2/2}{d+ (\delta d /4)/3}\right) = 4\exp\left(\frac{-\delta^2d}{32+8\delta/3} \right).
 \end{align*} 
Then by taking the union bound for all $a \in [K]$ such that $a \neq z_i$ we get 
 \[\mathbb{P}\Big( \bigcup_{a:a \neq z_i} \Big\{ \Big| \sum_{j=1}^{i-1} Y_{ija}\Big|+\Big|\sum_{j=i+1}^{n} Y_{ija}^*\Big| > \frac{\delta d}{2} \Big\} \Big) \leq 4(K-1)\exp \left(\frac{-\delta^2d}{32+8\delta/3} \right) .\]
Since $(\sum_{j=1}^{i-1}Y_{i ja},\sum_{j=i+1}^{n} Y_{i ja}^*)$, $1\le i\le n$, are independent conditioning on $\Psi^{(0)}$, by Bernstein's inequality, 
 \begin{align*}
  \sum_{i=1}^{n} \mathbbm{1}\Big(  \bigcup_{a:a \neq z_i} \Big\{ \Big| \sum_{j=1}^{i-1} Y_{ija}\Big|+\Big|\sum_{j=i+1}^{n} Y_{ija}^*\Big| > \frac{\delta d}{2} \Big\} \Big) \leq 8(K-1)n \exp\left(\frac{-\delta^2d}{32+8\delta/3}\right)
 \end{align*}
occurs with probability at least $1 - n^{-r_1}$ for some absolute constant $r_1$. The same bound holds for $\sum_{i=1}^{n} \mathbbm{1}\Big(  \bigcup_{a:a \neq z_i} \Big\{ \Big| \sum_{j=1}^{i-1} Y_{ija}^*\Big|+\Big|\sum_{j=i+1}^{n} Y_{ija} \Big| > \frac{\delta d}{2} \Big\} \Big)$. Therefore, the following event 
\[\mathcal{A}_1 = \left\{\sum_{i=1}^n \mathbbm{1}\Big( \bigcup_{a:a \neq z_i}  \Big\{ |r_{ia} ^{(0)} | > \delta d
 \Big\} \Big) \leq 16(K-1)n \exp\left(\frac{-\delta^2d}{32+8\delta/3}\right) \right\} \]
occurs with probability at least $1 - n^{-r_1}$ for some absolute constant $r_1$ conditioning on $\Psi^{(0)}$. By the law of total probability the event $\mathcal{A}_1$ occurs with probability at least $1 - n^{-r_1}$ for some absolute constant $r_1$. Clearly, it follows from \eqref{bound of Psi1 related to noise} that $\mathcal{A}_1$ implies \[ ||\Psi^{(1)} - Z||_1 \leq 
32(K-1) n \exp\left(\frac{-\delta^2d}{32+8\delta/3}\right)
\]
with high probability.
\end{proof}

Using the above lemmas, we are now ready to prove Proposition~\ref{prop:parameter estimation} and Theorem~\ref{unknown pa} by induction.

\medskip

\begin{proof}[Proof of Proposition~\ref{prop:parameter estimation} and Theorem~\ref{unknown pa}]

We will use induction to prove the results for both block parameter estimation and label estimation beyond first step. First we assume that for $s>1$, we have 
\begin{equation}\label{induction assumption}
||\Psi^{(s-1)} - Z||_1 =: n G_{s-1} \leq n\exp(-cd)
\end{equation}
for some constant $c$ only dependent on $\varepsilon$ and $K$. Note that the constant $c$ here does not depend on the number of iterations $s$. We will specify how to choose the constant $c$ afterward.

\medskip

\noindent \textbf{Beyond the first step: block parameter estimation}.
From \eqref{update of pq in proof}, the estimate of $p$ in the $s$-th iteration is given by
\[  p^{(s)} = \frac{\sum_{i<j} \sum_a \Psi_{i a}^{(s-1)} \Psi_{j a}^{(s-1)} A_{i j}}{ \sum_{i<j} \sum_a \Psi_{i a}^{(s-1)} \Psi_{j a}^{(s-1)} }.
\]
First let us define a ``better" estimation of $p^{(s)}$:\[
\underline{p}^{(s)} = \frac{\sum_{i<j} \sum_a Z_{i a} Z_{j a} A_{i j}}{ \sum_{i<j} \sum_a Z_{i a} Z_{j a} }.
\] 
The reason why we say $\underline{p}^{(s)}$ is a ``better" estimation is that we use the ground truth $Z$ to replace the estimated $\Psi^{(s-1)}$ as in the expression of $p^{(s)}$. From the Bernstein inequality we have \begin{equation}\label{bound of underline ps}
|\underline{p}^{(s)} - p| \leq p\exp(-cd)
\end{equation}
occurs with probability at least $1 - n^{-r_1}$ for some absolute constant $r_1$. To show $p^{(s)}$ is close to $\underline{p}^{(s)}$, we will show both the numerators and denominators are close.

To quantify the numerator of $p^{(s)}$, we can decompose it by
\begin{align*}
\sum_{i<j} \sum_a \Psi_{i a}^{(s-1)} \Psi_{j a}^{(s-1)} A_{i j} &= \sum_{i<j} \sum_a Z_{i a} Z_{j a} A_{i j} + \sum_{i<j} \sum_a [\Psi_{i a}^{(s-1)} \Psi_{j a}^{(s-1)} -  Z_{i a} Z_{j a}] A_{i j} \\ &= \sum_{i<j} \sum_a Z_{i a} Z_{j a} A_{i j} + \sum_{i<j} \sum_a [\Psi_{i a}^{(s-1)} \Psi_{j a}^{(s-1)} -  Z_{i a} Z_{j a}] A_{i j}^{\prime} \\ &+  \sum_{i<j} \sum_a [\Psi_{i a}^{(s-1)} \Psi_{j a}^{(s-1)} -  Z_{i a} Z_{j a}] (A_{i j} - A_{i j}^{\prime}).
\end{align*}
To bound the second term, from \eqref{induction assumption} and the assumption that each row and column sum of $A^{\prime}$ is bounded by $C_0 d$, we have
\begin{align}\label{second bound in ps}
    \big|\sum_{i<j} \sum_a \Psi_{i a}^{(s-1)} \Psi_{j a}^{(s-1)} A_{i j}^{\prime} - \sum_{i<j} \sum_a Z_{i a} Z_{j a} A_{i j}^{\prime} \big| = \big| \sum_{i<j} \sum_a (\Psi_{i a}^{(s-1)} - Z_{i a} )(\Psi_{j a}^{(s-1)} - Z_{j a} )A_{i j}^{\prime} \nonumber \\ + \sum_{i<j} \sum_a (\Psi_{i a}^{(s-1)} - Z_{i a} )Z_{j a} A_{i j}^{\prime} + \sum_{i<j} \sum_a Z_{i a} (\Psi_{j a}^{(s-1)} - Z_{j a} )A_{i j}^{\prime} \big| \nonumber \\ \leq 3||\Psi_{(s-1)} - Z||_1 C_0 d \leq 3n \exp(-cd) C_0 d.
\end{align}
To bound the third term, first notice it follows from Lemma~\ref{lemma:quantity of removed edges} that \begin{align}\label{first bound in ps}
   \sum_{i<j} \sum_a \Psi_{i a}^{(s-1)} \Psi_{j a}^{(s-1)} (A_{i j} - A_{i j}^{\prime}) \leq \sum_{i<j}  (\sum_a \Psi_{i a}^{(s-1)})( \sum_a \Psi_{j a}^{(s-1)}) (A_{ i j} - A^{\prime}_{i j}) \nonumber \\ = \sum_{i<j} (A_{ i j} - A^{\prime}_{i j}) \leq 2n\exp(-C_1 d) . 
\end{align}
Similarly as in \eqref{first bound in ps} we have the same bound for $\sum_{i<j} \sum_a Z_{ia}Z_{ja} (A_{ij} - A^{\prime}_{i j})$ and it holds that 
\begin{equation}\label{third bound in ps}
 \big| \sum_{i<j} \sum_a [\Psi_{i a}^{(s-1)} \Psi_{j a}^{(s-1)} -  Z_{i a} Z_{j a}] (A_{i j} - A_{i j}^{\prime}) \big| \leq 4n\exp(-C_1 d).
\end{equation}
From \eqref{second bound in ps} and \eqref{third bound in ps}, the numerator of $p^{(s)}$ satisfies  
\begin{align}\label{estimation of numerator of ps}
\big| \sum_{i<j} \sum_a \Psi_{i a}^{(s-1)} \Psi_{j a}^{(s-1)} A_{i j} -   \sum_{i<j} \sum_a Z_{i a} Z_{j a} A_{i j}\big| \leq 3n\exp(-c d)C_0 d + 4n\exp(-C_1 d).
\end{align}
Moreover, notice that the denominator of $p^{(s)}$ satisfies:
\begin{align}\label{estimation of denominator of ps}
\big| \sum_{i<j} \sum_a \Psi_{i a}^{(s-1)} \Psi_{j a}^{(s-1)} - \sum_{i<j} \sum_a Z_{i a} Z_{j a} \big| = \big| \sum_{i<j} \sum_a (\Psi_{i a}^{(s-1)} - Z_{i a} )(\Psi_{j a}^{(s-1)} - Z_{j a} ) \nonumber \\ + \sum_{i<j} \sum_a (\Psi_{i a}^{(s-1)} - Z_{i a} )Z_{j a} + \sum_{i<j} \sum_a Z_{i a} (\Psi_{j a}^{(s-1)} - Z_{j a} ) \big| \nonumber \\ \leq 3||\Psi_{(s-1)} - Z||_1 n \leq 3n^2 \exp(-cd).
\end{align}
Therefore, from \eqref{estimation of numerator of ps}, \eqref{estimation of denominator of ps} and \eqref{bound of underline ps} we get
\[ |p^{(s)} - p| \leq |p^{(s)} - \underline{p}^{(s)}| + | \underline{p}^{(s)} - p| \leq p\exp(-c_3 d),\]
where $c_3$ is a constant only depending on $\varepsilon$ and $K$. 
By the similar argument we have 
\[|q^{(s)} - q| \leq q\exp(-c_4 d),\]
where $c_4$ is a constant only depending on $\varepsilon$ and $K$. 
Since
\[t^{(s)} = \frac{1}{2} \log \frac{p^{(s)}}{ q^{(s)}}, \quad \lambda^{(s)} = \frac{1}{2t^{(s)}} \log (p^{(s)} - q^{(s)} ),\]
it follows directly that 
\begin{equation}\label{estimatino of ts and lambdas}
    |t^{(s)} - t| \leq t\exp(-c_5 d),\quad
    |\lambda^{(s)} - \lambda| \leq \lambda \exp(-c_6 d)
,
\end{equation}
where $c_5$ and $c_6$ are the constants only depending on $\varepsilon$ and $K$.

\medskip
\noindent \textbf{Beyond the first step: label estimation}. 
Similarly as in the first iteration, from \eqref{update of Psi in proof} and \eqref{estimatino of ts and lambdas} we know $\Psi_{i}^{(s)} = Z_i$ after threshold if any only if the following inequality  \[ \sum_{j:j\neq i} (\Psi_{j z_i}^{(s-1)} - \Psi_{ja}^{(s-1)}) (A_{ij} -  \lambda^{(s)}) > 0
\] holds for any $a \in [K]$ such that $a \neq z_i$. By the similar decomposition for the left-hand side of this inequality, we can define the signal and noise terms in the $s$-th iteration as
\begin{align*}
   \text{SIG}_{ia}^{(s-1)} :=& \sum_{j:j\neq i} (\Psi_{j z_i}^{(s-1)} - \Psi_{ja}^{(s-1)}) (E[A_{ij}] - \lambda^{(s)}), \\
    r_{ia}^{(s-1)} :=& \sum_{j:j\neq i} (\Psi_{j z_i}^{(s-1)} - \Psi_{ja}^{(s-1)}) ( A_{ij} - E[A_{ij}]).  
\end{align*}

To quantify the signal term, first we state the sign of two expressions. It follows from Lemma~\ref{lemma:Propositions of the parameter lambda} and \eqref{estimatino of ts and lambdas} that 
\begin{equation}\label{sign of two expressions}
    p - \lambda^{(s)} >0, \quad q - \lambda^{(s)}<0.
\end{equation}
Second we give the upper bounds of two expressions. By the triangle inequality, from \eqref{induction assumption} we have: 
\begin{align*}
    \Big|\sum_{j:j\neq i,z_j = z_i} (\Psi_{j z_i}^{(s-1)} - \Psi_{ja}^{(s-1)}) -  \sum_{j:j\neq i,z_j = z_i} (Z_{j z_i} - Z_{ja}) \Big| \leq || \Psi^{(s-1)} - Z||_1 = nG_{s-1}.
\end{align*}
Since $ \sum_{j:j\neq i,z_j = z_i} (Z_{j z_i} - Z_{ja}) = \frac{n}{K}-1$, we have \begin{align}\label{upper bound of first expressions}
    \Big|\sum_{j:j\neq i,z_j = z_i} (\Psi_{j z_i}^{(s-1)} - \Psi_{ja}^{(s-1)}) - (n/K - 1 ) \Big| \leq nG_{s-1}.
\end{align}
Similarly we can derive the other upper bound: 
\begin{align}\label{upper bound of second expressions}
    \Big|\sum_{j:j\neq i,z_j \neq z_i} (\Psi_{j z_i}^{(s-1)} - \Psi_{ja}^{(s-1)}) - (-n/K ) \Big| \leq nG_{s-1}.
\end{align}
Therefore from \eqref{sign of two expressions}, \eqref{upper bound of first expressions} and \eqref{upper bound of second expressions}
we can quantify the signal term by
\begin{align*}
    SIG_{ia}^{(s-1)} = \sum_{j:j\neq i,z_j = z_i} (\Psi_{j z_i}^{(s-1)} - \Psi_{ja}^{(s-1)})(p - \lambda^{(s)}) + \sum_{j:j\neq i, z_j \neq z_i } (\Psi_{j z_i}^{(s-1)} - \Psi_{ja}^{(s-1)})(q - \lambda^{(s)}) \\
    \geq (n/K-1 - nG_{s-1})(p - \lambda^{(s)}) + (-n/K+nG_{s-1})(q - \lambda^{(s)}) \\
    \geq [\frac{n}{K} - nG_{s-1} - C_4](p-q) \geq [\frac{n}{K} - n\exp(-cd) - C_4](p-q),
\end{align*}
where $C_4$ is an absolute constant. 

To analyze the noise terms, let us define a small constant $\delta^{\prime}$:
\[ \delta^{\prime} = \frac{C_3 (p-q)}{2 [p + (K-1)q]},
\]
where $\delta^{\prime}$ satisfies $\delta^{\prime} d < SIG_{ia}^{(s-1)}$ and $C_3$ is an absolute constant. We can decompose the noise terms by
\begin{align*}
r^{(s-1)}_{ia} = \sum_{j:j \neq i} (A _{i j} - A^{\prime}_{i j} )(\Psi^{(s-1)}_{j z_i} - \Psi^{(s-1)}_{ja} - Z_{j z_i} +Z_{ja} ) + \sum_{j:j \neq i} (A^{\prime}_{i j} - P_{i j} )(\Psi^{(s-1)}_{j z_i} \\- \Psi^{(s-1)}_{ja} - Z_{j z_i} +Z_{ja} ) + \sum_{j:j \neq i} (A_{i j} - P_{i j} )(Z_{j z_i} - Z_{ja})  =: r_{1ia}^{(s-1)} + r_{2ia}^{(s-1)} + r_{3ia}.
\end{align*}
If $r_{ia}^{(s-1)}>\delta^{\prime} d$, it implies that at least one components above is larger than $\delta^{\prime} d/3$. Let $\delta_1 = \delta_2 = \delta_3 = \delta^{\prime}/3$, we have the fact that
\begin{align*}
\mathbbm{1}\Big(\bigcup_{a:a \neq z_i}  
 \Big\{ |r_{ia} ^{(s-1)} | > \delta^{\prime} d
 \Big\} \Big) \leq \mathbbm{1}\Big(\bigcup_{a:a \neq z_i}  
 \Big\{ |r_{1ia} ^{(s-1)} | > \delta_1 d
 \Big\} \Big)+\mathbbm{1}\Big(\bigcup_{a:a \neq z_i}  
 \Big\{ |r_{2ia} ^{(s-1)} | > \delta_2 d
 \Big\} \Big) \\+\mathbbm{1}\Big(\bigcup_{a:a \neq z_i}  
 \Big\{ |r_{3ia}  | > \delta_3 d
 \Big\} \Big).
 \end{align*}
 
By using the similar argument as in the first iteration, the following event 
\begin{equation}\label{first bound in Psis}
\mathcal{A}_2 = \left\{\sum_{i=1}^n \mathbbm{1}\Big( \bigcup_{a:a \neq z_i}  \Big\{ |r_{3ia}| > \delta_3 d
 \Big\} \Big) \leq 16(K-1)n \exp\left(\frac{-\delta_3^2d}{32+8\delta_3/3}\right) \right\} 
 \end{equation}
occurs with probability at least $1 - n^{-r_1}$ for some absolute constant $r_1$.

To deal with the first noise component $r_{1ia}^{(s-1)}$, first notice that the following inequality holds by Lemma~\ref{lemma:quantity of removed edges}:\[
 \sum_i \sum_a |r_{1ia}^{(s-1)}| \leq  \sum_i \sum_a \sum_{j:j \neq i} 2(A_{ij} - A_{ij}^{\prime}) \leq 4nK \exp(- C_1 d).
\]
This implies
\[\sum_{i}\sum_{a} \mathbbm{1}(|r^{(s-1)}_{1ia}| > \delta_1 d) \leq \frac{ 4nK\exp(-C_1 d)}{\delta_1 d}.
\]
Consequently we have
\begin{equation}\label{second bound in Psis}
\sum_{i=1}^n \mathbbm{1}\Big(\bigcup_{a:a \neq z_i}  
 \Big\{ |r_{1ia} ^{(s-1)} | > \delta_1 d
 \Big\} \Big) \leq \sum_i \sum_a \mathbbm{1}(|r^{(s-1)}_{1ia}| > \delta_1 d) \leq \frac{ 4nK\exp(-C_1 d)}{\delta_1 d}.
 \end{equation}

To deal with the second noise component $r_{2ia}^{(s-1)}$, first notice that the following inequality holds by \eqref{induction assumption} and Lemma~\ref{lemma:concentration of regularized adjacency matrices}:
\begin{align*}
\sum_i \sum_a |r_{2ia}^{(s-1)}|^2 &= \sum_a \sum_{i} \left( \sum_{j: j \neq i} (A^{\prime}_{i j} - P_{i j} )(\Psi^{(s-1)}_{j z_i} - \Psi^{(s-1)}_{ja} -Z_{j z_i} +Z_{ja} ) \right)^2 \\
& \leq \sum_a ||A^{\prime} - P||^2 \cdot 2nG_{s-1} \leq 2 C_2 d K nG_{s-1}.
\end{align*}
This implies \[
\sum_{i}\sum_{a} \mathbbm{1}(|r^{(s-1)}_{2ia}| > \delta_2 d) \leq \frac{ 2C_2 dK nG_{s-1} }{\delta_2^2 d^2},
\]
and consequently,
\begin{equation}\label{third bound in Psis}
 \sum_{i=1}^n \mathbbm{1}\Big(\bigcup_{a:a \neq z_i}  
 \Big\{ |r_{2ia} ^{(s-1)} | > \delta_1 d
 \Big\} \Big) \leq \sum_{i}\sum_{a} \mathbbm{1}(|r^{(s-1)}_{2ia}| > \delta_2 d) \leq \frac{ 2 C_2 d K nG_{s-1} }{\delta_2^2 d^2}.
 \end{equation}
Therefore, it follows from \eqref{first bound in Psis}, \eqref{second bound in Psis} and \eqref{third bound in Psis} that 
\begin{align*}
||\Psi^{(s)} - Z||_1 \leq 2 \sum_{i} \mathbbm{1}\Big( \bigcup_{a:a \neq z_i}  \Big\{ |r_{ia} ^{(s-1)} | > \delta^{\prime} d
 \Big\} \Big) 
 \leq  \frac{4C_2 K}{ \delta_2^2 d} || \Psi^{(s-1)}- Z ||_1 \\+ \frac{8nK\exp(-C_1 d)}{\delta_1 d} +  32(K-1)n \exp\left(\frac{-\delta_3^2d}{32+8\delta_3/3}\right).
\end{align*}
Since we assume that $d$ is larger than a constant that depending only on $\epsilon$ and $K$, the inductive assumption in \eqref{induction assumption} automatically satisfied for $\Psi^{(s)}$ as long as we choose a suitable constant $c$ such that $c < \frac{\delta^2}{32 + 8\delta / 3}$, $c<C_1$ and $c< \frac{\delta_3^2}{32 + 8\delta_3 / 3}$. For example, we can choose $c = \frac{1}{2} \min\{\frac{\delta^2}{32 + 8\delta / 3}, C_1, \frac{\delta_3^2}{32 + 8\delta_3 / 3}\}$. The proof is complete.
\end{proof}

Now from all the quantities above we are ready to prove Theorem~\ref{dist}.

\medskip

\begin{proof}[Proof of Theorem \protect\ref{dist}]
For notation simplicity, we use $\Psi$ to denote $\Psi^{(s-1)}$ in this proof. The estimate of $p$ in the $s$-th iteration is
\[  p^{(s)} = \frac{\sum_{i<j} \sum_a \Psi_{i a} \Psi_{j a} A_{i j}}{ \sum_{i<j} \sum_a \Psi_{i a} \Psi_{j a} }.\]
Similarly as in equation~\eqref{estimation of numerator of ps}, the numerator of $p^{(s)}$ satisfies  
\begin{align}\label{analysis1 of p in dist}
\big| \sum_{i<j} \sum_a \Psi_{i a} \Psi_{j a} A_{i j} -   \sum_{i<j} \sum_a Z_{i a} Z_{j a} A_{i j}\big| \leq 3n\exp(-c d)C_0 d + 6n\exp(-C_1 d)
\end{align}
with high probability. As in equation~\eqref{estimation of denominator of ps}, the denominator of $p^{(s)}$ satisfies  
\begin{align}\label{analysis2 of p in dist}
\big| \sum_{i<j} \sum_a \Psi_{i a} \Psi_{j a} - \sum_{i<j} \sum_a Z_{i a} Z_{j a} \big|  \leq 3n^2 \exp(-cd).
\end{align}
with high probability. By Berry-Esseen theorem the asymptotic normality holds for $ \sum_{i<j} \sum_a Z_{i a} Z_{j a} A_{i j}$  since $n^2 \rho_n \to \infty$. To be specific, we have the following asymptotic result: 
\begin{equation}\label{analysis3 of p in dist}
\sqrt{\frac{n^2}{2Kp}}\left( \frac{\sum_{i<j} \sum_a Z_{i a} Z_{j a} A_{i j}}{ \sum_{i<j} \sum_a Z_{i a} Z_{j a} } - p\right) \to N(0,1).
\end{equation}
To obtain the asymptotic normality of $p^{(s)}$, from \eqref{analysis1 of p in dist}, \eqref{analysis2 of p in dist} and \eqref{analysis3 of p in dist}, by Slutsky's theorem we need \[\sqrt{\frac{n^2}{p}} \left( \frac{nd \exp(-cd)}{n^2} \right) = o(1),\] which holds when $d\ge C \log n$ for some large constant $C$ only depending on $\varepsilon$ and $K$. The analysis for $q^{(s)}$ is analogical, and by the multidimensional version of Slutsky's theorem, we showed that $p^{(s)}$ and $q^{(s)}$ are jointly asymptotically normally distributed: 
\[ \begin{pmatrix}n (p^{(s)} - p)/\sqrt{p}\\n(q^{(s)} - q)/ \sqrt{q} \end{pmatrix} \to N\left(\begin{pmatrix} 0\\0\end{pmatrix},\begin{pmatrix} 2K & 0\\0 & 2K/(K-1)\end{pmatrix}\right).
\]
The proof is complete.
\end{proof}

\section{Proofs of Results for DCSBM in Section~\ref{sec:theory dcsbm}}
\label{all_proof2}

The proof of the results under DCSBM is similar to that under SBM, except that we need to track the updates of the degree parameters $\theta=(\theta_1,...,\theta_n)^\top$ carefully. Instead of only highlighting the main differences, for clarity, we present the complete proof for the DCSBM case here to make it easier for the reader to follow. 

First, we introduce some notations used in the main proof. Denote 
$$D_i = \sum_{j=1}^n A_{i j }, \quad d_i = \mathbb{E} [ D_i], \quad d_{\min} = \min_{1\le i\le n} d_i, \quad d_{\max} = \max_{1\le i\le n} d_i.$$ 
Recall that $d = np / K + n(K-1)q/K$ is the expected average degree of the network. In the sparse regime, we must carefully control the irregularity of sparse networks. To that end, denote by $A^{\prime}$ the adjacency matrix of the network obtained from the observed network by removing some edges (in an arbitrary way) of nodes with degrees greater than $C_0 d$ so that all node degrees of the new network are at most $C_0 d$, where $C_0>0$ is a large fixed absolute constant. In this paper we choose $C_0$ such that $C_0 d \geq \max_i 20n \theta_i^2 p$. Note that $A'$ is only used for the proofs, and our algorithm does not need to specify it.

Now, we state and prove all the lemmas needed in the main proof. In the following lemmas, we assume that the assumptions of Proposition~\ref{prop:parameter estimation in DCSBM} and Theorem~\ref{unknown pa (dcsbm)} are satisfied.

\begin{lemma}[Concentration inequality in DCSBM]
\label{lemma: Concentration inequality (dcsbm)}
Define $S = \{ i \in [n], \sum_j A_{ij} \geq C_0 d \}$ and $H_i = \sum_j |A_{ij} - \mathbb{E} A_{ij}| \mathbbm{1}\{ i \in S\}$. Then with probability at least $1 - \exp(-C_0 d /4)$, we have
\begin{equation}
   \sum_i H_i \leq n C_0 d \exp(-C_0 d / 4).
\end{equation}
\end{lemma}

\begin{proof}
This lemma is analogous to the Lemma C.5 in \cite{zhang2020theoretical}. Denote $\max_i \theta_i^2 p$ by $p_{\text{max}}$. Note that $ \mathbb{E} \left[ \sum_j |A_{ij} - \mathbb{E} A_{ij}| \right] \leq 2n p_{\text{max}} \leq C_0 d /10.$  For any $s \leq C_0 d$, we have 
\begin{align*}
\mathbb{P}(H_i > s) & \leq \mathbb{P} \left( \sum_j |A_{ij} - \mathbb{E} A_{ij}| -   \mathbb{E} \left[ \sum_j |A_{ij} - \mathbb{E} A_{ij}| \right] > s - 2np_{\text{max}} \right) \\
&  \leq \exp\left(  -\frac{ (s-2np_{\text{max}})^2 / 2 }{np_{\text{max}}+ (s-2np_{\text{max}})/3}  \right)\\
& \leq \exp(-s/2)
\end{align*}
by implementing Bernstein inequality. Applying Bernstein inequality again we have
\begin{align*}
\mathbb{P}(H_i>0) & = \mathbb{P} \left( \sum_i A_{ij} \geq C_0 d  \right) \\ & \leq \mathbb{P} \left( \sum_i A_{ij} - \mathbb{E}\left[ \sum_i A_{ij}\right] \geq 9C_0 d /10  \right) \\ & \leq \exp\left(  -\frac{(9C_0 d /10)^2/2}{ C_0 d /20 +(9C_0 d /10) / 3 }\right) \\ & \leq \exp( - 21/40 C_0 d).
\end{align*}
Thus, we are able to bound $\mathbb{E}[H_i]$ by 
\begin{align*}
\mathbb{E}[H_i] & \leq \int_{0}^{C_0 d} \mathbb{P}(H_i > 0) ds + \int_{C_0 d}^{\infty}  \mathbb{P}(H_i > s) ds \\ & \leq C_0 d \exp( - 21/40 C_0 d) + \int_{C_0 d}^{\infty}  \exp(-s/2) ds \\ &  \leq C_0 d \exp( - C_0 d / 2).
\end{align*}
By Markov inequality, we have
\begin{align*}
\mathbb{P}\left( \sum_i H_i \geq n C_0 d \exp(-C_0 d / 4) \right) \leq \frac{ n \mathbb{E}[H_1] }{n C_0 d \exp(-C_0 d / 4)} \leq \exp(-C_0 d /4).
\end{align*}
The proof is complete.
\end{proof}

 \begin{lemma}[Number of removed edges in DCSBM] 
\label{lemma:quantity of removed edges (dcsbm)}
Assume that $ d \geq \kappa $ for a large constant $\kappa>0$. Then with high probability $1 - n^{-r}$ for some absolute constant $r$, we have  
\[\sum_{i ,j} (A_{i j} - A_{i j}^{\prime}) \leq 2n \exp(- C_1 d),\]
where $C_1$ is an absolute constant.
\end{lemma}

\begin{proof}
Case 1: $ \kappa \leq d \leq \eta \log n$ for some constant $\eta \in (0,1)$. Let $d_{\text{max}} = \max_i d_i$. It follows from Proposition 1.12 of \cite{benaych2019largest} that with high probability, \[\sum_{i,j} (A_{i j} - A^{\prime}_{i j}) \leq 2 \sum_{k = C _0 d+1}^{\infty} k n  \exp(-f(k)) ,\]
where \[f(x) = x\log\left(\frac{x}{d_{\text{max}}}\right) - (x-d_{\text{max}}) - \log\sqrt{2\pi x}.
\]
Here we can use the results from \cite{benaych2019largest} since the sum of all elements in the adjacency matrix considered in our case is stochastically bounded by that in \cite{benaych2019largest}. Notice that for $k > C_0 d$, since $C_0 d \geq 20 d_{\text{max}}$ we have
$$f(k) - \log k \geq k ( \log 20 - 1) - \log \sqrt{2 \pi k} - \log k \geq k( \log 20 - 1)/2.$$ 
This implies the following inequality
\[ \sum_{i,j} (A_{i j} - A^{\prime}_{i j}) \leq 2n  \sum_{k = C_0 d+1}^{\infty} \exp\left(-\frac{\log 20 -1}{2} k\right) \leq 2 n \exp(- C_1 d)\]
occurs with high probability for some constant $C_1$.

Case 2: $d >  \eta \log n$. Define $S = \{ i \in [n], \sum_j A_{ij} \geq C_0 d \}$, $H_i = \sum_j |A_{ij} - \mathbb{E} A_{ij}| \mathbbm{1}\{ i \in S\}$ and $M_i = \sum_j |A_{ij}^{\prime} - \mathbb{E} A_{ij} | \mathbbm{1}\{ i \in S\}$. From Lemma~\ref{lemma: Concentration inequality (dcsbm)} we know that with probability at least $1 - \exp(-C_0 d/4)$, it holds that 
\begin{equation}\label{first bound in removed edges (dcsbm)}
   \sum_i H_i \leq n C_0 d \exp(-C_0 d / 4).
\end{equation}
In addition, from Lemma~\ref{lemma: Concentration inequality (dcsbm)} we know \[\mathbb{P}(i \in S) = \mathbb{P}(H_i>0) \leq \exp(-21C_0 d / 40).\]
Therefore, $\mathbb{E}[M_i] \leq (20np_{\text{max}} + np_{\text{max}}) \exp(-21C_0 d / 40 ) \leq 21C_0d/20 \exp(-21C_0 d / 40 ) $. By Markov inequality, we have 
\begin{equation}\label{second bound in removed edges (dcsbm)}
  \mathbb{P}\left( \sum_i M_i > 21n C_0 d/20 \exp(-11 C_0 d / 40)\right) \leq \exp(-C_0 d /4).
\end{equation}
It follows from \eqref{first bound in removed edges (dcsbm)} and \eqref{second bound in removed edges (dcsbm)} that 
\begin{align*}
    \sum_{i ,j} (A_{i j} - A_{i j}^{\prime}) & \leq \sum_i H_i +  \sum_i M_i \\
    & \leq n C_0 d \exp(-C_0 d / 4) +  21n C_0 d/20 \exp(-11 C_0 d / 40) \\
    & \leq 2 n \exp(- C_1 d)
\end{align*}
occurs with high probability for some constant $C_1$.
\end{proof}



\begin{lemma}[Propositions of the parameter $\lambda$ in DCSBM]
\label{lemma:Propositions of the parameter lambda (dcsbm)}
Under the assumption that $p>q>0$ and $p \asymp q \asymp p - q \asymp \rho_n$, we have the following bounds for $\lambda$:
\begin{align*}
    q< \lambda < p, \quad
    \lambda - q \geq C_3 \rho_n, \quad
    \frac{p+q}{2} - \lambda \geq C_4 \rho_n,
\end{align*}
where $C_3$ and $C_4$ are absolute constants.
\end{lemma}

\begin{proof}
This lemma is the direct result from Proposition 1 and 2 in \cite{sarkar2021random} by using the inequality: $0 < x - \log(1+x) \leq x^2$ for $x>0$.
\end{proof}

The next lemma provides crude bounds for the parameters of the threshold BCAVI after the first iteration in DCSBM.

\begin{lemma}[First step: Parameter estimation in DCSBM]
\label{lemma:parameter estimation (dcsbm)}
There exist constants $C,c_{1},c_{2}$ only depending on $\varepsilon$ and $K$ such that if $n \rho_n > C$ then with high probability $1-n^{-r}$ for some constant $r>0$,  
\[ t^{(1)} \ge c_{1 },\quad \lambda^{(1)} \le c_{2 } \rho_n.\]
\end{lemma}

\begin{proof}
By the Assumption~\ref{ass:theta}, it follows that
\begin{align*}
d_i &=  \sum_{j:j \neq i, z_j = z_i} \theta_i \theta_j p + \sum_{j: z_j \neq z_i}  \theta_i \theta_j q = \theta_i \frac{np + n(K-1) q}{K} +  \mathcal{E}
= \theta_i d + \mathcal{E},
\end{align*}
where $|\mathcal{E}| < r_1 n^{2/3} \rho_n$ holds for all $i$ with probability at least $1 - n^{-r_2}$ for some absolute constant $r_1>0$ and $r_2>0$.

When $s=1$, it follows from \eqref{update of pq in dc in proof} that the first update of $p$ is:
\[
    p^{(1)} = \frac{\sum_{i<j} \sum_a \Psi_{i a}^{(0)} \Psi_{j a}^{(0)} A_{i j}}{ \sum_{i<j} \sum_a \Psi_{i a}^{(0)} \Psi_{j a}^{(0)} \theta_i^{(0)} \theta_j^{(0)} }, 
\] where $\Psi^{(0)} = Z^{(0)}$ is the random initialization with error rate $\varepsilon$. For degree parameters, as we state in Section~\ref{sec:initialization} we initialize it by \[
\theta^{(0)}_i = \frac{D_i n}{\sum_i D_i}, i = 1,2,...,n.
\]
We will analyze the numerator and denominator of $p^{(1)}$ separately and combine them to have an estimation of $p^{(1)}$.

First we analyze the denominator of $p^{(1)}$. We will bound $\sum_i D_i$ and $ \sum_i \Psi_{ia}^{(0)} D_i$ separately. Notice that 

\begin{align}
\label{E of sum Di}
\mathbb{E} \big[ \sum_i D_i \big] = \sum_i d_i  = nd + \mathcal{E},
\end{align}
where  $|\mathcal{E}| < r_1 n^{5/3} \rho_n$ with probability at least $1 - n^{-r_2}$ for some absolute constant $r_1>0$ and $r_2>0$. The variance of $\sum_i D_i$ is bounded by:
\begin{align}
\label{Var of sum Di}
 \text{Var}\big[  \sum_i D_i \big] \leq r_1 n^2 \rho_n,
\end{align}
where $r_1>0$ is an absolute constant. By choosing $k = n^{2/3} \rho_n^{1/2}$ in the Chebyshev’s inequality, it follows from \eqref{E of sum Di} and \eqref{Var of sum Di} that 
\begin{align}
\label{bound of sum Di}
\sum_i D_i = nd + \mathcal{E},
\end{align}
where  $|\mathcal{E}| < r_1 n^{5/3} \rho_n$ with probability at least $1 - n^{-r_2}$ for some absolute constant $r_1>0$ and $r_2>0$. Then we compute the mean and variance of $ \sum_i \Psi_{ia}^{(0)} D_i$. The mean is given by:
\begin{align}
\label{E of sum Di Psi}
\mathbb{E} \big[  \sum_i \Psi_{ia}^{(0)} D_i \big] =  \sum_i \mathbb{E} \big[ \Psi_{ia}^{(0)} \big] d_i = \frac{n d}{K} + \mathcal{E},
\end{align}
where  $|\mathcal{E}| < r_1 n^{5/3} \rho_n $ with probability at least $1 - n^{-r_2}$ for some absolute constant $r_1>0$ and $r_2>0$. The variance is bounded by:
\begin{align}
\label{Var of sum Di Psi}
\nonumber \text{Var}\big[   \sum_i \Psi_{ia}^{(0)} D_i  \big] &= \text{Var} \left[  \mathbb{E} \big[   \sum_i \Psi_{ia}^{(0)} D_i  | \Psi^{(0)}\big]  \right] +  \mathbb{E} \left[  \text{Var} \big[   \sum_i \Psi_{ia}^{(0)} D_i  | \Psi^{(0)}\big]  \right] \\
& \leq r_1 n d^2,
\end{align}
where $r_1>0$ is an absolute constant. By choosing $k = n^{1/6}$ in the Chebyshev’s inequality, it follows from \eqref{E of sum Di Psi} and \eqref{Var of sum Di Psi} that 
\begin{align}
\label{bound of sum Di Psi}
   \sum_i \Psi_{ia}^{(0)} D_i = \frac{nd}{K} + \mathcal{E},
\end{align}
where  $|\mathcal{E}| < r_1 n^{5/3} \rho_n $ with probability at least $1 - n^{-r_2}$ for some absolute constant $r_1>0$ and $r_2>0$. 

Now we are ready to bound the denominator of $p^{(1)}$. It follows from \eqref{bound of sum Di} and \eqref{bound of sum Di Psi} that the denominator is:
\begin{align}\label{eq:denominator of p1 (dcsbm)}
   \sum_{i<j} \sum_a \Psi_{i a}^{(0)} \Psi_{j a}^{(0)} \theta_i^{(0)} \theta_j^{(0)}  &= \frac{1}{2}\sum_a  \left( \sum_{i } \Psi_{ia}^{(0)} \theta_i^{(0)} \right)^2 - \frac{1}{2}\sum_a \sum_{i} ( \Psi_{ia}^{(0)} \theta_i^{(0)} )^2 \nonumber \\
   &= \frac{ K (nd/K)^2 n^2}{2 n^2 d^2} +  \mathcal{E}  \nonumber \\
   &= \frac{n^2}{2K} + \mathcal{E},
\end{align}
where $|\mathcal{E}|< r_1 n ^{5/3}$ with probability at least $1 - n^{-r_2}$ for some absolute constant $r_1>0$ and $r_2>0$. Notice here we can put the term $\sum_a \sum_{i} ( \Psi_{ia}^{(0)} \theta_i^{(0)} )^2$ within $\mathcal{E}$ because it is obvious that $D_i \leq n^{1/3} d_i$ holds for all $i$ with high probability.

Second we analyze the numerator of $p^{(1)}$. By replacing $A_{ij}$ with $P_{ij}+(A_{ij}-P_{ij})$, we decompose the numerator of $p^{(1)}$ as the sum of signal and noise terms:
\begin{eqnarray*}
\text{signal of $p^{(1)}$} &:=&\sum_{i<j} \sum_a \Psi_{i a}^{(0)} \Psi_{j a}^{(0)} P_{i j},\\
\text{noise of $p^{(1)}$} &:=&\sum_{i<j} \sum_a \Psi_{i a}^{(0)} \Psi_{j a}^{(0)} (A_{i j} - P_{i j}).
\end{eqnarray*}
Since $P_{ij} = \theta_i \theta_j p$ when $z_i=z_j$ and $P_{ij} = \theta_i \theta_j q$ when $z_i \neq z_j$, the signal can be written as \[\sum_{i<j} \sum_a \Psi_{i a}^{(0)} \Psi_{j a}^{(0)} P_{i j} = \sum_a \sum_{i<j : z_i = z_j} \Psi_{i a}^{(0)} \Psi_{j a}^{(0)} \theta_i \theta_j p + \sum_a \sum_{i<j : z_i \neq z_j} \Psi_{i a}^{(0)} \Psi_{j a}^{(0)} \theta_i \theta_j q.\]
By the assumption of $\Psi^{(0)}$ and $\theta$, we have 
\begin{align*}
    \mathbb{E} \Big[\sum_{i:z_i=a} \Psi^{(0)}_{ia} \theta_i \Big] = \frac{n}{K}(1 - \epsilon) +  \mathcal{E}, \quad \mathbb{E} \Big[\sum_{i:z_i=b, z_i \neq a} \Psi^{(0)}_{ia} \theta_i \Big] = \frac{n \epsilon}{K(K-1)} +  \mathcal{E},
\end{align*}
where $|\mathcal{E}|<  r_1 n ^{2/3}$ with probability at least $1 - n^{-r_2}$ for some absolute constant $r_1>0$ and $r_2>0$. Therefore we can apply the Chebyshev’s inequalities for each $\sum_{i:z_i=a} \Psi^{(0)}_{i a} \theta_i$ and $\sum_{i:z_i=b, z_i \neq a} \Psi^{(0)}_{ia} \theta_i$ as in the analysis of the denominator of $p^{(1)}$. Then by taking the union bound we have 
\begin{align}\label{eq:signal of p1 (dcsbm)}
    \text{signal of $p^{(1)}$} =& \sum_{i<j} \sum_a \Psi_{i a}^{(0)} \Psi_{j a}^{(0)} P_{i j} = \sum_a \sum_{i<j : z_i = z_j} \Psi_{i a}^{(0)} \Psi_{j a}^{(0)} \theta_i \theta_j p + \sum_a \sum_{i<j : z_i \neq z_j} \Psi_{i a}^{(0)} \Psi_{j a}^{(0)} \theta_i \theta_j q  \nonumber \\ =& K\left[ \frac{1}{2}[ \frac{n}{K}(1 - \varepsilon)]^2 + \frac{1}{2}[\frac{n \varepsilon}{K(K-1)}]^2(K-1) \right]p  \nonumber \\ &+ K\left[ \frac{n}{K}(1 - \varepsilon) \frac{n \varepsilon}{K(K-1)} (K-1) + [\frac{n \varepsilon}{K(K-1)}]^2 \binom{K-1}{2} \right]q + \mathcal{E},
\end{align}
where $|\mathcal{E}|< r_1 n ^{5/3} \rho_n$ with probability at least $1 - n^{-r_2}$ for some absolute constant $r_1>0$ and $r_2>0$.

We can analyze the noise of $p^{(1)}$ conditioning on $\Psi^{(0)}$. Since
\begin{equation*}
      \mathbb{E}\left[ \sum_{i<j} \sum_a \Psi_{i a}^{(0)} \Psi_{j a}^{(0)} (A_{i j} - P_{i j}) | \Psi^{(0)} \right] = 0,  
\end{equation*}
we have
\begin{align*}
    \text{Var}\left[\sum_{i<j} \sum_a \Psi_{i a}^{(0)} \Psi_{j a}^{(0)} (A_{i j} - P_{i j}) \right] &= \mathbb{E}\left[ \text{Var} \left[ \sum_{i<j} \sum_a \Psi_{i a}^{(0)} \Psi_{j a}^{(0)} (A_{i j} - P_{i j}) | \Psi^{(0)} \right] \right] \\ &\leq r_1 n^2 K \rho_n,
\end{align*}
where $r_1$ is an absolute constant. By Chebyshev’s inequality, for any real number $k>0$, we have 
\begin{align} \label{eq:noise of p1 (dcsbm)}
 \mathbb{P} \Big( 
\Big| \sum_{i<j} \sum_a \Psi_{i a}^{(0)} \Psi_{j a}^{(0)} (A_{i j} - P_{i j}) \Big| < kn \sqrt{r_1 K \rho_n} \Big) \geq 1 - \frac{1}{k^2}.
\end{align}
By choosing $k = n^{3/4}\sqrt{\rho_n}$, the noise term is bounded by $ \sqrt{r_1 K} n^{7/4}\rho_n$ with probability at least $1 - \frac{1}{n^{3/2} \rho_n}$ for some absolute constant $r_1$.

It follows from \eqref{eq:denominator of p1 (dcsbm)}, \eqref{eq:signal of p1 (dcsbm)} and \eqref{eq:noise of p1 (dcsbm)} that 
\begin{equation}\label{eq: estimation of p1 (dcsbm)}
 p^{(1)} = \left[  (1 - \varepsilon)^2 + \frac{\varepsilon^2}{K-1}\right]p+  \left[2(1 - \varepsilon)\varepsilon + \frac{(K-2) \varepsilon^2 }{K-1}\right]q+ \mathcal{E}, 
\end{equation}
where $|\mathcal{E}| < r_1 n^{-1/4}p_n$ with probability at least $1 - n^{-r_2}$ for some absolute constant $r_1$ and $r_2$.

A similar analysis gives
\begin{eqnarray}\label{eq: estimation of q1 (dcsbm)}
    q^{(1)} = \left[  K-2 + (1 - \varepsilon)^2 + \frac{\varepsilon^2}{K-1}\right] \frac{q}{K-1} + \left[2(1 - \varepsilon)\varepsilon + \frac{(K-2) \varepsilon^2 }{K-1}\right]\frac{p}{K-1}+   \mathcal{E}, 
\end{eqnarray}
where $|\mathcal{E|} < r_1 n^{-1/4}p_n$ with probability at least $1 - n^{-r_2}$ for some absolute constant $r_1$ and $r_2$.

Notice here $p^{(1)}$ is always larger than $q^{(1)}$ with high probability. To verify this, first notice the sum of coefficients in front of $p$ and $q$ is $1$ in both \eqref{eq: estimation of p1 (dcsbm)} and \eqref{eq: estimation of q1 (dcsbm)}. Therefore, we only need show that \[
\left[  (1 - \varepsilon)^2 + \frac{\varepsilon^2}{K-1}\right] > \left[2(1 - \varepsilon)\varepsilon + \frac{(K-2) \varepsilon^2 }{K-1}\right] \frac{1}{K-1}
,\]
which is equivalent to \[
(K-1)(1 - \varepsilon)^2 + \frac{\varepsilon^2}{K-1} > 2\varepsilon(1- \varepsilon).\]
The last inequality holds by fundamental inequalities. The similar calculation also yields that 
\begin{equation}
\label{eq: estimation of p1 - q1 (dcsbm)}
p^{(1)} - q^{(1)} =  \left( (K-1)(1 - \varepsilon)^2 + \frac{\varepsilon^2}{K-1} - 2\varepsilon(1- \varepsilon) \right)\frac{p-q}{K-1} + \mathcal{E}, 
\end{equation}
where $|\mathcal{E|} < r_1 n^{-1/4}p_n$ with probability at least $1 - n^{-r_2}$ for some absolute constant $r_1$ and $r_2$.

Since 
\[t^{(1)} = \frac{1}{2} \log \frac{p^{(1)}}{ q^{(1)} }, \quad \lambda^{(1)} = \frac{1}{2t^{(1)}} (p^{(1)} - q^{(1)}),\]
it follows directly from \eqref{eq: estimation of p1 (dcsbm)}, \eqref{eq: estimation of q1 (dcsbm)} and \eqref{eq: estimation of p1 - q1 (dcsbm)} that 
\begin{equation}\label{eq: estimation of t1 and lambda1 (dcsbm)}
    t^{(1)} \ge c_{1}, \quad \lambda^{(1)} \le c_{2}\rho_n
\end{equation} 
for some constants $c_1$, $c_2$ only depending on $\varepsilon$ and $K$.
\end{proof}

The next lemma provides the accuracy of the label estimates after the first iteration in DCSBM.

\begin{lemma}[First step: Label estimation in DCSBM]\label{lem:label first step (dcsbm)}
There exist constants $C, c>0$ only depending on $\varepsilon$ and $K$ such that if $n\rho_n \ge C$ then with high probability $1-n^{-r}$ for some constant $r>0$,
$$||\Psi^{(1)}- Z||_1 \leq n\exp(-c d).$$
\end{lemma}
\begin{proof}
To deal with the estimation of $\Psi$ in the first iteration, we will first quantify the probability of misclassification for each node and then derive the upper bound of $||\Psi^{(1)} - Z||$.

For the node $i$, from \eqref{update of Psi in dc in proof} we know $\Psi_{i}^{(1)} = Z_i$ after threshold if any only if the following inequality 
 \[ \sum_{j:j\neq i} \Psi_{j z_i}^{(0)} (A_{ij} -   \theta_i^{(0)} \theta_j^{(0)} \lambda^{(1)}) > \sum_{j:j\neq i} \Psi_{ja}^{(0)} (A_{ij} -   \theta_i^{(0)} \theta_j^{(0)} \lambda^{(1)})
\]
holds for any $a \in [K]$ such that $a \neq z_i$. This is equivalent to \[ \sum_{j:j\neq i} (\Psi_{j z_i}^{(0)} - \Psi_{ja}^{(0)}) (A_{ij} -   \theta_i^{(0)} \theta_j^{(0)} \lambda^{(1)}) > 0
\] holds for any $a \in [K]$ such that $a \neq z_i$. We can decompose the left-hand side of the above inequality by the sum of signal and noise terms. We will first quantify the signal term and then show that most of the noise terms can be upper bounded by the signal term. Denote the signal and noise terms in the first iteration by $SIG_{ia}^{(0)}$, $ r_{1ia}^{(0)}$ and $r_{2ia}^{(0)}$ defined as:

\begin{align*}
   \text{SIG}_{ia}^{(0)} :=& \sum_{j:j\neq i} (\Psi_{j z_i}^{(0)} - \Psi_{ja}^{(0)}) (E[A_{ij}] - \theta_i \theta_j \lambda^{(1)}), \\
    r_{1ia}^{(0)} :=& \sum_{j:j\neq i} (\Psi_{j z_i}^{(0)} - \Psi_{ja}^{(0)}) ( A_{ij} - E[A_{ij}]) := \sum_{j:j \neq i} Y_{ija}, \\
    r_{2ia}^{(0)} :=& \sum_{j:j\neq i} (\Psi_{j z_i}^{(0)} - \Psi_{ja}^{(0)}) (\theta_i \theta_j \lambda^{(1)} - \theta_i^{(0)} \theta_j^{(0)} \lambda^{(1)}).
\end{align*}

Similarly as in the derivation of \eqref{eq:signal of p1 (dcsbm)}, if follows from Lemma~\ref{lemma:parameter estimation (dcsbm)} that 
\begin{align*}
    \text{SIG}_{ia}^{(0)} &= 
    \sum_{j:j\neq i} (\Psi_{j z_i}^{(0)} - \Psi_{ja}^{(0)}) (E[A_{ij}] - \lambda^{(1)}) \\ &= \sum_{j:j\neq i,z_j = z_i} (\Psi_{j z_i}^{(0)} - \Psi_{ja}^{(0)}) \theta_i \theta_j p + \sum_{j:j\neq i, z_j \neq z_i } (\Psi_{j z_i}^{(0)} - \Psi_{ja}^{(0)}) \theta_i \theta_j q -  \sum_{j:j\neq i} (\Psi_{j z_i}^{(0)} - \Psi_{ja}^{(0)})\lambda^{(1)}\\
    & = \theta_i \frac{n}{K} [ 1- \varepsilon - \frac{\varepsilon}{ K-1 }][p-q] + \mathcal{E},
\end{align*}
where $|\mathcal{E}|<r_1 n^{2/3}\rho_n$ with probability at least $1 - n^{-r_2}$ for some absolute constant $r_1$ and $r_2$. Here the upper bound of $|\mathcal{E}|$ is a union bound for all $i \in [n]$ and $a \in [K]$ such that $a \neq z_i$.

To bound the noise term $r_{2ia}^{(0)}$, notice that 
\begin{align*}
   | r_{2ia}^{(0)} |&= \Big| \theta_i \sum_{j:j\neq i} (\Psi_{j z_i}^{(0)} - \Psi_{ja}^{(0)}) \theta_j \lambda^{(1)} - \theta_i^{(0)}  \sum_{j:j\neq i} (\Psi_{j z_i}^{(0)} - \Psi_{ja}^{(0)})  \theta_j^{(0)} \lambda^{(1)} \Big| \\
   & \leq (\theta_i + \theta^{(0)}_i ) \lambda^{(1)} r_1 n^{2/3} \leq r_1 n^{5/6} \rho_n
\end{align*}
for some constant $r_1 >0$. The last inequality holds because it is obvious that $|\theta_i - \theta^{(0)}_i| \leq n^{1/6}$ for all $i$.

Now we are ready to bound the noise terms $r_{1ia}^{(0)}$. Denote
\[ \delta = \frac{[1 - \varepsilon - \frac{\varepsilon}{K-1}](p-q)}{2 [p + (K-1)q]},
\]
and note that $\delta d_i < SIG_{ia}^{(0)}$. As we discussed, if $|r_{1ia}^{(0)}| < \delta d_i$ holds for all $a \in [K]$ and $a \neq z_i$, then $|r_{1ia}^{(0)}|< SIG_{ia}^{(0)}$, which implies $\Psi_i^{(1)} = Z_i$. Otherwise $||\Psi_i^{(1)} -Z_i  ||_1 \leq 2$.  For an event $E$, denote by $\mathbbm{1}(E)$ the indicator of $E$. Therefore, we have the fact that
\begin{equation}\label{bound of Psi1 related to noise (dcsbm)}
||\Psi^{(1)} - Z||_1 = \sum_{i} || \Psi^{(1)}_i - Z_i ||_1 \leq 2 \sum_{i} \mathbbm{1}\Big( \bigcup_{a:a \neq z_i}  \Big\{ |r_{1ia} ^{(0)} | > \delta d_i
 \Big\} \Big).
\end{equation}
First we analyze the noise terms $r_{1ia}^{(0)}$ conditioning on $\Psi^{(0)}$. 
To deal with the dependency among rows of the adjacency matrix due to symmetry, 
let $\{Y_{i ja}^*\}$ be an independent copy of $\{Y_{i ja}\}$. Then by the triangle inequality, 
 \begin{align*}
 \mathbbm{1}\Big(\bigcup_{a:a \neq z_i}  
 \Big\{ |r_{1ia} ^{(0)} | > \delta d_i
 \Big\} \Big) &= \mathbbm{1} \Big( \bigcup_{a:a \neq z_i} \Big\{ |\sum_{j: j \neq i} Y_{ija }|>\delta d_i \Big\} \Big)  \leq \mathbbm{1}\Big( \bigcup_{a:a \neq z_i} \Big\{ \Big| \sum_{j=1}^{i-1} Y_{ija}\Big|+\Big|\sum_{j=i+1}^{n} Y_{ija}\Big|> \delta d_i \Big\} \Big) \\ 
 & \leq \mathbbm{1}\Big( \bigcup_{a:a \neq z_i} \Big\{ \Big| \sum_{j=1}^{i-1} Y_{ija}\Big|+\Big|\sum_{j=i+1}^{n} Y_{ija}\Big| + \Big| \sum_{j=1}^{i-1} Y_{ija}^*\Big| +\Big|\sum_{j=i+1}^{n} Y_{ija}^*\Big| > \delta d_i \Big\} \Big) \\ 
 & \leq \mathbbm{1}\Big( \bigcup_{a:a \neq z_i} \Big\{ \Big| \sum_{j=1}^{i-1} Y_{ija}\Big|+\Big|\sum_{j=i+1}^{n} Y_{ija}^*\Big| > \frac{\delta d_i}{2} \Big\} \Big) \\ & \quad + \mathbbm{1}\Big( \bigcup_{a:a \neq z_i} \Big\{ \Big| \sum_{j=1}^{i-1} Y_{ija}^*\Big|+\Big|\sum_{j=i+1}^{n} Y_{ija}\Big| > \frac{\delta d_i}{2} \Big\} \Big).
 \end{align*}
 Applying Bernstein's inequality for $ \sum_{j: j\neq i} Y_{ija} $ with $\mathbb{E}[Y_{ija}] = 0$, $|Y_{ija}|<1$ and $\sum_{j:j \neq i}\mathbb{E}[Y_{ija}^2] \leq d_i$, we get
 \begin{align*}
 \mathbb{P}\Big(\Big|\sum_{j=1}^{i-1} Y_{ija}\Big|+\Big|\sum_{j=i+1}^{n} Y_{ija}^*\Big| > \frac{\delta d_i}{2}\Big) 
  &\leq \mathbb{P}\Big(\Big|\sum_{j=1}^{i-1} Y_{ija}\Big|>\frac{\delta d_i}{4}\Big) + \mathbb{P}\Big(\Big|\sum_{j=i+1}^{n} Y_{ija}^*\Big|>\frac{\delta d_i}{4}\Big) \\
  &\leq 4\exp\left(\frac{-(\delta d_i/4)^2/2}{d_i+ (\delta d_i /4)/3}\right) = 4\exp\left(\frac{-\delta^2d_i}{32+8\delta/3} \right).
 \end{align*} 
Then by taking the union bound for all $a \in [K]$ such that $a \neq z_i$ we get 
 \[\mathbb{P}\Big( \bigcup_{a:a \neq z_i} \Big\{ \Big| \sum_{j=1}^{i-1} Y_{ija}\Big|+\Big|\sum_{j=i+1}^{n} Y_{ija}^*\Big| > \frac{\delta d_i}{2} \Big\} \Big) \leq 4(K-1)\exp \left(\frac{-\delta^2d_i}{32+8\delta/3} \right) .\]
Since $(\sum_{j=1}^{i-1}Y_{i ja},\sum_{j=i+1}^{n} Y_{i ja}^*)$, $1\le i\le n$, are independent conditioning on $\Psi^{(0)}$, by Bernstein's inequality, 
 \begin{align*}
  \sum_{i=1}^{n} \mathbbm{1}\Big(  \bigcup_{a:a \neq z_i} \Big\{ \Big| \sum_{j=1}^{i-1} Y_{ija}\Big|+\Big|\sum_{j=i+1}^{n} Y_{ija}^*\Big| > \frac{\delta d_i}{2} \Big\} \Big) \leq 8(K-1) \sum_{i} \exp\left(\frac{-\delta^2d_i}{32+8\delta/3}\right)
 \end{align*}
occurs with probability at least $1 - n^{-r_1}$ for some absolute constant $r_1$. The same bound holds for $\sum_{i=1}^{n} \mathbbm{1}\Big(  \bigcup_{a:a \neq z_i} \Big\{ \Big| \sum_{j=1}^{i-1} Y_{ija}^*\Big|+\Big|\sum_{j=i+1}^{n} Y_{ija} \Big| > \frac{\delta d_i}{2} \Big\} \Big)$. Therefore, the following event 
\[\mathcal{A}_1 = \left\{\sum_{i=1}^n \mathbbm{1}\Big( \bigcup_{a:a \neq z_i}  \Big\{ |r_{1ia} ^{(0)} | > \delta d_i
 \Big\} \Big) \leq 16(K-1) \sum_i \exp\left(\frac{-\delta^2d_i}{32+8\delta/3}\right) \right\} \]
occurs with probability at least $1 - n^{-r_1}$ for some absolute constant $r_1$ conditioning on $\Psi^{(0)}$. By the law of total probability the event $\mathcal{A}_1$ occurs with probability at least $1 - n^{-r_1}$ for some absolute constant $r_1$. Clearly, it follows from \eqref{bound of Psi1 related to noise (dcsbm)} that $\mathcal{A}_1$ implies \[ ||\Psi^{(1)} - Z||_1 \leq  n \exp(-cd)
\]
with high probability for some constant $ c>0$ only depending on $\varepsilon$ and $K$.
\end{proof}

\begin{lemma}[First step: Degree parameter estimation]\label{lem:degree first step}
There exist constants $C$ only depending on $\varepsilon$ and $K$ such that if $n\rho_n \ge C$ then with high probability $1-n^{-r}$ for some constant $r>0$, we have that the following inequalities holds for all $i$:
$$|\theta_i^{(1)} - \theta_i^{(0)}| \leq \frac{1}{4}\theta_i^{(0)}.$$
\end{lemma}
\begin{proof}

From \eqref{update of theta in proof} the first update of $\theta$ is: 
\[
\frac{\sum_j A_{i j}}{ \theta_i^{(1)}} = \sum_{j:j \neq i } \sum_{a \neq b} \Psi_{i a}^{(0)} \Psi_{j b}^{(0)} \theta_j^{(0)} q^{(1)} + \sum_{j:j \neq i } \sum_{a} \Psi_{i a}^{(0)} \Psi_{j a}^{(0)} \theta_j^{(0)} p^{(1)}, \quad i \in [n].\]

We will analyze $\theta^{(1)}$ in each estimated community. First we fix $a$ and analyze $\theta_i^{(1)}$ when $\Psi_{ia}^{(0)} = 1$. It follows from \eqref{eq: estimation of p1 (dcsbm)} and \eqref{eq: estimation of q1 (dcsbm)} that the above update equation can be simplified to:
\begin{align}
\label{bound first update of theta}
\nonumber \frac{\sum_j A_{i j}}{ \theta_i^{(1)}}  &= \sum_{j:j \neq i } \sum_{b: b\neq a} \Psi_{j b}^{(0)} \theta_j^{(0)} q^{(1)} + \sum_{j:j \neq i }  \Psi_{j a}^{(0)} \theta_j^{(0)} p^{(1)} \\
&= \frac{n(K-1)}{K}q^{(1)} + \frac{np^{(1)}}{K} + \mathcal{E} = d + \mathcal{E},
\end{align}
where $|\mathcal{E}|<r_1 n^{3/4}\rho_n$ with probability at least $1 - n^{-r_2}$ for some absolute constant $r_1$ and $r_2$. Therefore, it follows from \eqref{bound of sum Di} and \eqref{bound first update of theta} that 
\begin{equation}
\label{theta_1 from theta_0}
 \theta_i^{(1)} = \theta_i^{(0)} (1 + \mathcal{E} ),
 \end{equation}
where $|\mathcal{E}|<r_1 n^{ - 1/4}$ with probability at least $1 - n^{-r_2}$ for some absolute constant $r_1$ and $r_2$. Since we further update $\theta_i^{(1)}$ by \eqref{threshold theta in dcsbm in proof}, we only need show $\sum_i \Psi_{ia}^{(1)} \theta_i^{(1)}$ is close to $n/K$. For this purpose we can just show $\sum_i \Psi_{i a}^{(1)} \theta_i^{(0)}$ is close to $n/K$ according to \eqref{theta_1 from theta_0}. To evaluate $\sum_i \Psi_{i a}^{(1)} \theta_i^{(0)}$, first we show that $\sum_{i j} \Psi_{ia}^{(1)} A_{i j}$ is close to $nd/K$.

\begin{align*}
\sum_{i j} \Psi_{ia}^{(1)} A_{i j} &= \sum_{i j} (\Psi_{ia}^{(1)} - Z_{ia}) A_{i j} +  \sum_{i j}  Z_{ia} A_{i j} \\
&= \sum_{i j} (\Psi_{ia}^{(1)} - Z_{ia}) (A_{i j} - A_{ij}^{\prime}) +  \sum_{i j}(\Psi_{ia}^{(1)} - Z_{ia})  A_{ij}^{\prime} +  \sum_{i j}  Z_{ia} A_{i j}.
\end{align*}

By Lemma~\ref{lemma:quantity of removed edges (dcsbm)} the first term is bounded by $4n\exp(-C_1 d)$. By Lemma~\ref{lem:label first step (dcsbm)} the second term is bounded by $n \exp(- cd)$. Therefore, it follows that
\begin{equation}
\label{bound of numerator in theta_1 Psi_1}
\big| \sum_{i j} \Psi_{ia}^{(1)} A_{i j} - \frac{nd}{K} \big| \leq 4n\exp(-C_1 d) + n \exp(-cd) + r_1 n^{5/3}\rho_n,
 \end{equation}
for some absolute constant $r_1$. Then it follows from \eqref{bound of numerator in theta_1 Psi_1}, \eqref{bound of sum Di} and \eqref{theta_1 from theta_0} that after rescaling we have:
\[
 |\theta_i^{(1)} - \theta_i^{(0)}| \leq \frac{1}{4}\theta_i^{(0)}
\]
as long as $n\rho_n \ge C$ for some constants $C$ only depending on $\varepsilon$ and $K$.
\end{proof}

Using the above lemmas, we are now ready to prove Proposition~\ref{prop:parameter estimation in DCSBM} and Theorem~\ref{unknown pa (dcsbm)} by induction.

\medskip

\begin{proof}[Proof of Proposition~\ref{prop:parameter estimation in DCSBM} and Theorem~\ref{unknown pa (dcsbm)}]

We will use induction to prove the results for block parameter estimation, degree parameter estimation and label estimation beyond first step. First we assume that for $s>1$, we have 
\begin{equation}\label{induction assumption (dcsbm)}
||\Psi^{(s-1)} - Z||_1 =: n G_{s-1} \leq n\exp(-cd)
\end{equation}
for some constant $c$ only dependent on $\varepsilon$ and $K$. Note that the constant $c$ here does not depend on the number of iterations $s$. We will specify how to choose the constant $c$ afterwards. 

Second we assume that for $s>1$, we have 
\begin{equation}
  \label{induction assumption of theta}  
   |\theta_i^{(s-1)} - \theta_i^{(0)}| \leq \frac{1}{4}\theta_i^{(0)}
\end{equation}
holds for all $i$. Now we will analyze the block parameter.

\medskip

\noindent \textbf{Beyond first step: block parameter estimation}\\

The estimate of $p$ in the $s$-th iteration is
\[  p^{(s)} = \frac{\sum_{i<j} \sum_a \Psi_{i a}^{(s-1)} \Psi_{j a}^{(s-1)} A_{i j}}{ \sum_{i<j} \sum_a \Psi_{i a}^{(s-1)} \Psi_{j a}^{(s-1)} \theta_i^{(s-1)} \theta_j^{(s-1)} }.
\]
First let us define a ``better" estimation of $p^{(s)}$:\[
\underline{p}^{(s)} = \frac{\sum_{i<j} \sum_a Z_{i a} Z_{j a} A_{i j}}{ \sum_{i<j} \sum_a Z_{i a} Z_{j a} \theta_i \theta_j}.
\] 
The reason why we say $\underline{p}^{(s)}$ is a ``better" estimation is that we use the ground truth $(Z,\theta)$ to replace the estimated $(\Psi^{(s-1)}, \theta^{(s-1)})$ as in the expression of $p^{(s)}$. From the Bernstein inequality we have \begin{equation}\label{bound of underline ps (dcsbm)}
|\underline{p}^{(s)} - p| \leq p\exp(-r_1 d)
\end{equation}
occurs with probability at least $1 - n^{-r_2}$ for some absolute constant $r_1$ and $r_2$. To show $p^{(s)}$ is close to $\underline{p}^{(s)}$, we will show both the numerators and denominators are close.

To quantify the numerator of $p^{(s)}$, we can decompose it by
\begin{align*}
\sum_{i<j} \sum_a \Psi_{i a}^{(s-1)} \Psi_{j a}^{(s-1)} A_{i j} &= \sum_{i<j} \sum_a Z_{i a} Z_{j a} A_{i j} + \sum_{i<j} \sum_a [\Psi_{i a}^{(s-1)} \Psi_{j a}^{(s-1)} -  Z_{i a} Z_{j a}] A_{i j} \\ &= \sum_{i<j} \sum_a Z_{i a} Z_{j a} A_{i j} + \sum_{i<j} \sum_a [\Psi_{i a}^{(s-1)} \Psi_{j a}^{(s-1)} -  Z_{i a} Z_{j a}] A_{i j}^{\prime} \\ &+  \sum_{i<j} \sum_a [\Psi_{i a}^{(s-1)} \Psi_{j a}^{(s-1)} -  Z_{i a} Z_{j a}] (A_{i j} - A_{i j}^{\prime}).
\end{align*}
To bound the second term, from \eqref{induction assumption (dcsbm)} and the assumption that each row and column sum of $A^{\prime}$ is bounded by $C_0 d$, we have
\begin{align}\label{second bound in ps (dcsbm)}
    \big|\sum_{i<j} \sum_a \Psi_{i a}^{(s-1)} \Psi_{j a}^{(s-1)} A_{i j}^{\prime} - \sum_{i<j} \sum_a Z_{i a} Z_{j a} A_{i j}^{\prime} \big| = \big| \sum_{i<j} \sum_a (\Psi_{i a}^{(s-1)} - Z_{i a} )(\Psi_{j a}^{(s-1)} - Z_{j a} )A_{i j}^{\prime} \nonumber \\ + \sum_{i<j} \sum_a (\Psi_{i a}^{(s-1)} - Z_{i a} )Z_{j a} A_{i j}^{\prime} + \sum_{i<j} \sum_a Z_{i a} (\Psi_{j a}^{(s-1)} - Z_{j a} )A_{i j}^{\prime} \big| \nonumber \\ \leq 3||\Psi_{(s-1)} - Z||_1 C_0 d \leq 3n \exp(-cd) C_0 d.
\end{align}
To bound the third term, first notice it follows from Lemma~\ref{lemma:quantity of removed edges (dcsbm)} that \begin{align}\label{first bound in ps (dcsbm)}
   \sum_{i<j} \sum_a \Psi_{i a}^{(s-1)} \Psi_{j a}^{(s-1)} (A_{i j} - A_{i j}^{\prime}) \leq \sum_{i<j}  (\sum_a \Psi_{i a}^{(s-1)})( \sum_a \Psi_{j a}^{(s-1)}) (A_{ i j} - A^{\prime}_{i j}) \nonumber \\ = \sum_{i<j} (A_{ i j} - A^{\prime}_{i j}) \leq 2n\exp(-C_1 d) . 
\end{align}
Similarly as in \eqref{first bound in ps (dcsbm)} we have the same bound for $\sum_{i<j} \sum_a Z_{ia}Z_{ja} (A_{ij} - A^{\prime}_{i j})$ and it holds that 
\begin{equation}\label{third bound in ps (dcsbm)}
 \big| \sum_{i<j} \sum_a [\Psi_{i a}^{(s-1)} \Psi_{j a}^{(s-1)} -  Z_{i a} Z_{j a}] (A_{i j} - A_{i j}^{\prime}) \big| \leq 4n\exp(-C_1 d).
\end{equation}
From \eqref{second bound in ps (dcsbm)} and \eqref{third bound in ps (dcsbm)}, the numerator of $p^{(s)}$ satisfies  
\begin{align}\label{estimation of numerator of ps (dcsbm)}
\big| \sum_{i<j} \sum_a \Psi_{i a}^{(s-1)} \Psi_{j a}^{(s-1)} A_{i j} -   \sum_{i<j} \sum_a Z_{i a} Z_{j a} A_{i j}\big| \leq 3n\exp(-c d)C_0 d + 4n\exp(-C_1 d).
\end{align}
Moreover, notice that the denominator of $p^{(s)}$ is:
\begin{align}\label{eq:denominator of ps (dcsbm)}
   \sum_{i<j} \sum_a \Psi_{i a}^{(s-1)} \Psi_{j a}^{(s-1)} \theta_i^{(s-1)} \theta_j^{(s-1)}  &= \frac{1}{2}\sum_a  \left( \sum_{i } \Psi_{ia}^{(s-1)} \theta_i^{(s-1)} \right)^2 - \frac{1}{2}\sum_a \sum_{i} ( \Psi_{ia}^{(s-1)} \theta_i^{(s-1)} )^2 \nonumber \\
   &= \frac{n^2 }{2K} - \frac{1}{2}\sum_a \sum_{i} ( \Psi_{ia}^{(s-1)} \theta_i^{(s-1)} )^2  \nonumber \\
   &= \frac{n^2}{2K} + \mathcal{E},
\end{align}
where $|\mathcal{E}|< r_1 n ^{5/3}$ for some absolute constant $r_1$. Notice here the second equation holds since  $\sum_{i } \Psi_{ia}^{(s-1)} \theta_i^{(s-1)} = n/K$ for all $a$ after rescaling. The third equation holds since we can put the term $\sum_a \sum_{i} ( \Psi_{ia}^{(s-1)} \theta_i^{(s-1)} )^2$ within the error term by the induction \eqref{induction assumption of theta}.

Therefore, from \eqref{estimation of numerator of ps (dcsbm)}, \eqref{eq:denominator of ps (dcsbm)} and \eqref{bound of underline ps (dcsbm)} we get
\begin{equation}
\label{estimation of ps (dcsbm)}
|p^{(s)} - p| \leq |p^{(s)} - \underline{p}^{(s)}| + | \underline{p}^{(s)} - p| \leq p [ \exp(-c_3 d) + c_3^{\prime} n^{-1/3}],
\end{equation}
where $c_3$ and $c_3^{\prime}$ are constants only depending on $\varepsilon$ and $K$. 
By the similar argument we have 
\begin{equation}
\label{estimation of qs (dcsbm)}
|q^{(s)} - q| \leq q[ \exp(-c_4 d) + c_4^{\prime} n^{-1/3}],
\end{equation}
where $c_4$ and $c_4^{\prime}$ are constants only depending on $\varepsilon$ and $K$. 
Since
\[t^{(s)} = \frac{1}{2} \log \frac{p^{(s)}}{ q^{(s)}}, \quad \lambda^{(s)} = \frac{1}{2t^{(s)}} \log (p^{(s)} - q^{(s)} ),\]
it follows directly that 
\begin{equation}\label{estimatino of ts and lambdas (dcsbm)}
    |t^{(s)} - t| \leq t[ \exp(-c_5 d) + c_5^{\prime} n^{-1/3}],\quad
    |\lambda^{(s)} - \lambda| \leq \lambda [ \exp(-c_6 d) + c_6^{\prime} n^{-1/3}]
,
\end{equation}
where $c_5$, $c_5^{\prime}$, $c_6$ and $c_6^{\prime}$ are the constants only depending on $\varepsilon$ and $K$.

\medskip

\noindent \textbf{Beyond first step: label estimation}\\

Similarly as in the first iteration we know $\Psi_{i}^{(s)} = Z_i$ after threshold if any only if the following inequality  \[ \sum_{j:j\neq i} (\Psi_{j z_i}^{(s-1)} - \Psi_{ja}^{(s-1)}) (A_{ij} -  \theta_i^{(s-1)} \theta_j^{(s-1)} \lambda^{(s)}) > 0
\] holds for any $a \in [K]$ such that $a \neq z_i$. By the similar decomposition for the left-hand side of this inequality, we can define the signal and noise terms in the $s$-th iteration as
\begin{align*}
   \text{SIG}_{ia}^{(s-1)} :=& \sum_{j:j\neq i} (\Psi_{j z_i}^{(s-1)} - \Psi_{ja}^{(s-1)}) (E[A_{ij}] - \theta_i \theta_j \lambda^{(s)}), \\
    r_{ia}^{(s-1)} :=& \sum_{j:j\neq i} (\Psi_{j z_i}^{(s-1)} - \Psi_{ja}^{(s-1)}) ( A_{ij} - E[A_{ij}]), \\
    R_{ia}^{(s-1)} :=& \sum_{j:j\neq i} (\Psi_{j z_i}^{(s-1)} - \Psi_{ja}^{(s-1)}) (\theta_i \theta_j \lambda^{(s)} - \theta_i^{(s-1)} \theta_j^{(s-1)} \lambda^{(s)}).
\end{align*}

To quantify the signal term, first we state the sign of two expressions. It follows from Lemma~\ref{lemma:Propositions of the parameter lambda (dcsbm)} and \eqref{estimatino of ts and lambdas (dcsbm)} that 
\begin{equation}\label{sign of two expressions （dcsbm)}
    p - \lambda^{(s)} >0, \quad q - \lambda^{(s)}<0.
\end{equation}
Second we give the upper bounds of two expressions. By the triangle inequality, from \eqref{induction assumption (dcsbm)} we have: 
\begin{align*}
    \Big|\sum_{j:j\neq i,z_j = z_i} (\Psi_{j z_i}^{(s-1)} \theta_j - \Psi_{ja}^{(s-1)}\theta_j)  -  \sum_{j:j\neq i,z_j = z_i} (Z_{j z_i} \theta_j - Z_{ja}\theta_j )  \Big| \leq r_1 || \Psi^{(s-1)} - Z||_1 = r_1 nG_{s-1},
\end{align*}
where $r_1 >0$ is an absolute constant.
Since $ \sum_{j:j\neq i,z_j = z_i} (Z_{j z_i} \theta_j - Z_{ja} \theta_j) = \frac{n}{K} + \mathcal{E} $ where $|\mathcal{E}| \leq r_1 n^{2/3}$ for some absolute constant $r_1$, we have \begin{align}\label{upper bound of first expressions (dcsbm)}
    \Big|\sum_{j:j\neq i,z_j = z_i} (\Psi_{j z_i}^{(s-1)} \theta_j - \Psi_{ja}^{(s-1)} \theta_j ) - n/K \Big| \leq n r_1 (G_{s-1} +n^{-1/3}).
\end{align}
Similarly we can derive the other upper bound: 
\begin{align}\label{upper bound of second expressions (dcsbm)}
    \Big|\sum_{j:j\neq i,z_j \neq z_i} (\Psi_{j z_i}^{(s-1)} \theta_j - \Psi_{ja}^{(s-1)} \theta_j ) - (-n/K ) \Big| \leq n r_1 (G_{s-1} +n^{-1/3}) .
\end{align}
Therefore from \eqref{sign of two expressions （dcsbm)}, \eqref{upper bound of first expressions (dcsbm)} and \eqref{upper bound of second expressions (dcsbm)}
we can quantify the signal term by
\begin{align*}
    SIG_{ia}^{(s-1)} = \sum_{j:j\neq i,z_j = z_i} (\Psi_{j z_i}^{(s-1)} - \Psi_{ja}^{(s-1)})(p - \lambda^{(s)})\theta_i\theta_j + \sum_{j:j\neq i, z_j \neq z_i } (\Psi_{j z_i}^{(s-1)} - \Psi_{ja}^{(s-1)})(q - \lambda^{(s)})\theta_i\theta_j \\
    \geq \theta_i \left( n/K- n r_1 (G_{s-1} +n^{-1/3}) \right)(p - \lambda^{(s)}) + \theta_i \left(-n/K+n r_1 (G_{s-1} +n^{-1/3}) \right)(q - \lambda^{(s)}) \\
    \geq \theta_i [\frac{n}{K} - n r_1 (G_{s-1} +n^{-1/3})](p-q) \geq \theta_i[\frac{n}{K} - n r_1 (\exp(-cd) +n^{-1/3})](p-q),
\end{align*}
where $r_1$ is an absolute constant. Let us define a small constant $\delta^{\prime}$:
\[ \delta^{\prime} = \frac{C_3(p-q)}{2 [p + (K-1)q]},
\]
where $\delta^{\prime}$ satisfies $2 \delta^{\prime} d_i < SIG_{ia}^{(s-1)}$ and $C_3$ is an absolute constant. We will show $R_{ia}^{(s-1)}$ and $r_{ia}^{(s-1)}$ can be bounded by $\delta^{\prime} d_i$ (at least for most of the nodes).

To bound the noise term $R_{ia}^{(s-1)}$, notice that
\begin{align*}
        |R_{ia}^{(s-1)}| &= \Big| \theta_i \sum_{j:j\neq i} (\Psi_{j z_i}^{(s-1)} - \Psi_{ja}^{(s-1)}) \theta_j \lambda^{(s)}  -  \theta_i^{(s-1)}\sum_{j:j\neq i} (\Psi_{j z_i}^{(s-1)} - \Psi_{ja}^{(s-1)}) \theta_j^{(s-1)} \lambda^{(s)} \Big| \\
        & \leq r_1 \theta_i n G_{s-1} \rho_n \leq \delta^{\prime} d_i ,
\end{align*}
where $r_1$ is an absolute constant. The last inequality holds since we assume that $d$ is larger than a constant that depending on $\varepsilon$ and $K$.

To bound the noise term $r_{ia}^{(s-1)}$ we can decompose it by
\begin{align*}
r^{(s-1)}_{ia} = \sum_{j:j \neq i} (A _{i j} - A^{\prime}_{i j} )(\Psi^{(s-1)}_{j z_i} - \Psi^{(s-1)}_{ja} - Z_{j z_i} +Z_{ja} ) + \sum_{j:j \neq i} (A^{\prime}_{i j} - P_{i j} )(\Psi^{(s-1)}_{j z_i} \\- \Psi^{(s-1)}_{ja} - Z_{j z_i} +Z_{ja} ) + \sum_{j:j \neq i} (A_{i j} - P_{i j} )(Z_{j z_i} - Z_{ja})  =: r_{1ia}^{(s-1)} + r_{2ia}^{(s-1)} + r_{3ia}.
\end{align*}
If $r_{ia}^{(s-1)}>\delta^{\prime} d_i$, it implies that at least one components above is larger than $\delta^{\prime} d_i/3$. Let $\delta_1 = \delta_2 = \delta_3 = \delta^{\prime}/3$, we have the fact that
\begin{align*}
\mathbbm{1}\Big(\bigcup_{a:a \neq z_i}  
 \Big\{ |r_{ia} ^{(s-1)} | > \delta^{\prime} d_i
 \Big\} \Big) \leq \mathbbm{1}\Big(\bigcup_{a:a \neq z_i}  
 \Big\{ |r_{1ia} ^{(s-1)} | > \delta_1 d_i
 \Big\} \Big)+\mathbbm{1}\Big(\bigcup_{a:a \neq z_i}  
 \Big\{ |r_{2ia} ^{(s-1)} | > \delta_2 d_i
 \Big\} \Big) \\+\mathbbm{1}\Big(\bigcup_{a:a \neq z_i}  
 \Big\{ |r_{3ia}  | > \delta_3 d_i
 \Big\} \Big).
 \end{align*}
 
By using the similar argument as in the first iteration, the following event 
\begin{equation}\label{first bound in Psis (dcsbm)}
\mathcal{A}_2 = \left\{\sum_{i=1}^n \mathbbm{1}\Big( \bigcup_{a:a \neq z_i}  \Big\{ |r_{3ia}| > \delta_3 d_i
 \Big\} \Big) \leq 16(K-1) \sum_i \exp\left(\frac{-\delta_3^2d_i}{32+8\delta_3/3}\right) \right\} 
 \end{equation}
occurs with probability at least $1 - n^{-r_1}$ for some absolute constant $r_1$.

To deal with the first noise component $r_{1ia}^{(s-1)}$, first notice that the following inequality holds by Lemma~\ref{lemma:quantity of removed edges (dcsbm)}:\[
 \sum_i \sum_a |r_{1ia}^{(s-1)}| \leq  \sum_i \sum_a \sum_{j:j \neq i} 2(A_{ij} - A_{ij}^{\prime}) \leq 4nK \exp(- C_1 d).
\]
This implies
\[\sum_{i}\sum_{a} \mathbbm{1}(|r^{(s-1)}_{1ia}| > \delta_1 d_i) \leq \frac{ 4nK\exp(-C_1 d)}{\delta_1 d_{\text{min}}}.
\]
Consequently we have
\begin{equation}\label{second bound in Psis (dcsbm)}
\sum_{i=1}^n \mathbbm{1}\Big(\bigcup_{a:a \neq z_i}  
 \Big\{ |r_{1ia} ^{(s-1)} | > \delta_1 d_i
 \Big\} \Big) \leq \sum_i \sum_a \mathbbm{1}(|r^{(s-1)}_{1ia}| > \delta_1 d_i) \leq \frac{ 4nK\exp(-C_1 d)}{\delta_1 d_{\text{min}}}.
 \end{equation}

To deal with the second noise component $r_{2ia}^{(s-1)}$, first notice that the following inequality holds by \eqref{induction assumption (dcsbm)} and Lemma~\ref{lemma:concentration of regularized adjacency matrices}:
\begin{align*}
\sum_i \sum_a |r_{2ia}^{(s-1)}|^2 &= \sum_a \sum_{i} \left( \sum_{j: j \neq i} (A^{\prime}_{i j} - P_{i j} )(\Psi^{(s-1)}_{j z_i} - \Psi^{(s-1)}_{ja} -Z_{j z_i} +Z_{ja} ) \right)^2 \\
& \leq \sum_a ||A^{\prime} - P||^2 \cdot 2nG_{s-1} \leq 2 C_2 d K nG_{s-1}.
\end{align*}
This implies \[
\sum_{i}\sum_{a} \mathbbm{1}(|r^{(s-1)}_{2ia}| > \delta_2 d_i) \leq \frac{ 2C_2 dK nG_{s-1} }{\delta_2^2 d_{\text{min}}^2},
\]
and consequently,
\begin{equation}\label{third bound in Psis (dcsbm)}
 \sum_{i=1}^n \mathbbm{1}\Big(\bigcup_{a:a \neq z_i}  
 \Big\{ |r_{2ia} ^{(s-1)} | > \delta_1 d_i
 \Big\} \Big) \leq \sum_{i}\sum_{a} \mathbbm{1}(|r^{(s-1)}_{2ia}| > \delta_2 d_i) \leq \frac{ 2 C_2 d K nG_{s-1} }{\delta_2^2 d_{\text{min}}^2}.
 \end{equation}
Therefore, it follows from \eqref{first bound in Psis (dcsbm)}, \eqref{second bound in Psis (dcsbm)} and \eqref{third bound in Psis (dcsbm)} that 
\begin{align*}
||\Psi^{(s)} - Z||_1 \leq 2 \sum_{i} \mathbbm{1}\Big( \bigcup_{a:a \neq z_i}  \Big\{ |r_{ia} ^{(s-1)} | > \delta^{\prime} d_i
 \Big\} \Big) 
 \leq  \frac{4C_2 d K}{ \delta_2^2 d_{\min}^2} || \Psi^{(s-1)}- Z ||_1 \\+ \frac{8nK\exp(-C_1 d)}{\delta_1 d_{\text{min}}} +  32(K-1)\sum_i \exp\left(\frac{-\delta_3^2d_i}{32+8\delta_3/3}\right).
\end{align*}
Since we assume that $d$ is larger than a constant that depending only on $\epsilon$ and $K$, the inductive assumption in \eqref{induction assumption (dcsbm)} automatically satisfied for $\Psi^{(s)}$ as long as we choose a suitable constant $c$ such that $c < \frac{\delta^2}{32 + 8\delta / 3}$, $c<C_1$ and $c< \frac{\delta_3^2}{32 + 8\delta_3 / 3}$. For example, we can choose $c = \frac{1}{2} \min\{\frac{\delta^2}{32 + 8\delta / 3}, C_1, \frac{\delta_3^2}{32 + 8\delta_3 / 3}\}$.

\medskip

\noindent \textbf{Beyond first step: Degree parameter estimation}\\

From \eqref{update of theta in proof} the update of $\theta$ at $s$-iteration is: 
\[ \frac{\sum_j A_{i j}}{ \theta_i^{(s)}} = \sum_{j:j \neq i } \sum_{a \neq b} \Psi_{i a}^{(s-1)} \Psi_{j b}^{(s-1)} \theta_j^{(s-1)} q^{(s)} + \sum_{j:j \neq i } \sum_{a} \Psi_{i a}^{(s-1)} \Psi_{j a}^{(s-1)} \theta_j^{(s-1)} p^{(s)}, \quad i \in [n].\]
We will analyze $\theta^{(s)}$ in each estimated community. First we fix $a$ and analyze $\theta_i^{(s)}$ when $\Psi_{ia}^{(s-1)} = 1$. It follows from \eqref{estimation of ps (dcsbm)} and \eqref{estimation of qs (dcsbm)} that the above update equation can be simplified to:
\begin{align}
\label{bound sth update of theta}
\nonumber \frac{\sum_j A_{i j}}{ \theta_i^{(s)}}  &= \sum_{j:j \neq i } \sum_{b: b\neq a} \Psi_{j b}^{(s-1)} \theta_j^{(s-1)} q^{(s)} + \sum_{j:j \neq i }  \Psi_{j a}^{(s-1)} \theta_j^{(s-1)} p^{(s)} \\
&= \frac{n(K-1)}{K}q^{(s)} + \frac{np^{(s)}}{K} - \theta_i^{(s-1)}p^{(s)} = d + \mathcal{E},
\end{align}
where $|\mathcal{E}|< n \rho_n [ \exp(- c_0 d) + c_0^{\prime} n^{-1/3}]$ for constants $c_0$ and $c_0^{\prime}$  only depending on $\varepsilon$ and $K$. Therefore, it follows from \eqref{bound of sum Di} and \eqref{bound sth update of theta} that 
\begin{equation}
\label{theta_s from theta_0}
 \theta_i^{(s)} = \theta_i^{(0)} (1 + \mathcal{E} ),
 \end{equation}
where $|\mathcal{E}|< \exp(- c_0 d) + c_0^{\prime} n^{-1/3}$. Since we further update $\theta_i^{(s)}$ by \eqref{threshold theta in dcsbm in proof}, we only need show $\sum_i \Psi_{ia}^{(s)} \theta_i^{(s)}$ is close to $n/K$. For this purpose we can just show $\sum_i \Psi_{i a}^{(s)} \theta_i^{(0)}$ is close to $n/K$ according to \eqref{theta_s from theta_0}. To evaluate $\sum_i \Psi_{i a}^{(s)} \theta_i^{(0)}$, first we show that $\sum_{i j} \Psi_{ia}^{(s)} A_{i j}$ is close to $nd/K$.

\begin{align*}
\sum_{i j} \Psi_{ia}^{(s)} A_{i j} &= \sum_{i j} (\Psi_{ia}^{(s)} - Z_{ia}) A_{i j} +  \sum_{i j}  Z_{ia} A_{i j} \\
&= \sum_{i j} (\Psi_{ia}^{(s)} - Z_{ia}) (A_{i j} - A_{ij}^{\prime}) +  \sum_{i j}(\Psi_{ia}^{(s)} - Z_{ia})  A_{ij}^{\prime} +  \sum_{i j}  Z_{ia} A_{i j}.
\end{align*}

By Lemma~\ref{lemma:quantity of removed edges (dcsbm)} the first term is bounded by $4n\exp(-C_1 d)$. In the previous section we proved that the second term is bounded by $n \exp(- cd)$. Therefore, it follows that
\begin{equation}
\label{bound of numerator in theta_0 Psi_s}
\big| \sum_{i j} \Psi_{ia}^{(s)} A_{i j} - \frac{nd}{K} \big| \leq 4n\exp(-C_1 d) + n \exp(-cd) + r_1 n^{5/3}\rho_n,
 \end{equation}
for some absolute constant $r_1$. Then it follows from \eqref{bound of numerator in theta_0 Psi_s}, \eqref{bound of sum Di} and \eqref{theta_s from theta_0} that after rescaling we have:
\[
 |\theta_i^{(s)} - \theta_i^{(0)}| \leq \frac{1}{4}\theta_i^{(0)}
\]
as long as $n\rho_n \ge C$ for some constants $C$ only depending on $\varepsilon$ and $K$. The proof is complete.
\end{proof}

\section{Numerical results for networks generated from DCSBM}\label{sec: simu of dcsbm}
In this section, we provide the results of numerical study in DCSBM. We follow the same network settings as described in Section~\ref{simulation}. In addition, we generate each element of $\theta$ from $\text{Beta}(2,1/3)$, the Beta distribution with parameters $2$ and $1/3$, respectively. For the spectral clustering initialization settings as in Figures~\ref{avd-dc} and \ref{Uavd-dc}, we use regularized spectral clustering \citep{qin2013regularized} in DCSBM. The qualitative behaviors of the two versions of T-BCAVI based on SBM and DCSBM are very similar for networks drawn from SBM and DCSBM, respectively; for a more detailed discussion about these results, see the discussions in Section~\ref{simulation}.

\begin{figure}[ht]
\begin{subfigure}{.33\textwidth}
  \centering
  \includegraphics[width=1\linewidth]{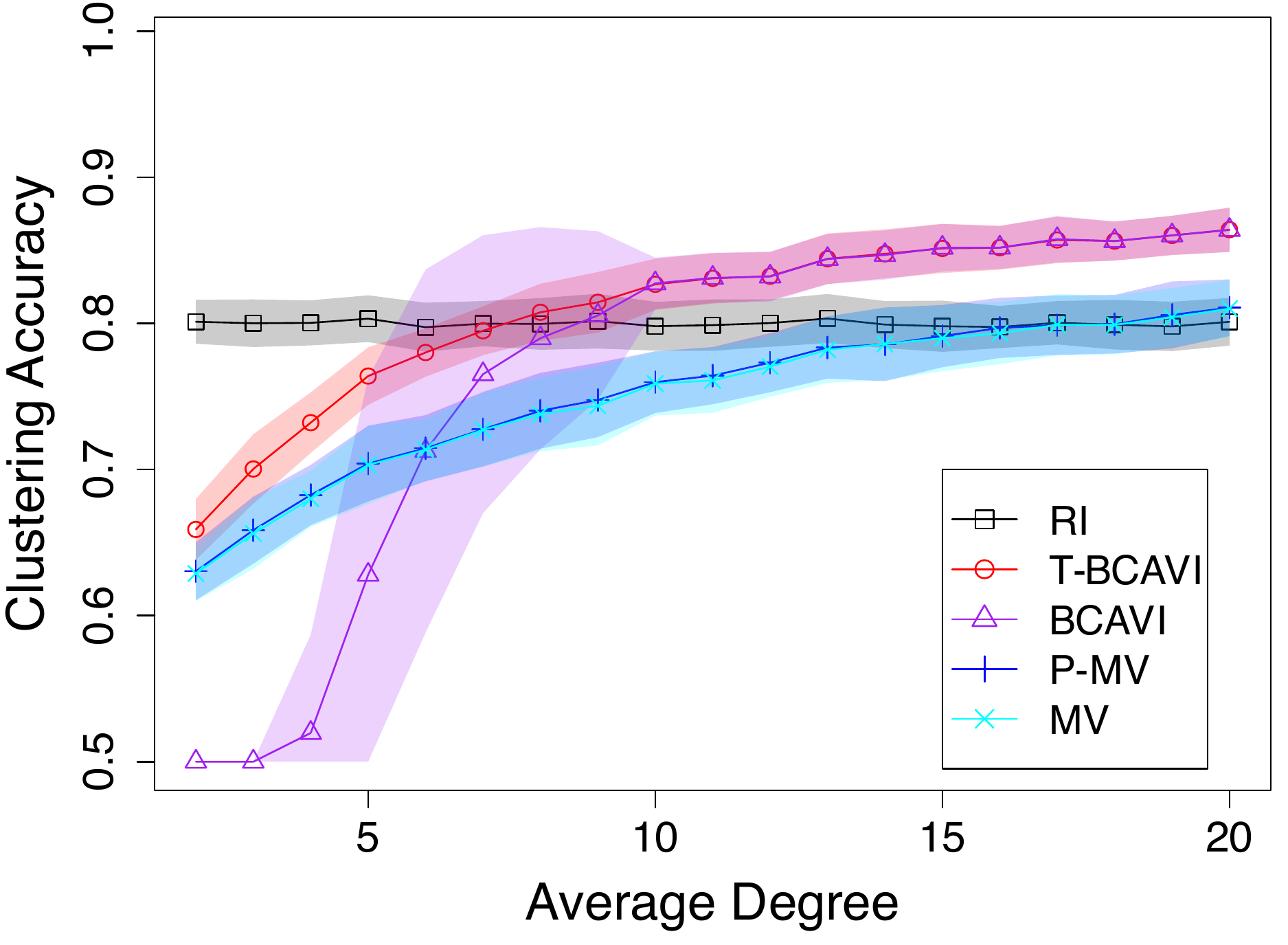}  
  \caption{$\varepsilon = 0.2$}
  \label{eps0.2-dc}
\end{subfigure}
\begin{subfigure}{.33\textwidth}
  \centering
  \includegraphics[width=1\linewidth]{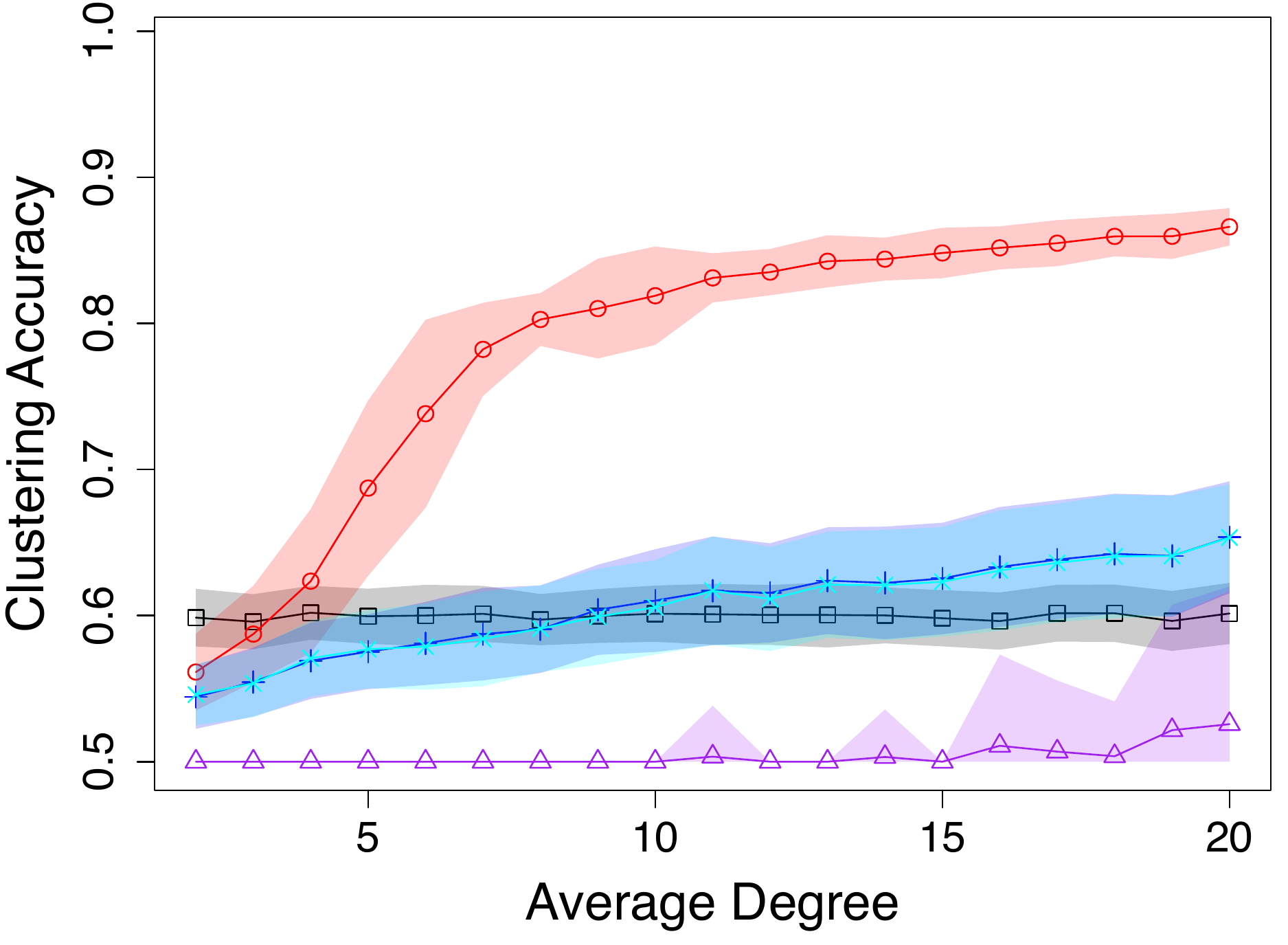}  
  \caption{$\varepsilon = 0.4$}
  \label{eps0.4-dc}
\end{subfigure}
\begin{subfigure}{.33\textwidth}
  \centering
  \includegraphics[width=1\linewidth]{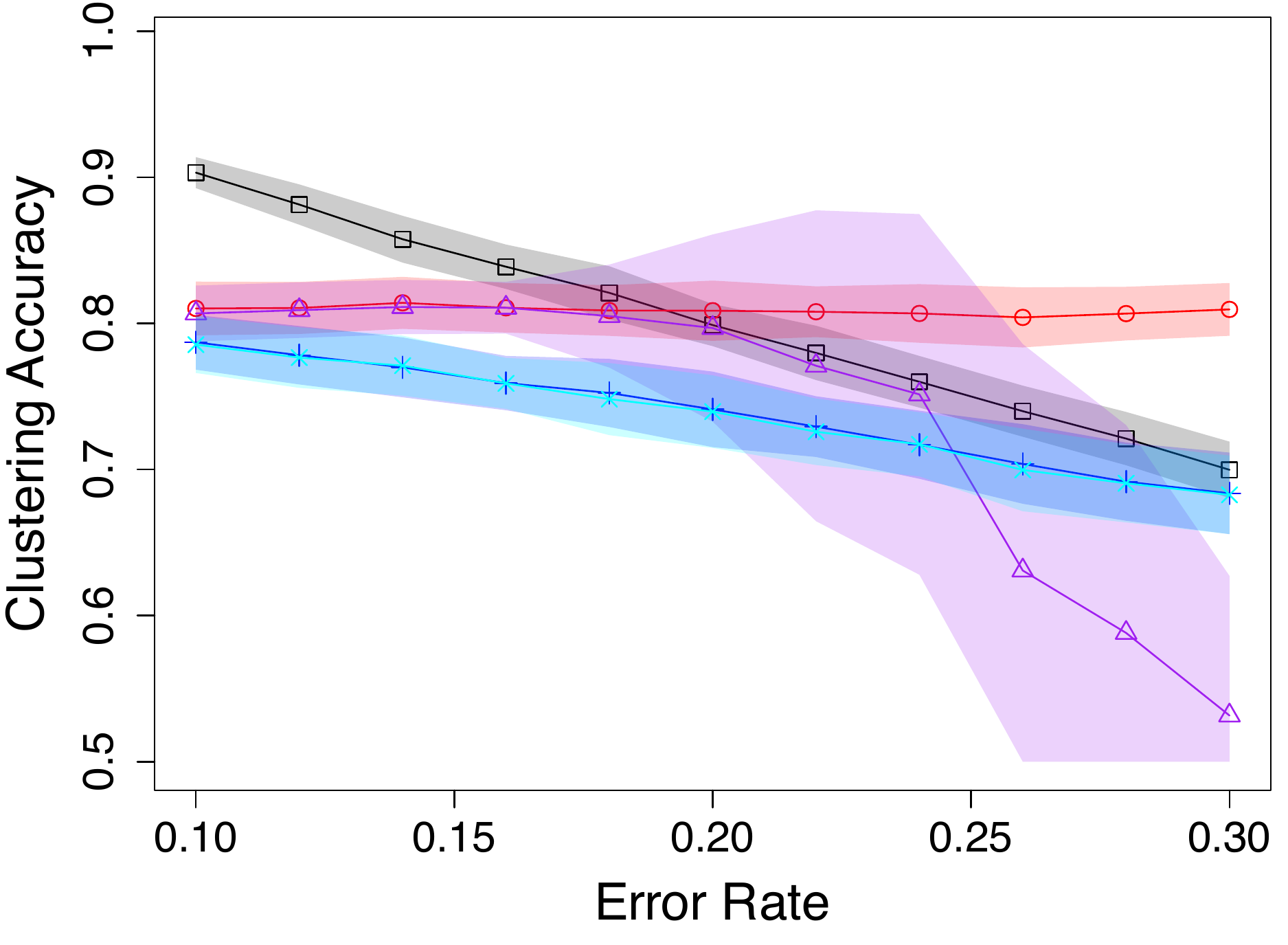}  
  \caption{$d = 8$}
  \label{avd8-dc}
\end{subfigure}
\begin{subfigure}{.33\textwidth}
  \centering
  \includegraphics[width=1\linewidth]{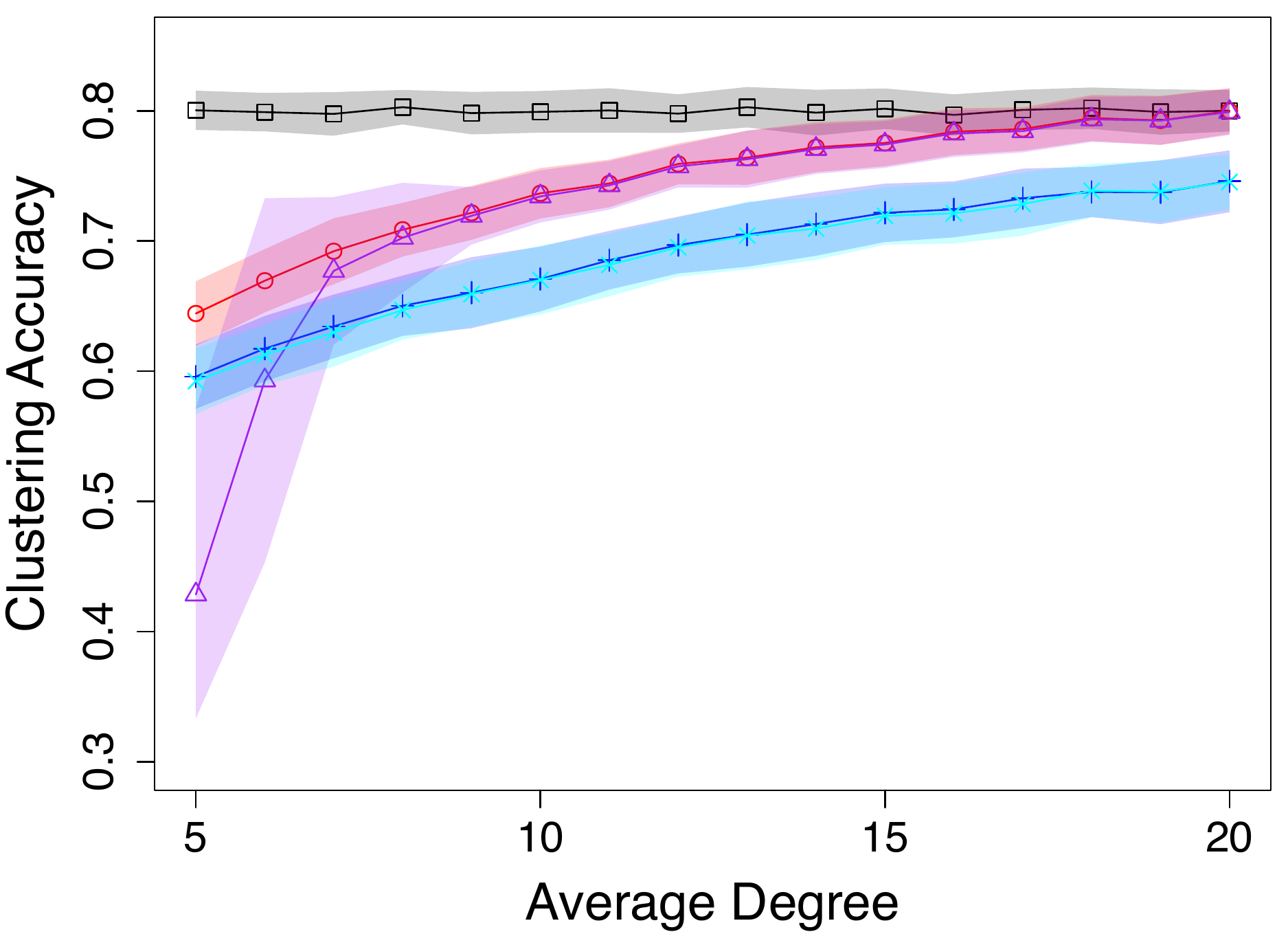}  
  \caption{$\varepsilon = 0.2$}
  \label{eps0.2(k=3)-dc}
\end{subfigure}
\begin{subfigure}{.33\textwidth}
  \centering
  \includegraphics[width=1\linewidth]{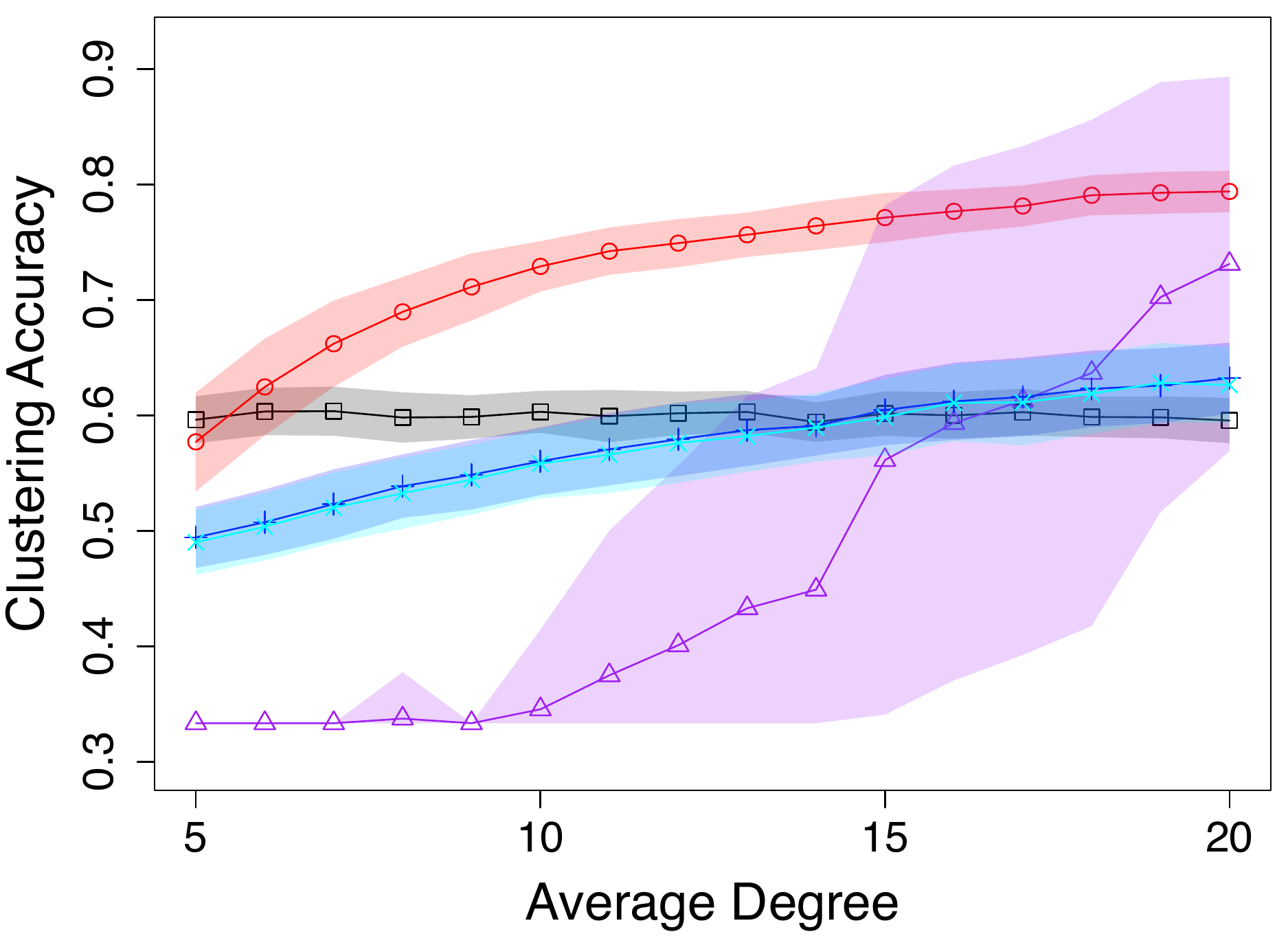}  
  \caption{$\varepsilon = 0.4$}
  \label{eps0.4(k=3)-dc}
\end{subfigure}
\begin{subfigure}{.33\textwidth}
  \centering
  \includegraphics[width=1\linewidth]{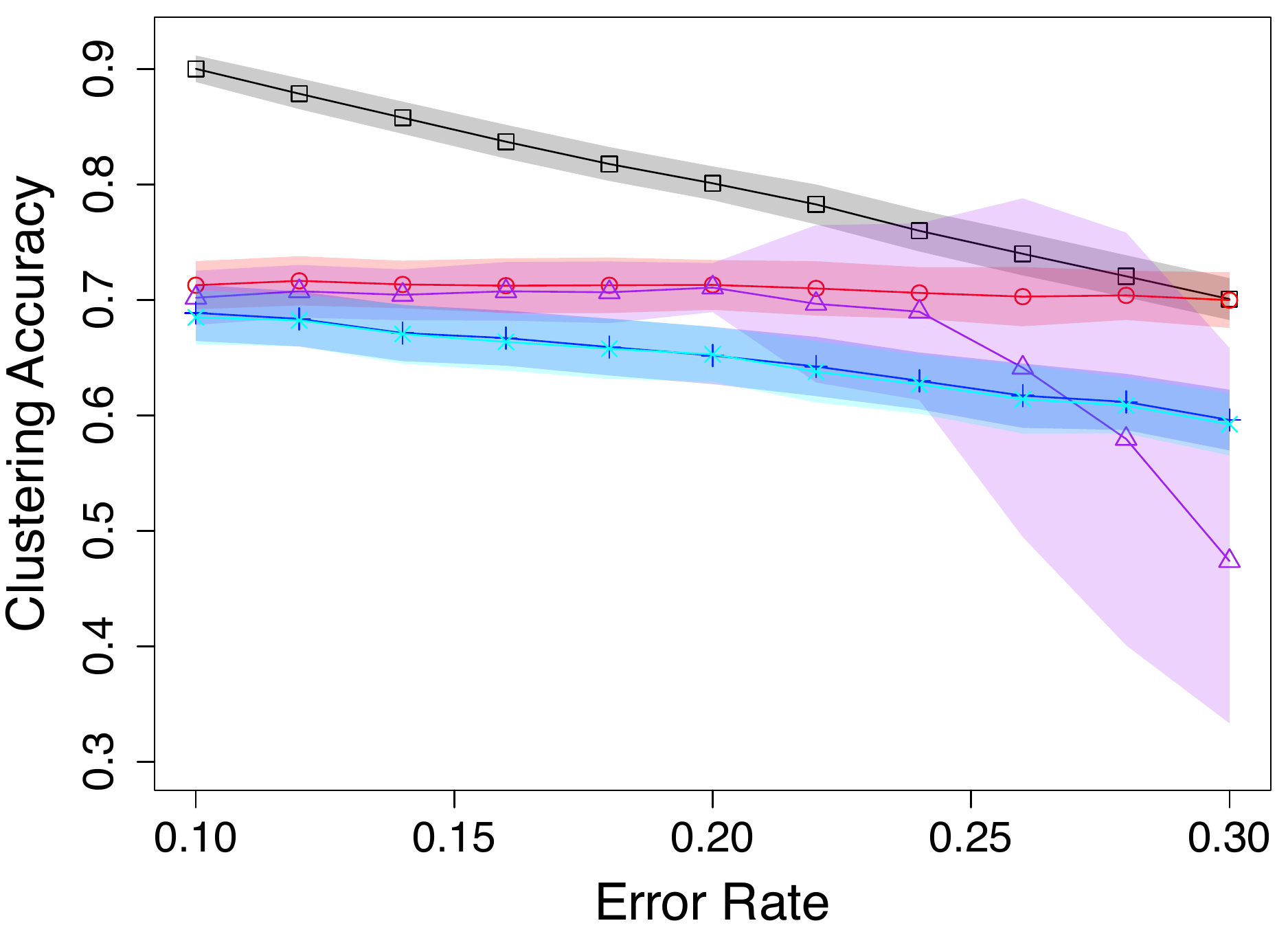}  
  \caption{$d = 8$}
  \label{avd8(k=3)-dc}
\end{subfigure}
\caption{Performance of Threshold BCAVI (T-BCAVI), the classical BCAVI, majority vote (MV), and majority vote with penalization (P-MV) in balanced settings. (a)-(c): Networks are generated from DCSBM with $n=600$ nodes, $K=2$ communities of sizes $n_1 = n_2 = 300$. (d)-(f): Networks are generated from DCSBM with $n=600$ nodes, $K=3$ communities of sizes $n_1 = n_2 = n_3 = 200$. Initializations are generated from true node labels according to Assumption~\protect\ref{ass:perturb} with error rate $\varepsilon$, resulting in actual clustering initialization accuracy (RI) of approximately $1-\varepsilon$.}
\label{eps-dc}
\end{figure}

\begin{figure}[ht]
\begin{subfigure}{.33\textwidth}
  \centering
  \includegraphics[width=1\linewidth]{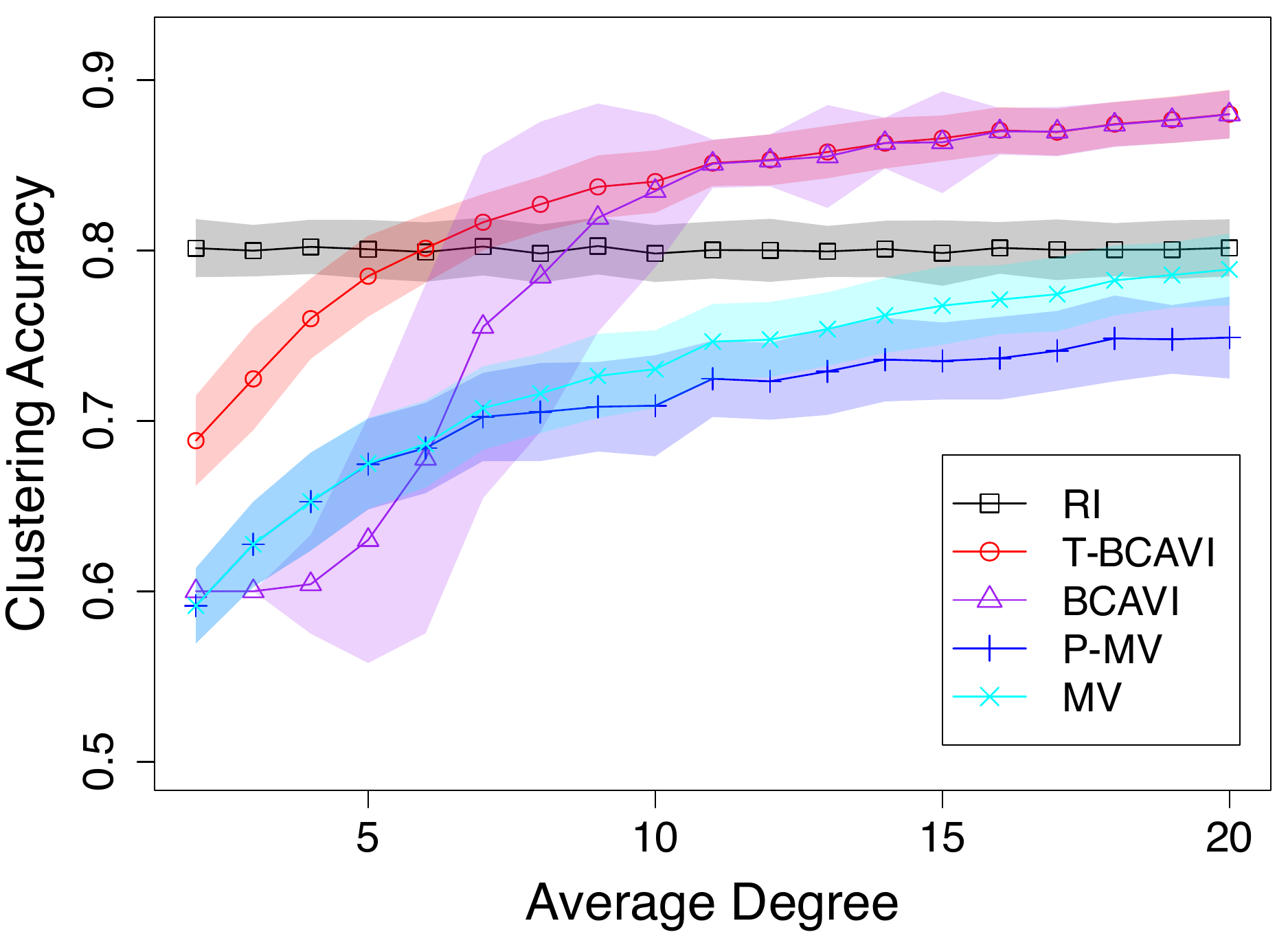}  
  \caption{$\varepsilon = 0.2$}
  \label{Ueps0.2-dc}
\end{subfigure}
\begin{subfigure}{.33\textwidth}
  \centering
  \includegraphics[width=1\linewidth]{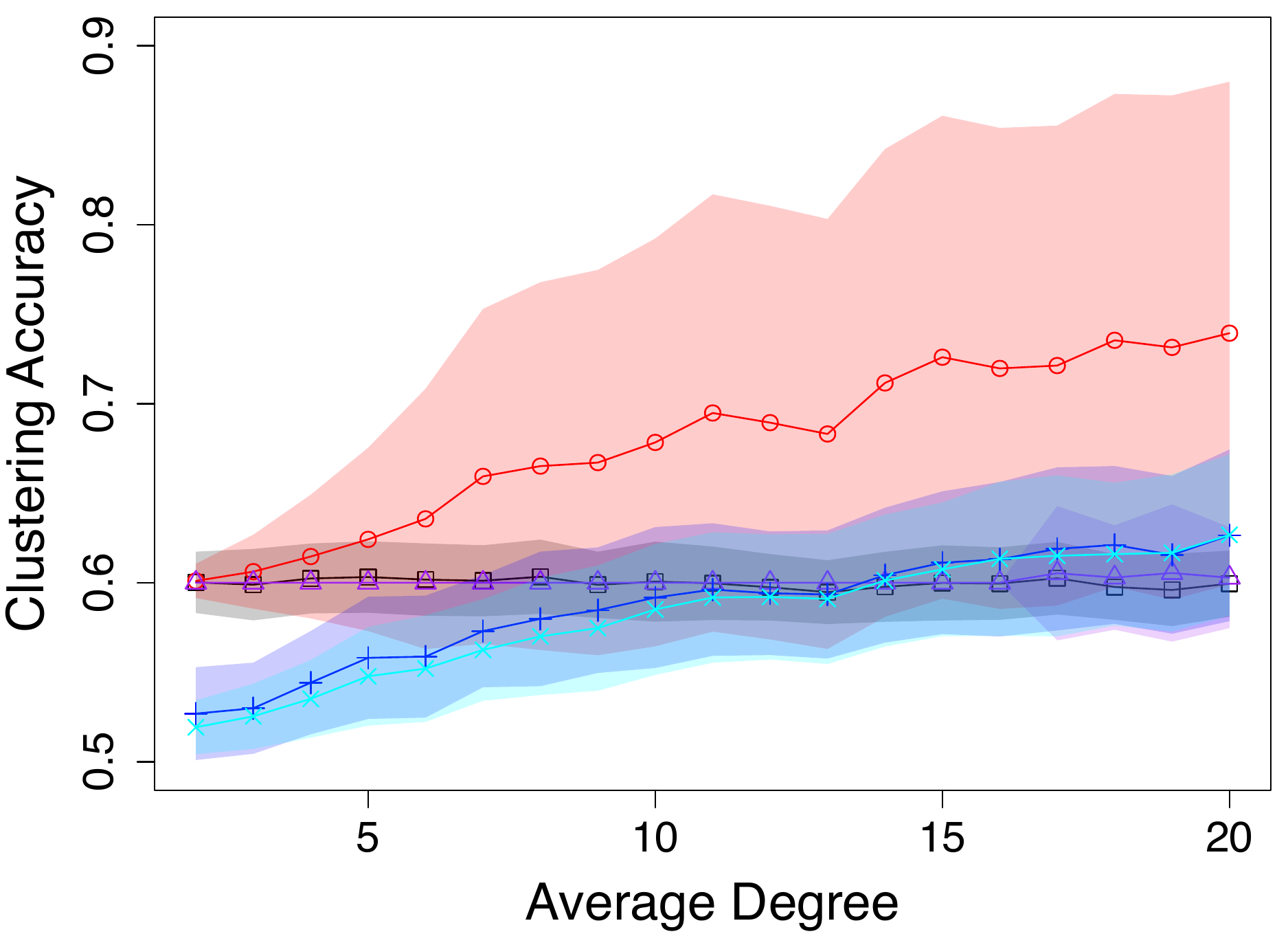}  
  \caption{$\varepsilon = 0.4$}
  \label{Ueps0.4-dc}
\end{subfigure}
\begin{subfigure}{.33\textwidth}
  \centering
  \includegraphics[width=1\linewidth]{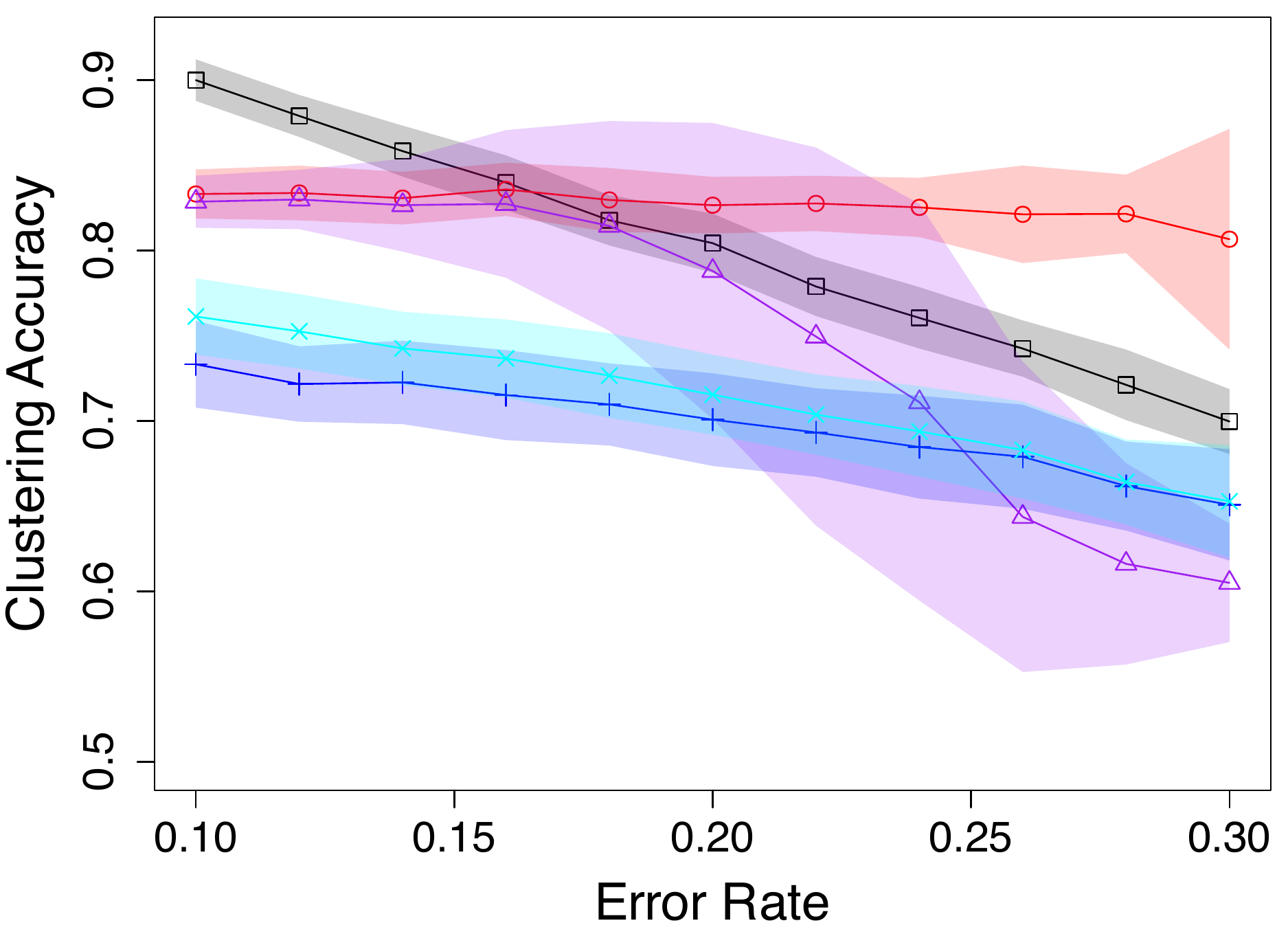}  
  \caption{$d = 8 $}
  \label{Uavd8-dc}
\end{subfigure}
\begin{subfigure}{.33\textwidth}
  \centering
  \includegraphics[width=1\linewidth]{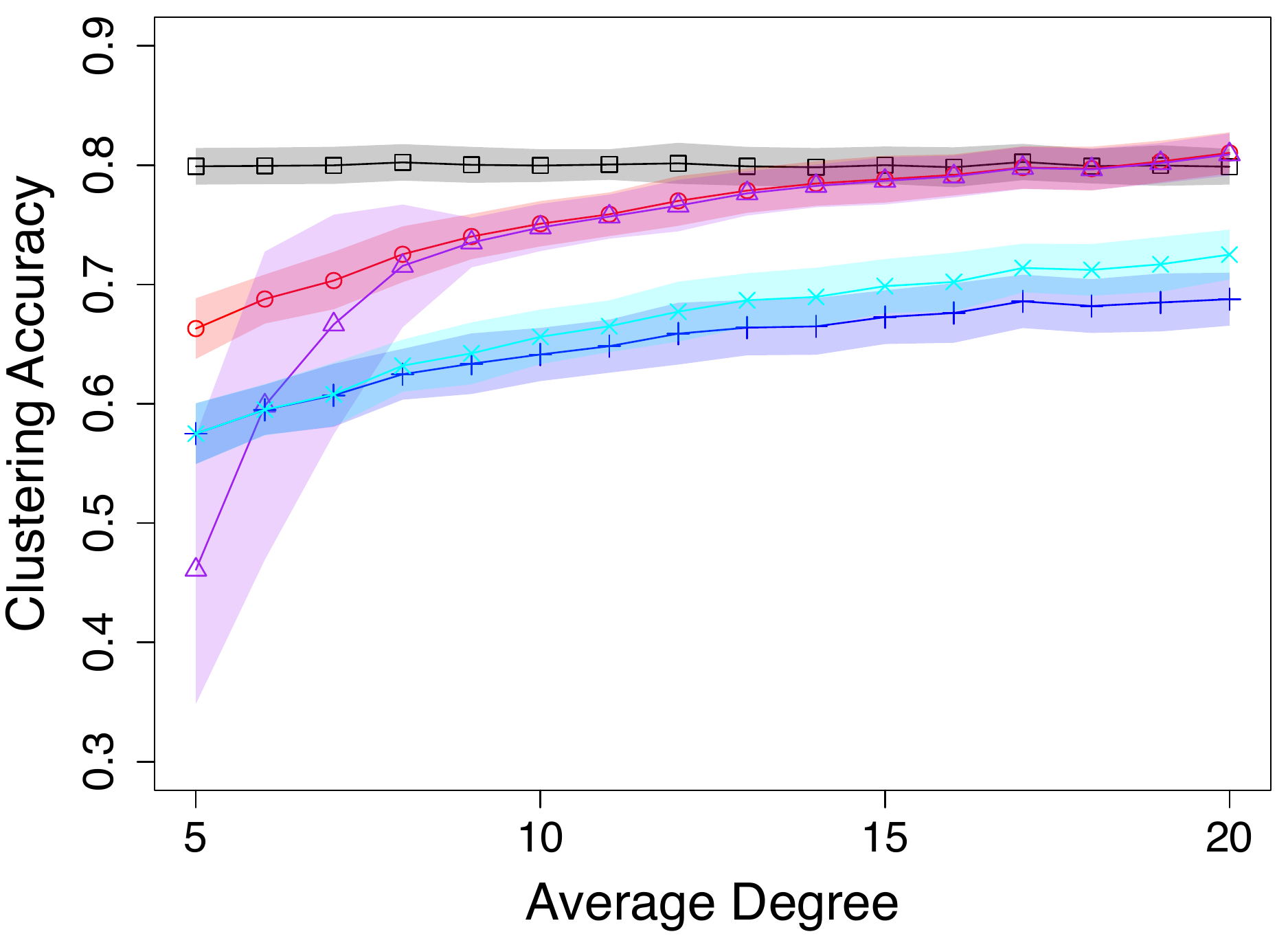}  
  \caption{$\varepsilon = 0.2$}
  \label{Ueps0.2(k=3)-dc}
\end{subfigure}
\begin{subfigure}{.33\textwidth}
  \centering
  \includegraphics[width=1\linewidth]{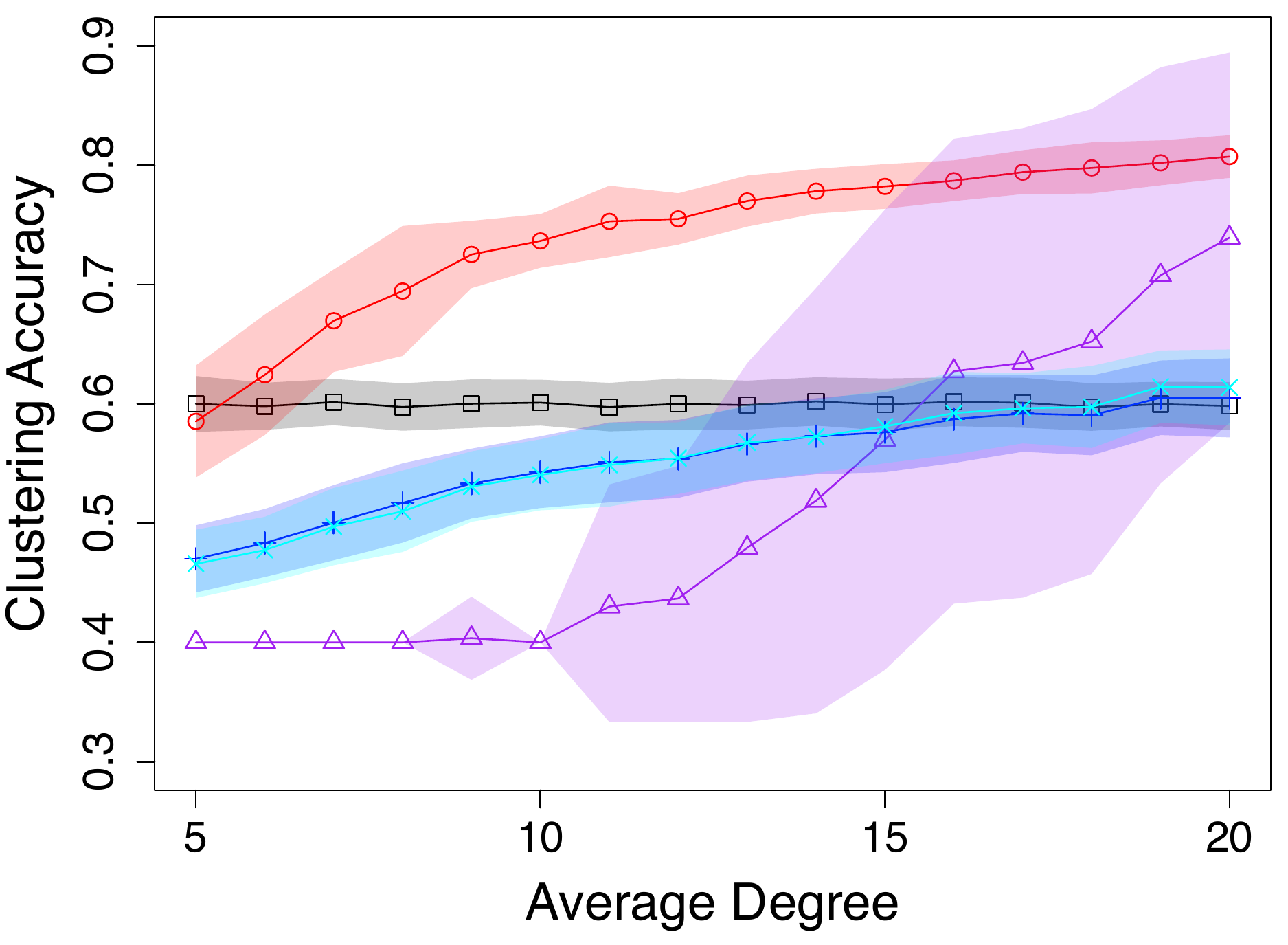} 
  \caption{$\varepsilon = 0.4$}
  \label{Ueps0.4(k=3)-dc}
\end{subfigure}
\begin{subfigure}{.33\textwidth}
  \centering
  \includegraphics[width=1\linewidth]{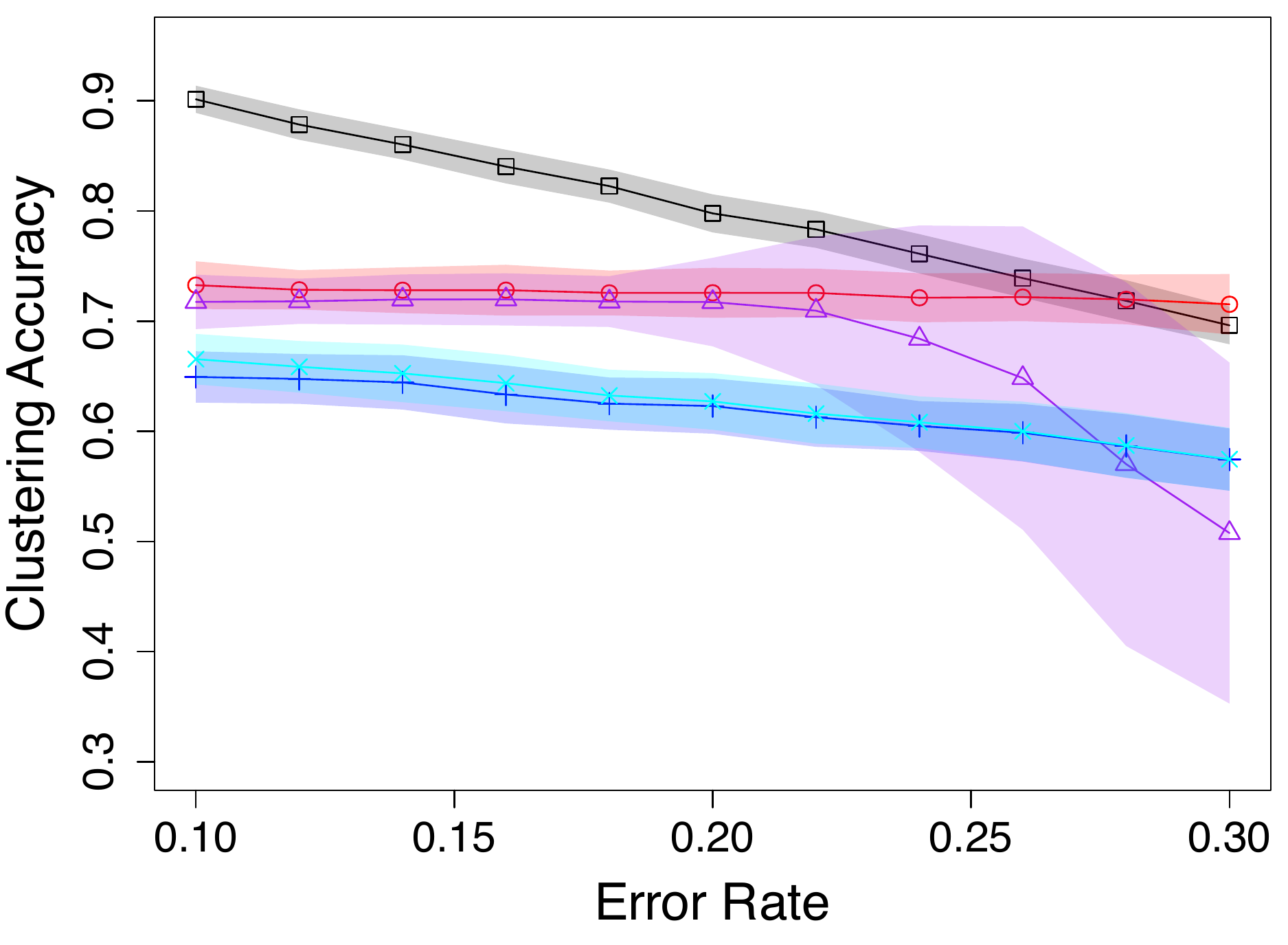}  
  \caption{$d = 8$}
  \label{Uavd8(k=3)-dc}
\end{subfigure}
\caption{Performance of Threshold BCAVI (T-BCAVI), the classical BCAVI, majority vote (MV), and majority vote with penalization (P-MV) in unbalanced settings. (a)-(c): Networks are generated from DCSBM with $n=600$ nodes, $K=2$ communities of sizes $n_1 = 240, n_2 = 360$. (d)-(f): Networks are generated from DCSBM with $n=600$ nodes, $K=3$ communities of sizes $n_1 = 150, n_2 =210, n_3 = 240$. Initializations are generated from true node labels according to Assumption~\protect\ref{ass:perturb} with error rate $\varepsilon$, resulting in actual clustering initialization accuracy (RI) of approximately $1-\varepsilon$.}
\label{Ueps-dc}
\end{figure}

\begin{figure}[ht]
\begin{subfigure}{.33\textwidth}
  \centering
  \includegraphics[width=1\linewidth]{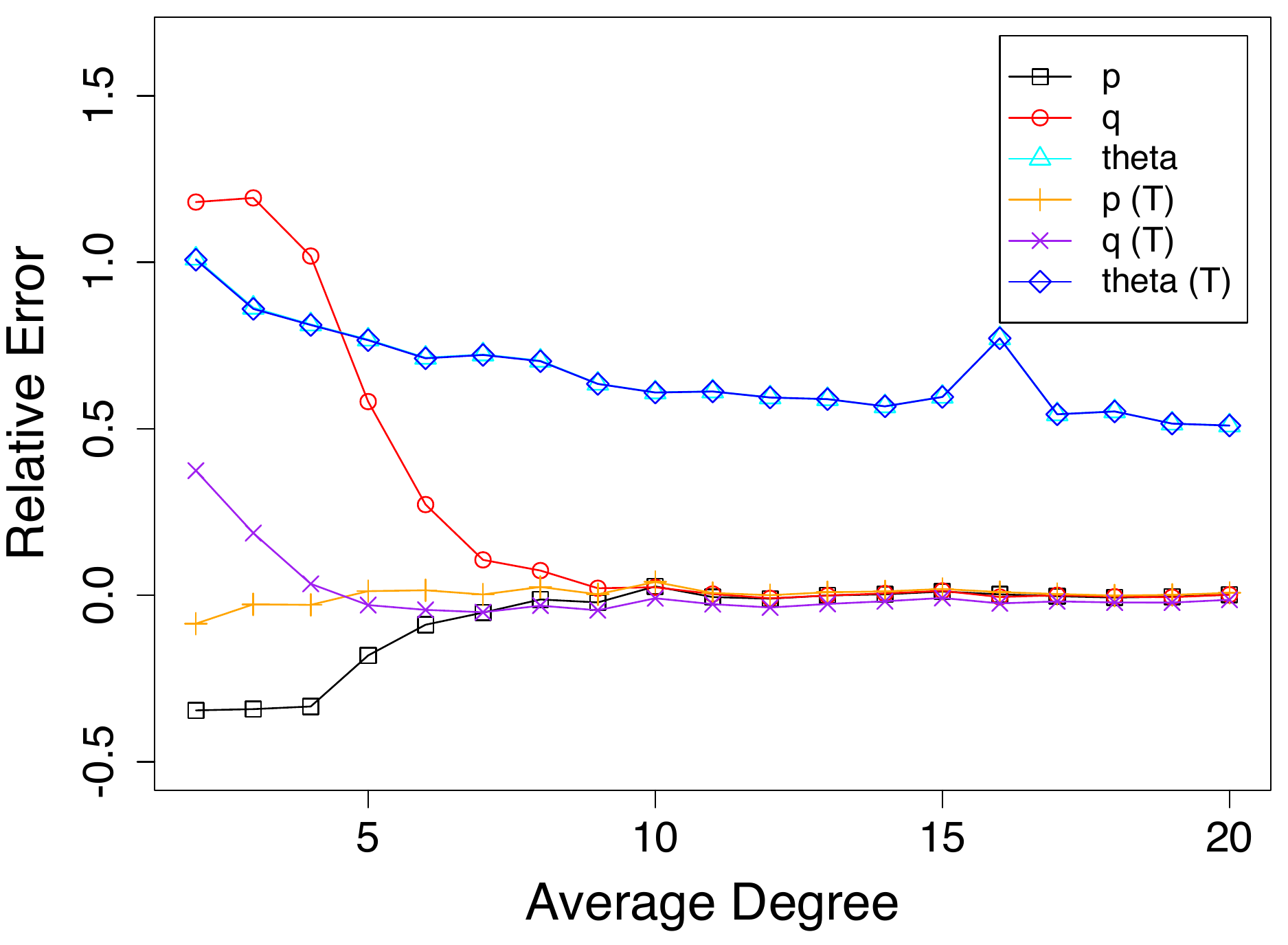}  
  \caption{$\varepsilon = 0.2$}
  \label{REeps0.2-dc}
\end{subfigure}
\begin{subfigure}{.33\textwidth}
  \centering
  \includegraphics[width=1\linewidth]{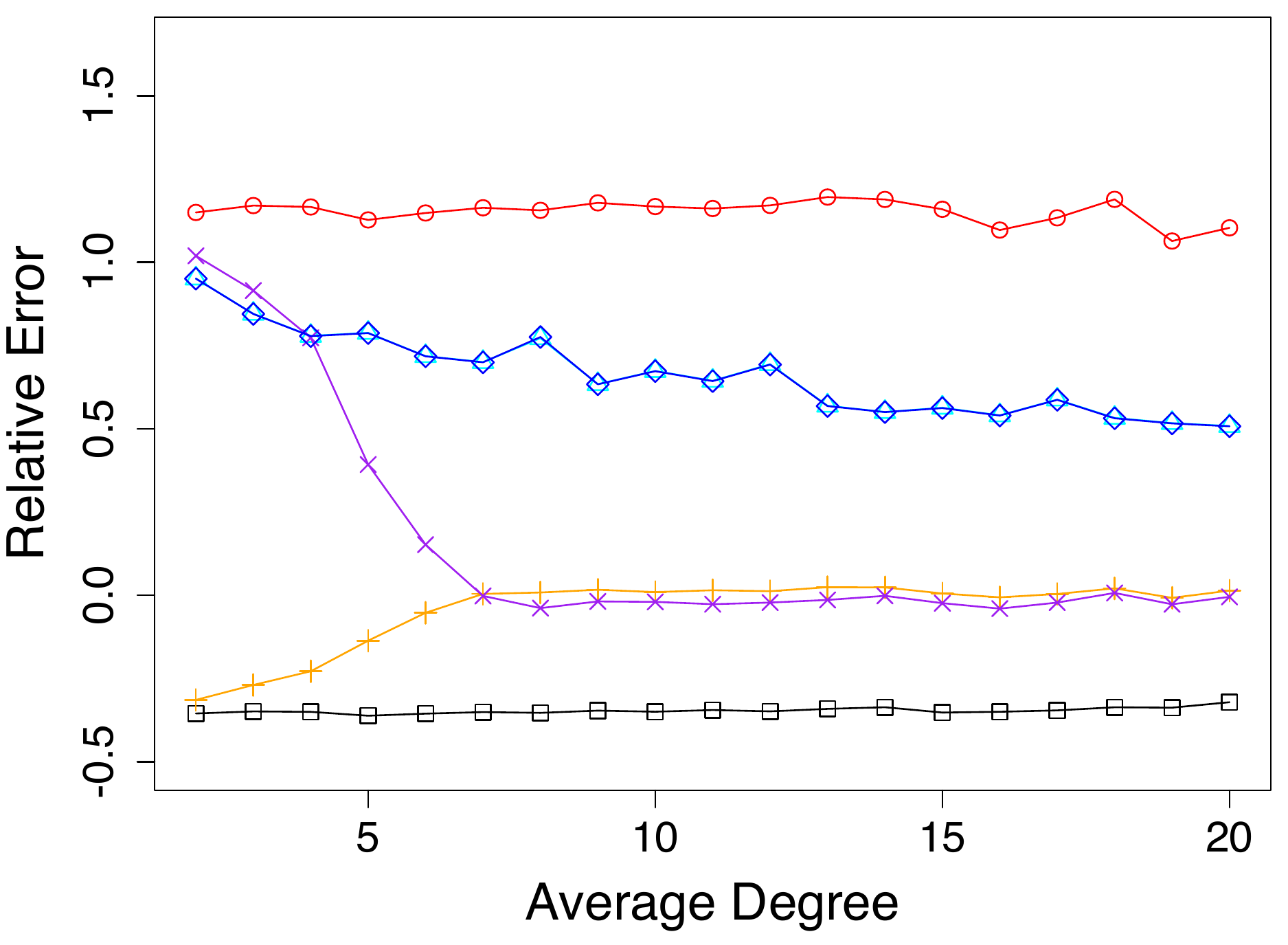}  
  \caption{$\varepsilon = 0.4$}
  \label{REeps0.4-dc}
\end{subfigure}
\begin{subfigure}{.33\textwidth}
  \centering
  \includegraphics[width=1\linewidth]{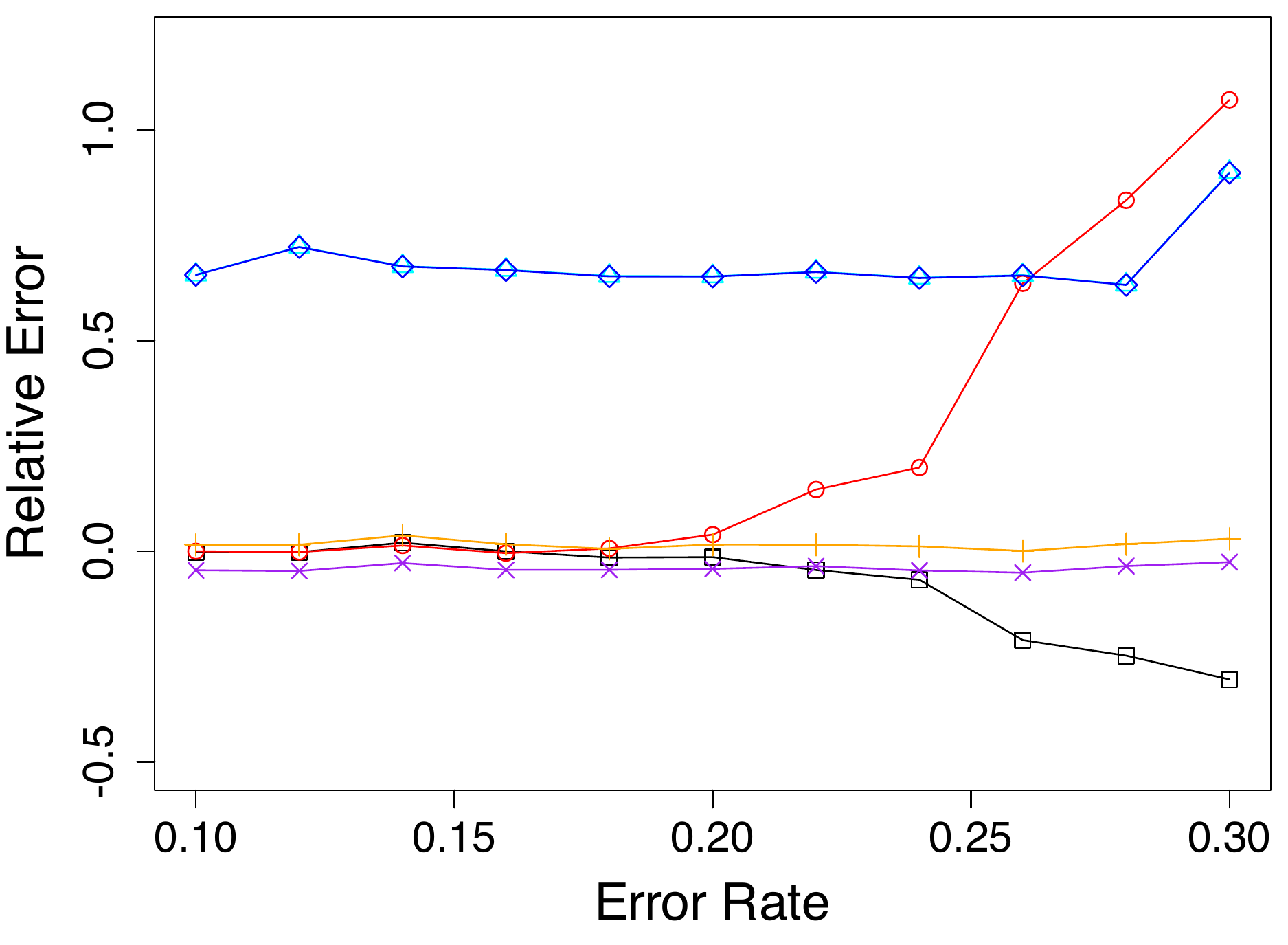}  
  \caption{$d = 8$}
  \label{REavd8-dc}
\end{subfigure}
\begin{subfigure}{.33\textwidth}
  \centering
  \includegraphics[width=1\linewidth]{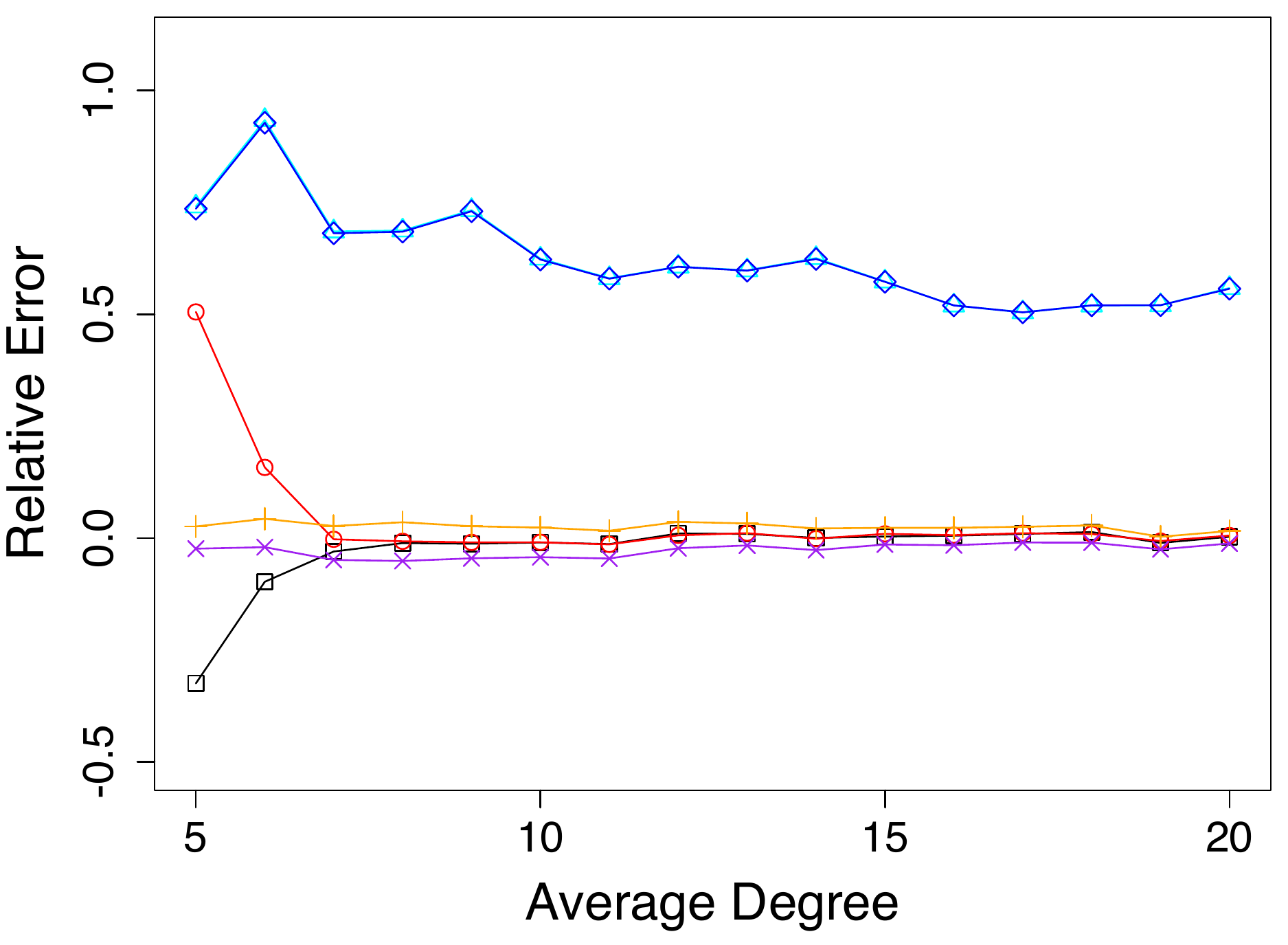}  
  \caption{$\varepsilon = 0.2$}
  \label{REeps0.2(k=3)-dc}
\end{subfigure}
\begin{subfigure}{.33\textwidth}
  \centering
  \includegraphics[width=1\linewidth]{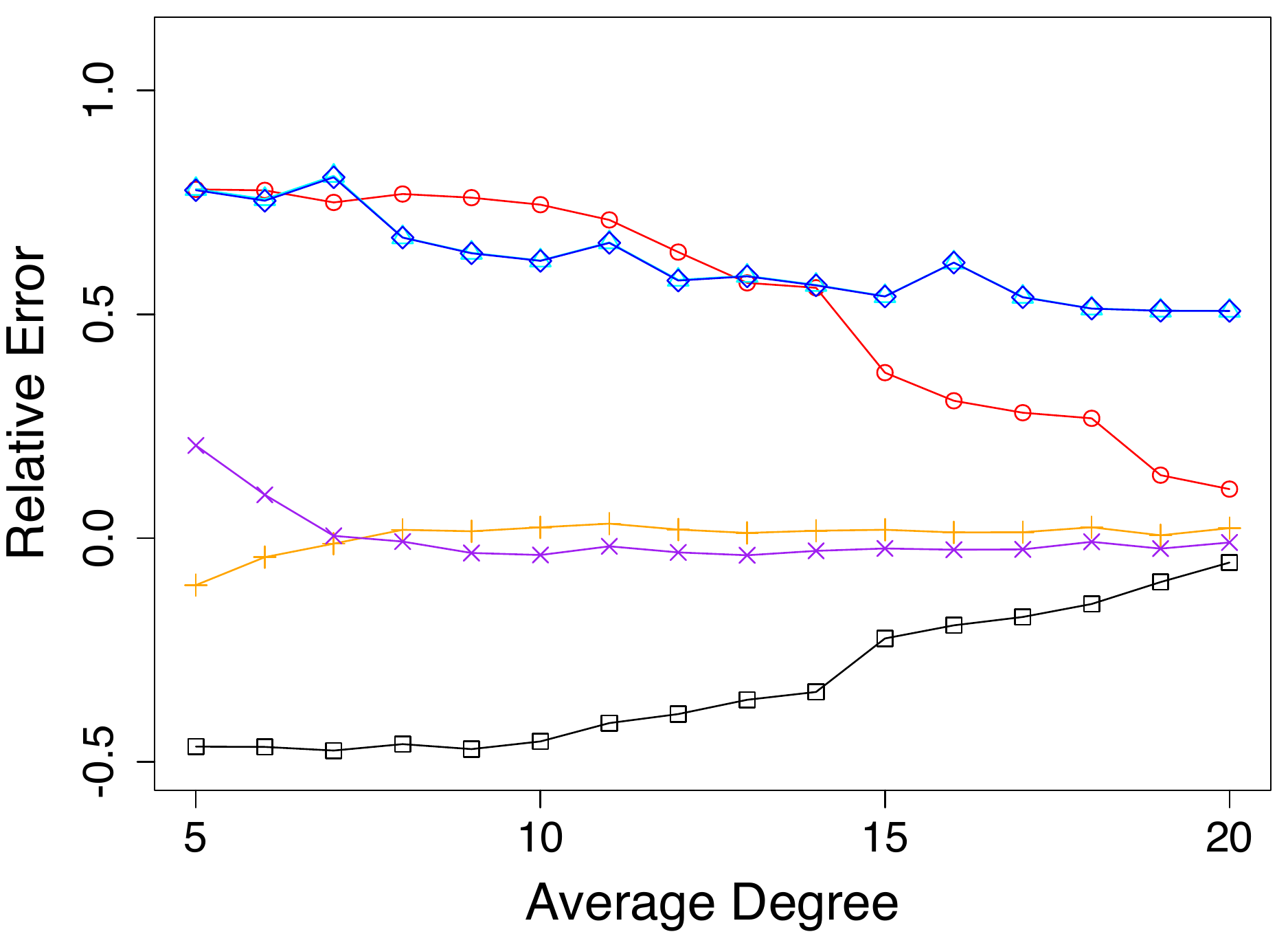}  
  \caption{$\varepsilon = 0.4$}
  \label{REeps0.4(k=3)-dc}
\end{subfigure}
\begin{subfigure}{.33\textwidth}
  \centering
  \includegraphics[width=1\linewidth]{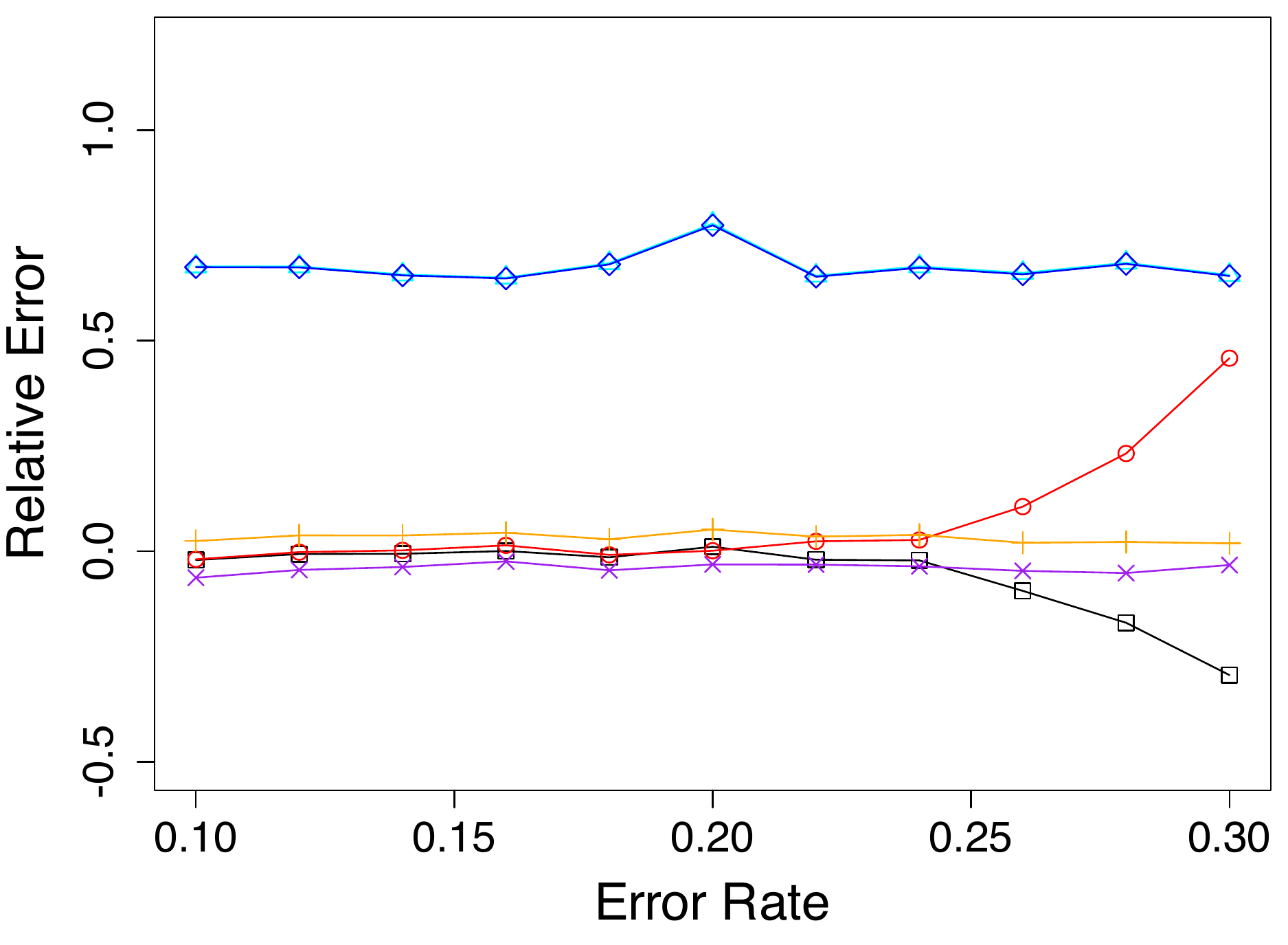}  
  \caption{$d = 8$}
  \label{REavd8(k=3)-dc}
\end{subfigure}
\caption{Relative errors of parameter estimation by Threshold BCAVI (T) and the classical BCAVI in balanced settings. (a)-(c): Networks are generated from DCSBM with $n=600$ nodes, $K=2$ communities of sizes $n_1 = n_2 = 300$. (d)-(f): Networks are generated from DCSBM with $n=600$ nodes, $K=3$ communities of sizes $n_1 = n_2 = n_3 = 200$. Initializations are generated from true node labels according to Assumption~\protect\ref{ass:perturb} with error rate $\varepsilon$.}
\label{REperturb-dc}
\end{figure}

\begin{figure}[ht]
\begin{subfigure}{.33\textwidth}
  \centering
  \includegraphics[width=1\linewidth]{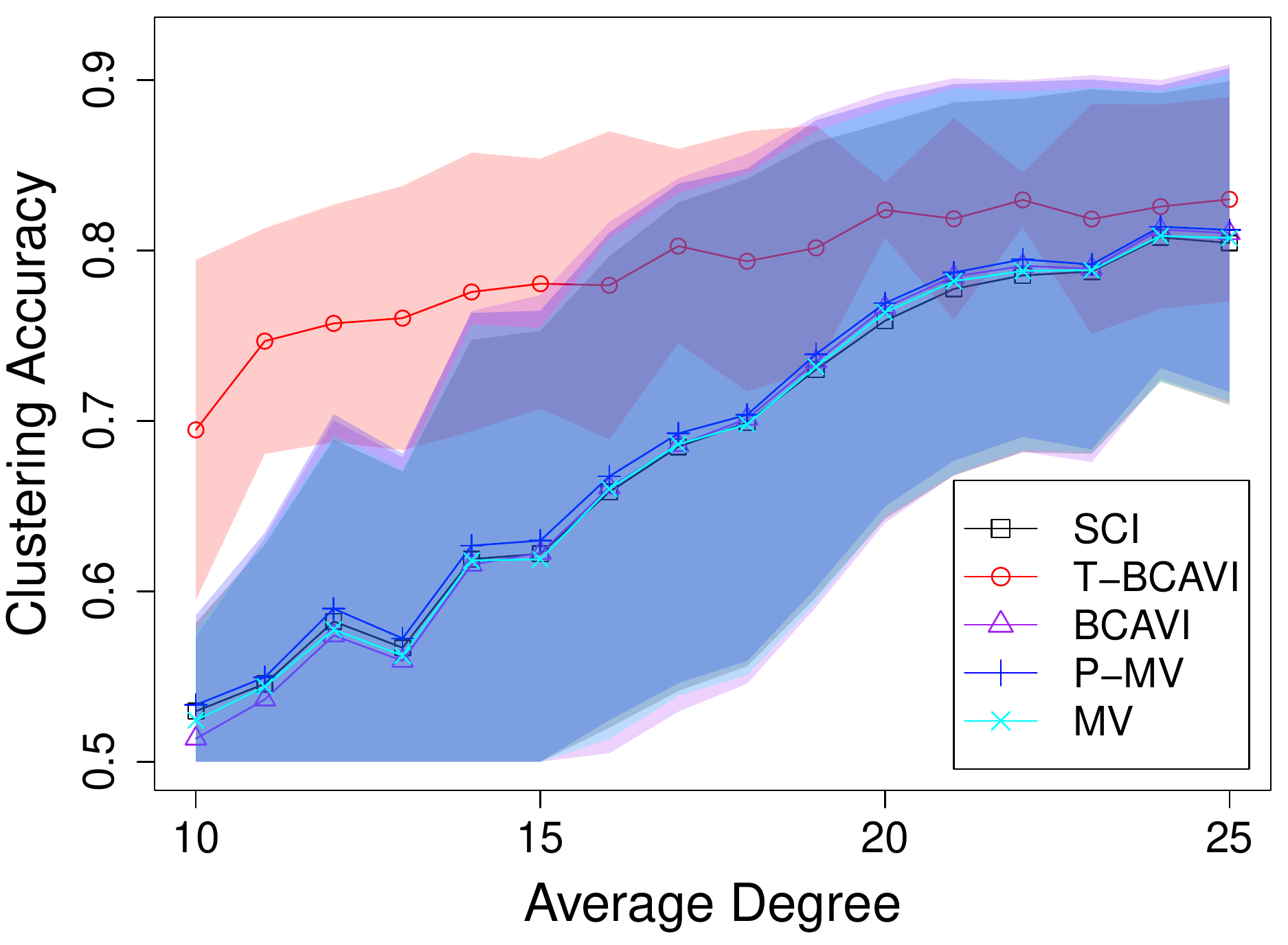}  
  \caption{$\tau = 0.5$}
  \label{samplep0.5-dc}
\end{subfigure}
\begin{subfigure}{.33\textwidth}
  \centering
  \includegraphics[width=1\linewidth]{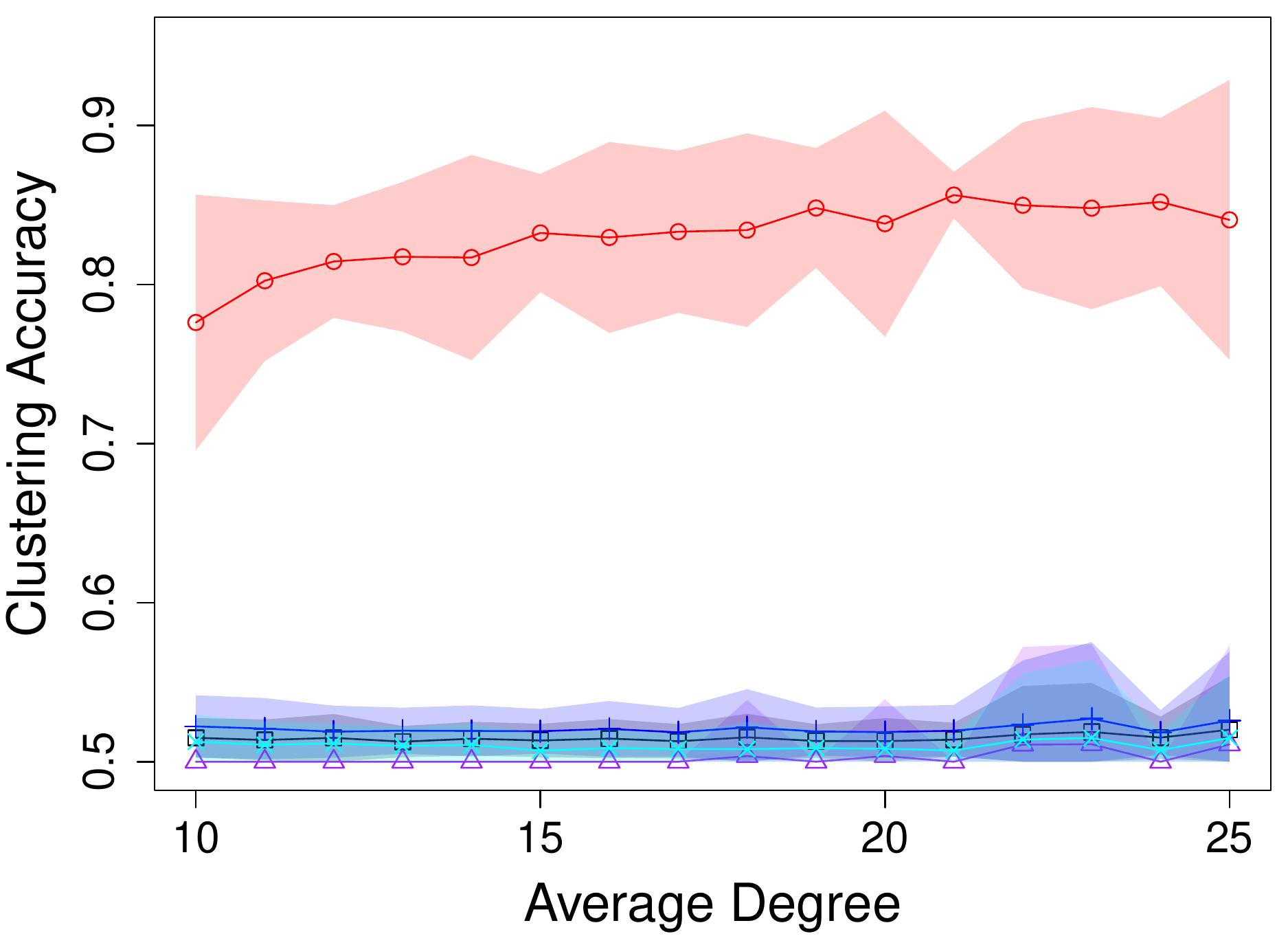}  
  \caption{$\tau = 0.2$}
  \label{samplep0.2-dc}
\end{subfigure}
\begin{subfigure}{.33\textwidth}
  \centering
  \includegraphics[width=1\linewidth]{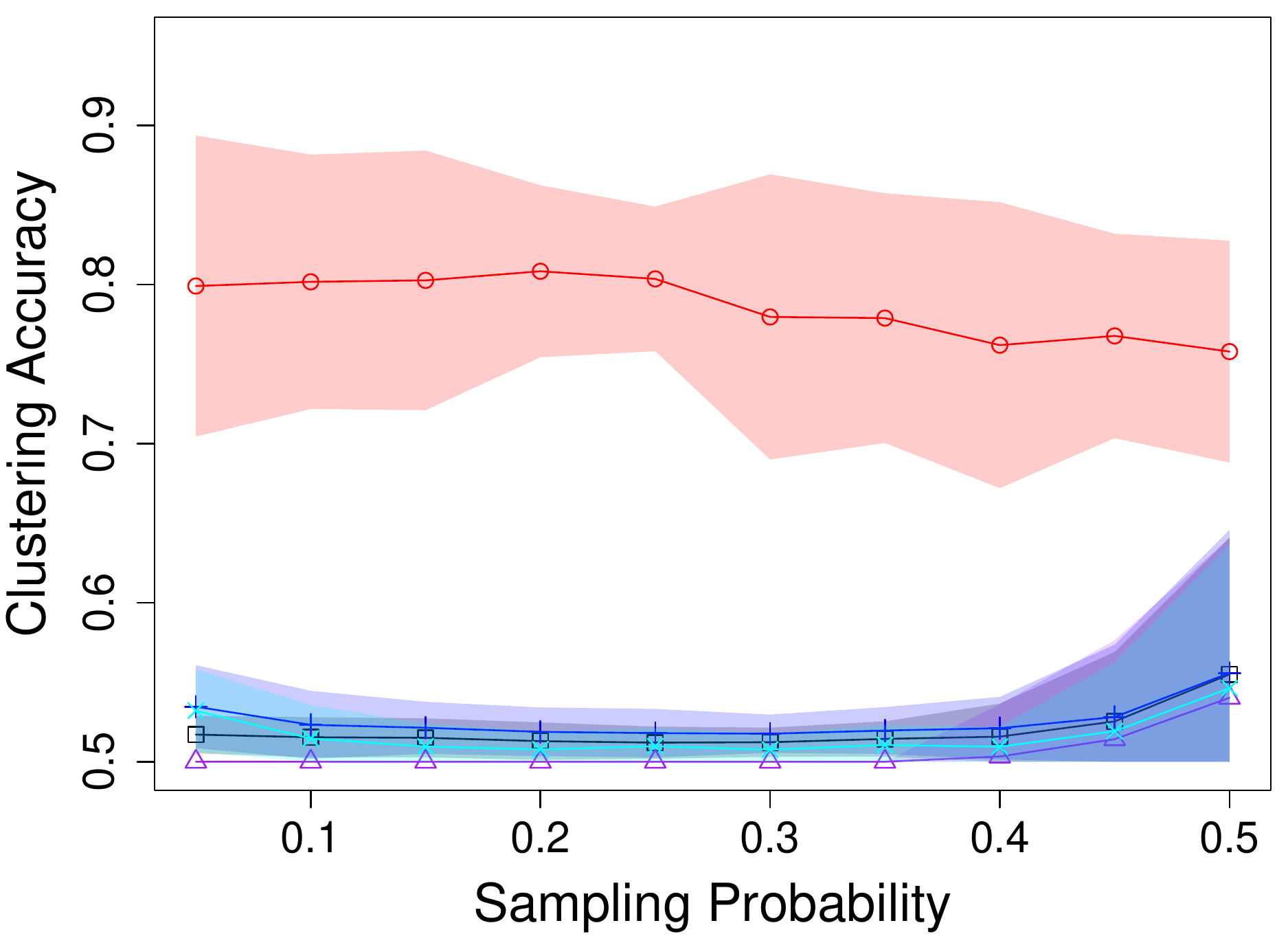}  
  \caption{ $d = 12$}
  \label{sample_avd12-dc}
\end{subfigure}
\begin{subfigure}{.33\textwidth}
  \centering
  \includegraphics[width=1\linewidth]{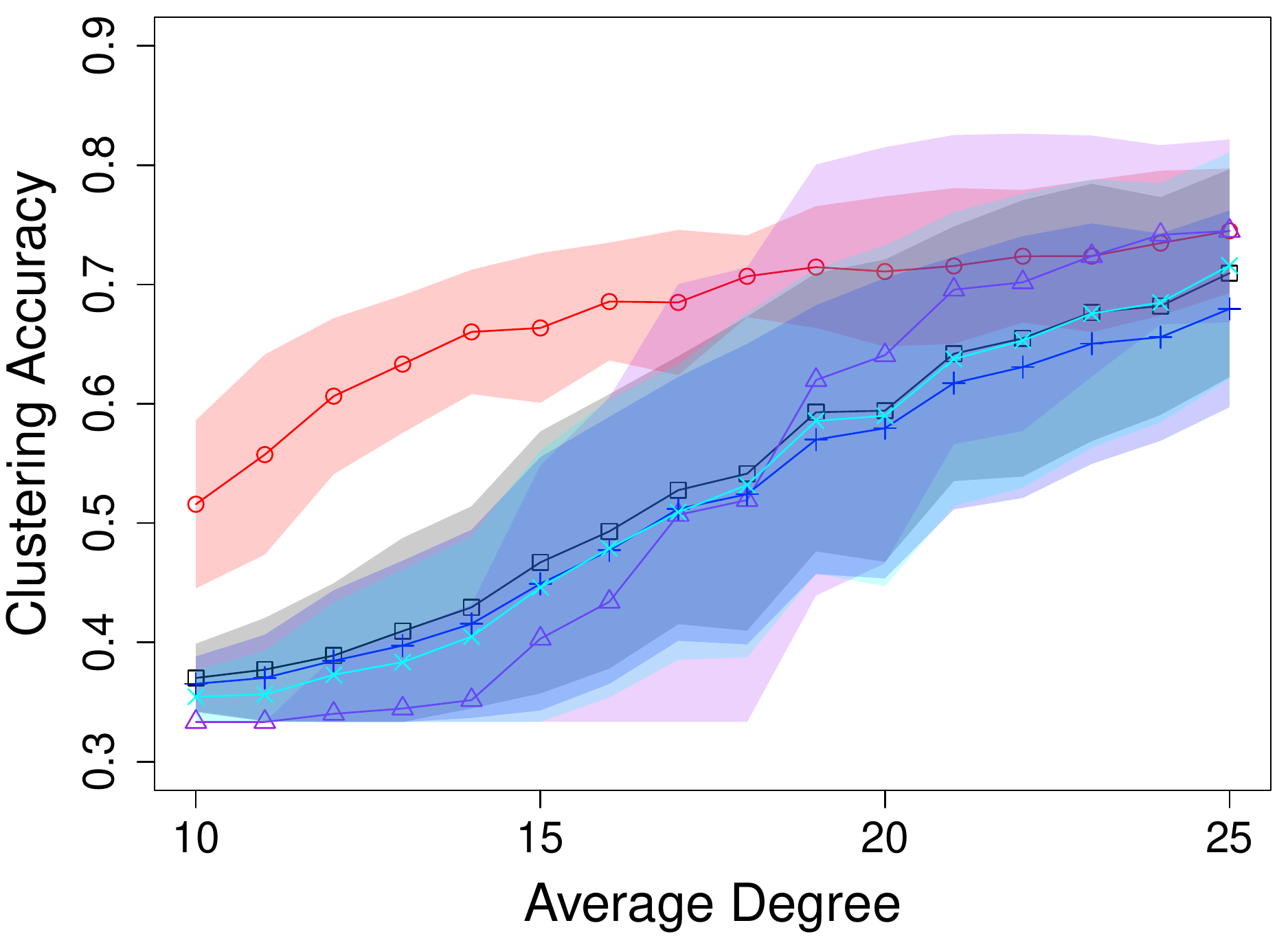}  
  \caption{$\tau = 0.5$}
  \label{samplep0.5(k=3)-dc}
\end{subfigure}
\begin{subfigure}{.33\textwidth}
  \centering
  \includegraphics[width=1\linewidth]{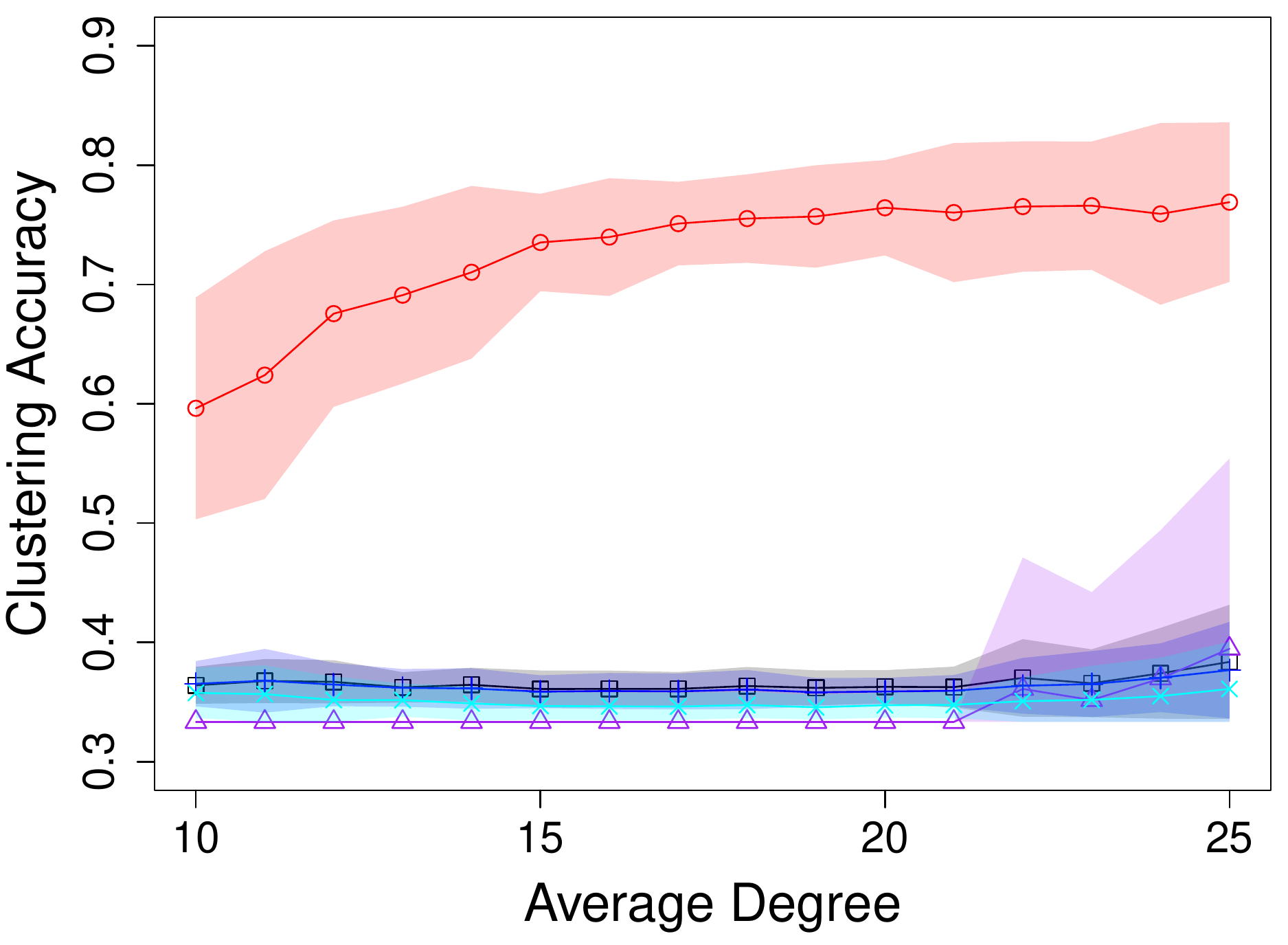}  
  \caption{$\tau = 0.25$}
  \label{samplep0.25(k=3)-dc}
\end{subfigure}
\begin{subfigure}{.33\textwidth}
  \centering
  \includegraphics[width=1\linewidth]{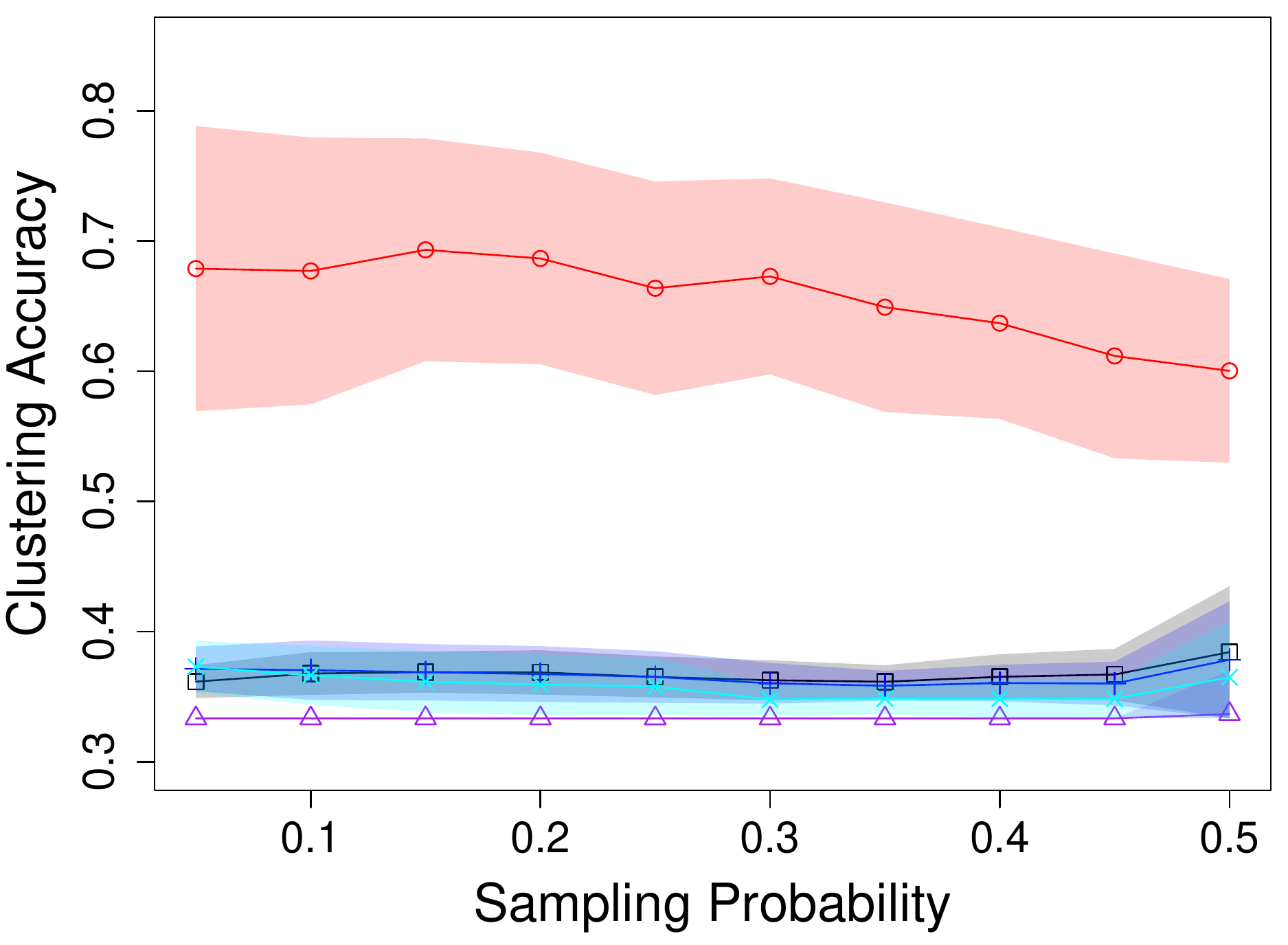}  
  \caption{$d = 12$}
  \label{sample_avd12(k=3)-dc}
\end{subfigure}

\caption{Performance of Threshold BCAVI (T-BCAVI), the classical BCAVI, majority vote (MV), and majority vote with penalization (P-MV) in balanced settings. (a)-(c): Networks are generated from DCSBM with $n=600$ nodes, $K=2$ communities of sizes $n_1 = n_2 = 300$. (d)-(f): Networks are generated from DCSBM with $n=600$ nodes, $K=3$ communities of sizes $n_1 = n_2 = n_3 = 200$. Initializations are computed by spectral clustering (SCI) applied to sampled sub-networks $A^{(\text{init})}$ with sampling probability $\tau$  while T-BCAVI and BCAVI are performed on remaining sub-networks $A-A^{(\text{init})}$.}
\label{avd-dc}
\end{figure}

\begin{figure}[ht]
\begin{subfigure}{.33\textwidth}
  \centering
  \includegraphics[width=1\linewidth]{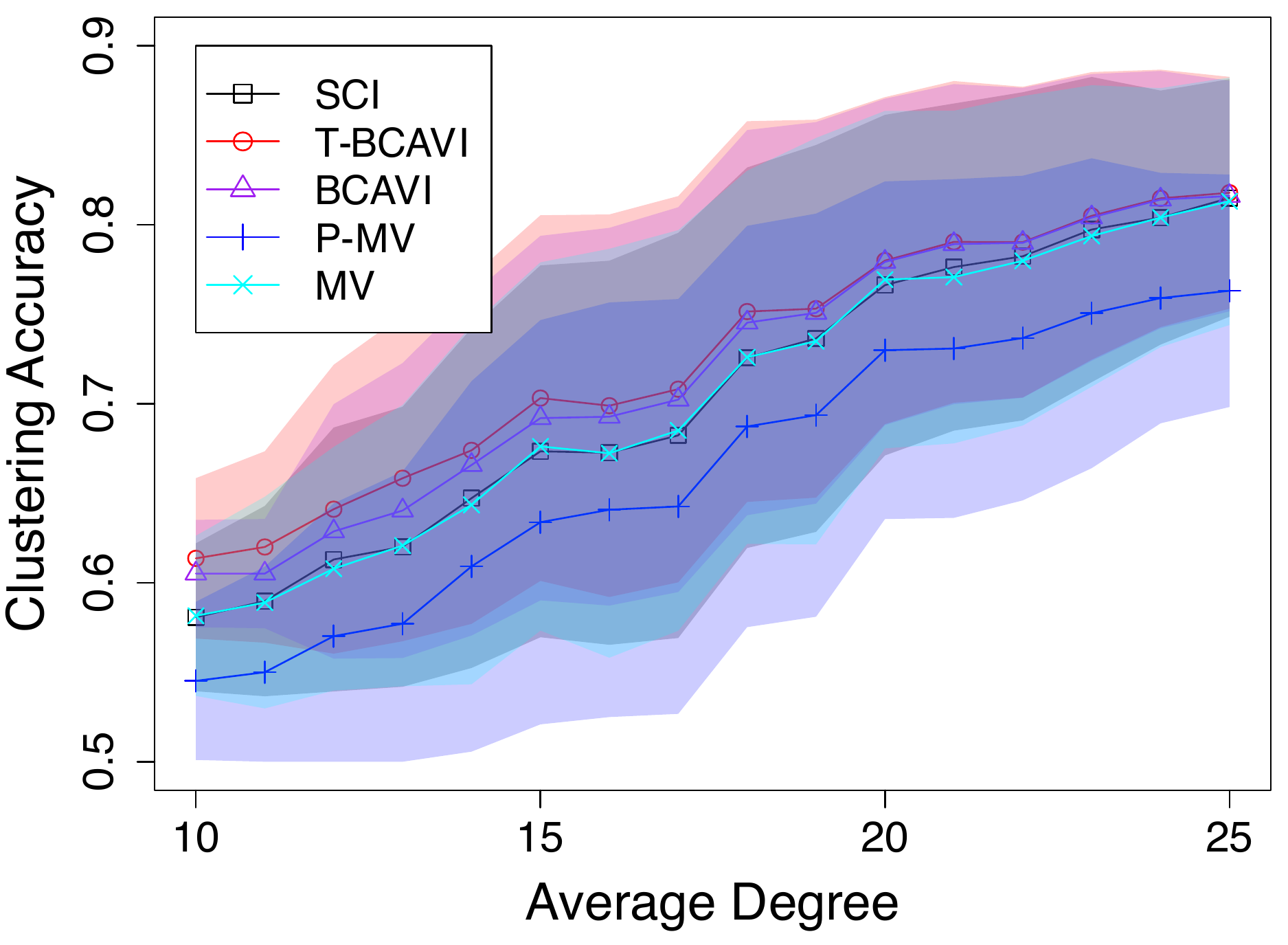}  
  \caption{$\tau = 0.5$}
  \label{Usamplep0.5-dc}
\end{subfigure}
\begin{subfigure}{.33\textwidth}
  \centering
  \includegraphics[width=1\linewidth]{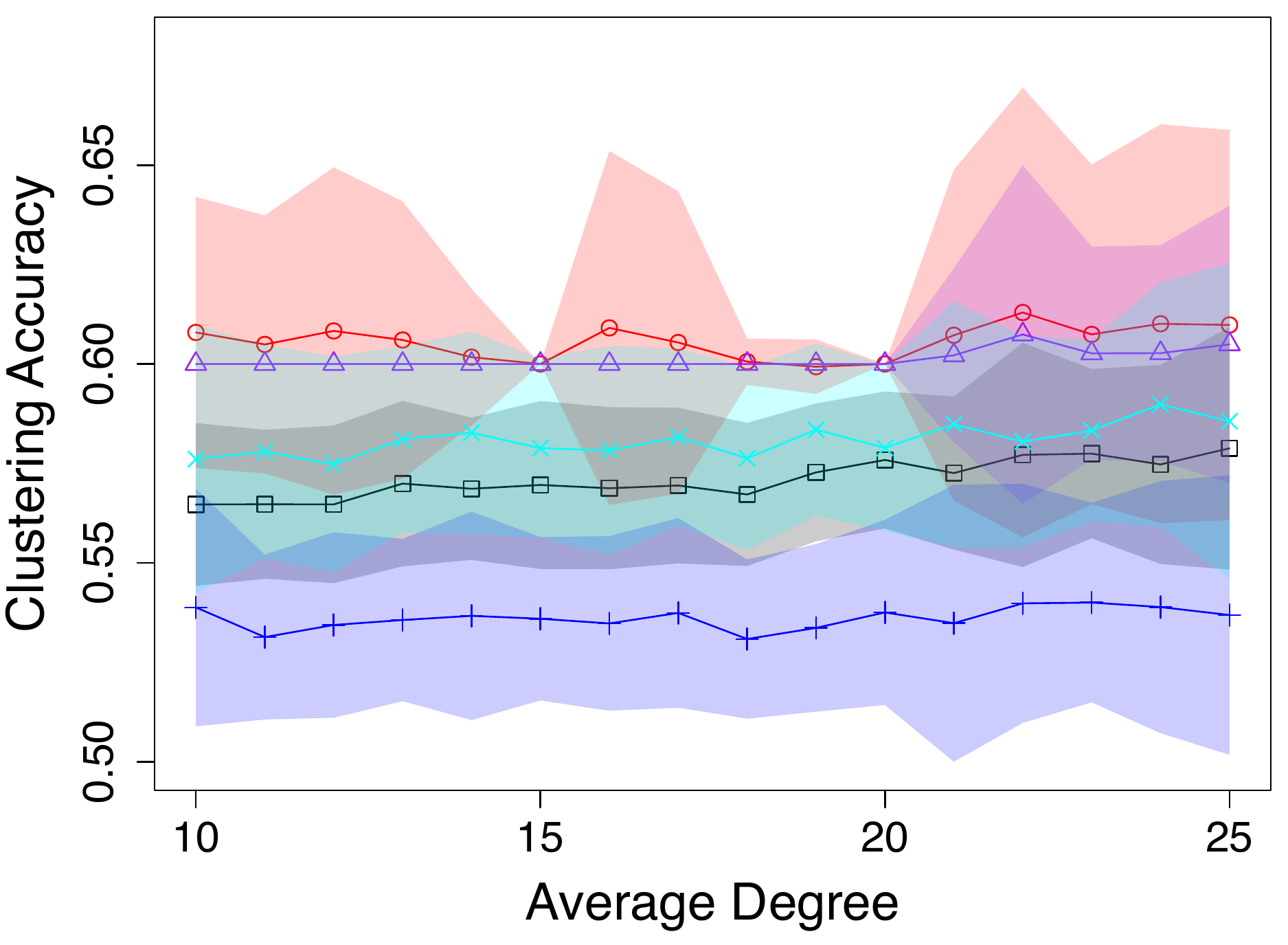}  
  \caption{$\tau = 0.2$}
  \label{Usamplep0.2-dc}
\end{subfigure}
\begin{subfigure}{.33\textwidth}
  \centering
  \includegraphics[width=1\linewidth]{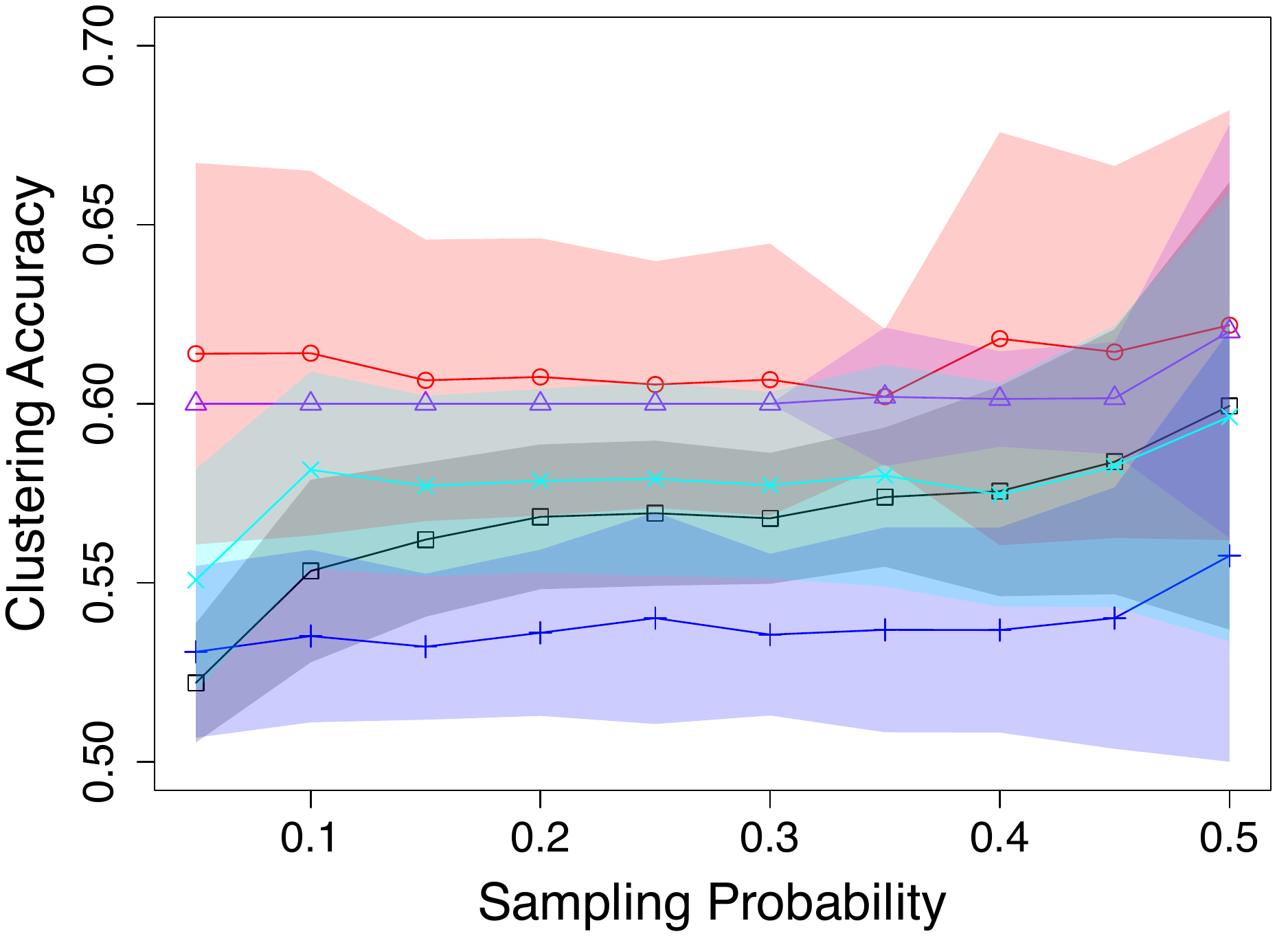}  
  \caption{$d = 12$}
  \label{Usample_avd12-dc}
\end{subfigure}
\begin{subfigure}{.33\textwidth}
  \centering
  \includegraphics[width=1\linewidth]{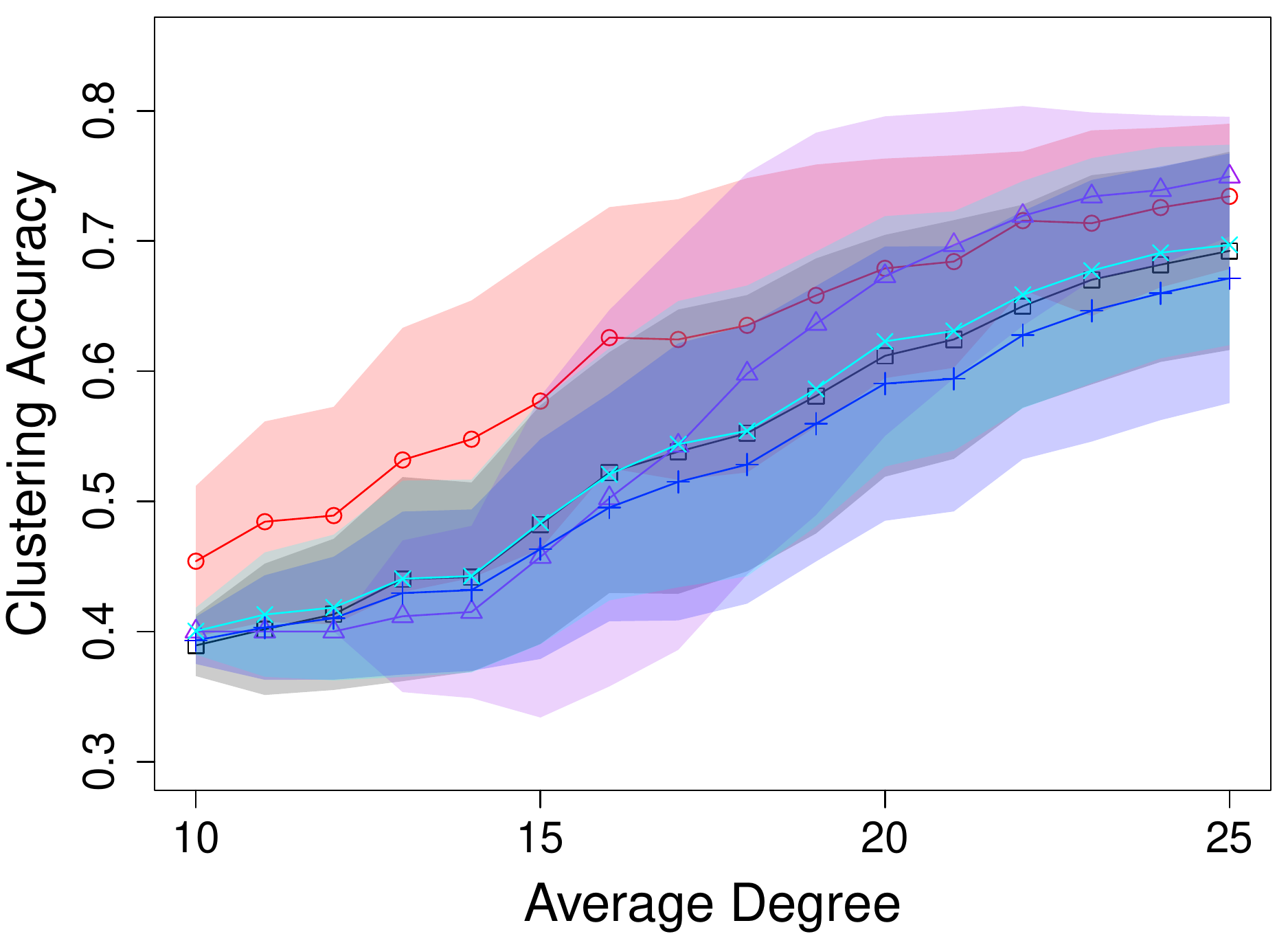}  
  \caption{$\tau = 0.5$}
  \label{Usamplep0.5(k=3)-dc}
\end{subfigure}
\begin{subfigure}{.33\textwidth}
  \centering
  \includegraphics[width=1\linewidth]{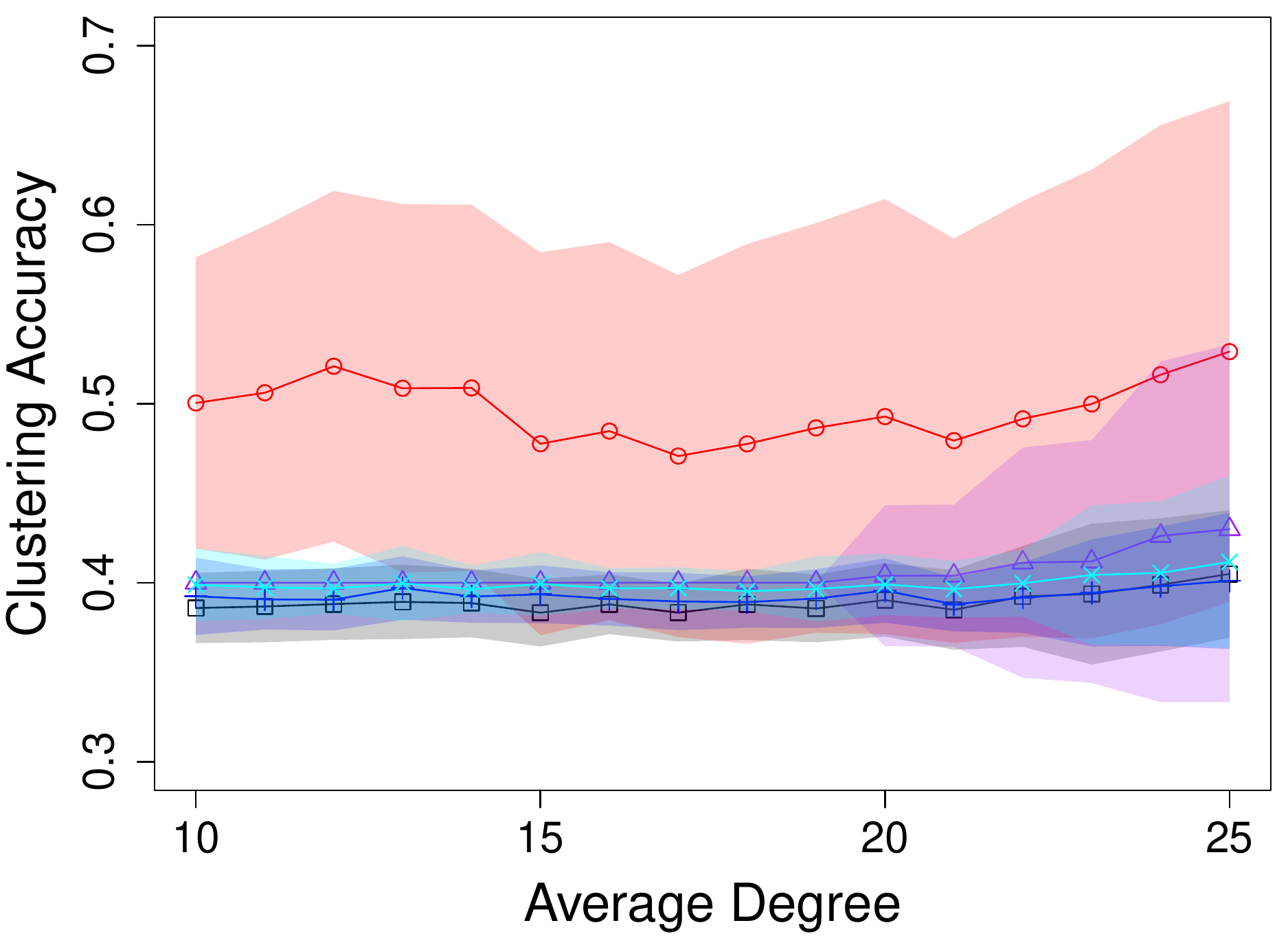}  
  \caption{$\tau = 0.2$}
  \label{Usamplep0.25(k=3)-dc}
\end{subfigure}
\begin{subfigure}{.33\textwidth}
  \centering
  \includegraphics[width=1\linewidth]{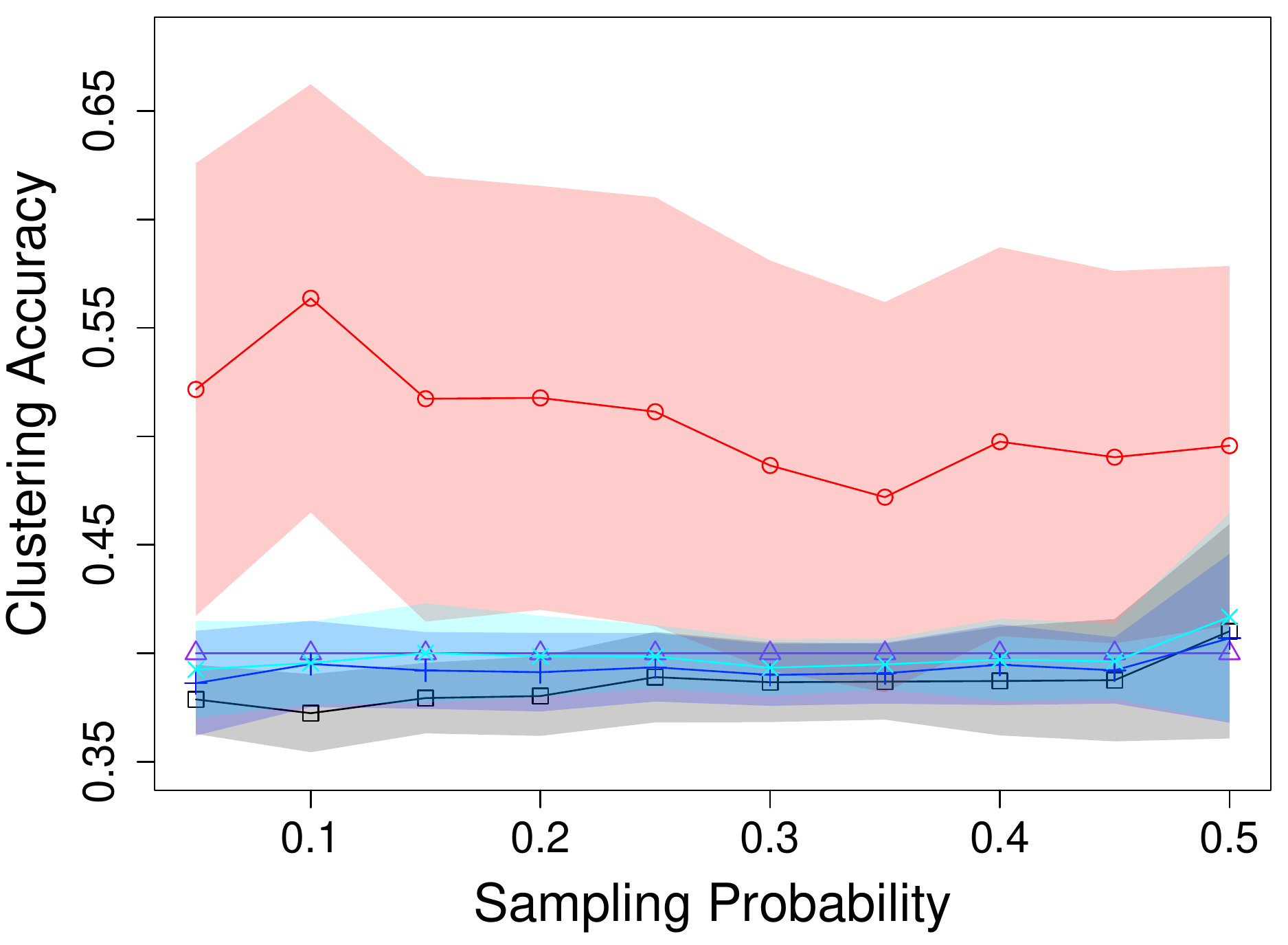}  
  \caption{$d = 12$}
  \label{Usample_avd12(k=3)-dc}
\end{subfigure}

\caption{Performance of Threshold BCAVI (T-BCAVI), the classical BCAVI, majority vote (MV), and majority vote with penalization (P-MV) in unbalanced settings. (a)-(c): Networks are generated from DCSBM with $n=600$ nodes, $K=2$ communities of sizes $n_1 = 240, n_2 = 360$. (d)-(f): Networks are generated from DCSBM with $n=600$ nodes, $K=3$ communities of sizes $n_1 = 150, n_2 =210, n_3 = 240$. Initializations are computed by spectral clustering (SCI) applied to sampled sub-networks $A^{(\text{init})}$ with sampling probability $\tau$  while T-BCAVI and BCAVI are performed on remaining sub-networks $A-A^{(\text{init})}$.}
\label{Uavd-dc}
\end{figure}

\vskip 0.2in
\bibliography{my_paper}

\end{document}